\documentclass{article} 
\usepackage{iclr2027_conference,times}

\usepackage{url}
\usepackage{upgreek}

\usepackage{amsmath}
\usepackage{amssymb}
\usepackage{amsthm} 
\usepackage{nicefrac}
\usepackage{comment}
\usepackage{algorithm, algorithmicx, algpseudocode}
\usepackage{subcaption}
\usepackage{enumitem}
\usepackage{xcolor}
\usepackage{tikz}
\usepackage{wrapfig}
\usepackage{booktabs}
\usepackage{colortbl}
\usepackage{mathabx}

\usetikzlibrary{arrows.meta,positioning,shapes.geometric,calc}

\usepackage{titletoc}

\usepackage[normalem]{ulem}
\usepackage{xcolor}
\usepackage{tikz}
\usetikzlibrary{shapes.geometric, shadows.blur, arrows.meta, positioning}
\usepackage{hyperref}
\hypersetup{
    colorlinks=true,
    linkcolor=blue,
    citecolor=blue,
    urlcolor=blue
}
\usepackage{multirow}

\theoremstyle{definition}

\theoremstyle{remark}

\usepackage{listings}
\definecolor{codegreen}{rgb}{0,0.6,0}
\definecolor{codegray}{rgb}{0.5,0.5,0.5}
\definecolor{codepurple}{rgb}{0.58,0,0.82}
\definecolor{backcolour}{rgb}{0.95,0.95,0.92}

\lstdefinestyle{mystyle}{
    backgroundcolor=\color{backcolour},   
    commentstyle=\color{codegreen},
    keywordstyle=\color{magenta},
    numberstyle=\tiny\color{codegray},
    stringstyle=\color{codepurple},
    basicstyle=\ttfamily\footnotesize,
    breakatwhitespace=false,         
    breaklines=true,                 
    captionpos=b,                    
    keepspaces=true,                 
    numbers=left,                    
    numbersep=5pt,                  
    showspaces=false,                
    showstringspaces=false,
    showtabs=false,                  
    tabsize=2
}
\usepackage{cleveref}
\usepackage{autonum}

\definecolor{codegreen}{rgb}{0,0.6,0}
\definecolor{codegray}{rgb}{0.5,0.5,0.5}
\definecolor{codepurple}{rgb}{0.58,0,0.82}
\definecolor{backcolour}{rgb}{0.97,0.97,0.97}
\lstdefinestyle{pythonstyle}{
    backgroundcolor=\color{backcolour},   
    commentstyle=\color{codegreen}\itshape,
    keywordstyle=\color{magenta}\bfseries,
    numberstyle=\tiny\color{codegray},
    stringstyle=\color{codepurple},
    basicstyle=\ttfamily\scriptsize,
    breakatwhitespace=false,         
    breaklines=true,                 
    captionpos=b,                    
    keepspaces=true,                 
    numbers=left,                    
    numbersep=5pt,                  
    showspaces=false,                
    showstringspaces=false,
    showtabs=false,                  
    tabsize=2,
    language=Python
}

\definecolor{colorP}{HTML}{0D47A1}
\definecolor{colorW}{HTML}{1B5E20}
\definecolor{colorV}{HTML}{B71C1C}

\newcommand{\colP}[1]{{\textcolor{colorP}{#1}}}
\newcommand{\colW}[1]{{\textcolor{colorW}{#1}}}
\newcommand{\colV}[1]{{\textcolor{colorV}{#1}}}

\usepackage{xcolor}

\makeatletter
\newcommand{\appendixpart}[2][]{%
  \refstepcounter{part}%
  \addtocontents{toc}{%
    \protect\vspace{1.2em}%
    \protect\noindent{\protect\large\protect\bfseries Part \Roman{part}: #2\protect\vspace{0.25em}\protect\newline}%
    \protect\hrule height 0.6pt%
    \protect\vspace{0.5em}%
  }%
  \vspace{2.0em}%
  \noindent
  \fboxrule=0.6pt\fboxsep=10pt%
  \colorbox{gray!10}{%
    \begin{minipage}{\dimexpr\textwidth-2\fboxsep-2\fboxrule\relax}
      {\LARGE\bfseries Part \Roman{part}: #2}\par
      \if\relax\detokenize{#1}\relax\else
        \vspace{0.4em}%
        {\normalsize\itshape\color{black!75}#1}\par
      \fi
    \end{minipage}%
  }%
  \vspace{1.5em}%
}
\makeatother

\crefname{theorem}{theorem}{theorems}
\Crefname{theorem}{Theorem}{Theorems}

\crefname{proposition}{proposition}{propositions}
\Crefname{proposition}{Proposition}{Propositions}

\crefname{lemma}{lemma}{lemmas}
\Crefname{lemma}{Lemma}{Lemmas}

\crefname{corollary}{corollary}{corollaries}
\Crefname{corollary}{Corollary}{Corollaries}

\crefname{definition}{definition}{definitions}
\Crefname{definition}{Definition}{Definitions}

\crefname{assumption}{assumption}{assumptions}
\Crefname{assumption}{Assumption}{Assumptions}

\crefname{remark}{remark}{remarks}
\Crefname{remark}{Remark}{Remarks}

\newcommand{\mcx}{\mathcal{X}}
\newcommand{\nset}{\mathbb{N}}

\newcommand{\Id}{\mathrm{Id}}

\newcommand{\argmax}{\mathrm{argmax}}

\newcommand{\rset}{\mathbb{R}}
\newcommand{\Dir}{\mathrm{Dir}}

\newcommand{\rmd}{\mathrm{d}}
\newcommand{\Beta}{\mathrm{Beta}}

\newcommand{\Var}{\mathrm{Var}}

\newcommand{\vareps}{\varepsilon}

\newcommand{\Ber}{\mathrm{Ber}}
\newcommand{\Jax}{\textsc{Jax}}

\newcommand{\Pbb}{\mathbb{P}}
\newcommand{\rmc}{\mathrm{C}}
\newcommand{\mcl}{\mathcal{L}}
\newcommand{\KL}{\mathrm{KL}}
\newcommand{\Qbb}{\mathbb{Q}}
\newcommand{\rmD}{\mathrm{D}}
\newcommand{\mcp}{\mathcal{P}}
\newcommand{\stopgrad}{\mathrm{sg}}
\newcommand{\Ebb}{\mathbb{E}}

\newcommand{\Cov}{\operatorname{Cov}}

\newcommand{\diag}{\operatorname{diag}}
\newcommand{\masktoken}{[\textsc{mask}]~}
\newcommand{\bostoken}{\text{[\textsc{bos}]}}

\newcommand{\Categorical}{\mathrm{Cat}}
\newcommand{\GammaDist}{\mathrm{Gamma}}

\newcommand{\1}{\mathbf{1}}

\newcommand{\Cat}{\mathrm{Cat}}

\newcommand{\pdata}[1]{p_{\text{data} #1}}
\newcommand{\Tvar}{\mathrm{TVar}}

\newcommand{\topp}{\mathrm{TOP}}
\newcommand{\temp}{\mathrm{TEMP}}
\newcommand{\startt}{\mathrm{start}}
\newcommand{\finish}{\mathrm{end}}
\newcommand{\freq}{\mathrm{FREQ}}
\newcommand{\freqloc}{\mathrm{FREQLOC}}
\newcommand{\figdir}{figs}

\title{Simplex Diffusion Models}

\usepackage{xcolor}
\usepackage[most]{tcolorbox}
\tcbuselibrary{theorems}
\definecolor{brightgrey}{HTML}{f5f5f6}
\definecolor{softblue}{HTML}{f0f5f8}

\newtcbtheorem[auto counter, number within=section]{propositionbeaut}{Proposition}{
    enhanced,
    theorem style=plain,
    fonttitle=\bfseries\upshape,
    fontupper=\itshape,
    colback=softblue,
    colframe=softblue,
    coltitle=black,
    arc=0mm,
    boxrule=0pt,
    top=0.6em, bottom=0.6em, left=0.8em, right=0.8em,
    before skip=0.8em, after skip=0.8em,
}{propbeaut}

\newtcbtheorem[auto counter, number within=section]{lemmabeaut}{Lemma}{
    enhanced,
    theorem style=plain,
    fonttitle=\bfseries\upshape,
    fontupper=\itshape,
    colback=softblue,
    colframe=softblue,
    coltitle=black,
    arc=0mm,
    boxrule=0pt,
    top=0.6em, bottom=0.6em, left=0.8em, right=0.8em,
    before skip=0.8em, after skip=0.8em,
}{lemmabeaut}

\newtcbtheorem[auto counter, number within=section]{theorembeaut}{Theorem}{
    enhanced,
    theorem style=plain,
    fonttitle=\bfseries\upshape,
    fontupper=\itshape,
    colback=softblue,
    colframe=softblue,
    coltitle=black,
    arc=0mm,
    boxrule=0pt,
    top=0.6em, bottom=0.6em, left=0.8em, right=0.8em,
    before skip=0.8em, after skip=0.8em,
}{thmbeaut}

\newtcbtheorem[auto counter, number within=section]{corollarybeaut}{Corollary}{
    enhanced,
    theorem style=plain,
    fonttitle=\bfseries\upshape,
    fontupper=\itshape,
    colback=softblue,
    colframe=softblue,
    coltitle=black,
    arc=0mm,
    boxrule=0pt,
    top=0.6em, bottom=0.6em, left=0.8em, right=0.8em,
    before skip=0.8em, after skip=0.8em,
}{corobeaut}

\newtcbtheorem[auto counter, number within=section]{definitionbeaut}{Definition}{
    enhanced,
    theorem style=plain,
    fonttitle=\bfseries\upshape,
    fontupper=\itshape,
    colback=softblue,
    colframe=softblue,
    coltitle=black,
    arc=0mm,
    boxrule=0pt,
    top=0.6em, bottom=0.6em, left=0.8em, right=0.8em,
    before skip=0.8em, after skip=0.8em,
}{definitionbeaut}

\makeatletter
\crefname{tcb@cnt@propositionbeaut}{Proposition}{Propositions}
\Crefname{tcb@cnt@propositionbeaut}{Proposition}{Propositions}
\crefname{tcb@cnt@lemmabeaut}{Lemma}{Lemmas}
\Crefname{tcb@cnt@lemmabeaut}{Lemma}{Lemmas}
\crefname{tcb@cnt@theorembeaut}{Theorem}{Theorems}
\Crefname{tcb@cnt@theorembeaut}{Theorem}{Theorems}
\crefname{tcb@cnt@corollarybeaut}{Corollary}{Corollaries}
\Crefname{tcb@cnt@corollarybeaut}{Corollary}{Corollaries}
\makeatother
\crefname{tcb@cnt@definitionbeaut}{Definition}{Definitions}
\Crefname{tcb@cnt@definitionbeaut}{Definition}{Definitions}
\makeatother

\author{Justin Deschenaux \\
Google DeepMind \\
EPFL \\
\And
Alexandre Galashov \\
Google DeepMind \\
UCL Gatsby \\
\And
Andrew Campbell \\
Google DeepMind \\
\And
Li Kevin Wenliang \\
Google DeepMind \\
\And
James Thornton \\
Google DeepMind \\
\And
Arnaud Doucet \\
Google DeepMind \\
\And
Valentin De Bortoli \\
Google DeepMind \\
}

\iclrfinalcopy 
\begin{document}

\maketitle

\begin{abstract}
Diffusion models have revolutionized generative modeling for continuous data through the gradual refinement of a belief state. This iterative refinement has not yet carried over to \emph{discrete} diffusion models, which discard uncertainty at intermediate steps through categorical sampling (\emph{information collapse}).
We propose \emph{Simplex Diffusion Models} (SDMs), a framework that lifts the diffusion process to the probability simplex to represent beliefs over categories. SDMs admit probability paths with closed-form reverse transitions and can be trained with a simple cross-entropy loss. Contrary to earlier proposals such as Dirichlet Flow Matching which requires integrating an ordinary differential equation, we introduce a DDIM-like sampler with a tunable level of stochasticity. Because SDMs operate on samples on the simplex, they can carry uncertainty across denoising steps, which mitigates information collapse.
On OpenWebText, SDMs are competitive with strong  Discrete Diffusion baselines, achieving $17.0$ GenPPL at $5.46$ unigram entropy in 64 sampling steps, close to real validation data.  
Even without Self-Conditioning (SC), SDMs outperform masked and uniform diffusion (with SC or predictor-corrector sampling) on code generation (TinyGSM, $T=0.1$; $49.0\%$ vs.\ $45.8\%$).
Distilled down to 8 steps, SDMs solve $32.1\%$ of GSM8K problems, more than distilled Discrete Diffusion models with 128 steps ($21.4\%$).  
\end{abstract}
\begin{figure}[h!]
    \centering
    \includegraphics[width=\linewidth]{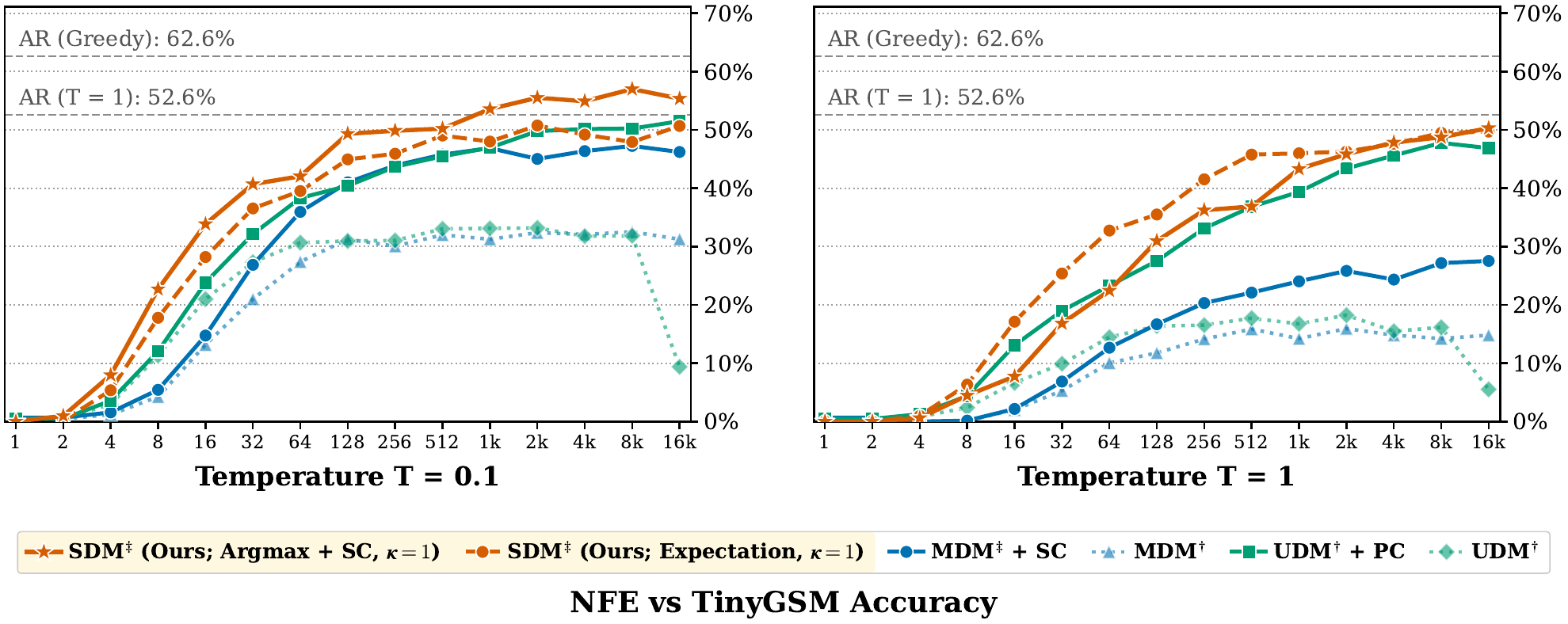}
    \caption{
    \textbf{SDMs outperform Discrete Diffusion on TinyGSM.}
    \emph{Left:} Accuracy as a function of the number of function evaluation (NFE) at low temperature ($T=0.1$). \emph{Right:} Accuracy against NFEs at $T=1$. 
    Without Self-Conditioning (SC), SDMs with the expected embedding outperform masked diffusion with SC (MDM + SC) and uniform diffusion with predictor-corrector sampling (UDM + PC). At 512 NFE, SDMs reach $49.0\%$ at $T = 0.1$, vs.\ $45.8\%$ for MDM$^\ddagger$ + SC, and $45.8\%$ at $T = 1$, vs.\ $36.8\%$ for UDM$^\dagger$ + PC.
    With SC, SDM (Argmax + SC) is the best diffusion variant at $T=0.1$ from 32 NFE onwards, and solves $57\%$ of problems with 8k NFEs. Autoregressive models with greedy decoding is best (62.6\%). All SDMs use churn $\kappa = 1$ and the adaptive sampling time grid (\Cref{sec:adaptive_time_schedule}), and we compare with the best baseline configurations (\Cref{sec:exp_tinygsm}). 
    $^\dagger$/$^\ddagger$: trained with the uniform/adaptive time sampler.
    The accuracy drop for UDM at 16k is consistent with typical floating-point rounding errors, which can accumulate at very small step sizes \citep{karras2022elucidating}.
    }
    \label{fig:tinygsm-nfe-vs-accuracy}
\end{figure}
\section{Introduction}

Denoising diffusion models \citep{ho2020denoising, song2020score} are state-of-the-art generative models in continuous domains  \citep{ho2022imagen, rombach2022highresolutionimagesynthesislatent, ramesh2022hierarchical,watson2023novo}. However, a vast portion of real-world data---including text, biological sequences, and graph structures---is discrete, and extending diffusion to those modalities remains challenging \citep{austin2021structured, hoogeboom2021argmaxflowsmultinomialdiffusion, campbell2022continuous}.

Existing diffusion models for categorical data generally fall into three categories. First, \emph{categorical diffusion models} such as masked or uniform diffusion models \citep{sahoo2024simple,shi2024simplified, ou2024absorbingdiscretediffusionsecretly} define the corruption process directly on sequences of discrete tokens. While theoretically sound, they suffer from \emph{information collapse}. Indeed, as intermediate states are discrete, the denoiser's uncertainty is discarded at each update. Hence, these models  rely on Self-Conditioning \citep{chen2022analog} or loopholing \citep{jo2025loopholing} to carry beliefs across sampling steps.

Second, \emph{continuous relaxations} diffuse one-hot vectors \citep{lee2026one, roos2026categoricalflowmaps, potaptchik2026discreteflowmaps}, learned embeddings \citep{dieleman2022continuous,gulrajani2023likelihood,deschenaux2026languagemodelinghypersphericalflows,chemseddine2026sphericalflowssamplingcategorical, yang2026continuousdiffusionscalescompetitively}, or frozen embeddings \citep{hu2026elfembeddedlanguageflows, shen2026codarcontinuousdiffusionlanguage}. However, the identity of the clean category tends to be destroyed abruptly, within a short window during the forward process, and this worsens as the dimension of the encoding grows \citep{pynadath2025candi,shabalin2026gaussian}.

Finally, a third line of work diffuses directly on the probability simplex, whose elements are distributions over categories \citep{richemond2022categorical,cheng2024categorical,davis2024fisher,stark2024dirichlet,cheng2025alpha,williams2025simplex,boget2026unrestrainedsimplexdenoisingdiscrete,chandra2025unification}. However, they typically sample by numerically integrating an ODE or SDE: Dirichlet Flow Matching (DFM; \citealp{stark2024dirichlet}), the closest to our work, scales poorly to large vocabularies (\Cref{tab:tinygsm_dfm_apdx}). While the concurrent Simplax model \citep{sakurai2026simplex} does not rely on numerical integration, their  sampling procedure operates on categorical samples rather than the continuous representation itself.

\paragraph{Contributions.}
We propose Simplex Diffusion Models (SDMs), whose intermediate states are distributions over tokens.
\begin{enumerate}[leftmargin=*,itemsep=1pt]
\item Similar to DFM, SDMs use a Dirichlet forward process, but parameterized so that its mean and concentration are set separately. This path admits closed-form reverse transitions, inducing a DDIM-style sampler \citep{song2020denoising}, with a  \emph{churn} parameter that controls the sampling stochasticity. SDMs are trained with a simple cross-entropy objective. Because SDMs operate on samples on the simplex, they can carry uncertainty across denoising steps, mitigating information collapse. 
\item We show that SDMs unify discrete and continuous diffusions. A category sampled from the diffused simplex state has the same marginal distribution as if it were corrupted by Discrete Diffusion (\Cref{propbeaut:induced-forward}). Viewing the inverse Dirichlet concentration as a \emph{temperature}, SDMs reduce to Discrete Diffusion at high temperature, and follow a deterministic path with Gaussian fluctuations at low temperature (\Cref{propbeaut:temperature-limits,} and \Cref{sec:gaussian_limit_linear}).
\item Empirically, even without Self-Conditioning (SC), \textbf{SDMs outperform masked and uniform diffusion} (using either SC or predictor-corrector sampling) on TinyGSM ($T=0.1$; $49.0\%$ vs.\ $45.8\%$; \Cref{fig:tinygsm-nfe-vs-accuracy}). Distilled down to 8 steps, SDMs solve $32.1\%$ of the GSM8K problems, more than distilled Discrete Diffusion models with 128 steps ($21.4\%$; \Cref{tab:tinygsm_distill_idlm}). On Sudoku, SDMs with SC match the best Discrete Diffusion model ($99.1\%$ vs.\ $99.3\%$). On OpenWebText, SDMs are competitive with UDM and MDMs with all those methods achieving a GenPPL/Entropy frontier that matches real data after logit shaping interventions. This result calls into question the validity of OpenWebText in assessing the unconditional text generation capabilities of small scale models. We also validate the flexibility of our approach by achieving competitive performance on quality and diversity metrics for molecule generation, see \Cref{sec:appendix-genmol}.
\end{enumerate}

\section{Background}
\label{sec:background}
\paragraph{Notation.}
Let $\mcx = \{1, \dots, N\}$ be the finite categorical vocabulary. Let $\Delta_N$ denote the probability simplex on $\mathbb{R}^N$, and $\mathbf{1} = (1, \dots, 1)^\top$ be the all-ones vector. For $i \in \mcx$, let $e_i \in \Delta_N$ be the standard basis (one-hot) vector corresponding to state $i$, and let $\odot$ denote the Hadamard product. Let $t \mapsto \alpha_t$ be non-increasing on $[0,1]$ with $\alpha_0=1$, $\alpha_1=0$, and $0<\alpha_t<1$,  for $t\in (0,1)$. We denote by $\pi = (\pi_1, \dots, \pi_N) \in \Delta_N$ a reference prior distribution. We denote by $\Dir$, $\Beta$, $\GammaDist$, $\Ber$ and $\Categorical$ the Dirichlet, Beta, Gamma, Bernoulli and Categorical distributions respectively. We recall basic facts about these distributions in \Cref{app:basics}.
\paragraph{Discrete Diffusion Models.}
Discrete Diffusion models (DDMs; \citealp{sohl2015deep, austin2021structured, campbell2022continuous, sahoo2024simple,  shi2024simplified, ou2024absorbingdiscretediffusionsecretly}) are generative models over discrete spaces. DDMs define a conditional corruption process 
\begin{equation}
\label{eq:discrete_forward}
    p_{t|0}(x_t | x_0) = \mathrm{Cat}\left(x_t ; \alpha_t e_{x_0} + (1-\alpha_t) \pi\right) , 
\end{equation}
that induces marginal distributions $\left(p_t\right)_{t \in [0, 1]}$, given by $p_t(x_t)=\sum_{x_0} p_{t|0}(x_t|x_0)p_0(x_0)$, bridging $p_0 = p_\text{data}$ to $p_1 = \pi$. The generative process is defined in terms of transitions $p_{s|t}$ with $s \leq t$. The transition $p_{s|t}$ is \emph{compatible} if \mbox{$p_s(x_s) = \sum_{x_t}p_{s|t}(x_s|x_t) p_t(x_t)$}. Thus, one can generate $x_0 \sim \pdata{}$ by applying $p_{s|t}$ for $n$ iterations, starting from $x_1 \sim p_1 $. For a time grid $0=t_0 < t_1 < ... < t_n = 1$,
\begin{equation}
\label{eq:discrete_ancestral_sampling}
    \textstyle p_0(x_0) = \sum\nolimits_{\left(x_{t_1}, ..., x_{t_{n-1}},x_1\right)} p_{1} (x_1) \prod_{i=1}^n p_{t_{i-1}|t_i}(x_{t_{i-1}}|x_{t_i}).
\end{equation}

While using \emph{marginal} transitions $p_{s|t}(x_s|x_t) = \sum_{x_0} p_{s|0, t}(x_s | x_0, x_t)p_{0|t}(x_0|x_t)$ is possible, recent work instead defines \emph{bridge} (approximate) transitions $\hat{p}_{s|t}^\theta(x_s|x_t) :=  p_{s|0, t}(\cdot | \mathbf{x}_\theta(t, x_t), x_t)$, where \mbox{$\mathbf{x}_\theta(t, x_t): [0,1] \times \mcx \rightarrow \Delta_N$} is learned
\citep{gourevitch2026uniformdiffusionmodelsrevisited}. 
One obtains $\theta$ by maximizing an expected Evidence Lower Bound (ELBO). 

\paragraph{Dirichlet Flow Matching.}
Dirichlet llow matching (DFM; \citealp{stark2024dirichlet})  extends Flow Matching (FM; \citealp{lipman2022flow})  to the simplex $\Delta_N$ to model categorical data. The data distribution $\pdata{}=(\pdata{, 1},\cdots,\pdata{, N})$ on $\mcx$ is first lifted to  an atomic measure on the simplex which we also denote as $ p_\text{data}$ in a slight abuse of notation for $P \in \Delta_N$
\begin{equation}
   \label{eq:form_pi_0}
\textstyle     p_\text{data} \left( P \right) = \sum_{i=1}^{N} \pdata{, i} \updelta_{e_i}(P) .
\end{equation}
\citet{stark2024dirichlet} define a probability path $(p_t)_{t\in[0,1]}$ between the distribution $p_0= p_\text{data}$ (at time $t=0$) and the uniform Dirichlet prior $p_1=\Dir(\cdot; \mathbf{1})$ (at $t=1$) using
\begin{equation}
\label{eq:dfm_path}
     p_t(P_t) = \sum\nolimits_{P_0} p_{t|0}(P_t \mid P_0) p_\text{data}(P_0)\quad~\text{where}~\quad p_{t|0}(P_t | P_0) = \Dir \big(P_t; \mathbf{1} + h_t P_0 \big),
\end{equation}
with $t \mapsto h_t$ a decreasing function with $h_1 = 0$ and $\lim_{t \to 0} h_t = \infty$. 
A velocity field $u_t(P_t)$ ensuring (approximately) that $\rmd P_t =u_t(P_t) \rmd t $ satisfies $P_t \sim p_t$ for $P_1 \sim p_1$, hence $P_0 \sim p_\text{data}$, is then learned through cross-entropy.

\section{Simplex Diffusion}\label{sec:simplicialdiffusion}
\subsection{Probability Path and Bridge}
As in Dirichlet Flow Matching \citep{stark2024dirichlet}, we build a probability path on the simplex by considering $P_0 \sim  \pdata{}$, see \eqref{eq:form_pi_0}, and Dirichlet distributions to define $p_{t|0}(P_t|P_0)$. 

\paragraph{Probability Path.}
Let $t \mapsto c_t$ be a positive function and $t \mapsto \alpha_t$ a non-increasing function such that $\alpha_0=1$ and $\alpha_1=0$. We consider the corruption mechanism defined by 
\begin{equation}
\label{eq:forward_process_simplicial}
    p_{t|0}(P_t|P_0) = \Dir(P_t;\beta_t(P_0,\pi)) , \qquad \beta_t(P_0,\pi) = c_t (\alpha_t P_0 + (1-\alpha_t) \pi) , 
\end{equation}
where we recall that $\pi \in \Delta_N$ is a reference prior distribution on $\mcx$ chosen by the user\footnote{In practice, we will use a distribution $\pi$ with full support on $\mcx$. Otherwise, this would mean that for data tokens $x_0$ not in the support of $\pi$, any sample $P_t$ from $p_{t|0}(P_t|P_0=e_{x_0})$ would reveal $x_0$ almost surely for all $t<1$. The same argument rules out using a simplex version of masked diffusions.}. 
In particular $p_{1|0}(P_1|P_0)= \Dir(P_1;c_1 \pi)$ is independent of $P_0$. We have $\mathbb{E}[P_t|P_0]=(1-\alpha_t)\pi+\alpha_t P_0:=\bar{P}_t$ and $\Cov[P_t|P_0]=(\diag(\bar{P}_t)-\bar{P}_t\bar{P}^\top_t)/(c_t+1)$. We refer to $c_t$ as the concentration parameter.

This forward model on the simplex induces a forward model on the discrete state-space commonly used in Discrete Diffusions/Flow Matching \citep{campbell2024generativeflowsdiscretestatespaces,sahoo2024simple,shi2024simplified}.

\begin{propositionbeaut}{Induced forward}{induced-forward}
    Let $t \in [0,1]$ and $P_t$ satisfying  \eqref{eq:forward_process_simplicial}.
    In addition, let $x_t$ be such that $p_{t|0}(x_t| P_t, P_0) =P_{t,x_t}$. Then, recalling that $P_0 =e_{x_0}$, we have that 
    \begin{equation}
        p_{t|0}(x_t | P_0) = \alpha_t \updelta_{x_0}(x_t) + (1-\alpha_t) \pi_{x_t}.
    \end{equation}
\end{propositionbeaut}

We now propose an alternative representation of \eqref{eq:forward_process_simplicial} enabling us to sample easily from $p_{t|0}(P_t \mid P_0)$.

\begin{propositionbeaut}{Interpolation path}{interpolatingpath}
    Let  $t \in [0,1]$. Let $W_t \sim \Beta\left(c_t \alpha_t, c_t(1-\alpha_t)\right)$ and let $V_t \sim \Dir(c_t(1-\alpha_t)\pi)$ be independent of $W_t$. Define
   \begin{equation}
   \label{eq:interpolation_w}
        P_t = W_t P_0 + (1-W_t)V_t. 
    \end{equation}
    Then the conditional density of $P_t$ given $P_0$ is given by $p_{t|0}(P_t \mid P_0) = \Dir\left(P_t;\beta_t(P_0,\pi)\right)$.
\end{propositionbeaut}
\eqref{eq:interpolation_w} shows that, for $t\in (0,1)$, $P_t$ is a random convex combination (with interpolation weight $W_t \in (0,1)$) of a data sample $P_0$ and a noise sample $V_t$.

The concentration parameter $c_t$ acts inversely to temperature: larger values concentrate probability mass around the mean, whereas smaller values push it toward the simplex vertices. Specializing momentarily to $c_t = \varepsilon / (1 - \alpha_t)$ for $t > 0$, we interpret $\varepsilon$ as an \emph{inverse temperature}. In \Cref{sec:theoreticalproperties}, we formally prove that as $\varepsilon \to 0$ (high temperature), the conditional distribution of $P_t \mid P_0$ concentrates entirely on the vertices. Hence \emph{Discrete Diffusion models emerge as a limiting case of Simplex Diffusion Models}. Conversely, as $\varepsilon \to \infty$, the distribution concentrates on the mean, $\alpha_t P_0 + (1 - \alpha_t)\pi$. \Cref{fig:sdm-forward} illustrates these regimes  across varying $\varepsilon > 0$ and $t \in [0, 1]$.

\paragraph{Backward Bridges.} We now return to a general concentration schedule $(c_t)_{t\in [0,1]}$. In order to obtain a generative model, in the spirit of DDIM \citep{song2020denoising}, we need to identify a conditional distribution $p_{s|0,t}$ satisfying the following backward compatibility condition for any $s, t \in [0,1]$ with $s < t$ and $P_s \in \Delta_N$
\begin{equation}
\label{eq:backward_compatibility_simplicial}
    p_{s|0}(P_s|P_0) = \int_{\Delta_N} p_{s|0,t}(P_s|P_0, P_t) p_{t|0}(P_t|P_0) \rmd P_t . 
\end{equation}
Leveraging classical properties of Dirichlet distributions, we propose such a backward bridge model. For $t\in(0,1]$, write $\beta_t(P_0,\pi)=a_t P_0+b_t\pi$ so that $a_t=c_t\alpha_t$ and $b_t=c_t(1-\alpha_t)$. The quantity $b_t$ can be thought of as the concentration associated to $\pi$ at time $t$. Let $r_{s,t}=\min\left\{1,b_s/b_t\right\}$.

\begin{propositionbeaut}{Simplex Transition}{simplicialtransition}
Let $\kappa \in [0,1)$ and $s, t \in (0,1]$ with $s < t$. Denote $\rho^\kappa_{s,t}=(1-\kappa)r_{s,t}$. For any $P_0, P_s, P_t \in \Delta_N$, consider $\colP{P_{s,t}^\kappa}=(\colP{P_{s,t,1}^\kappa},\cdots,\colP{P_{s,t,N}^\kappa})$, where
\begin{equation}\label{eq:beta_shrinkage_Pt}
    \colP{P_{s,t,i}^\kappa} =\frac{B_i P_{t,i}}{\sum_{j=1}^N B_j P_{t,j}}
\end{equation}
with independent $B_i \sim \Beta\left(\rho_{s,t}^{\kappa}\beta_t(P_0,\pi)_i,(1-\rho_{s,t}^{\kappa})\beta_t(P_0,\pi)_i\right)$,  for $i=1,\cdots,N$. Moreover, let $\colW{W_{s,t}^{\kappa}}$ and $\colV{V_{s,t}^{\kappa}}$ be given by
\begin{equation}\label{eq:intermediate_posterior_simplicial}
    \colW{W_{s,t}^{\kappa}} \sim \Beta\left(\rho_{s,t}^{\kappa}c_t,\;c_s-\rho_{s,t}^{\kappa}c_t\right),\qquad
    \colV{V_{s,t}^{\kappa}} \sim \Dir\left(\beta_s(P_0,\pi) - \rho_{s,t}^{\kappa}\beta_t(P_0,\pi)\right),
\end{equation}
the random variables $(B_i,\colW{W_{s,t}^{\kappa}},\colV{V_{s,t}^{\kappa}})$ being all independent. Finally let
\begin{equation}\label{eq:backwardtransitionPs|P_0,P_t}
    P_s =\colW{W_{s,t}^{\kappa}} \colP{P_{s,t}^\kappa} + (1-\colW{W_{s,t}^{\kappa}}) \colV{V_{s,t}^{\kappa}}.
\end{equation}
Then the induced transition kernel denoted $p_{s|0,t}$ satisfies the compatibility condition~\eqref{eq:backward_compatibility_simplicial}. For the limiting case $\kappa=1$, we define $p_{s|0,t}(P_s|P_0,P_t)=p_{s|0}(P_s|P_0)$.
\end{propositionbeaut}

\begin{figure}[h]
\resizebox{1.00\linewidth}{!}{%
\includegraphics[width=0.99\linewidth]{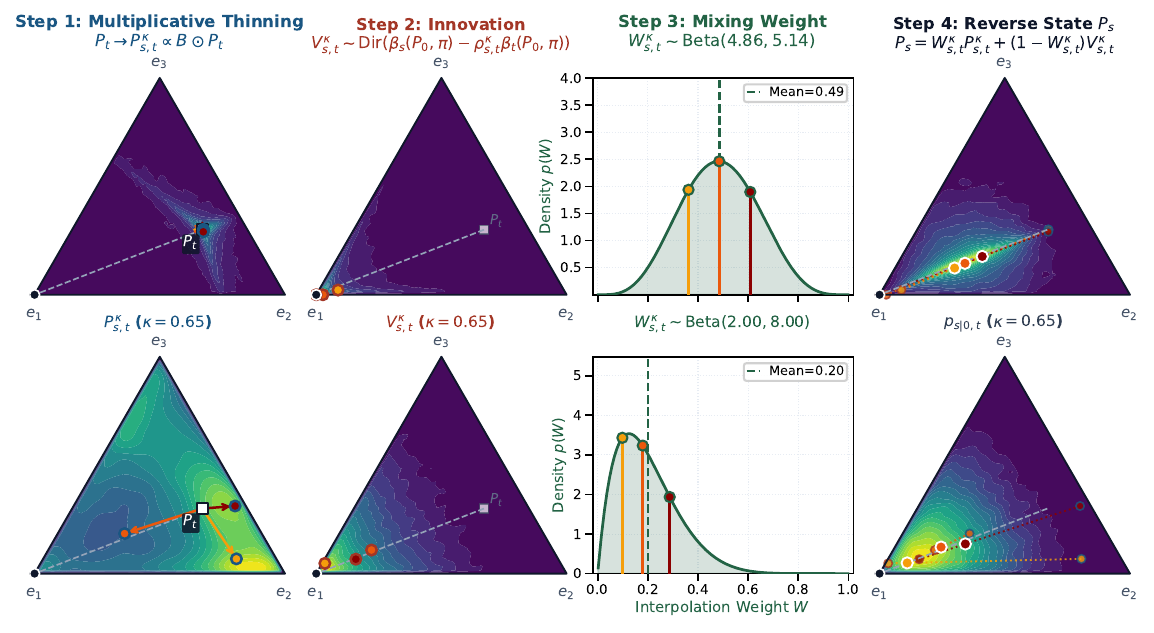}
}
\caption{\textbf{Left to right: different stages of a backward step during inference}. We assume that at the current step, $P_0=e_{x_0}$ where $x_0 \sim \hat{P}_\theta(t, P_t)$ is such that $P_0 = e_1$. The top row corresponds to a low churn $\kappa=0.15$ and the bottom row corresponds to a high churn of $\kappa=0.65$. We assume an interpolation schedule $\alpha_t = 1-t$ and concentration schedule $c_t = \vareps / (1-\alpha_t)$ with $\vareps=4.0$. Additionally, $s=0.40$ and $t=0.70$. The dots with colors yellow, orange, red correspond to three different samples from the current step. In the \uline{first step}, using the churn $\kappa \in [0,1]$, we sample $\colP{P_{s,t}^\kappa}$. This corresponds to a \emph{multiplicative} noising step around $P_t$, with a higher churn corresponding to a higher noise. In the \uline{second step}, we sample the \emph{innovation} variable $\colV{V_{s,t}^\kappa}$, a Dirichlet random variable concentrated around the prediction $P_0$. In the \uline{third step} we sample from a $\Beta$ random variable $\colW{W_{s,t}^\kappa}$ corresponding to the mixing weight in the interpolation between the thinned variable $\colP{P_{s,t}^\kappa}$ and the innovation $\colV{V_{s,t}^\kappa}$. For high churns, the mixing weight $\colW{W_{s,t}^\kappa}$ is close to $0$, meaning that we favor the innovation random variable $\colV{V_{s,t}^\kappa}$. Finally in the \uline{fourth step}, we visualize the interpolation between the innovation $\colV{V_{s,t}^\kappa}$ and the thinned variable $\colP{P_{s,t}^\kappa}$ using the sampled $\Beta$ mixing weight $\colW{W_{s,t}^\kappa}$.}
    \label{fig:sdm-evol}
\end{figure}

The simplex transition outputs a new state $P_s$ in \eqref{eq:backwardtransitionPs|P_0,P_t} which is a random convex combination of an innovation term $\colV{V^\kappa_{s,t}}$ and a thinned version $\colP{P^\kappa_{s,t}}$ of the current state $P_t$, with interpolation weight $\colW{W_{s,t}^\kappa}$. Indeed for $P_t \sim \Dir(\beta_t(P_0,\pi))$, one can show that $\colP{P^\kappa_{s,t}} \sim \Dir(\rho_{s,t}^{\kappa}\beta_t(P_0,\pi))$. The \emph{churn} parameter $\kappa\in[0,1]$ controls how much additional randomness is introduced in the reverse transition at inference time. We present those different steps in \Cref{fig:sdm-evol}.

When $\kappa=0$, we retain the largest fraction of the current state that is compatible with the prescribed concentration schedule. In this case $\rho_{s,t}^0=r_{s,t}$,  $\colW{W_{s,t}^{0}}
    \sim
    \Beta\left(
        r_{s,t}c_t,\,
        c_s-r_{s,t}c_t
    \right)$ and $\colV{V_{s,t}^{0}}
    \sim
    \Dir\left(
        \beta_s(P_0,\pi)
        -
        r_{s,t}\beta_t(P_0,\pi)
    \right)$.
The corresponding update is
\begin{equation}
\label{eq:posterior_kappa_zero}
    P_s
    =
    \colW{W_{s,t}^{0}}\colP{P_{s,t}^{0}}
    +
    (1-\colW{W_{s,t}^{0}})\colV{V_{s,t}^{0}}.
\end{equation}
For the concentration $c_t=\vareps/(1-\alpha_t)$, we have
$b_t=c_t(1-\alpha_t)=\vareps$, hence $r_{s,t}=1$.
In that case $\colP{P_{s,t}^{0}}=P_t$ and
$\colV{V_{s,t}^{0}}=e_{x_0}$, so
\eqref{eq:posterior_kappa_zero} reduces to $P_s=\colW{W_{s,t}^{0}}P_t+ (1-\colW{W_{s,t}^{0}})e_{x_0}$.  
If $\vareps \to +\infty$, \eqref{eq:posterior_kappa_zero} becomes 
\begin{equation}
    P_s = \frac{1-\alpha_s}{1-\alpha_t} P_t + \frac{\alpha_s - \alpha_t}{1-\alpha_t} e_{x_0} , 
\end{equation}
which is the exact deterministic DDIM rule. We refer to \Cref{propbeaut:temperature-limits} for a theoretical analysis of those temperature limits.
Equipped with a compatible backward transition mechanism, we can now define our training and inference algorithms.

\subsection{Training and Inference}\label{sec:training_inference_simplicial}
Define $0 = t_0 < \dots < t_M= 1$. At inference time, we will follow a DDIM approach \citep{song2020denoising}: ideally we would generate data by starting from $P_{t_M}\sim p_1(\cdot)$ and $P_{t_k}\sim p_{t_k|t_{k+1}}(\cdot|P_{t_{k+1}})$ for $k=M-1,...,0$, where for $0\leq s<t\leq 1$
\begin{equation}\label{eq:backward_kernel}
    p_{s|t}(P_s|P_t) = \int p_{s|0,t}(P_s|P_0, P_t) p_{0|t}(P_0 |P_t) \rmd P_0,
\end{equation}
with $p_{0|t}(P_0 |P_t)$ the posterior distribution of $P_0$ given $P_t$. Here $p_1(P_1)=\int p_{1|0}(P_1\mid P_0) p(P_0)\rmd P_0=\Dir(P_1;c_1 \pi)$.
If we had access to the true backward transitions \eqref{eq:backward_kernel}, then this procedure would return samples from $P_0$, hence from the data distribution. Since $P_0$ is supported on the vertices of the simplex, then $p_{0|t}(P_0 \mid P_t)$ is a categorical distribution on those vertices which we approximate by a denoiser $\hat P_\theta(t,P_t)$ such that $\hat P_{\theta,i}(t,P_t) \approx \mathbb{P}(P_0=e_i\mid P_t)=\mathbb{P}(x_0=i\mid P_t)$. For each time step $k=M-1,\dots,0$, we then sample $P_{t_k}\sim p_{t_k \mid 0,t_{k+1}}(\cdot\mid P_0 ,P_{t_{k+1}})$ where $P_0=e_{\widetilde{x}_0}$ for $\widetilde{x}_0 \sim \hat{P}_\theta(t_{k+1},P_{t_{k+1}})$; i.e., we approximate $p_{0|t}(P_0 |P_t)$ by $p^\theta_{0|t}(P_0 |P_t)=\sum_{i=1}^{N}\hat P_{\theta,i}(t,P_t)\updelta_{e_i}(P_0)$.
To learn the denoiser $\hat{P}_\theta$ from $P_t$, we minimize the cross entropy loss
\begin{equation}
\mathcal{L}(\theta)=- \mathbb{E}_{t \sim w(t), x_0 \sim \pdata{}, P_0=e_{x_0}, P_t \sim \text{Dir}(\beta_t(P_0, \pi))} [ \log \big(\hat{P}_\theta(t,P_t)_{x_0}\big) ], \label{eq:xentropy}
\end{equation}
as $\sum_{i=1}^N P_{0,i} \log \big(\hat{P}_\theta(t,P_t)_i\big) =\log \big(\hat{P}_\theta(t,P_t)_{x_0}\big)$. 
In \Cref{app:ELBO}, we show that this loss corresponds to a negative evidence lower bound (ELBO) if $w(t)$ is selected as the uniform distribution on the set $\{t_1,...,t_M\}$.
Given a token embedding matrix $E \in \mathbb{R}^{N \times d}$, the denoiser $\hat{P}_\theta(t, P_t)$ embeds the simplex state $P_t$ either via its expected embedding $P_t^\top E$ or, to avoid this dense matrix product on large vocabularies, via the most likely token embedding $E_{\operatorname{argmax} P_t}$; see \Cref{sec:sdm_input_variants} for details.
We summarize the training of Simplex Diffusion Models in \Cref{alg:training} and inference in \Cref{alg:inference}. These algorithms are described for a single token for the sake of simplicity. The extension to multiple tokens is described in \Cref{sec:multiple_tokens}.

\section{Concentration properties}\label{sec:theoreticalproperties}

\subsection{Temperature Limits}
In \Cref{sec:simplicialdiffusion}, we briefly discussed how the distribution of $P_t$ behaves in both the high and low temperature settings  $c_t=\vareps/(1-\alpha_t)$. We show it here formally.

\begin{propositionbeaut}{Temperature limits}{temperature-limits}
Let $c_t = \vareps / (1-\alpha_t)$ and assume that $P_0=e_{x_0}$.

\medskip
\noindent\textbf{(i) High-temperature limit.}
Fix $t\in(0,1)$. As $\vareps\to0$,
\begin{equation}
    P_t
    \overset{d}{\longrightarrow}
    B_tP_0+(1-B_t)V,
    \label{eq:high_temperature_forward_limit}
\end{equation}
where $B_t\sim\Ber(\alpha_t)$ and $V \sim \sum_{i=1}^N \pi_i\,\updelta_{e_i}$, independently. Thus, $P_t$ is supported on the vertices
of $\Delta_N$. Moreover, fix $0<s<t<1$ and consider the $\kappa=0$ backward bridge
\eqref{eq:posterior_kappa_zero}. For every fixed $P_t\in\Delta_N$,
as $\vareps\to0$,
\begin{equation}
    P_s
    \overset{d}{\longrightarrow}
    W_{s,t}^{0}P_t
    +
    \left(1-W_{s,t}^{0}\right)e_{x_0},
    \qquad
    W_{s,t}^{0}
    \sim
    \Ber\!\left(
    \frac{1-\alpha_s}{1-\alpha_t}
    \right).
    \label{eq:posterior_kappa_zero_epsilon_zero}
\end{equation}
Thus, the limiting transition
remains supported on simplex vertices.

\medskip
\noindent\textbf{(ii) Low-temperature limit.}
Fix $t\in(0,1)$. As $\vareps\to\infty$,
\begin{equation}
    P_t
    \overset{\mathbb P}{\longrightarrow}
    \bar P_t
    :=
    \alpha_tP_0+(1-\alpha_t)\pi.
    \label{eq:low_temperature_deterministic_limit}
\end{equation}
\end{propositionbeaut}

\Cref{propbeaut:temperature-limits} shows that the inverse temperature
interpolates between two very different regimes.  At high
temperature, the simplex-valued state collapses onto the
vertices, and with $\kappa=0$ it reduces the reverse bridge to a transition between the current and clean vertices. Thus the discrete-diffusion
interpretation is recovered for both the probability path and the backward bridge. At low temperature, by contrast, the randomness disappears and the forward state interpolates deterministically between $P_0$ and $\pi$. However, the deterministic limit in
\eqref{eq:low_temperature_deterministic_limit} only describes the
first-order behavior as $\vareps\to\infty$.  In \Cref{sec:gaussian_limit_linear} we identify a Gaussian fluctuation limit result, drawing connections between the limit $\vareps \to +\infty$ and Gaussian diffusion models.
\subsection{Concentration Schedule via Normalized Variance}
\label{sec:concentration_schedule}
While \Cref{propbeaut:temperature-limits} establishes that SDMs interpolate between Discrete Diffusion (high temperature) and the mean interpolation (low temperature), it remains unclear how to select $t \mapsto c_t$ in practice. Here we focus on the setting where $\pi = \1 / N$, i.e., the uniform prior. Rather than tuning an unbounded $c_t \in (0, \infty)$, we define $c_t$ in terms of the fraction $t \mapsto \nu_t \in (0, 1)$ of the maximal total variance ($\Tvar$), defined by 
\begin{equation}
    \Tvar(t) := \operatorname{Tr}\left[ \Cov\left( P_t \mid P_0\right)\right] = \frac{1}{c_t + 1} \left[ 1 - \| \alpha_t P_0 + (1-\alpha_t) \pi \|_2^2 \right] = \frac{N-1}{N}\frac{1-\alpha_t^2}{c_t + 1}.
\end{equation}
Because $c_t \in (0, \infty)$, the total variance is bounded above by $\frac{N-1}{N} (1 - \alpha_t^2)$. Setting \mbox{$\nu_t := \frac{1}{c_t + 1}$} (equivalently \mbox{$c_t = \frac{1}{\nu_t} - 1$}), the total variance is exactly the fraction $\nu_t$ of this bound. In practice, we experiment with two variants (one constant, one piecewise linear):
\begin{equation}
    \nu^\text{cst.}_t = \nu_0, \qquad 
    \nu_t^\text{cst.-lin} = \begin{cases} 
    \nu_0, & \text{if } t < \ell, \\ 
    \nu_0 + \frac{t - \ell}{1-\ell}(\nu_1 - \nu_0), & \text{if } t \ge \ell.
    \end{cases}
    \label{eq:constant_linear_schedule}
\end{equation}
This parameterization is more interpretable than the raw value of $c_t$. At $t=1$, setting $c_1 = N$ induces a uniform distribution on the simplex, and a larger $c_1$ concentrates around $\pi$. However, for $t < 1$, the mean $\alpha_t P_0 + (1-\alpha_t)\pi$ moves towards the vertex $P_0$, and a large $c_t$ collapses $P_t$ onto the mean, which reveals the clean token almost deterministically. In contrast, $\Tvar$ operates on a bounded interval with clear extrema: it approaches its upper bound when $P_t$ collapses onto the vertices, and vanishes when $P_t$ collapses onto its mean ($c_t \to \infty$). 
\section{Experiments}
\begin{wraptable}{r}{0.50\textwidth}
    \vspace{-45pt}
    \caption{\textbf{Accuracy (\%)} on Sudoku in 180 steps, varying the sampling schedule. $^\dagger$trained with uniform time sampler, $^\ddagger$trained with adaptive time sampler. For each column, we \underline{underline} the best and \textbf{bold} the highest accuracy for continuous methods. We show SDM with normalized variance $\nu_t^\text{cst.-lin.}$, with \mbox{$\nu_0=0.4$}, \mbox{$\nu_1=0.75$}, \mbox{$\ell=0.8$}. Values are mean $\pm$ std over 5 seeds.}
    \label{tab:sudoku_hard_table_main}
    \centering
    \small
    \setlength{\tabcolsep}{6pt}
    \renewcommand{\arraystretch}{1.05}
    {%
        \newcommand{\tabrow}{\hspace*{0.8em}}
        \begin{tabular}{l cc}
        \toprule
        Model & Linear & Cosine \\
        \midrule
        \multicolumn{3}{@{}l@{}}{\textit{Discrete}} \\
        \tabrow AR (Greedy) & $\hphantom{0}3.1_{{\color{gray}\scriptscriptstyle\pm 0.5}}$ & -- \\
        \tabrow MDM$^\ddagger$ & $69.4_{{\color{gray}\scriptscriptstyle\pm 4.8}}$ & $69.7_{{\color{gray}\scriptscriptstyle\pm 4.0}}$ \\
        \tabrow MDM$^\ddagger$ (+SC) & $94.7_{{\color{gray}\scriptscriptstyle\pm 0.8}}$ & $\underline{99.3}_{{\color{gray}\scriptscriptstyle\pm 0.2}}$ \\
        \tabrow UDM$^\ddagger$ & $82.1_{{\color{gray}\scriptscriptstyle\pm 1.7}}$ & $82.4_{{\color{gray}\scriptscriptstyle\pm 1.7}}$ \\
        \tabrow UDM$^\dagger$ (+SC) & $97.8_{{\color{gray}\scriptscriptstyle\pm 1.7}}$ & $98.1_{{\color{gray}\scriptscriptstyle\pm 1.2}}$ \\
        \midrule
        \multicolumn{3}{@{}l@{}}{\textit{Continuous}} \\
        \tabrow FLM$^\ddagger$ & $75.3_{{\color{gray}\scriptscriptstyle\pm 3.4}}$ & $74.9_{{\color{gray}\scriptscriptstyle\pm 3.5}}$ \\
        \tabrow $\mathbb{S}$-FLM$^\ddagger$ & $87.3_{{\color{gray}\scriptscriptstyle\pm 1.5}}$ & $87.3_{{\color{gray}\scriptscriptstyle\pm 1.7}}$ \\
        \rowcolor{gray!15}
        \tabrow SDM$^\dagger$ ($\kappa = 1.0$) & $88.4_{{\color{gray}\scriptscriptstyle\pm 2.2}}$ & $88.9_{{\color{gray}\scriptscriptstyle\pm 2.4}}$ \\
        \rowcolor{gray!15}
        \tabrow SDM$^\dagger$ (+SC; $\kappa = 1.0$) & $\underline{\textbf{99.1}}_{{\color{gray}\scriptscriptstyle\pm 0.2}}$ & $\textbf{98.7}_{{\color{gray}\scriptscriptstyle\pm 0.3}}$ \\
        \bottomrule
        \end{tabular}%
    }
    \vspace{-20pt}
\end{wraptable}
We compare SDMs with autoregressive models (AR), masked and uniform diffusion models (MDMs, UDMs), and Flow Language Models (FLM) and Hyperspherical Flows ($\mathbb{S}$-FLM). We evaluate on Sudoku \citep{alp2024sudoku, benhamu2025acceleratedsamplingmaskeddiffusion, kim2026finetuningmaskeddiffusionprovable}, code generation on TinyGSM \citep{liu2023tinygsmachieving80gsm8k, kim2026stoptrainingworstprogressive}, language modeling on OpenWebText (OWT) \citep{Gokaslan2019OpenWeb}, language understanding; \Cref{sec:appendix-setup-lu}, and molecule generation following GenMol (\citealp{lee2025genmol}; \Cref{sec:appendix-setup-genmol}). Prior work mostly compares models by Generative Perplexity (Gen.\ PPL) on OWT, which does not always correlate with downstream performance such as functional correctness in code \citep{deschenaux2024promisesoutlookschallengesdiffusion, feng2025theoreticalbenefitlimitationdiffusion, franca2026hackinggenerativeperplexityunconditional, velickovic2026perplexitytellrightwrong}. Thus, we primarily compare models on TinyGSM.
%
\subsection{Reasoning on Sudoku}
\label{sec:exp_sudoku}
\paragraph{Experimental Setup.}
We train on 200k grids with 30/81 clues revealed and evaluate on 5k unseen examples. We train a modified Diffusion Transformer (DiT) \citep{peebles2023scalablediffusionmodelstransformers} as in \citet{lou2023discrete} with embedding dimension 512, 8 layers and 8 attention heads (28.6M parameters) for 50k steps. We use 180 sampling steps for the diffusion variants. See \Cref{sec:appendix-setup-sudoku} for further details.
\paragraph{Results.}
SDMs are the strongest continuous method on Sudoku and match the Discrete Diffusion models (\Cref{tab:sudoku_hard_table_main}). Without Self-Conditioning (SC), SDMs outperform MDMs and UDMs with ancestral sampling and FLMs, and match $\mathbb{S}$-FLMs. With SC, SDMs reach 99.1\%, outperforming the continuous baselines and within 0.2 points of the best result overall. SDMs also outperform Dirichlet Flow Matching (DFM) by a wide margin (76.7\%; \Cref{tab:sudoku_hard_dfm_apdx}).
Setting $\kappa = 1$ works best. Without SC, increasing $\kappa$ from $0$ to $1$ improves the accuracy from 73.9\% to 88.4\% (\Cref{tab:sudoku_sdm_churn_0_0,tab:sudoku_sdm_churn_0_2,tab:sudoku_sdm_churn_1_0}). With SC, SDMs perform similarly across churns.
Training with the adaptive time sampler (\Cref{sec:adaptive_time_sampler}) generally improves MDMs, UDMs, and SDMs at low churn.
%
%
\subsection{Code Generation on TinyGSM}
\label{sec:exp_tinygsm}
\paragraph{Experimental Setup.}
We train on TinyGSM, a dataset of 11.8M synthetic math word problems with executable Python solutions, and evaluate on the GSM8K test set by executing one generated solution per problem. We tokenize with the SmolLM tokenizer, and use a context length of 512. We train a 12-layer DiT with hidden dimension 768 and 12 attention heads (167.9M parameters) for 250k steps ($\approx 77\%$ for SC variants, to match the training FLOPs; \Cref{sec:ext_sc_loopholing}) with a batch size of 512. See \Cref{sec:appendix-setup-tinygsm} for further details. With ODE sampling, FLM, $\mathbb{S}$-FLM and DFM perform worse than the discrete baselines on TinyGSM, so we defer them to \Cref{app:additional-exp}.
\paragraph{Results.}
We compare SDMs against MDMs and UDMs with the ancestral sampler, with a Predictor-Corrector (PC) sampler (\Cref{sec:pc_sampler}) and with Self-Conditioning (SC; \Cref{sec:ext_sc_loopholing}). Since the expected embedding is costly with a large vocabulary, we also train SDMs that embed only the token $\argmax~P_t$, optionally combined with SC (Argmax + SC; \Cref{sec:sdm_input_variants}).
With the expected embedding and without SC, SDMs outperform all diffusion baselines at 512 steps (\Cref{fig:tinygsm-nfe-vs-accuracy}), reaching $45.8\%$ at $T = 1$ (vs.\ $36.8\%$ for UDM + PC) and $49.0\%$ at $T = 0.1$ (vs.\ $45.8\%$ for MDM + SC). With Argmax + SC, SDMs perform best at $T = 0.1$ from 32 function evaluations (NFE) onward and reach $57.0\%$ at 8k NFE ($50.2\%$ at 512 NFE).
\citet{chemseddine2026sphericalflowssamplingcategorical} train Spherical Flow (SF) with the same setup and find that PC sampling and SC are essential. With 512 NFE ($T = 1$), the accuracy increases from $6.4\%$ with ODE sampling to $41.7\%$ with PC and SC. In this setting, SDMs without SC outperform SFs without SC ($45.8\%$ vs.\ $32.4\%$) and perform similarly to SFs with SC (\Cref{tab:tinygsm_spherical_flows}).
\paragraph{Time Schedule, Churn and Blockwise Generation.}
We train with an adaptive time sampler that upweights noise levels where the model learns most, using a density fitted to the training loss (\Cref{sec:adaptive_time_sampler}), and derive the sampling time grid from the same density (\Cref{sec:adaptive_time_schedule}, \citealp{dieleman2022continuous, raya2026noiseschedulinginformationguidedallocation}).
%
With the expected embedding, this adaptive schedule improves SDMs ($\nu_0=0.2$, $\nu_1=0.5$, $\ell=0.2$, $\kappa=1$) over the linear grid by 9--11 points at 512 steps and 18--19 points at 64 steps; it does not help the Argmax inputs (e.g.\ Argmax + SC at $T=1$ and 512 steps: $40.3\%$ adaptive vs.\ $43.7\%$ linear for $\nu_0=0.2$, $\nu_1=0.75$, $\ell=0.2$) (\Cref{tab:tinygsm_sdm_sdm_t10_s512,tab:tinygsm_sdm_sdm_t01_s512,tab:tinygsm_sdm_sdm_t10_s64,tab:tinygsm_sdm_sdm_t01_s64}). The adaptive time sampler can improve MDM + SC but not UDM + PC, so we report MDM + SC trained with the adaptive time sampler and UDM + PC trained with the uniform time sampler (\Cref{app:additional-exp-tinygsm}).
As on Sudoku, SDMs perform best with $\kappa = 1$ ($45.8\%$ vs.\ $12.6\%$ at $\kappa = 0$, $T = 1$, 512 steps). SFs likewise needs stochastic sampling. Finally, full-sequence SDMs outperform a block-autoregressive variant \citep{han2023ssd, arriola2025block} with block size 32 at matched NFE, despite its stronger left-to-right bias (\Cref{fig:tinygsm_block_vs_fullseq_nfe}).
\paragraph{Distillation and Auto-Guidance.}
We distill the SDM with the expected embedding (no SC; $\nu_t^\text{cst.-lin.}$ with $\nu_0=0.2$, $\nu_1=0.5$, $\ell=0.2$) into a few-step generator (\Cref{app:distillation}). With 8 NFEs, the distilled SDM solves $32.1\%$ of the problems at $T = 1$, more than IDLM \citep{li2026idlm} with 128 NFEs ($21.4\%$; \Cref{tab:tinygsm_distill_idlm}) and the undistilled SDM with 8 NFEs ($\approx 6\%$; \Cref{fig:tinygsm_distilled_vs_sdm_nfe}). The accuracy improves up to $\approx 40\%$ at 128 NFEs and then plateaus. Distillation reduces the diversity, which we measure using the \emph{Abstract Syntax Trees} (AST) of the generated programs (\Cref{app:ast_diversity}). The AST diversity score decreases from $35.3$ for the undistilled SDM with 512 NFEs ($40.0$ with 8 NFEs) to $17.5$ for the distilled SDM with 8 NFEs at $T = 1$ (higher is more diverse; \Cref{tab:tinygsm_baselines_ast_diversity}). See \Cref{app:distillation_exp} for more details. Finally, guiding SDMs away from a weaker checkpoint (Auto-Guidance; \citealp{karras2024guiding}) improves accuracy in all settings we tested, mostly in the low churn and high temperature regime (\Cref{app:autoguidance}).
\subsection{Language Modeling on OpenWebText (OWT)}
\paragraph{Experimental Setup.}
We train on OWT \citep{Gokaslan2019OpenWeb} for unconditional language modeling, and evaluate the generation quality via Generative Perplexity (GenPPL) under a pretrained GPT-2-large model \citep{Radford2019LanguageMA} alongside unigram entropy ($H_1$). We tokenize with the GPT-2 tokenizer, and use a context length of 1024. We train a 12-layer DiT with hidden dimension 768 and 12 attention heads (167.9M parameters) for 1M steps with a batch size of 512. 
\paragraph{Results.} 
With a few sampling interventions that improve the diffusion variants (MDMs, UDMs and SDMs), SDMs reach the same Pareto frontier as MDMs and UDMs. Similar interventions also improve AR models, whose frontier reaches the real OWT validation data. These results highlight that heuristic optimization can surpass the empirical validation point of OWT through simple logit shaping, thereby illustrating the limitation of generative frontiers \citep{pynadath2026generative}.  We briefly describe the four \emph{logit shaping} interventions we introduce. First, we consider top-p sampling as  in \citet{holtzman2019curious}. Second we introduce a power-law logit temperature annealing during sampling similarly to \citet{team2026diffusiongemma,chang2022maskgitmaskedgenerativeimage}. The third intervention which is by far the most impactful is a \emph{sequence frequency penalty}, see \eqref{eq:count_mdm_udm} and \eqref{eq:count_simplex} for more details. Note that this intervention benefits all experimental setups including MDM, UDM, SDM and AR. Finally, SDM is the only model to benefit from a  \emph{local frequency penalty} which bridges the gap between SDMs and state-of-the-art MDM and UDM for text generation, achieving a GenPPL/Entropy trade-off that is close to real validation data (GenPPL/Entropy: $17.0$/$5.46$). More details on the interventions and additional results can be found in \Cref{app:additional-exp-owt}. 
Our main results including the Pareto frontiers for all models we consider are presented in  \Cref{fig:owt_pareto_h1_and_sdm_ablation}. Non-cherry picked samples are shown in \Cref{sec:additional_samples}. 
Finally, in \Cref{sec:language_understanding}, we present additional non-generative evaluation results on language understanding by evaluating our models on ARC-easy \citep{clark2018think}, PIQA \citep{bisk2020piqa} and HellaSwag \citep{zellers2019hellaswag}.
\begin{figure*}[t]
  \centering
  \includegraphics[width=0.49\textwidth]{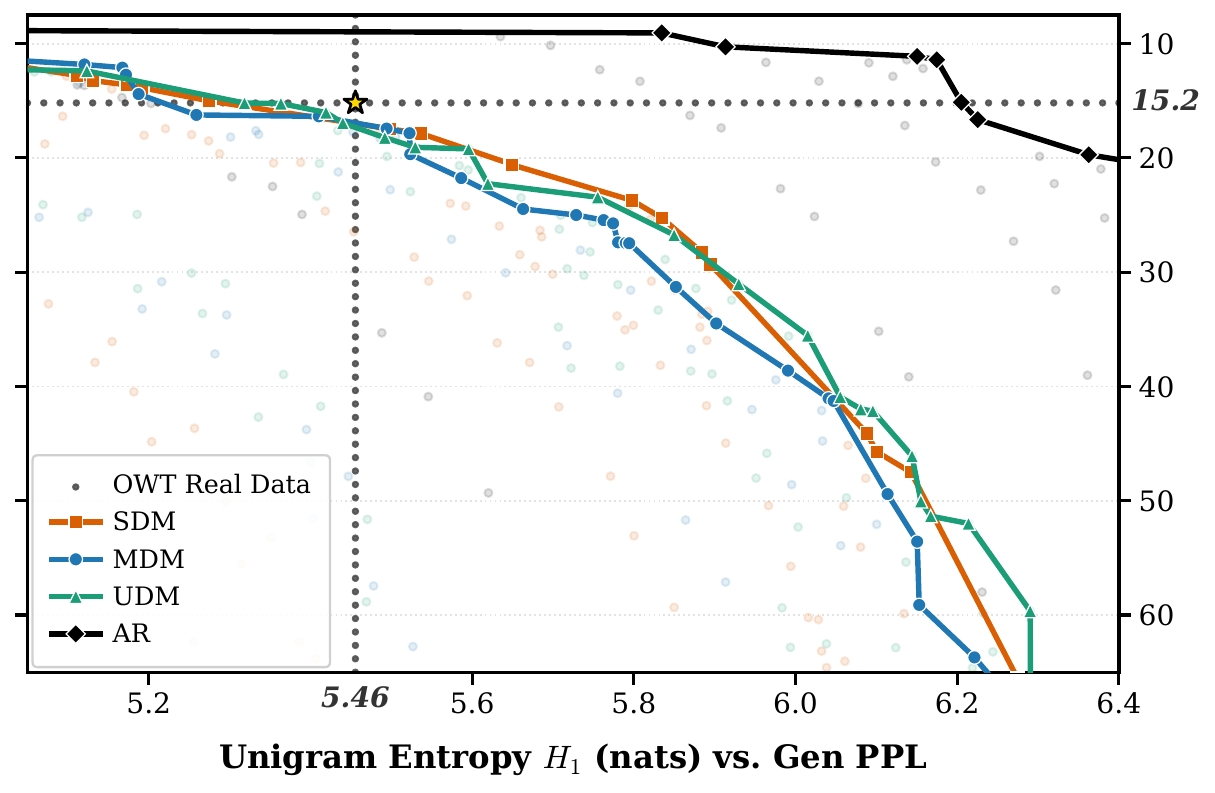}\hfill
  \includegraphics[width=0.49\textwidth]{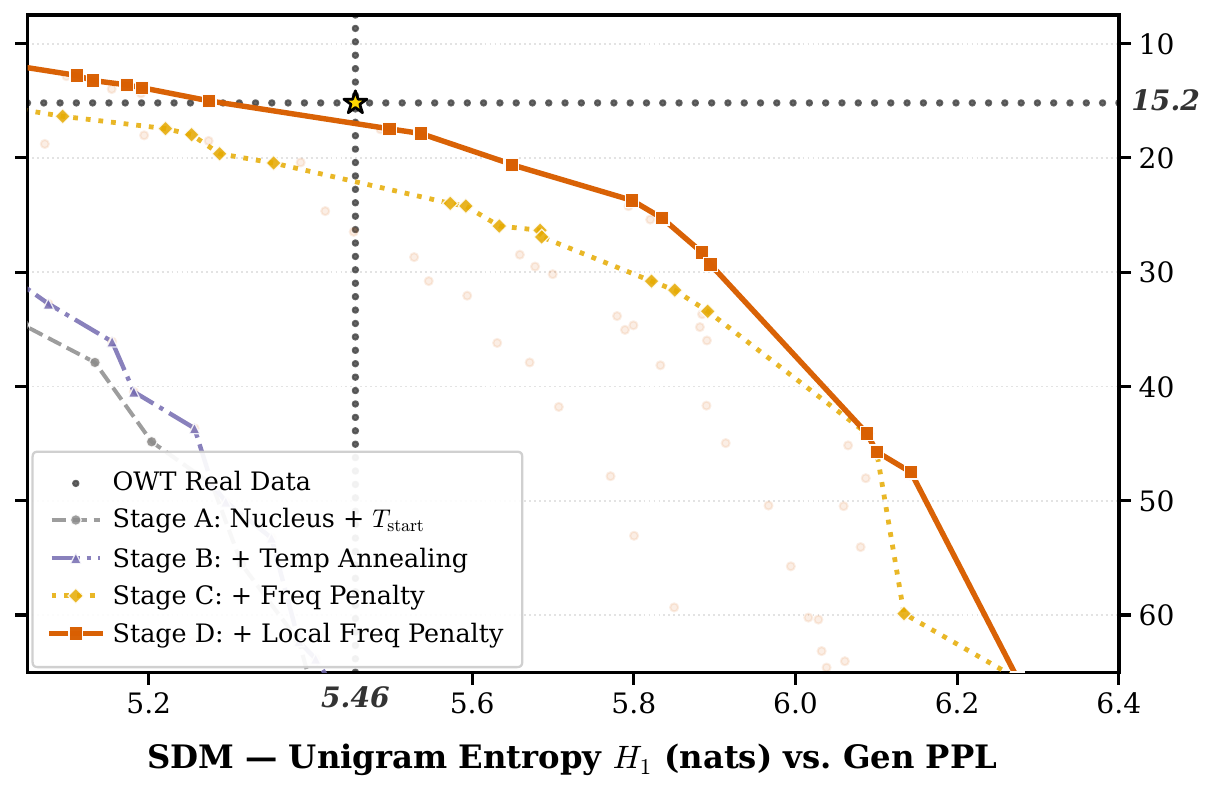}
  \vspace{-6pt}
  \caption{\textbf{OWT Pareto Frontiers of Entropy $H_1$ vs.\ GPT-2 GenPPL.}
Dotted crosshairs and (\textcolor{yellow!60!orange}{$\bigstar$}) denote measured OpenWebText validation data ($H_1 = 5.46\text{ nats}$, $\text{Gen PPL} = 15.2$).
\textbf{Left:} Head-to-head Pareto frontiers comparing 64-step \textbf{\textcolor{orange!85!black}{SDM}}, \textbf{\textcolor{cyan!60!blue}{MDM}}, and \textbf{\textcolor{teal!85!black}{UDM}} against 1024-step \textbf{AR} after sampler interventions.
\textbf{Right:} Progressive ablation on 64-step \textbf{\textcolor{orange!85!black}{SDM}} showing gains from \textbf{Stage A} (nucleus $\text{top-}p$) $\to$ \textbf{Stage B} (power-law logit temperature annealing) $\to$ \textbf{Stage C} (sequence frequency penalty) $\to$ \textbf{Stage D} (local frequency penalty).}
\label{fig:owt_pareto_h1_and_sdm_ablation}
\end{figure*}
\section{Conclusion}
This paper introduces Simplex Diffusion Models (SDMs) for discrete data. Unlike standard Discrete Diffusion models, SDMs can maintain a continuous distributional belief state throughout the generative process, circumventing information collapse even without Self-Conditioning, with which SDMs remain complementary. Our proposed generative model introduces a flexible DDIM-type sampler that does not require discretizing an ODE or SDE.

When compared to strong baselines, SDMs are competitive on unconditional text generation and molecular synthesis, and state-of-the-art on coding. Nevertheless, our experiments remain limited to small models, and whether these conclusions remain valid at scale is an open question. Further experimentation is needed to assess the viability of the approach on large-scale language generation tasks. 
More broadly, SDMs suggest that there is a false dichotomy between discrete and continuous diffusion language models. Discrete Diffusion Models emerge at the high temperature limit, and with the Argmax input (\Cref{sec:sdm_input_variants}), the denoiser processes a discrete sequence, while the generative process itself remains continuous on the simplex.

\section*{AI Use Statement}
We acknowledge the use of AI models during the preparation of this manuscript. AI tools were used to assist in the writing of the manuscript, the derivation of the proofs and the writing of the code. Experiments were run with the use of agentic tools. The main ideas were all introduced without relying on AI tools. Interpretation of the results and first writing of the manuscript were entirely done with no LLM assistance.  The TinyGSM dataset is fully generated by GPT-3.5 (as explained by \citealp{liu2023tinygsmachieving80gsm8k}). The authors take full responsibility for all content, scientific claims, citations, and conclusions presented in this work.

\section*{Reproducibility Statement}
To ensure the full reproducibility of our theoretical and empirical results, we provide comprehensive documentation across the paper and its supplementary material:
\begin{itemize}
    \item \textbf{Theoretical Claims and Proofs:} Complete mathematical proofs for all propositions (\Cref{propbeaut:induced-forward,propbeaut:interpolatingpath,propbeaut:simplicialtransition,propbeaut:temperature-limits}) and extended results, including Gaussian fluctuation limits (\Cref{sec:gaussian_limit_linear}),
    the variational lower bound (\Cref{app:ELBO}), 
    and Simplex DMD distillation (\Cref{app:distillation}), are provided in \Cref{app:proofs}, 
    \Cref{app:ELBO},
    \Cref{app:distillation},    
    with all mathematical assumptions explicitly stated.
    \item \textbf{Algorithms and Pseudocode:} Self-contained algorithmic pseudocode is detailed in \Cref{alg:training} (Training), \Cref{alg:inference} (Inference), \Cref{alg:gamma_sampling} (Accelerated Marsaglia--Tsang Gamma Sampling; see also \Cref{app:samplingdirichlet}), and \Cref{alg:simplicial_dmd} (Simplex Distribution Matching Distillation), along with multi-token and classifier-free guidance extensions in \Cref{sec:multiple_tokens,sec:cfg}.
    \item \textbf{Model Architectures and Hyperparameters:} Detailed neural network architectures (Diffusion Transformer configurations, layer counts, embedding dimensions, attention heads), optimizer parameters (Adam/AdamW learning rates, warmup, EMA decay, gradient clipping), adaptive time schedules (\Cref{sec:adaptive_time_sampler}), denoiser input parameterizations (\Cref{sec:sdm_input_variants}), and variance schedules ($\nu_t$; \Cref{sec:concentration_schedule}) are specified in  \Cref{sec:adaptive_time_sampler,sec:sdm_input_variants,sec:appendix-setup-sudoku,sec:appendix-setup-tinygsm,sec:appendix-setup-lu,sec:appendix-setup-genmol}.
    \item \textbf{Experimental Details and Full Sweeps:} Full experimental ablations and hyperparameter sweeps are documented in \Cref{app:additional-exp} and in the full result tables (\Cref{tab:sudoku_sdm_churn_0_0,tab:sudoku_sdm_churn_0_2,tab:sudoku_sdm_churn_1_0,tab:tinygsm_sdm_sdm_t10_s512,tab:tinygsm_sdm_sdm_t01_s512,tab:tinygsm_sdm_sdm_t10_s64,tab:tinygsm_sdm_sdm_t01_s64}), including exact mean and standard deviation values across random seeds.
\end{itemize}

\bibliography{iclr2027_conference}

\begin{thebibliography}{140}
\providecommand{\natexlab}[1]{#1}
\providecommand{\url}[1]{\texttt{#1}}
\expandafter\ifx\csname urlstyle\endcsname\relax
  \providecommand{\doi}[1]{doi: #1}\else
  \providecommand{\doi}{doi: \begingroup \urlstyle{rm}\Url}\fi

\bibitem[Alamdari et~al.(2023)Alamdari, Thakkar, van~den Berg, Lu, Fusi, Amini, and Yang]{Alamdari2023}
Sarah Alamdari, Nitya Thakkar, Rianne van~den Berg, Alex Lu, Nicolo Fusi, Ava Amini, and Kevin Yang.
\newblock Protein generation with evolutionary diffusion: Sequence is all you need.
\newblock In \emph{NeurIPS Generative AI and Biology (GenBio) Workshop}, 2023.

\bibitem[Allal et~al.(2025)Allal, Lozhkov, Bakouch, Blázquez, Penedo, Tunstall, Marafioti, Kydlíček, Lajarín, Srivastav, Lochner, Fahlgren, Nguyen, Fourrier, Burtenshaw, Larcher, Zhao, Zakka, Morlon, Raffel, von Werra, and Wolf]{allal2025smollm2smolgoesbig}
Loubna~Ben Allal, Anton Lozhkov, Elie Bakouch, Gabriel~Martín Blázquez, Guilherme Penedo, Lewis Tunstall, Andrés Marafioti, Hynek Kydlíček, Agustín~Piqueres Lajarín, Vaibhav Srivastav, Joshua Lochner, Caleb Fahlgren, Xuan-Son Nguyen, Clémentine Fourrier, Ben Burtenshaw, Hugo Larcher, Haojun Zhao, Cyril Zakka, Mathieu Morlon, Colin Raffel, Leandro von Werra, and Thomas Wolf.
\newblock Smollm2: When smol goes big -- data-centric training of a small language model.
\newblock \emph{arXiv preprint arXiv:2502.02737}, 2025.

\bibitem[Alp(2024)]{alp2024sudoku}
Ali Alp.
\newblock Sudoku puzzle generator.
\newblock \url{https://github.com/alicommit-malp/sudoku}, 2024.

\bibitem[Arriola et~al.(2025)Arriola, Gokaslan, Chiu, Yang, Qi, Han, Sahoo, and Kuleshov]{arriola2025block}
Marianne Arriola, Aaron Gokaslan, Justin~T. Chiu, Zhihan Yang, Zhixuan Qi, Jiaqi Han, Subham~Sekhar Sahoo, and Volodymyr Kuleshov.
\newblock Block diffusion: Interpolating between autoregressive and diffusion language models.
\newblock In \emph{International Conference on Learning Representations}, 2025.

\bibitem[Austin et~al.(2021)Austin, Johnson, Ho, Tarlow, and Van Den~Berg]{austin2021structured}
Jacob Austin, Daniel~D Johnson, Jonathan Ho, Daniel Tarlow, and Rianne Van Den~Berg.
\newblock Structured denoising diffusion models in discrete state-spaces.
\newblock In \emph{Advances in Neural Information Processing Systems}, 2021.

\bibitem[Avdeyev et~al.(2023)Avdeyev, Shi, Tan, Dudnyk, and Zhou]{avdeyev2023dirichlet}
Pavel Avdeyev, Chenlai Shi, Yuhao Tan, Kseniia Dudnyk, and Jian Zhou.
\newblock Dirichlet diffusion score model for biological sequence generation.
\newblock In \emph{International Conference on Machine Learning}, 2023.

\bibitem[Baker et~al.(2018)Baker, Fearnhead, Fox, and Nemeth]{baker2018large}
Jack Baker, Paul Fearnhead, Emily Fox, and Christopher Nemeth.
\newblock Large-scale stochastic sampling from the probability simplex.
\newblock \emph{Advances in Neural Information Processing Systems}, 31, 2018.

\bibitem[Batzolis et~al.(2026)Batzolis, Girolami, and Ambrogioni]{batzolis2026cobitlanguagemodelingbitstream}
Georgios Batzolis, Mark Girolami, and Luca Ambrogioni.
\newblock {CoBit}: Language modeling with bitstream diffusion.
\newblock \emph{arXiv preprint arXiv:2605.07013}, 2026.

\bibitem[Ben-Hamu et~al.(2025)Ben-Hamu, Gat, Severo, Nolte, and Karrer]{benhamu2025acceleratedsamplingmaskeddiffusion}
Heli Ben-Hamu, Itai Gat, Daniel Severo, Niklas Nolte, and Brian Karrer.
\newblock Accelerated sampling from masked diffusion models via entropy bounded unmasking.
\newblock In \emph{Advances in Neural Information Processing Systems}, 2025.

\bibitem[Benton et~al.(2024)Benton, Shi, De~Bortoli, Deligiannidis, and Doucet]{benton2024denoising}
Joe Benton, Yuyang Shi, Valentin De~Bortoli, George Deligiannidis, and Arnaud Doucet.
\newblock From denoising diffusions to denoising {M}arkov models.
\newblock \emph{Journal of the Royal Statistical Society Series \textup{B}: Statistical Methodology}, 86\penalty0 (2):\penalty0 286--301, 2024.

\bibitem[Bickerton et~al.(2012)Bickerton, Paolini, Besnard, Muresan, and Hopkins]{bickerton2012quantifying}
G~Richard Bickerton, Gaia~V Paolini, J{\'e}r{\'e}my Besnard, Sorel Muresan, and Andrew~L Hopkins.
\newblock Quantifying the chemical beauty of drugs.
\newblock \emph{Nature chemistry}, 4\penalty0 (2):\penalty0 90--98, 2012.

\bibitem[Bisk et~al.(2020)Bisk, Zellers, Gao, and Choi]{bisk2020piqa}
Yonatan Bisk, Rowan Zellers, Jianfeng Gao, and Yejin Choi.
\newblock Piqa: Reasoning about physical commonsense in natural language.
\newblock In \emph{Proceedings of the AAAI Conference on Artificial Intelligence}, volume~34, pp.\  7432--7439, 2020.

\bibitem[Boget \& Kalousis(2026)Boget and Kalousis]{boget2026unrestrainedsimplexdenoisingdiscrete}
Yoann Boget and Alexandros Kalousis.
\newblock Unrestrained simplex denoising for discrete data: A non-{M}arkovian approach applied to graph generation.
\newblock \emph{arXiv preprint arXiv:2603.28572}, 2026.

\bibitem[Boll et~al.(2024)Boll, Gonzalez-Alvarado, and Schn{\"o}rr]{boll2024generative}
Bastian Boll, Daniel Gonzalez-Alvarado, and Christoph Schn{\"o}rr.
\newblock Generative modeling of discrete joint distributions by e-geodesic flow matching on assignment manifolds.
\newblock \emph{arXiv preprint arXiv:2402.07846}, 2024.

\bibitem[Boll et~al.(2025)Boll, Gonzalez-Alvarado, Petra, and Schn{\"o}rr]{boll2025generative}
Bastian Boll, Daniel Gonzalez-Alvarado, Stefania Petra, and Christoph Schn{\"o}rr.
\newblock Generative assignment flows for representing and learning joint distributions of discrete data.
\newblock \emph{Journal of Mathematical Imaging and Vision}, 67\penalty0 (3):\penalty0 34, 2025.

\bibitem[Bradbury et~al.(2018)Bradbury, Frostig, Hawkins, Johnson, Leary, Maclaurin, Necula, Paszke, Vander{P}las, Wanderman-{M}ilne, and Zhang]{jax2018github}
James Bradbury, Roy Frostig, Peter Hawkins, Matthew~James Johnson, Chris Leary, Dougal Maclaurin, George Necula, Adam Paszke, Jake Vander{P}las, Skye Wanderman-{M}ilne, and Qiao Zhang.
\newblock {JAX}: composable transformations of {P}ython+{N}um{P}y programs, 2018.
\newblock URL \url{http://github.com/google/jax}.

\bibitem[Campbell et~al.(2022)Campbell, Benton, De~Bortoli, Rainforth, Deligiannidis, and Doucet]{campbell2022continuous}
Andrew Campbell, Joe Benton, Valentin De~Bortoli, Thomas Rainforth, George Deligiannidis, and Arnaud Doucet.
\newblock A continuous time framework for discrete denoising models.
\newblock In \emph{Advances in Neural Information Processing Systems}, 2022.

\bibitem[Campbell et~al.(2024)Campbell, Yim, Barzilay, Rainforth, and Jaakkola]{campbell2024generativeflowsdiscretestatespaces}
Andrew Campbell, Jason Yim, Regina Barzilay, Tom Rainforth, and Tommi Jaakkola.
\newblock Generative flows on discrete state-spaces: Enabling multimodal flows with applications to protein co-design.
\newblock In \emph{International Conference on Machine Learning}, 2024.

\bibitem[Chandra et~al.(2026)Chandra, Li, Amin, Ali, Rollins, Ober, Raghu, and Wilson]{chandra2025unification}
Nuria~Alina Chandra, Yucen~Lily Li, Alan~N Amin, Alex Ali, Joshua Rollins, Sebastian~W Ober, Aniruddh Raghu, and Andrew~Gordon Wilson.
\newblock A unification of discrete, {G}aussian, and simplicial diffusion.
\newblock In \emph{International Conference on Learning Representations}, 2026.

\bibitem[Chang et~al.(2022)Chang, Zhang, Jiang, Liu, and Freeman]{chang2022maskgitmaskedgenerativeimage}
Huiwen Chang, Han Zhang, Lu~Jiang, Ce~Liu, and William~T Freeman.
\newblock {MaskGIT}: Masked generative image transformer.
\newblock In \emph{IEEE/CVF Conference on Computer Vision and Pattern Recognition}, 2022.

\bibitem[Chemseddine et~al.(2026)Chemseddine, Kornhardt, and Steidl]{chemseddine2026sphericalflowssamplingcategorical}
Jannis Chemseddine, Gregor Kornhardt, and Gabriele Steidl.
\newblock Spherical flows for sampling categorical data.
\newblock \emph{arXiv preprint arXiv:2605.05629}, 2026.

\bibitem[Chen \& Lipman(2024)Chen and Lipman]{chen2023flow}
Ricky~TQ Chen and Yaron Lipman.
\newblock Flow matching on general geometries.
\newblock In \emph{International Conference on Learning Representations}, 2024.

\bibitem[Chen et~al.(2018)Chen, Rubanova, Bettencourt, and Duvenaud]{chen2018neural}
Ricky~TQ Chen, Yulia Rubanova, Jesse Bettencourt, and David~K Duvenaud.
\newblock Neural ordinary differential equations.
\newblock In \emph{Advances in Neural Information Processing Systems}, 2018.

\bibitem[Chen et~al.(2022)Chen, Zhang, and Hinton]{chen2022analog}
Ting Chen, Ruixiang Zhang, and Geoffrey Hinton.
\newblock Analog bits: Generating discrete data using diffusion models with self-conditioning.
\newblock In \emph{International Conference on Learning Representations}, 2022.

\bibitem[Chen et~al.(2026)Chen, Liang, Sui, Guo, Cheng, You, and Liu]{chen2026langflowcontinuousdiffusionrivals}
Yuxin Chen, Chumeng Liang, Hangke Sui, Ruihan Guo, Chaoran Cheng, Jiaxuan You, and Ge~Liu.
\newblock {LangFlow}: Continuous diffusion rivals discrete in language modeling.
\newblock \emph{arXiv preprint arXiv:2604.11748}, 2026.

\bibitem[Cheng et~al.(2024)Cheng, Li, Peng, and Liu]{cheng2024categorical}
Chaoran Cheng, Jiahan Li, Jian Peng, and Ge~Liu.
\newblock Categorical flow matching on statistical manifolds.
\newblock In \emph{Advances in Neural Information Processing Systems}, 2024.

\bibitem[Cheng et~al.(2025)Cheng, Li, Fan, and Liu]{cheng2025alpha}
Chaoran Cheng, Jiahan Li, Jiajun Fan, and Ge~Liu.
\newblock $\alpha$-flow: A unified framework for continuous-state discrete flow matching models.
\newblock \emph{arXiv preprint arXiv:2504.10283}, 2025.

\bibitem[Clark et~al.(2018)Clark, Cowhey, Etzioni, Khot, Sabharwal, Schoenick, and Tafjord]{clark2018think}
Peter Clark, Isaac Cowhey, Oren Etzioni, Tushar Khot, Ashish Sabharwal, Carissa Schoenick, and Oyvind Tafjord.
\newblock Think you have solved question answering? try arc, the ai2 reasoning challenge.
\newblock \emph{arXiv preprint arXiv:1803.05457}, 2018.

\bibitem[Cobbe et~al.(2021)Cobbe, Kosaraju, Bavarian, Chen, Jun, Kaiser, Plappert, Tworek, Hilton, Nakano, Hesse, and Schulman]{cobbe2021trainingverifierssolvemath}
Karl Cobbe, Vineet Kosaraju, Mohammad Bavarian, Mark Chen, Heewoo Jun, Lukasz Kaiser, Matthias Plappert, Jerry Tworek, Jacob Hilton, Reiichiro Nakano, Christopher Hesse, and John Schulman.
\newblock Training verifiers to solve math word problems.
\newblock \emph{arXiv preprint arXiv:2110.14168}, 2021.

\bibitem[Cox(1972)]{cox1972numerical}
Maurice~G Cox.
\newblock The numerical evaluation of {B}-splines.
\newblock \emph{IMA Journal of Applied Mathematics}, 10\penalty0 (2):\penalty0 134--149, 1972.

\bibitem[Davis et~al.(2024)Davis, Kessler, Petrache, Ceylan, Bronstein, and Bose]{davis2024fisher}
Oscar Davis, Samuel Kessler, Mircea Petrache, {\.I}smail~{\.I} Ceylan, Michael Bronstein, and Avishek~J Bose.
\newblock {Fisher} flow matching for generative modeling over discrete data.
\newblock In \emph{Advances in Neural Information Processing Systems}, 2024.

\bibitem[de~Boor(1972)]{deboor1972calculating}
Carl de~Boor.
\newblock On calculating with b-splines.
\newblock \emph{Journal of Approximation Theory}, 6\penalty0 (1):\penalty0 50--62, 1972.

\bibitem[De~Bortoli et~al.(2022)De~Bortoli, Mathieu, Hutchinson, Thornton, Teh, and Doucet]{de2022riemannian}
Valentin De~Bortoli, Emile Mathieu, Michael Hutchinson, James Thornton, Yee~Whye Teh, and Arnaud Doucet.
\newblock Riemannian score-based generative modelling.
\newblock In \emph{Advances in Neural Information Processing Systems}, 2022.

\bibitem[Deschenaux \& Gulcehre(2024)Deschenaux and Gulcehre]{deschenaux2024promisesoutlookschallengesdiffusion}
Justin Deschenaux and Caglar Gulcehre.
\newblock Promises, outlooks and challenges of diffusion language modeling.
\newblock \emph{arXiv preprint arXiv:2406.11473}, 2024.

\bibitem[Deschenaux \& Gulcehre(2025)Deschenaux and Gulcehre]{deschenaux2025beyond}
Justin Deschenaux and Caglar Gulcehre.
\newblock Beyond autoregression: Fast {LLM}s via self-distillation through time.
\newblock In \emph{International Conference on Learning Representations}, volume 2025, pp.\  5007--5045, 2025.

\bibitem[Deschenaux \& Gulcehre(2026)Deschenaux and Gulcehre]{deschenaux2026languagemodelinghypersphericalflows}
Justin Deschenaux and Caglar Gulcehre.
\newblock Language modeling with hyperspherical flows.
\newblock \emph{arXiv preprint arXiv:2605.11125}, 2026.

\bibitem[Deschenaux et~al.(2026{\natexlab{a}})Deschenaux, Gulcehre, and Sahoo]{deschenaux2026diffusiondualitychapterii}
Justin Deschenaux, Caglar Gulcehre, and Subham~Sekhar Sahoo.
\newblock The diffusion duality, chapter ii: $\psi$-samplers.
\newblock In \emph{International Conference on Learning Representations}, 2026{\natexlab{a}}.

\bibitem[Deschenaux et~al.(2026{\natexlab{b}})Deschenaux, Tran, and Gulcehre]{deschenaux2026partition}
Justin Deschenaux, Lan Tran, and Caglar Gulcehre.
\newblock Partition generative modeling: Masked modeling without masks.
\newblock In \emph{International Conference on Learning Representations}, volume 2026, pp.\  132963--132988, 2026{\natexlab{b}}.

\bibitem[Devlin et~al.(2019)Devlin, Chang, Lee, and Toutanova]{devlin-etal-2019-bert}
Jacob Devlin, Ming-Wei Chang, Kenton Lee, and Kristina Toutanova.
\newblock {BERT}: Pre-training of deep bidirectional transformers for language understanding.
\newblock In \emph{Proceedings of the Conference of the North {A}merican Chapter of the Association for Computational Linguistics: Human Language Technologies, Volume 1 (Long and Short Papers)}, pp.\  4171--4186, 2019.

\bibitem[Devroye(1996)]{devroye1996random}
Luc Devroye.
\newblock Random variate generation in one line of code.
\newblock In \emph{Proceedings of the 28th Conference on Winter Simulation}, pp.\  265--272, 1996.

\bibitem[Dieleman et~al.(2022)Dieleman, Sartran, Roshannai, Savinov, Ganin, Richemond, Doucet, Strudel, Dyer, Durkan, Hawthorne, Leblond, Grathwohl, and Adler]{dieleman2022continuous}
Sander Dieleman, Laurent Sartran, Arman Roshannai, Nikolay Savinov, Yaroslav Ganin, Pierre~H Richemond, Arnaud Doucet, Robin Strudel, Chris Dyer, Conor Durkan, Curtis Hawthorne, R\'emi Leblond, Will Grathwohl, and Jonas Adler.
\newblock Continuous diffusion for categorical data.
\newblock \emph{arXiv preprint arXiv:2211.15089}, 2022.

\bibitem[{DiffusionGemma Team} et~al.(2026){DiffusionGemma Team}, Ta{\"\i}ga, Assiene, Calandriello, Chaabouni, Gante, von Glehn, Keating, Knutsen, Kukla, et~al.]{team2026diffusiongemma}
{DiffusionGemma Team}, Adrien~Ali Ta{\"\i}ga, James Assiene, Daniele Calandriello, Rahma Chaabouni, Jo{\~a}o Gante, Tamara von Glehn, Nate Keating, Chris Knutsen, Martin Kukla, et~al.
\newblock Diffusion{G}emma technical report.
\newblock \emph{arXiv preprint arXiv:2608.00146}, 2026.

\bibitem[Dunn \& Koes(2024)Dunn and Koes]{dunn2024mixed}
Ian Dunn and David~Ryan Koes.
\newblock Mixed continuous and categorical flow matching for 3{D} de novo molecule generation.
\newblock \emph{arXiv preprint arXiv:2404.19739}, 2024.

\bibitem[Ertl \& Schuffenhauer(2009)Ertl and Schuffenhauer]{ertl2009estimation}
Peter Ertl and Ansgar Schuffenhauer.
\newblock Estimation of synthetic accessibility score of drug-like molecules based on molecular complexity and fragment contributions.
\newblock \emph{Journal of Cheminformatics}, 1:\penalty0 1--11, 2009.

\bibitem[Fathi et~al.(2025)Fathi, Scholak, and Noel]{fathi2025unifyingautoregressivediffusionbasedsequence}
Nima Fathi, Torsten Scholak, and Pierre-Andre Noel.
\newblock Unifying autoregressive and diffusion-based sequence generation.
\newblock In \emph{Conference on Language Modeling}, 2025.

\bibitem[Feng et~al.(2025)Feng, Geng, Guan, Wu, Wang, and He]{feng2025theoreticalbenefitlimitationdiffusion}
Guhao Feng, Yihan Geng, Jian Guan, Wei Wu, Liwei Wang, and Di~He.
\newblock Theoretical benefit and limitation of diffusion language model.
\newblock In \emph{Advances in Neural Information Processing Systems}, 2025.

\bibitem[Fishman et~al.(2023)Fishman, Klarner, Mathieu, Hutchinson, and De~Bortoli]{fishman2023metropolis}
Nic Fishman, Leo Klarner, Emile Mathieu, Michael Hutchinson, and Valentin De~Bortoli.
\newblock Metropolis sampling for constrained diffusion models.
\newblock In \emph{Advances in Neural Information Processing Systems}, 2023.

\bibitem[Fishman et~al.(2024)Fishman, Klarner, De~Bortoli, Mathieu, and Hutchinson]{fishman2023diffusion}
Nic Fishman, Leo Klarner, Valentin De~Bortoli, Emile Mathieu, and Michael Hutchinson.
\newblock Diffusion models for constrained domains.
\newblock \emph{Transactions on Machine Learning Research}, 2024.

\bibitem[Floto et~al.(2023)Floto, Jonsson, Nica, Sanner, and Zhu]{floto2023diffusion}
Griffin Floto, Thorsteinn Jonsson, Mihai Nica, Scott Sanner, and Eric~Zhengyu Zhu.
\newblock Diffusion on the probability simplex.
\newblock In \emph{International Conference on Machine Learning Worshop on Sampling and Optimization in Discrete Space}, 2023.

\bibitem[Franca \& Tong(2026)Franca and Tong]{franca2026hackinggenerativeperplexityunconditional}
Antonio Franca and Alexander Tong.
\newblock Hacking generative perplexity: Why unconditional text evaluation needs distributional metrics.
\newblock \emph{arXiv preprint arXiv:2606.08417}, 2026.

\bibitem[Gokaslan \& Cohen(2019)Gokaslan and Cohen]{Gokaslan2019OpenWeb}
Aaron Gokaslan and Vanya Cohen.
\newblock Openwebtext corpus.
\newblock \url{http://Skylion007.github.io/OpenWebTextCorpus}, 2019.

\bibitem[Gourevitch et~al.(2026)Gourevitch, Janati, Shariatian, Simsekli, Moulines, Xing, and Durmus]{gourevitch2026uniformdiffusionmodelsrevisited}
Samson Gourevitch, Yazid Janati, Dario Shariatian, Umut Simsekli, Eric Moulines, Eric~P. Xing, and Alain Durmus.
\newblock Uniform diffusion models revisited: Leave-one-out denoiser and absorbing state reformulation.
\newblock \emph{arXiv preprint arXiv:2605.22765}, 2026.

\bibitem[Grathwohl et~al.(2018)Grathwohl, Chen, Bettencourt, Sutskever, and Duvenaud]{grathwohl2018ffjord}
Will Grathwohl, Ricky~TQ Chen, Jesse Bettencourt, Ilya Sutskever, and David Duvenaud.
\newblock Ffjord: Free-form continuous dynamics for scalable reversible generative models.
\newblock In \emph{International Conference on Learning Representations}, 2018.

\bibitem[Grathwohl et~al.(2021)Grathwohl, Swersky, Hashemi, Duvenaud, and Maddison]{grathwohl2021oopsitookgradient}
Will Grathwohl, Kevin Swersky, Milad Hashemi, David Duvenaud, and Chris Maddison.
\newblock Oops i took a gradient: Scalable sampling for discrete distributions.
\newblock In \emph{International Conference on Machine Learning}, 2021.

\bibitem[Greaves(2026)]{greaves2026extended}
Dylan Greaves.
\newblock Extended one-liners for the {B}eta, {G}amma, and {D}irichlet distributions with shape parameters below one.
\newblock \emph{arXiv preprint arXiv:2604.11199}, 2026.

\bibitem[Gulrajani \& Hashimoto(2023)Gulrajani and Hashimoto]{gulrajani2023likelihood}
Ishaan Gulrajani and Tatsunori~B Hashimoto.
\newblock Likelihood-based diffusion language models.
\newblock In \emph{Advances in Neural Information Processing Systems}, 2023.

\bibitem[Han et~al.(2023)Han, Kumar, and Tsvetkov]{han2023ssd}
Xiaochuang Han, Sachin Kumar, and Yulia Tsvetkov.
\newblock Ssd-{LM}: Semi-autoregressive simplex-based diffusion language model for text generation and modular control.
\newblock In \emph{Proceedings of the 61st Annual Meeting of the Association for Computational Linguistics (Volume 1: Long Papers)}, pp.\  11575--11596, 2023.

\bibitem[Haussmann \& Pardoux(1986)Haussmann and Pardoux]{haussmann1986time}
Ulrich~G Haussmann and Etienne Pardoux.
\newblock Time reversal of diffusions.
\newblock \emph{The Annals of Probability}, 14\penalty0 (3):\penalty0 1188--1205, 1986.

\bibitem[Haviv et~al.(2025)Haviv, Pooladian, Pe'er, and Amos]{haviv2024wasserstein}
Doron Haviv, Aram-Alexandre Pooladian, Dana Pe'er, and Brandon Amos.
\newblock Wasserstein flow matching: Generative modeling over families of distributions.
\newblock In \emph{International Conference on Machine Learning}, 2025.

\bibitem[Ho \& Salimans(2022)Ho and Salimans]{ho2022classifier}
Jonathan Ho and Tim Salimans.
\newblock Classifier-free diffusion guidance.
\newblock \emph{arXiv preprint arXiv:2207.12598}, 2022.

\bibitem[Ho et~al.(2020)Ho, Jain, and Abbeel]{ho2020denoising}
Jonathan Ho, Ajay Jain, and Pieter Abbeel.
\newblock Denoising diffusion probabilistic models.
\newblock In \emph{Advances in Neural Information Processing Systems}, 2020.

\bibitem[Ho et~al.(2022)Ho, Chan, Saharia, Whang, Gao, Gritsenko, Kingma, Poole, Norouzi, Fleet, and Salimans]{ho2022imagen}
Jonathan Ho, William Chan, Chitwan Saharia, Jay Whang, Ruiqi Gao, Alexey Gritsenko, Diederik~P Kingma, Ben Poole, Mohammad Norouzi, David~J Fleet, and Tim Salimans.
\newblock Imagen video: High definition video generation with diffusion models.
\newblock \emph{arXiv preprint arXiv:2210.02303}, 2022.

\bibitem[Holtzman et~al.(2019)Holtzman, Buys, Du, Forbes, and Choi]{holtzman2019curious}
Ari Holtzman, Jan Buys, Li~Du, Maxwell Forbes, and Yejin Choi.
\newblock The curious case of neural text degeneration.
\newblock \emph{arXiv preprint arXiv:1904.09751}, 2019.

\bibitem[Holtzman et~al.(2021)Holtzman, West, Shwartz, Choi, and Zettlemoyer]{holtzman2021surface}
Ari Holtzman, Peter West, Vered Shwartz, Yejin Choi, and Luke Zettlemoyer.
\newblock Surface form competition: Why the highest probability answer isn’t always right.
\newblock In \emph{Proceedings of the 2021 conference on empirical methods in natural language processing}, pp.\  7038--7051, 2021.

\bibitem[Hoogeboom et~al.(2021)Hoogeboom, Nielsen, Jaini, Forr{\'e}, and Welling]{hoogeboom2021argmaxflowsmultinomialdiffusion}
Emiel Hoogeboom, Didrik Nielsen, Priyank Jaini, Patrick Forr{\'e}, and Max Welling.
\newblock Argmax flows and multinomial diffusion: Learning categorical distributions.
\newblock In \emph{Advances in Neural Information Processing Systems}, volume~34, pp.\  12454--12465, 2021.

\bibitem[Hoogeboom et~al.(2026)Hoogeboom, Ruhe, Heek, Mensink, and Salimans]{hoogeboom2026beyond}
Emiel Hoogeboom, David Ruhe, Jonathan Heek, Thomas Mensink, and Tim Salimans.
\newblock Beyond single tokens: Distilling discrete diffusion models via discrete {MMD}.
\newblock \emph{arXiv preprint arXiv:2603.20155}, 2026.

\bibitem[Hu et~al.(2026)Hu, Qiu, Lu, Zhao, Li, Kim, Andreas, and He]{hu2026elfembeddedlanguageflows}
Keya Hu, Linlu Qiu, Yiyang Lu, Hanhong Zhao, Tianhong Li, Yoon Kim, Jacob Andreas, and Kaiming He.
\newblock {ELF}: Embedded language flows.
\newblock \emph{arXiv preprint arXiv:2605.10938}, 2026.

\bibitem[Huang et~al.(2022)Huang, Aghajohari, Bose, Panangaden, and Courville]{huang2022riemannian}
Chin-Wei Huang, Milad Aghajohari, Joey Bose, Prakash Panangaden, and Aaron~C Courville.
\newblock Riemannian diffusion models.
\newblock In \emph{Advances in Neural Information Processing Systems}, 2022.

\bibitem[Jo \& Hwang(2025)Jo and Hwang]{jo2025continuous}
Jaehyeong Jo and Sung~Ju Hwang.
\newblock Continuous diffusion model for language modeling.
\newblock In \emph{Advances in Neural Information Processing Systems}, 2025.

\bibitem[Jo et~al.(2026)Jo, Yoon, Deschenaux, Gulcehre, and Ahn]{jo2025loopholing}
Mingyu Jo, Jaesik Yoon, Justin Deschenaux, Caglar Gulcehre, and Sungjin Ahn.
\newblock Loopholing discrete diffusion: Deterministic bypass of the sampling wall.
\newblock In \emph{International Conference on Learning Representations}, 2026.

\bibitem[Karras et~al.(2022)Karras, Aittala, Aila, and Laine]{karras2022elucidating}
Tero Karras, Miika Aittala, Timo Aila, and Samuli Laine.
\newblock Elucidating the design space of diffusion-based generative models.
\newblock In \emph{Advances in Neural Information Processing Systems}, 2022.

\bibitem[Karras et~al.(2024)Karras, Aittala, Kynk{\"a}{\"a}nniemi, Lehtinen, Aila, and Laine]{karras2024guiding}
Tero Karras, Miika Aittala, Tuomas Kynk{\"a}{\"a}nniemi, Jaakko Lehtinen, Timo Aila, and Samuli Laine.
\newblock Guiding a diffusion model with a bad version of itself.
\newblock In \emph{Advances in Neural Information Processing Systems}, 2024.

\bibitem[Kim et~al.(2026{\natexlab{a}})Kim, Geuter, Alvarez-Melis, Kakade, and Chen]{kim2026stoptrainingworstprogressive}
Jaeyeon Kim, Jonathan Geuter, David Alvarez-Melis, Sham~M. Kakade, and Sitan Chen.
\newblock Stop training for the worst: Progressive unmasking accelerates masked diffusion training.
\newblock In \emph{International Conference on Machine Learning}, 2026{\natexlab{a}}.

\bibitem[Kim et~al.(2026{\natexlab{b}})Kim, Kim, Lee, Pan, Kim, Kakade, and Chen]{kim2026finetuningmaskeddiffusionprovable}
Jaeyeon Kim, Seunggeun Kim, Taekyun Lee, David~Z. Pan, Hyeji Kim, Sham Kakade, and Sitan Chen.
\newblock Fine-tuning masked diffusion for provable self-correction.
\newblock In \emph{International Conference on Machine Learning}, 2026{\natexlab{b}}.

\bibitem[Kingma \& Gao(2023)Kingma and Gao]{kingma2023understandingdiffusionobjectiveselbo}
Diederik Kingma and Ruiqi Gao.
\newblock Understanding diffusion objectives as the {ELBO} with simple data augmentation.
\newblock In \emph{Advances in Neural Information Processing Systems}, 2023.

\bibitem[Kingma et~al.(2021)Kingma, Salimans, Poole, and Ho]{kingma2023variationaldiffusionmodels}
Diederik Kingma, Tim Salimans, Ben Poole, and Jonathan Ho.
\newblock Variational diffusion models.
\newblock In \emph{Advances in Neural Information Processing Systems}, 2021.

\bibitem[Kingma \& Ba(2014)Kingma and Ba]{kingma2014adam}
Diederik~P Kingma and Jimmy Ba.
\newblock Adam: A method for stochastic optimization.
\newblock \emph{arXiv preprint arXiv:1412.6980}, 2014.

\bibitem[Kornilov et~al.(2026)Kornilov, Li, Mavrin, Leonov, Gushchin, Burnaev, Koshelev, and Korotin]{kornilov2025universal}
Nikita Kornilov, David Li, Tikhon Mavrin, Aleksei Leonov, Nikita Gushchin, Evgeny Burnaev, Iaroslav Koshelev, and Alexander Korotin.
\newblock Universal inverse distillation for matching models with real-data supervision (no gans).
\newblock In \emph{International Conference on Learning Representations}, 2026.

\bibitem[Lee et~al.(2026)Lee, Yoo, Agarwal, Shah, Huang, Raghunathan, Hong, Boffi, and Kim]{lee2026one}
Chanhyuk Lee, Jaehoon Yoo, Manan Agarwal, Sheel Shah, Jerry Huang, Aditi Raghunathan, Seunghoon Hong, Nicholas~M Boffi, and Jinwoo Kim.
\newblock One-step language modeling via continuous denoising.
\newblock \emph{arXiv preprint arXiv:2602.16813}, 2026.

\bibitem[Lee et~al.(2025)Lee, Kreis, Veccham, Liu, Reidenbach, Paliwal, Nie, and Vahdat]{lee2025genmol}
Seul Lee, Karsten Kreis, Srimukh~Prasad Veccham, Meng Liu, Danny Reidenbach, Saee Paliwal, Weili Nie, and Arash Vahdat.
\newblock Genmol: A drug discovery generalist with discrete diffusion.
\newblock In \emph{International Conference on Machine Learning}, 2025.

\bibitem[Lezama et~al.(2023)Lezama, Salimans, Jiang, Chang, Ho, and Essa]{lezama2023discrete}
Jose Lezama, Tim Salimans, Lu~Jiang, Huiwen Chang, Jonathan Ho, and Irfan Essa.
\newblock Discrete predictor-corrector diffusion models for image synthesis.
\newblock In \emph{International Conference on Learning Representations}, 2023.

\bibitem[Li et~al.(2026)Li, Gushchin, Abulkhanov, Moulines, Oseledets, Panov, and Korotin]{li2026idlm}
David Li, Nikita Gushchin, Dmitry Abulkhanov, Eric Moulines, Ivan Oseledets, Maxim Panov, and Alexander Korotin.
\newblock {IDLM}: Inverse-distilled diffusion language models.
\newblock In \emph{International Conference on Machine Learning}, 2026.

\bibitem[Lipman et~al.(2023)Lipman, Chen, Ben-Hamu, Nickel, and Le]{lipman2022flow}
Yaron Lipman, Ricky~TQ Chen, Heli Ben-Hamu, Maximilian Nickel, and Matt Le.
\newblock Flow matching for generative modeling.
\newblock In \emph{International Conference on Learning Representations}, 2023.

\bibitem[Liu et~al.(2023{\natexlab{a}})Liu, Bubeck, Eldan, Kulkarni, Li, Nguyen, Ward, and Zhang]{liu2023tinygsmachieving80gsm8k}
Bingbin Liu, Sebastien Bubeck, Ronen Eldan, Janardhan Kulkarni, Yuanzhi Li, Anh Nguyen, Rachel Ward, and Yi~Zhang.
\newblock {TinyGSM}: achieving {$>80$\%} on {GSM8k} with small language models.
\newblock \emph{arXiv preprint arXiv:2312.09241}, 2023{\natexlab{a}}.

\bibitem[Liu et~al.(2023{\natexlab{b}})Liu, Chen, Theodorou, and Tao]{liu2023mirror}
Guan-Horng Liu, Tianrong Chen, Evangelos Theodorou, and Molei Tao.
\newblock Mirror diffusion models for constrained and watermarked generation.
\newblock In \emph{Advances in Neural Information Processing Systems}, 2023{\natexlab{b}}.

\bibitem[Liu et~al.(2025)Liu, Nam, Campbell, St{\"a}rk, Xu, Jaakkola, and G{\'o}mez-Bombarelli]{liu2025thinkgeneratediscretediffusion}
Sulin Liu, Juno Nam, Andrew Campbell, Hannes St{\"a}rk, Yilun Xu, Tommi Jaakkola, and Rafael G{\'o}mez-Bombarelli.
\newblock Think while you generate: Discrete diffusion with planned denoising.
\newblock In \emph{International Conference on Learning Representations}, 2025.

\bibitem[Liu et~al.(2026)Liu, Zhao, Xie, Ye, Jiao, Hu, Cao, and Liu]{liu2026balancingunderstandinggenerationdiscrete}
Yue Liu, Yuzhong Zhao, Zheyong Xie, Qixiang Ye, Jianbin Jiao, Yao Hu, Shaosheng Cao, and Yunfan Liu.
\newblock Balancing understanding and generation in discrete diffusion models.
\newblock In \emph{International Conference on Machine Learning}, 2026.

\bibitem[Lou \& Ermon(2023)Lou and Ermon]{lou2023reflected}
Aaron Lou and Stefano Ermon.
\newblock Reflected diffusion models.
\newblock In \emph{International Conference on Machine Learning}, 2023.

\bibitem[Lou et~al.(2024)Lou, Meng, and Ermon]{lou2023discrete}
Aaron Lou, Chenlin Meng, and Stefano Ermon.
\newblock Discrete diffusion language modeling by estimating the ratios of the data distribution.
\newblock In \emph{International Conference on Machine Learning}, 2024.

\bibitem[Luo et~al.(2023)Luo, Hu, Zhang, Sun, Li, and Zhang]{luo2023diff}
Weijian Luo, Tianyang Hu, Shifeng Zhang, Jiacheng Sun, Zhenguo Li, and Zhihua Zhang.
\newblock {Diff-instruct}: A universal approach for transferring knowledge from pre-trained diffusion models.
\newblock In \emph{Advances in Neural Information Processing Systems}, 2023.

\bibitem[Mahabadi et~al.(2024)Mahabadi, Ivison, Tae, Henderson, Beltagy, Peters, and Cohan]{mahabadi2024tess}
Rabeeh~Karimi Mahabadi, Hamish Ivison, Jaesung Tae, James Henderson, Iz~Beltagy, Matthew~E Peters, and Arman Cohan.
\newblock Tess: Text-to-text self-conditioned simplex diffusion.
\newblock In \emph{Proceedings of the 18th Conference of the European Chapter of the Association for Computational Linguistics (Volume 1: Long Papers)}, pp.\  2347--2361, 2024.

\bibitem[Marsaglia \& Tsang(2000)Marsaglia and Tsang]{marsaglia2000simple}
George Marsaglia and Wai~Wan Tsang.
\newblock A simple method for generating {G}amma variables.
\newblock \emph{ACM Transactions on Mathematical Software}, 26\penalty0 (3):\penalty0 363--372, 2000.

\bibitem[Ng et~al.(2011)Ng, Tian, and Tang]{ng2011dirichlet}
Kai~Wang Ng, Guo-Liang Tian, and Man-Lai Tang.
\newblock \emph{Dirichlet and Related Distributions: Theory, Methods and Applications}.
\newblock John Wiley \& Sons, 2011.

\bibitem[Noutahi et~al.(2024)Noutahi, Gabellini, Craig, Lim, and Tossou]{noutahi2024safe}
Emmanuel Noutahi, Cristian Gabellini, Michael Craig, Jonathan S.~C. Lim, and Prudencio Tossou.
\newblock Gotta be safe: a new framework for molecular design.
\newblock \emph{Digital Discovery}, 3\penalty0 (4):\penalty0 796--804, 04 2024.

\bibitem[Okhotin et~al.(2023)Okhotin, Molchanov, Arkhipkin, Bartosh, Ohanesian, Alanov, and Vetrov]{okhotin2023star}
Andrey Okhotin, Dmitry Molchanov, Vladimir Arkhipkin, Grigory Bartosh, Viktor Ohanesian, Aibek Alanov, and Dmitry~P Vetrov.
\newblock Star-shaped denoising diffusion probabilistic models.
\newblock In \emph{Advances in Neural Information Processing Systems}, 2023.

\bibitem[Ou et~al.(2025)Ou, Nie, Xue, Zhu, Sun, Li, and Li]{ou2024absorbingdiscretediffusionsecretly}
Jingyang Ou, Shen Nie, Kaiwen Xue, Fengqi Zhu, Jiacheng Sun, Zhenguo Li, and Chongxuan Li.
\newblock Your absorbing discrete diffusion secretly models the conditional distributions of clean data.
\newblock In \emph{International Conference on Learning Representations}, 2025.

\bibitem[Peebles \& Xie(2023)Peebles and Xie]{peebles2023scalablediffusionmodelstransformers}
William Peebles and Saining Xie.
\newblock Scalable diffusion models with transformers.
\newblock In \emph{IEEE/CVF International Conference on Computer Vision (ICCV)}, 2023.

\bibitem[Potaptchik et~al.(2026)Potaptchik, Yim, Saravanan, Holderrieth, Vanden-Eijnden, and Albergo]{potaptchik2026discreteflowmaps}
Peter Potaptchik, Jason Yim, Adhi Saravanan, Peter Holderrieth, Eric Vanden-Eijnden, and Michael~S. Albergo.
\newblock Discrete flow maps.
\newblock \emph{arXiv preprint arXiv:2604.09784}, 2026.

\bibitem[Pynadath et~al.(2026{\natexlab{a}})Pynadath, Shi, and Zhang]{pynadath2025candi}
Patrick Pynadath, Jiaxin Shi, and Ruqi Zhang.
\newblock Candi: Hybrid discrete-continuous diffusion models.
\newblock In \emph{International Conference on Machine Learning}, 2026{\natexlab{a}}.

\bibitem[Pynadath et~al.(2026{\natexlab{b}})Pynadath, Shi, and Zhang]{pynadath2026generative}
Patrick Pynadath, Jiaxin Shi, and Ruqi Zhang.
\newblock Generative frontiers: Why evaluation matters for diffusion language models.
\newblock \emph{arXiv preprint arXiv:2604.02718}, 2026{\natexlab{b}}.

\bibitem[Qin et~al.(2025)Qin, Madeira, Thanou, and Frossard]{qin2025defogdiscreteflowmatching}
Yiming Qin, Manuel Madeira, Dorina Thanou, and Pascal Frossard.
\newblock Defog: Discrete flow matching for graph generation.
\newblock In \emph{International Conference on Machine Learning}, 2025.

\bibitem[Radford et~al.(2019)Radford, Wu, Child, Luan, Amodei, and Sutskever]{Radford2019LanguageMA}
Alec Radford, Jeff Wu, Rewon Child, David Luan, Dario Amodei, and Ilya Sutskever.
\newblock Language models are unsupervised multitask learners.
\newblock 2019.
\newblock URL \url{https://api.semanticscholar.org/CorpusID:160025533}.

\bibitem[Ramesh et~al.(2022)Ramesh, Dhariwal, Nichol, Chu, and Chen]{ramesh2022hierarchical}
Aditya Ramesh, Prafulla Dhariwal, Alex Nichol, Casey Chu, and Mark Chen.
\newblock Hierarchical text-conditional image generation with clip latents.
\newblock \emph{arXiv preprint arXiv:2204.06125}, 2022.

\bibitem[Raya et~al.(2026)Raya, Nguyen, Batzolis, Takida, Stancevic, Murata, Lai, Mitsufuji, and Ambrogioni]{raya2026noiseschedulinginformationguidedallocation}
Gabriel Raya, Bac Nguyen, Georgios Batzolis, Yuhta Takida, Dejan Stancevic, Naoki Murata, Chieh-Hsin Lai, Yuki Mitsufuji, and Luca Ambrogioni.
\newblock Noise scheduling as information-guided allocation in diffusion training.
\newblock \emph{arXiv preprint arXiv:2602.18647}, 2026.

\bibitem[Richemond et~al.(2022)Richemond, Dieleman, and Doucet]{richemond2022categorical}
Pierre~H Richemond, Sander Dieleman, and Arnaud Doucet.
\newblock Categorical {SDE}s with simplex diffusion.
\newblock \emph{arXiv preprint arXiv:2210.14784}, 2022.

\bibitem[Rombach et~al.(2022)Rombach, Blattmann, Lorenz, Esser, and Ommer]{rombach2022highresolutionimagesynthesislatent}
Robin Rombach, Andreas Blattmann, Dominik Lorenz, Patrick Esser, and Bj{\"o}rn Ommer.
\newblock High-resolution image synthesis with latent diffusion models.
\newblock In \emph{Proceedings of the IEEE/CVF Conference on Computer Vision and Pattern Recognition}, 2022.

\bibitem[Roos et~al.(2026)Roos, Davis, Eijkelboom, Bronstein, Welling, Ceylan, Ambrogioni, and van~de Meent]{roos2026categoricalflowmaps}
Daan Roos, Oscar Davis, Floor Eijkelboom, Michael~M. Bronstein, Max Welling, {\.I}smail~{\.I}lkan Ceylan, Luca Ambrogioni, and Jan-Willem van~de Meent.
\newblock Categorical flow maps.
\newblock In \emph{International Conference on Machine Learning}, 2026.

\bibitem[Sahoo et~al.(2024)Sahoo, Arriola, Schiff, Gokaslan, Marroquin, Chiu, Rush, and Kuleshov]{sahoo2024simple}
Subham~S Sahoo, Marianne Arriola, Yair Schiff, Aaron Gokaslan, Edgar Marroquin, Justin~T Chiu, Alexander Rush, and Volodymyr Kuleshov.
\newblock Simple and effective masked diffusion language models.
\newblock In \emph{Advances in Neural Information Processing Systems}, 2024.

\bibitem[Sahoo et~al.(2025)Sahoo, Deschenaux, Gokaslan, Wang, Chiu, and Kuleshov]{sahoo2025diffusionduality}
Subham~Sekhar Sahoo, Justin Deschenaux, Aaron Gokaslan, Guanghan Wang, Justin Chiu, and Volodymyr Kuleshov.
\newblock The diffusion duality.
\newblock In \emph{International Conference on Machine Learning}, 2025.

\bibitem[Sahoo et~al.(2026)Sahoo, Lemercier, Yang, Deschenaux, Liu, Thickstun, and Jukic]{sahoo2026scalingmaskeddiffusionlanguage}
Subham~Sekhar Sahoo, Jean-Marie Lemercier, Zhihan Yang, Justin Deschenaux, Jingyu Liu, John Thickstun, and Ante Jukic.
\newblock Scaling beyond masked diffusion language models.
\newblock In \emph{International Conference on Machine Learning}, 2026.

\bibitem[Sakurai et~al.(2026)Sakurai, Pynadath, Hayakawa, Yoon, Yang, Chen, and Xu]{sakurai2026simplex}
Jinya Sakurai, Patrick Pynadath, Satoshi Hayakawa, Jaehong Yoon, Xulei Yang, Nancy~F Chen, and Xun Xu.
\newblock Simplex relaxation for discrete diffusion.
\newblock \emph{arXiv preprint arXiv:2608.10615}, 2026.

\bibitem[Salimans et~al.(2024)Salimans, Mensink, Heek, and Hoogeboom]{salimans2024multistep}
Tim Salimans, Thomas Mensink, Jonathan Heek, and Emiel Hoogeboom.
\newblock Multistep distillation of diffusion models via moment matching.
\newblock In \emph{Advances in Neural Information Processing Systems}, 2024.

\bibitem[Schiff et~al.(2025)Schiff, Sahoo, Phung, Wang, Boshar, Dalla-Torre, Almeida, Rush, Pierrot, and Kuleshov]{schiff2025simpleguidancemechanismsdiscrete}
Yair Schiff, Subham Sahoo, Hao Phung, Guanghan Wang, Sam Boshar, Hugo Dalla-Torre, Bernardo Almeida, Alexander Rush, Thomas Pierrot, and Volodymyr Kuleshov.
\newblock Simple guidance mechanisms for discrete diffusion models.
\newblock In \emph{International Conference on Learning Representations}, 2025.

\bibitem[Shabalin et~al.(2026)Shabalin, Elistratov, Meshchaninov, Sadrtdinov, and Vetrov]{shabalin2026gaussian}
Alexander Shabalin, Simon Elistratov, Viacheslav Meshchaninov, Ildus Sadrtdinov, and Dmitry Vetrov.
\newblock Why {G}aussian diffusion models fail on discrete data?
\newblock In \emph{Third Conference on Language Modeling}, 2026.

\bibitem[Shen et~al.(2026)Shen, Zhao, He, and Lin]{shen2026codarcontinuousdiffusionlanguage}
Junzhe Shen, Jieru Zhao, Ziwei He, and Zhouhan Lin.
\newblock {CoDAR}: Continuous diffusion language models are more powerful than you think.
\newblock \emph{arXiv preprint arXiv:2603.02547}, 2026.

\bibitem[Shi et~al.(2024)Shi, Han, Wang, Doucet, and Titsias]{shi2024simplified}
Jiaxin Shi, Kehang Han, Zhe Wang, Arnaud Doucet, and Michalis Titsias.
\newblock Simplified and generalized masked diffusion for discrete data.
\newblock In \emph{Advances in Neural Information Processing Systems}, 2024.

\bibitem[Sohl-Dickstein et~al.(2015)Sohl-Dickstein, Weiss, Maheswaranathan, and Ganguli]{sohl2015deep}
Jascha Sohl-Dickstein, Eric Weiss, Niru Maheswaranathan, and Surya Ganguli.
\newblock Deep unsupervised learning using nonequilibrium thermodynamics.
\newblock In \emph{International Conference on Machine Learning}, 2015.

\bibitem[Song et~al.(2021{\natexlab{a}})Song, Meng, and Ermon]{song2020denoising}
Jiaming Song, Chenlin Meng, and Stefano Ermon.
\newblock Denoising diffusion implicit models.
\newblock In \emph{International Conference on Learning Representations}, 2021{\natexlab{a}}.

\bibitem[Song et~al.(2021{\natexlab{b}})Song, Sohl-Dickstein, Kingma, Kumar, Ermon, and Poole]{song2020score}
Yang Song, Jascha Sohl-Dickstein, Diederik~P Kingma, Abhishek Kumar, Stefano Ermon, and Ben Poole.
\newblock Score-based generative modeling through stochastic differential equations.
\newblock In \emph{International Conference on Learning Representations}, 2021{\natexlab{b}}.

\bibitem[Song et~al.(2025)Song, Zhang, Pei, Gong, Yu, Zhang, Wang, Zhou, Liu, and Ma]{song2025shortlisting}
Yuxuan Song, Zhe Zhang, Yu~Pei, Jingjing Gong, Qiying Yu, Zheng Zhang, Mingxuan Wang, Hao Zhou, Jingjing Liu, and Wei-Ying Ma.
\newblock Shortlisting model: A streamlined simplex diffusion for discrete variable generation.
\newblock In \emph{Advances in Neural Information Processing Systems}, 2025.

\bibitem[Stark et~al.(2024)Stark, Jing, Wang, Corso, Berger, Barzilay, and Jaakkola]{stark2024dirichlet}
Hannes Stark, Bowen Jing, Chenyu Wang, Gabriele Corso, Bonnie Berger, Regina Barzilay, and Tommi Jaakkola.
\newblock Dirichlet flow matching with applications to {DNA} sequence design.
\newblock In \emph{International Conference on Machine Learning}, 2024.

\bibitem[Su et~al.(2024)Su, Ahmed, Lu, Pan, Bo, and Liu]{su2023roformerenhancedtransformerrotary}
Jianlin Su, Murtadha Ahmed, Yu~Lu, Shengfeng Pan, Wen Bo, and Yunfeng Liu.
\newblock {RoFormer}: Enhanced transformer with rotary position embedding.
\newblock \emph{Neurocomputing}, 568\penalty0 (C), 2024.

\bibitem[Sun et~al.(2023)Sun, Dai, Dai, Zhou, and Schuurmans]{sun2023discretelangevinsamplerwasserstein}
Haoran Sun, Hanjun Dai, Bo~Dai, Haomin Zhou, and Dale Schuurmans.
\newblock Discrete {L}angevin samplers via {W}asserstein gradient flow.
\newblock In \emph{International Conference on Artificial Intelligence and Statistics}, 2023.

\bibitem[Touvron et~al.(2023)Touvron, Lavril, Izacard, Martinet, Lachaux, Lacroix, Rozi{\`e}re, Goyal, Hambro, Azhar, et~al.]{touvron2023llama}
Hugo Touvron, Thibaut Lavril, Gautier Izacard, Xavier Martinet, Marie-Anne Lachaux, Timoth{\'e}e Lacroix, Baptiste Rozi{\`e}re, Naman Goyal, Eric Hambro, Faisal Azhar, et~al.
\newblock Llama: Open and efficient foundation language models.
\newblock \emph{arXiv preprint arXiv:2302.13971}, 2023.

\bibitem[Veličković et~al.(2026)Veličković, Barbero, Perivolaropoulos, Osindero, and Pascanu]{velickovic2026perplexitytellrightwrong}
Petar Veličković, Federico Barbero, Christos Perivolaropoulos, Simon Osindero, and Razvan Pascanu.
\newblock Perplexity cannot always tell right from wrong.
\newblock \emph{arXiv preprint arXiv:2601.22950}, 2026.

\bibitem[Vignac et~al.(2023)Vignac, Krawczuk, Siraudin, Wang, Cevher, and Frossard]{vignac2023digressdiscretedenoisingdiffusion}
Cl\'ement Vignac, Igor Krawczuk, Antoine Siraudin, Bohan Wang, Volkan Cevher, and Pascal Frossard.
\newblock {DiGress}: Discrete denoising diffusion for graph generation.
\newblock In \emph{International Conference on Learning Representations}, 2023.

\bibitem[von Rütte et~al.(2025)von Rütte, Fluri, Ding, Orvieto, Schölkopf, and Hofmann]{vonrutte2025generalizedinterpolatingdiscretediffusion}
Dimitri von Rütte, Janis Fluri, Yuhui Ding, Antonio Orvieto, Bernhard Schölkopf, and Thomas Hofmann.
\newblock Generalized interpolating discrete diffusion.
\newblock In \emph{International Conference on Machine Learning}, 2025.

\bibitem[von Rütte et~al.(2026)von Rütte, Fluri, Pooladzandi, Schölkopf, Hofmann, and Orvieto]{vonrutte2026scalingbehaviordiscretediffusion}
Dimitri von Rütte, Janis Fluri, Omead Pooladzandi, Bernhard Schölkopf, Thomas Hofmann, and Antonio Orvieto.
\newblock Scaling behavior of discrete diffusion language models.
\newblock In \emph{International Conference on Learning Representations}, 2026.

\bibitem[Wang et~al.(2025)Wang, Schiff, Sahoo, and Kuleshov]{wang2026remaskingdiscretediffusionmodels}
Guanghan Wang, Yair Schiff, Subham Sahoo, and Volodymyr Kuleshov.
\newblock Remasking discrete diffusion models with inference-time scaling.
\newblock In \emph{Advances in Neural Information Processing Systems}, 2025.

\bibitem[Wang et~al.(2026)Wang, Wang, Bai, Deng, Lin, and Song]{wang2026generalizeddiscretediffusionselfcorrection}
Linxuan Wang, Ziyi Wang, Yikun Bai, Wei Deng, Guang Lin, and Qifan Song.
\newblock Generalized discrete diffusion with self-correction.
\newblock In \emph{International Conference on Machine Learning}, 2026.

\bibitem[Watson et~al.(2023)Watson, Juergens, Bennett, Trippe, Yim, Eisenach, Ahern, Borst, Ragotte, Milles, et~al.]{watson2023novo}
Joseph~L Watson, David Juergens, Nathaniel~R Bennett, Brian~L Trippe, Jason Yim, Helen~E Eisenach, Woody Ahern, Andrew~J Borst, Robert~J Ragotte, Lukas~F Milles, et~al.
\newblock De novo design of protein structure and function with {RF}diffusion.
\newblock \emph{Nature}, 620\penalty0 (7976):\penalty0 1089--1100, 2023.

\bibitem[Williams et~al.(2026)Williams, Yeom-Song, Hartmann, and Klami]{williams2025simplex}
Bernardo Williams, Victor~M Yeom-Song, Marcelo Hartmann, and Arto Klami.
\newblock Simplex-to-{E}uclidean bijections for categorical flow matching.
\newblock In \emph{Artificial Intelligence and Statistics}, 2026.

\bibitem[Xu et~al.(2022)Xu, Liu, Yan, Cai, Li, and Li]{xu2206learning}
Jin Xu, Xiaojiang Liu, Jianhao Yan, Deng Cai, Huayang Li, and Jian Li.
\newblock Learning to break the loop: Analyzing and mitigating repetitions for neural text generation, 2022.
\newblock \emph{arXiv preprint arXiv:2206.02369}, 2022.

\bibitem[Yang et~al.(2024)Yang, Yang, Hui, Zheng, Yu, Zhou, Li, Li, Liu, Huang, Dong, Wei, Lin, Tang, Wang, Yang, Tu, Zhang, Ma, Yang, Xu, Zhou, Bai, He, Lin, Dang, Lu, Chen, Yang, Li, Xue, Ni, Zhang, Wang, Peng, Men, Gao, Lin, Wang, Bai, Tan, Zhu, Li, Liu, Ge, Deng, Zhou, Ren, Zhang, Wei, Ren, Liu, Fan, Yao, Zhang, Wan, Chu, Liu, Cui, Zhang, Guo, and Fan]{qwen2}
An~Yang, Baosong Yang, Binyuan Hui, Bo~Zheng, Bowen Yu, Chang Zhou, Chengpeng Li, Chengyuan Li, Dayiheng Liu, Fei Huang, Guanting Dong, Haoran Wei, Huan Lin, Jialong Tang, Jialin Wang, Jian Yang, Jianhong Tu, Jianwei Zhang, Jianxin Ma, Jianxin Yang, Jin Xu, Jingren Zhou, Jinze Bai, Jinzheng He, Junyang Lin, Kai Dang, Keming Lu, Keqin Chen, Kexin Yang, Mei Li, Mingfeng Xue, Na~Ni, Pei Zhang, Peng Wang, Ru~Peng, Rui Men, Ruize Gao, Runji Lin, Shijie Wang, Shuai Bai, Sinan Tan, Tianhang Zhu, Tianhao Li, Tianyu Liu, Wenbin Ge, Xiaodong Deng, Xiaohuan Zhou, Xingzhang Ren, Xinyu Zhang, Xipin Wei, Xuancheng Ren, Xuejing Liu, Yang Fan, Yang Yao, Yichang Zhang, Yu~Wan, Yunfei Chu, Yuqiong Liu, Zeyu Cui, Zhenru Zhang, Zhifang Guo, and Zhihao Fan.
\newblock Qwen2 technical report.
\newblock \emph{arXiv preprint arXiv:2407.10671}, 2024.

\bibitem[Yang et~al.(2026)Yang, Guo, Zhang, Sahoo, Chen, Vahdat, Mardani, and Thickstun]{yang2026continuousdiffusionscalescompetitively}
Zhihan Yang, Wei Guo, Shuibai Zhang, Subham~Sekhar Sahoo, Yongxin Chen, Arash Vahdat, Morteza Mardani, and John Thickstun.
\newblock Continuous diffusion scales competitively with discrete diffusion for language.
\newblock \emph{arXiv preprint arXiv:2605.18530}, 2026.

\bibitem[Yin et~al.(2024)Yin, Gharbi, Zhang, Shechtman, Durand, Freeman, and Park]{yin2024one}
Tianwei Yin, Micha{\"e}l Gharbi, Richard Zhang, Eli Shechtman, Fredo Durand, William~T Freeman, and Taesung Park.
\newblock One-step diffusion with distribution matching distillation.
\newblock In \emph{Proceedings of the IEEE/CVF Conference on Computer Vision and Pattern Recognition}, 2024.

\bibitem[Zellers et~al.(2019)Zellers, Holtzman, Bisk, Farhadi, and Choi]{zellers2019hellaswag}
Rowan Zellers, Ari Holtzman, Yonatan Bisk, Ali Farhadi, and Yejin Choi.
\newblock Hellaswag: Can a machine really finish your sentence?
\newblock In \emph{Proceedings of the 57th Annual Meeting of the Association for Computational Linguistics}, pp.\  4791--4800, 2019.

\bibitem[Zhang et~al.(2026)Zhang, von R{\"u}tte, Ding, and Hofmann]{zhang2026denoising}
Jinwei Zhang, Dimitri von R{\"u}tte, Yuhui Ding, and Thomas Hofmann.
\newblock Denoising is not the end: Discrete diffusion language models with self-correction.
\newblock In \emph{Workshop on Latent {\&} Implicit Thinking {\textendash} Going Beyond CoT Reasoning}, 2026.

\bibitem[Zhang \& Shasha(1989)Zhang and Shasha]{zhang1989simple}
Kaizhong Zhang and Dennis Shasha.
\newblock Simple fast algorithms for the editing distance between trees and related problems.
\newblock \emph{SIAM Journal on Computing}, 18\penalty0 (6):\penalty0 1245--1262, 1989.

\bibitem[Zhao et~al.(2025)Zhao, Shi, Chen, Druckmann, Mackey, and Linderman]{zhao2025informedcorrectorsdiscretediffusion}
Yixiu Zhao, Jiaxin Shi, Feng Chen, Shaul Druckmann, Lester Mackey, and Scott Linderman.
\newblock Informed correctors for discrete diffusion models.
\newblock In \emph{Advances in Neural Information Processing Systems}, 2025.

\end{thebibliography}
\bibliographystyle{iclr2027_conference}

\newpage
\appendix

\section*{Organization of the Supplementary}
\startcontents[appendices]
\printcontents[appendices]{l}{1}{\setcounter{tocdepth}{2}}

\newpage

\appendixpart[]{Algorithms and Implementation Details}

\section{Beta, Dirichlet and Gamma Distributions}\label{app:basics}
\subsection{Basic Definitions}
For sake of completeness, we recall elementary facts about the Beta, Dirichlet and Gamma distributions. 
The Beta distribution is a distribution on $[0,1]$. Given $a,b>0$, the associated Beta density $\Beta(x;a,b)$ is given for any $x \in [0,1]$ by 
\begin{equation}
    \Beta(x;a,b) = \frac{\Gamma(a+b)}{\Gamma(a) \Gamma(b)} x^{a-1}(1-x)^{b-1} ,
\end{equation}
where $\Gamma(\cdot)$ is the Gamma function. 
The Dirichlet distribution is a distribution on the simplex $\Delta_N$. Given $\beta = (\beta_1, \dots, \beta_N)$,
the associated Dirichlet density $\Dir(P;\beta)$  is given for any $P \in \Delta_N$ by
\begin{equation}
    \Dir(P;\beta) = \frac{\Gamma(\sum_{i=1}^{N}\beta_i)}{\prod_{i=1}^{N}\Gamma(\beta_i)} \prod_{i=1}^{N} p_i^{\beta_i - 1} ,
\end{equation}
 We have
\begin{equation}\label{eq:meanvarianceDirichlet}
\mathbb{E}[P]=\bar{P}=\left(\frac{\beta_1}{C},\dots,\frac{\beta_N}{C} \right),\quad \Cov[P]=\frac{\text{diag}(\bar{P})-\bar{P}\bar{P}^\top}{C+1},
\end{equation}
where $C=\sum_{i=1}^N \beta_i$ the concentration parameter of the Dirichlet distribution. 
The Gamma distribution is a distribution on $[0,\infty)$ defined for parameters $\alpha,\beta>0$ by
\begin{equation}
    \GammaDist(x;\alpha,\beta) = \frac{\beta^\alpha}{\Gamma(\alpha)}x^{\alpha-1}\exp(-\beta x).
\end{equation}
We will use extensively in practice the following fact. For $N$ independently distributed Gamma random variables $X_1 \sim \GammaDist(\alpha_1,1),\cdots,X_N \sim \GammaDist(\alpha_N,1)$, we have 
\begin{equation}
P=\left(\frac{X_1}{\sum_{i=1}^N X_i},\cdots,\frac{X_N}{\sum_{i=1}^N X_i} \right) \sim \Dir(\alpha_1,\cdots,\alpha_N).
\end{equation}

We will use the following convention throughout. Whenever a Gamma/Beta/Dirichlet parameter lies on the boundary (for instance, some Dirichlet coordinates are zero, or a formula is obtained as the limit $t\downarrow 0$), the corresponding distribution is understood in the natural degenerate/weak-limit sense. In particular, Dirichlet distributions with some zero coordinates are supported on the corresponding face of the simplex, Beta distributions with a zero parameter are degenerate at $0$ or $1$, and formulas involving $t=0$ are interpreted as limits of the same formulas for $t>0$.

\subsection{Some Useful Properties of Dirichlet Distributions}
We now recall two simple results for Dirichlet distributions. To keep the paper self-contained, we provide proofs below without any claim of originality.

\begin{propositionbeaut}{Sum of Dirichlet random variables}{sumdirichlet}
    Let $\alpha, \beta \in \rset^N$ with $\beta_i \geq \alpha_i \geq 0$ for any $i \in \{1, \dots, N\}$ and assume that $X \sim \Dir(\alpha)$.
    Denote $\gamma = \beta - \alpha$ and let $W \sim \Beta(\sum_{i=1}^{N} \alpha_i, \sum_{i=1}^{N} \gamma_i)$ independent of $X$. Let $V \sim \Dir(\gamma)$, independent of $X$ and $W$. Finally, let $Y$ be given by
    \begin{equation}
        Y = W X + (1 - W) V .
    \end{equation}
    Then, we have that $Y \sim \Dir(\beta)$. 
\end{propositionbeaut}

\begin{proof}
    Using the $\GammaDist$ characterization of Dirichlet random variables, there exist independent random variables $(Z_i)_{i=1}^{N}$ such that $Z_i \sim \GammaDist(\alpha_i, 1)$ and $X = \left(Z_i / \sum_{j=1}^{N} Z_j\right)_{i=1}^{N}$. Similarly, there exist independent random variables $(Z_i')_{i=1}^{N}$ such that $Z_i' \sim \GammaDist(\gamma_i, 1)$ and $V = \left(Z_i' / \sum_{j=1}^{N} Z_j'\right)_{i=1}^{N}$. We denote $\bar{Z} = \sum_{i=1}^{N} Z_i$, $\bar{Z}' = \sum_{i=1}^{N} Z_i'$, and $W = \bar{Z} / (\bar{Z} + \bar{Z}')$.
    Note that $\bar{Z}$ and $\bar{Z}'$ are independent of each other and independent of $V$ and $X$. The random variable $W$ is independent of $X$ and $V$ by Lukacs's Proportion-Sum Independence Theorem, see \cite{ng2011dirichlet}. In addition,  $\bar{Z} \sim \GammaDist(\sum_{i=1}^{N}\alpha_i,1)$ and $\bar{Z}' \sim \GammaDist(\sum_{i=1}^{N}\gamma_i,1)$ using the summation property of independent $\Gamma$ distributions. Now using the characterization of $\Beta$ distributions with $\GammaDist$ distributions, we get that $W \sim \Beta(\sum_{i=1}^{N} \alpha_i, \sum_{i=1}^{N} \gamma_i)$.  Finally, we have that
    \begin{equation}
        Y = WX + (1-W) V = \left( \frac{Z_i + Z_i'}{\sum_{j=1}^{N} (Z_j + Z_j')} \right)_{i=1}^N .
    \end{equation}
    We conclude the proof using the summation property of independent $\Gamma$ random variables and the $\Gamma$ characterization of Dirichlet random variables.
\end{proof}

\begin{propositionbeaut}{Dirichlet thinning}{thinningdirichlet}
    Let $X \sim \Dir(\alpha)$ with $\alpha \in \rset^N$ with $\alpha_i \geq 0$ for any $i \in \{1, \dots, N\}$.
    The construction has to be understood on the supporting face $I=\{i:\alpha_i>0\}$; coordinates outside $I$ are fixed at zero and omitted from the normalization. For $\rho \in (0,1]$ and $i \in I$, let $B_i \sim \Beta(\rho \alpha_i, (1-\rho) \alpha_i)$ be mutually independent random variables, also independent of $X$, and define $Y$ by
    \begin{equation}
      Y_i = \frac{B_i X_i}{\sum_{j\in I} B_j X_j}\quad\text{for }i\in I,\qquad Y_i=0\quad\text{for }i\notin I.
    \end{equation}
    Then $Y \sim \Dir(\rho \alpha)$.
\end{propositionbeaut}

\begin{proof}
    By restricting to the supporting face and relabelling its coordinates, it is enough to consider $\alpha_i>0$ for every $i\in\{1,\dots,N\}$. We sample mutually independent variables $Z_i\sim\Gamma(\alpha_i,1)$ and $B_i\sim\Beta(\rho\alpha_i,(1-\rho)\alpha_i)$, and define $X=(Z_i/\sum_{j=1}^{N}Z_j)_{i=1}^{N}$. This defines the joint distribution of $(X,(B_i)_{i=1}^N)$.
    For any $i \in \{1, \dots, N\}$, let $Z_i' = Z_i B_i$ and we have that 
    \begin{equation}
        Y = \left(\frac{B_i X_i}{\sum_{j=1}^{N} B_j X_j} \right)_{i=1}^{N} = \left( \frac{B_i Z_i}{\sum_{j=1}^{N} B_j Z_j} \right)_{i=1}^{N} = \left(\frac{Z_i'}{\sum_{j=1}^{N} Z_j'}\right)_{i=1}^{N} .
    \end{equation}
    We now use the fact that for independent random variables $A \sim \Beta(b, a-b)$ and $B \sim \GammaDist(a, 1)$ then $AB \sim \GammaDist(b, 1)$, using Lukacs's Proportion-Sum Independence Theorem, see \cite{ng2011dirichlet}. So for any $i \in \{1, \dots, N\}$, with $A = B_i$, $B = Z_i$, $a=\alpha_i$ and $b=\rho \alpha_i$.
    It follows that that $(Z_i')_{i=1}^{N}$ is a collection of independent random variables such that $Z_i' \sim \GammaDist(\rho \alpha_i,1)$. Therefore, using once again the $\GammaDist$ representation of Dirichlet random variables, $Y \sim \Dir(\rho \alpha)$.  
\end{proof}

\section{Training and Inference Algorithms}

\subsection{Training}\label{app:training}
The training algorithm is detailed in \Cref{alg:training}.

\begin{algorithm}
\caption{Training of Simplex Diffusion Model}
\label{alg:training}
\begin{algorithmic}[1]
\Require Data distribution $p_\text{data}$, Reference distribution $\pi \in \Delta_N$ (e.g., Uniform), Schedule $(\alpha_t)_{t\in[0,1]}$, Concentration $(c_t)_{t\in(0,1]}$, time-sampling density $w(t)$, and Model $\hat{P}_\theta: \Delta_N \times [0,1] \to \Delta_N$

\While{not converged}
    \State \textbf{1. Sample Data and Time}
    \State Sample $x_0 \sim p_\text{data}$ \Comment{$x_0 \in \{1, \dots, N\}$}
    \State Sample time $t\sim w(t)$ \Comment{Uniform by default}

    \State \textbf{2. Compute Corruption Parameters}
    \State $P_0 \gets e_{x_0}$ \Comment{$P_0 \in \{0,1\}^N$; i.e., $P_0=\text{OneHot}(x_0)$}
    \State $\beta_t \gets c_t\left(\alpha_t P_0+(1-\alpha_t)\pi\right)$ \Comment{Dirichlet parameters $\beta_t \in \rset_+^N$}
    
    \State \textbf{3. Sample Noisy Simplex State $P_t$}
    \State Sample $P_t \sim \Dir(\beta_t)$
    \State \textbf{4. Predict and Optimize}
    \State $\hat{P}_0 \gets \hat{P}_\theta(t,P_t)$ \Comment{Model prediction (Softmax output)}
    \State $\mathcal{L} \gets \text{CrossEntropy}(x_0, \hat{P}_0) = - \log (\hat{P}_{0,x_0})$
    
    \State $\theta \gets \theta - \eta \nabla_\theta \mathcal{L}$ \Comment{Gradient descent step}
\EndWhile
\end{algorithmic}
\end{algorithm}

\subsection{Inference}\label{app:inference}
The inference algorithm is described in \Cref{alg:inference}. For readability, this is presented in the case $0<\rho_{s,t}^\kappa<1$ with non-degenerate Beta and Dirichlet parameters (boundary cases are understood by continuity). 

\begin{algorithm}
\caption{Inference for Simplex Diffusion Model}
\label{alg:inference}
\begin{algorithmic}[1]

\Require Trained Model $\hat{P}_\theta$, Prior $\pi$,
Schedule $(\alpha_t)_{t\in[0,1]}$, Concentration $(c_t)_{t\in [0,1]}$, Churn $\kappa \in [0,1]$, Number of steps $M$, Time sequence $0=t_0<\cdots<t_M=1$.

\State \textbf{1. Initialize}
\State Sample $P_1\sim\Dir(c_1 \pi)$
\Comment{Sample from scaled prior}

\State \textbf{2. Reverse Generative Loop}

\For{$k=M,M-1,\dots,2$}

    \State $t\gets t_k,\qquad s\gets t_{k-1}$

    \State \textbf{a. Sample a Clean Vertex}
    \State Sample $\widetilde x_0\sim\Cat(\hat P_\theta(t,P_{t_k}))$
    \State $P_0\gets e_{\widetilde x_0}$

    \State \textbf{b. Compute
    $p_{s|0,t}(P_s\mid P_0,P_t)$ Parameters}

     \State
    $r_{s,t} \gets \min\left\{1,\dfrac{c_s(1-\alpha_s)}{c_t(1-\alpha_t)}\right\}$

    \State
    $\rho \gets (1-\kappa)r_{s,t}$

    \State
    $a_W \gets \rho c_t$

    \State
    $b_W \gets c_s-\rho c_t$

    \State
    $\beta_V \gets  \beta_s(P_0,\pi) - \rho\beta_t(P_0,\pi)$

    \State \textbf{c. Sample Variables}

    \State Sample
    $W\sim\Beta(a_W,b_W)$
    and
    $V\sim\Dir(\beta_V)$

    \State \textbf{d. Compute $P_{t_k}^{\kappa}$}

    \For{$i\in\{1,\dots,N\}$}
        \State Sample
        $B_i\sim
        \Beta\!\left(
            \rho\beta_t(P_0,\pi)_i,
            (1-\rho)\beta_t(P_0,\pi)_i
        \right)$
    \EndFor

    \State
    $P_{t_k}^{\kappa}
    \gets
    \dfrac{B\odot P_{t_k}}
    {\sum_{i=1}^N B_iP_{t_k,i}}$
    \Comment{Element-wise mult. \& normalize}

    \State \textbf{e. Update State}

    \State
    $P_{t_{k-1}}
    \gets
    WP_{t_k}^{\kappa}
    +(1-W)V$

\EndFor

\State \textbf{3. Endpoint}
\State Sample $x_{\mathrm{data}}\sim\Cat(\hat P_\theta(t_1,P_{t_1}))$
\State $P_{t_0}\gets e_{x_{\mathrm{data}}}$
\State \Return $x_{\mathrm{data}}$

\end{algorithmic}
\end{algorithm}

\section{Several Extensions}

\subsection{Extension to Multiple Tokens}
\label{sec:multiple_tokens}
So far, we have presented Simplex Diffusion Models for a single token taking values in \(\mcx=\{1,\dots,N\}\). We now discuss how to extend the
construction to sequences \(x_0=(x_0^1,\dots,x_0^L)\in \mcx^L\). To obtain a scalable solution, we consider a product of simplices and represent the state at time
\(t\) by
\begin{equation}
P_t = (P_t^1,\dots,P_t^L)\in (\Delta_N)^L,
\qquad
P_0^\ell = e_{x_0^\ell}, \quad \ell\in\{1,\dots,L\}.
\end{equation}
We then define the forward process independently across positions:
\begin{equation}
p_{t\mid 0}(P_t\mid P_0)
=
\prod_{\ell=1}^L
\Dir \bigl(P_t^\ell;\beta_t(P_0^\ell,\pi)\bigr),
\qquad
\beta_t(P_0^\ell,\pi)
=
c_t\bigl(\alpha_tP_0^\ell+(1-\alpha_t)\pi\bigr).
\end{equation}
The denoiser is now a joint sequence model
\begin{equation}
\hat P_\theta : (\Delta_N)^L\times [0,1]\to (\Delta_N)^L,
\qquad
\hat P_\theta(t,P_t)
=
(\hat P_0^1,\dots,\hat P_0^L),
\end{equation}
where each \(\hat P_0^\ell\in\Delta_N\) is the predicted clean distribution at
position \(\ell\). Importantly, although the forward and reverse kernels
factorize across positions, each
\(\hat P_0^\ell\) can depend on the full noisy sequence \(P_t\).

Training is performed with the sequence-level extension of the
cross-entropy objective. At inference time, \Cref{propbeaut:simplicialtransition} is applied independently at each position
after sampling a clean token from each predicted categorical distribution.
More precisely, for any $0<s<t\le 1$, we first draw, conditionally
independently across positions,
\begin{equation}
    \widetilde x_0^\ell\sim\Cat(\hat P_0^\ell),
    \qquad
    \widetilde P_0^\ell=e_{\widetilde x_0^\ell},
    \qquad \ell\in\{1,\dots,L\}.
\end{equation}
We then sample independently across positions
\begin{equation}
W^\kappa_{s,t,\ell}
\sim
\Beta\left(
\rho_{s,t}^\kappa c_t,\,
c_s-\rho_{s,t}^\kappa c_t
\right),
\qquad
V^\kappa_{s,t,\ell}
\sim
\Dir\left(
\beta_s(\widetilde P_0^\ell,\pi)
-
\rho_{s,t}^\kappa\beta_t(\widetilde P_0^\ell,\pi)
\right),
\end{equation}
where
\begin{equation}
\rho_{s,t}^\kappa
=
(1-\kappa)r_{s,t},
\qquad
r_{s,t}
=
\min\left\{
1,
\frac{c_s(1-\alpha_s)}
     {c_t(1-\alpha_t)}
\right\}.
\end{equation}
For each \(i\in\{1,\dots,N\}\), we also sample
\begin{equation}
B_i^\ell
\sim
\Beta\left(
\rho_{s,t}^\kappa\,
\beta_t(\widetilde P_0^\ell,\pi)_i,
(1-\rho_{s,t}^\kappa)\,
\beta_t(\widetilde P_0^\ell,\pi)_i
\right).
\end{equation}
We then define
\begin{equation}
(P_{s,t}^\ell)_i^\kappa
=
\frac{B_i^\ell P_{t,i}^\ell}
{\sum_{j=1}^N B_j^\ell P_{t,j}^\ell},
\qquad
P_s^\ell
=
W^\kappa_{s,t,\ell}(P_{s,t}^\ell)^\kappa
+
(1-W^\kappa_{s,t,\ell})V^\kappa_{s,t,\ell}.
\end{equation}
When \(\rho_{s,t}^\kappa=1\), we set
\((P_{s,t}^\ell)^\kappa=P_t^\ell\) directly and omit the
\(B_i^\ell\) variables.

This yields the factorized mixture reverse transition
\begin{equation}
p_{s\mid t}^\theta(P_s\mid P_t)
=
\prod_{\ell=1}^L
\left\{
\sum_{j=1}^N
\hat P_{0,j}^\ell
p_{s\mid 0,t}(P_s^\ell\mid P^\ell_0=e_j,P_t^\ell)
\right\}.
\end{equation}
After the final positive-time step, we compute
\((\hat P_0^1,\dots,\hat P_0^L)=\hat P_\theta(t_1,P_{t_1})\)
and sample
\(x_{\mathrm{data}}^\ell\sim\Categorical(\hat P_0^\ell)\)
independently across positions conditional on the prediction.

This factorized extension preserves the simplicity of the single-token
construction while allowing the neural predictor \(\hat P_\theta\) to exploit
full sequence context.

\subsection{Classifier-Free Guidance}
\label{sec:cfg}

We now describe how to incorporate classifier-free guidance into the generative
reverse process for the single-token case. The extension to multiple tokens is straightforward. Let \(c\in\mathcal C\) denote an external condition (e.g.,
class label, text prompt, side information). We replace the predictor
\(\hat P_\theta(t,P_t)\) by a conditional predictor $\hat P_\theta(t,c,P_t)$.

\paragraph{Training.}
During training, we sample a dropped condition \(\tilde c\) according to
\[
\tilde c
=
\begin{cases}
c, & \text{with probability } 1-p_{\mathrm{drop}},\\
\varnothing, & \text{with probability } p_{\mathrm{drop}},
\end{cases}
\]
where \(\varnothing\) denotes the null condition. We then optimize the same
objective as before, replacing \(c\) by \(\tilde c\). This gives
\[
L_{\mathrm{CFG}}(\theta)
=
\mathbb E_{t,x_0,P_t,\tilde c}
\left[
-
\log \hat P_{\theta}(t,\tilde c,P_t)_{x_0}
\right].
\]

\paragraph{Guided Plug-In Prediction.}
At inference time, given \(P_t\) and a target condition \(c\), we compute the
unconditional and conditional predictions
\[
\hat P_0^{\mathrm{u}}
=
\hat P_\theta(t,\varnothing,P_t),
\qquad
\hat P_0^{\mathrm{c}}
=
\hat P_\theta(t,c,P_t).
\]
For a guidance scale \(w\ge 0\), we define the guided predictor
\(\tilde P_0\in\Delta_N\) by 
\[
\tilde P_{0,i}
=
\frac{
(\hat P_{0,i}^{\mathrm{u}})^{1-w}
(\hat P_{0,i}^{\mathrm{c}})^w
}{
\sum_{j=1}^N
(\hat P_{0,j}^{\mathrm{u}})^{1-w}
(\hat P_{0,j}^{\mathrm{c}})^w
},
\qquad i\in\{1,\dots,N\}.
\]
The cases \(w=0\), \(w=1\), and
\(w>1\) correspond respectively to unconditional sampling, conditional
sampling, and classifier-free guidance.

\paragraph{Guided Reverse Transition.}
Classifier-free guidance is implemented by replacing the sample $P_0$ from 
\(\hat P_0\) in \Cref{propbeaut:simplicialtransition} and \Cref{alg:inference} by one from its guided version \(\tilde P_0\).

\section{Fast Sampling of Dirichlet Random Variables}\label{app:samplingdirichlet}

In order to sample from Dirichlet distribution efficiently one leverage the representation of Dirichlet distributions in terms of Gamma distributions. More precisely, for $Y \sim \Dir(\alpha)$ with $\alpha=(\alpha_1, \dots, \alpha_N)$ can be efficiently obtained by defining 
\begin{equation}
Y = \left( \frac{X_1}{\sum_{j=1}^N X_j}, \dots,  \frac{X_N}{\sum_{j=1}^N X_j}\right) ,
\end{equation}
with $X_i \sim \GammaDist(\alpha_i, 1)$. Hence in order to sample efficiently from a Dirichlet distribution one must efficiently sample from a $\GammaDist$ distribution. In order to do so, we leverage the Marsaglia--Tsang algorithm, \cite{marsaglia2000simple}; see also \Cref{alg:gamma_sampling}.
However, the current implementation of the Gamma sampler in \Jax \ is either slow or approximate \cite{jax2018github}, see \url{https://github.com/jax-ml/jax/issues/38141}. 

In order to generate a $\GammaDist$ random variable of shape $(d_1, \dots, d_n)$, the current official implementation splits an original key $d = \prod_{i=1}^n d_i$ times. Then it proceeds in using $\texttt{vmap}$ on those $d$ dimensions. The body of this parallelised function also contains key splitting. This incurs large memory overhead which makes the whole function slow. Instead, we propose an exact implementation which does not rely on $\texttt{vmap}$ and therefore do not require creating many keys which significantly reduce the memory overhead. We reproduce a minimal example of this overhead as well as our solution below and provide some performance benchmark on CPU hardware, see \Cref{fig:prng_patterns}. All benchmarks were conducted on an AMD EPYC 7B13 processor (32 physical cores, 2.45GHz base clock, 128MB L3 cache) with 117 GB of DDR4 RAM using \Jax{} and the XLA CPU backend in single precision (\texttt{float32}).

Finally, we highlight that a recent approach  by \cite{greaves2026extended} show that $\GammaDist$ distributions can be generated as extended one-liner, thereby resolving a conjecture of \citet{devroye1996random}. In practice, the approach leverages a Generalized Acceptance-Complement procedure with a Gaussian proposal and a triangular symmetric complement. In particular, the proposed sampler only requires three draws of independent random variables. As a consequence, this new sampler i) decreases drastically the number of calls to pseudo random number generators ii) bypasses the need of a while loop to sample from the distribution. Early benchmarks of this approach forecast a sampling speed-up of 1.5x over our fast implementation of the Marsaglia--Tsang algorithm.

\begin{algorithm}
\caption{Marsaglia \& Tsang Gamma Sampler}
\label{alg:gamma_sampling}
\begin{algorithmic}[1]
\Require PRNG Key, Shape parameter $\alpha > 0$
\Ensure Sample $x \sim \text{Gamma}(\alpha, 1)$

\State \textbf{1. Boosting for Small Alpha}
\If{$\alpha < 1$}
    \State $\alpha' \gets \alpha + 1$ \Comment{Boost variance for stability}
    \State $\text{is\_small} \gets \text{True}$
\Else
    \State $\alpha' \gets \alpha$
    \State $\text{is\_small} \gets \text{False}$
\EndIf

\State \textbf{2. Compute Constants}
\State $d \gets \alpha' - \frac{1}{3}$
\State $c \gets \frac{1}{\sqrt{9d}}$

\State \textbf{3. Rejection Sampling Loop}
\While{sample not accepted}
    \State \textbf{a. Generate Candidates}
    \State Sample $z \sim \mathcal{N}(0, 1)$ \Comment{Standard Normal}
    \State Sample $u \sim \mathcal{U}(0, 1)$ \Comment{Uniform}
    
    \State $v \gets 1 + c \cdot z$
    \If{$v \le 0$} \State \textbf{continue} \Comment{Reject negative support}
    \EndIf
    
    \State $v \gets v^3$
    \State $x_{\text{prop}} \gets d \cdot v$ \Comment{Proposed sample}
    
    \State \textbf{b. Acceptance Checks}
    \State \Comment{Squeeze Test (Fast Check)}
    \If{$u < 1 - 0.0331 \cdot z^4$}
        \State \textbf{break}
    \EndIf
    
    \State \Comment{Log-Likelihood Test (Exact Check)}
    \If{$\log(u) < 0.5 z^2 + d(1 - v + \log(v))$}
        \State \textbf{break}
    \EndIf
\EndWhile

\State \textbf{4. Final Correction}
\If{$\text{is\_small}$ is True}
    \State Sample $u_{\text{boost}} \sim \mathcal{U}(0, 1)$
    \State $x \gets x_{\text{prop}} \cdot u_{\text{boost}}^{1/\alpha}$ \Comment{Apply boosting correction}
\Else
    \State $x \gets x_{\text{prop}}$
\EndIf

\State \Return $x$
\end{algorithmic}
\end{algorithm}

\begin{figure*}[t]
\centering
\begin{minipage}[t]{0.48\textwidth}
\textbf{(a) JAX Native Pattern (\texttt{vmap} over scalar PRNG)}
\begin{lstlisting}[style=pythonstyle]
@partial(jax.jit, static_argnames=("N", "d"))
def jax_native_pattern(key, N: int, d: int):
  keys = jax.random.split(key, N)
  def _body(k):
    k1, k2 = jax.random.split(k, 2)
    z = jax.random.normal(k1, (d,))
    u = jax.random.uniform(k2, (d,))
    return z, u
  return jax.vmap(_body)(keys)
\end{lstlisting}
\end{minipage}
\hfill
\begin{minipage}[t]{0.48\textwidth}
\textbf{(b) Vectorized Pattern (Counter-based PRNG)}
\begin{lstlisting}[style=pythonstyle]
@partial(jax.jit, static_argnames=("N", "d"))
def vectorized_pattern(key, N: int, d: int):
  k1, k2 = jax.random.split(key, 2)
  z = jax.random.normal(k1, (N, d))
  u = jax.random.uniform(k2, (N, d))
  return z, u
\end{lstlisting}
\end{minipage}
\caption{Comparison of PRNG generation patterns in \textsc{Jax}. Pattern (a) derives per-element keys inside \texttt{vmap}, creating $120\,\text{MB}$ of DRAM traffic across Threefry rounds. Pattern (b) splits a single scalar key in CPU registers and generates $N \times d$ values via counter indexing, achieving a $3.8\times$ speedup with $N=10^6$ and $d=1$.}
\label{fig:prng_patterns}
\end{figure*}

\newpage

\appendixpart[]{Theory and Generalizations}

\section{Proofs and Additional Results}\label{app:proofs}

\subsection{Proof of \Cref{propbeaut:induced-forward}}
We have that for any $t \in [0,1]$
\begin{align}
    p_{t|0}(x_t|P_0) &= \int_{\Delta_N} p_{t|0}(x_t|P_t, P_0) p_{t|0}(P_t|P_0) \rmd P_t \\
    &= \int_{\Delta_N} P_{t,x_t} p_{t|0}(P_t|P_0) \rmd P_t = \mathbb{E}_{P_t|P_0}[P_{t,x_t}]. 
\end{align}
From \eqref{eq:meanvarianceDirichlet}, the mean of the Dirichlet distribution is given for any $x_t \in \mcx$ by
\begin{align}
    \mathbb{E}_{P_t|P_0}[P_{t,x_t}] = \alpha_t \updelta_{x_0}(x_t) + (1-\alpha_t) \pi_{x_t},
\end{align}
which concludes the proof.

\subsection{Proof of \Cref{propbeaut:interpolatingpath}}
Fix $t\in(0,1)$. Recall that $P_0=e_{x_0}$, which, under our boundary convention, may be viewed as the degenerate Dirichlet random
variable $P_0 \sim \Dir(c_t\alpha_t e_{x_0})$.
Let $W_t \sim \Beta\left(c_t\alpha_t,c_t(1-\alpha_t)\right)$ and $V_t\sim\Dir\left(c_t(1-\alpha_t)\pi\right)$ independently, and define
\begin{equation}
\label{eq:interpolation_w_appendix}
    P_t=W_tP_0+(1-W_t)V_t.
\end{equation}
The result is then an immediate application of
\Cref{propbeaut:sumdirichlet} with
\begin{equation}
    \alpha \leftarrow c_t\alpha_t e_{x_0},
    \qquad
    \beta \leftarrow
    c_t\left(\alpha_t e_{x_0}+(1-\alpha_t)\pi\right),
    \qquad
    \gamma=\beta-\alpha
    \leftarrow c_t(1-\alpha_t)\pi.
\end{equation}
Hence
\begin{equation}
    P_t
    \sim
    \Dir\left(
        c_t\left(\alpha_tP_0+(1-\alpha_t)\pi\right)
    \right)
    =
    \Dir\left(\beta_t(P_0,\pi)\right),
\end{equation}
which concludes the proof.

\subsection{Proof of \Cref{propbeaut:simplicialtransition}}
Since $\alpha_s\geq\alpha_t$, the definition of $r_{s,t}$ ensures that
\begin{equation}
    \gamma_{s,t}^{\kappa}(P_0,\pi)=\beta_s(P_0,\pi) - \rho_{s,t}^{\kappa}\beta_t(P_0,\pi) \geq0
\end{equation}
coordinatewise. We write $\rho=\rho_{s,t}^{\kappa}$.

We first apply \Cref{propbeaut:thinningdirichlet}. Since $P_t\mid P_0 \sim \Dir\bigl(\beta_t(P_0,\pi)\bigr)$ and $ B_i \sim \Beta\left(\rho\beta_t(P_0,\pi)_i, (1-\rho)\beta_t(P_0,\pi)_i \right)$ independently, the vector defined in
\eqref{eq:beta_shrinkage_Pt} satisfies
\begin{equation}
    P_{s,t}^{\kappa}\mid P_0 \sim \Dir\left(\rho\beta_t(P_0,\pi) \right).
\end{equation}
By construction, $V_{s,t}^{\kappa} \sim \Dir\left( \gamma_{s,t}^{\kappa}(P_0,\pi)\right)$ independently of $P_{s,t}^{\kappa}$, and $W_{s,t}^{\kappa} \sim \Beta\left( \sum_i \rho\beta_t(P_0,\pi)_i,  \sum_i \gamma_{s,t}^{\kappa}(P_0,\pi)_i \right)$.
Since $ \sum_i \rho\beta_t(P_0,\pi)_i=\rho c_t$ and $\sum_i \gamma_{s,t}^{\kappa}(P_0,\pi)_i=c_s-\rho c_t$, this is exactly the distribution of $W_{s,t}^{\kappa}$ in \eqref{eq:intermediate_posterior_simplicial}.

We can therefore apply \Cref{propbeaut:sumdirichlet} with $X\leftarrow P_{s,t}^{\kappa}$, $\alpha\leftarrow \rho\beta_t(P_0,\pi)$ and $\beta\leftarrow \beta_s(P_0,\pi)$ which gives
\begin{equation}
    P_s = W_{s,t}^{\kappa}P_{s,t}^{\kappa} + (1-W_{s,t}^{\kappa})V_{s,t}^{\kappa} \sim \Dir\bigl(\beta_s(P_0,\pi)\bigr).
\end{equation}
This is exactly the compatibility condition
\eqref{eq:backward_compatibility_simplicial}.

\subsection{A Representation of $P_t$}\label{sec:link_with_self_cond}
We provide here an explicit representation of $P_t$ under exact DDIM sampling.

\begin{propositionbeaut}{$P_t$ as a random convex combination}{selfconditioning}
Consider $c_t = \vareps / (1-\alpha_t)$. Let $\kappa=0$  and $(t_i)_{i=0}^{M}$ with $M \in \nset$ such that $0=t_0$ $< \dots < t_M=1$. Let $P_{t_M} \sim p_{t_M}$. For $i=M-1,\dots,0$, conditionally on $P_{t_{i+1}}$, sample $P_0^i\sim p_{0|t_{i+1}}(\cdot\mid P_{t_{i+1}})$, write $P_0^i=e_{x_0^i}$, and then sample $P_{t_i}$ from $p_{t_i|0,t_{i+1}}(\cdot\mid P_0^i,P_{t_{i+1}})$, equivalently using \eqref{eq:posterior_kappa_zero}. For any $i \in \{0, \dots, M\}$, we have $P_{t_i} \sim p_{t_i}$ and
\begin{equation}\label{eq:selfconditioningrep}
P_{t_i} = \sum_{j=i}^{M-1} L_{i,j} e_{x_0^j} + L_{i,M} P_{t_M},
\end{equation}
with $L_i = (L_{i,j})_{j=i}^M$ and $L_i \sim \Dir\left( c_{t_i} - c_{t_{i+1}}, \dots, c_{t_{M-1}} - c_{t_M}, c_{t_M} \right)$.
 \end{propositionbeaut}
 
Equation \eqref{eq:selfconditioningrep} shows that $P_{t_i}$ is a random convex combination of the successive posterior clean-state draws $(x^j_0)_{j=i}^{M-1}$ and the initial noise $P_{t_M}$. For the plug-in transitions, this exact representation need not hold, but it motivates the interpretation of $P_t$ as a belief-memory state. In \Cref{app:corollaryselfconditioning}, we investigate the high and low temperature limits of this representation. 

\begin{proof}
We first verify the marginal claim. It holds at $i=M$ by assumption. If $P_{t_{i+1}}\sim p_{t_{i+1}}$ and $P_0^i\sim p_{0|t_{i+1}}(\cdot\mid P_{t_{i+1}})$, then $(P_0^i,P_{t_{i+1}})$ has joint law $p(P_0)p_{t_{i+1}|0}(P_{t_{i+1}}\mid P_0)$. Integrating the exact kernel $p_{t_i|0,t_{i+1}}$ and using the compatibility condition \eqref{eq:backward_compatibility_simplicial} gives $P_{t_i}\sim p_{t_i}$. Backward induction therefore proves $P_{t_i}\sim p_{t_i}$ for every $i$.

The proof relies on the stick-breaking property of the Dirichlet distribution (a direct consequence of \Cref{propbeaut:sumdirichlet}). If a random vector $X \sim \Dir(\alpha_1, \dots, \alpha_K)$ and a random variable $W \sim \Beta(\sum_{k=1}^K \alpha_k, \alpha_0)$ are independent, then the augmented vector $(1-W, W X_1, \dots, W X_K)$ is distributed as $\Dir(\alpha_0, \alpha_1, \dots, \alpha_K)$. We proceed by backward induction on $i$, from $i=M-1$ down to $0$. To simplify notation, we define the parameter sequence $\gamma_i = \frac{\vareps}{1-\alpha_{t_i}} - \frac{\vareps}{1-\alpha_{t_{i+1}}}$ for $i \in \{0, \dots, M-1\}$, and $\gamma_M = \frac{\vareps}{1-\alpha_{t_M}}$. We start with $i = M-1$. By \eqref{eq:posterior_kappa_zero}, we have:
\begin{equation}
    P_{t_{M-1}} = W_{t_{M-1}, t_M}^{0} P_{t_M} + (1 - W_{t_{M-1}, t_M}^{0}) e_{x_0^{M-1}} ,
\end{equation}
where $W_{t_{M-1}, t_M}^{0} \sim \Beta\left(\gamma_M, \gamma_{M-1}\right)$. The vector $(1 - W_{t_{M-1}, t_M}^{0}, W_{t_{M-1}, t_M}^{0})$ follows a Dirichlet distribution with parameters $(\gamma_{M-1}, \gamma_M)$. Setting $L_{M-1, M-1} = 1 - W_{t_{M-1}, t_M}^{0}$ and $L_{M-1, M} = W_{t_{M-1}, t_M}^{0}$ satisfies the proposition for the base case.
Now assume the proposition holds for some $i+1 \leq M-1$; i.e.,
\begin{equation}
    P_{t_{i+1}} = \sum_{j=i+1}^{M-1} L_{i+1,j} e_{x_0^j} + L_{i+1,M} P_{t_M} ,
\end{equation}
where $L_{i+1}^\vareps \sim \Dir(\gamma_{i+1}, \dots, \gamma_M)$. Note that the sum of these concentration parameters is exactly $\sum_{k=i+1}^M \gamma_k = \frac{\vareps}{1-\alpha_{t_{i+1}}}$.
So using \Cref{propbeaut:simplicialtransition} and applying the backward transition to step $i$, we have:
\begin{equation}
    P_{t_i} = W_{t_i, t_{i+1}}^{0} P_{t_{i+1}} + (1 - W_{t_i, t_{i+1}}^{0}) e_{x_0^i} ,
\end{equation}
where $W_{t_i, t_{i+1}}^{0} \sim \Beta\left(\frac{\vareps}{1-\alpha_{t_{i+1}}}, \gamma_i\right) = \Beta\left(\sum_{k=i+1}^M \gamma_k, \gamma_i\right)$.
Therefore, we get 
\begin{equation}
    P_{t_i} = (1 - W_{t_i, t_{i+1}}^{0}) e_{x_0^i} + \sum_{j=i+1}^{M-1} \left( W_{t_i, t_{i+1}}^{0} L_{i+1,j} \right) e_{x_0^j} + \left( W_{t_i, t_{i+1}}^{0} L_{i+1,M} \right) P_{t_M} .
\end{equation}
We identify the new weights $L_i$ as follows:
$L_{i,i} = 1 - W_{t_i, t_{i+1}}^{0}$ and 
$L_{i,j} = W_{t_i, t_{i+1}}^{0} L_{i+1,j} \quad \text{for } j \in \{i+1, \dots, M\}$.
Because $L_{i+1} \sim \Dir(\gamma_{i+1}, \dots, \gamma_M)$ and $W_{t_i, t_{i+1}}^{0} \sim \Beta\left(\sum_{k=i+1}^M \gamma_k, \gamma_i\right)$, applying the stick-breaking property guarantees that the augmented vector $L_i$ is distributed as $\Dir(\gamma_i, \gamma_{i+1}, \dots, \gamma_M)$. This concludes the inductive step and the proof.
\end{proof}

\subsection{Illustration and Temperature Limits for \Cref{propbeaut:selfconditioning}}\label{app:corollaryselfconditioning}
We first present in \Cref{fig:weight_grid} the Self-Conditioning weights $L_{i,j}^{\varepsilon}$ appearing in \Cref{propbeaut:selfconditioning} for various $\varepsilon$ and index $i$.

\begin{figure}[t]
    \centering
    \includegraphics[width=\linewidth]{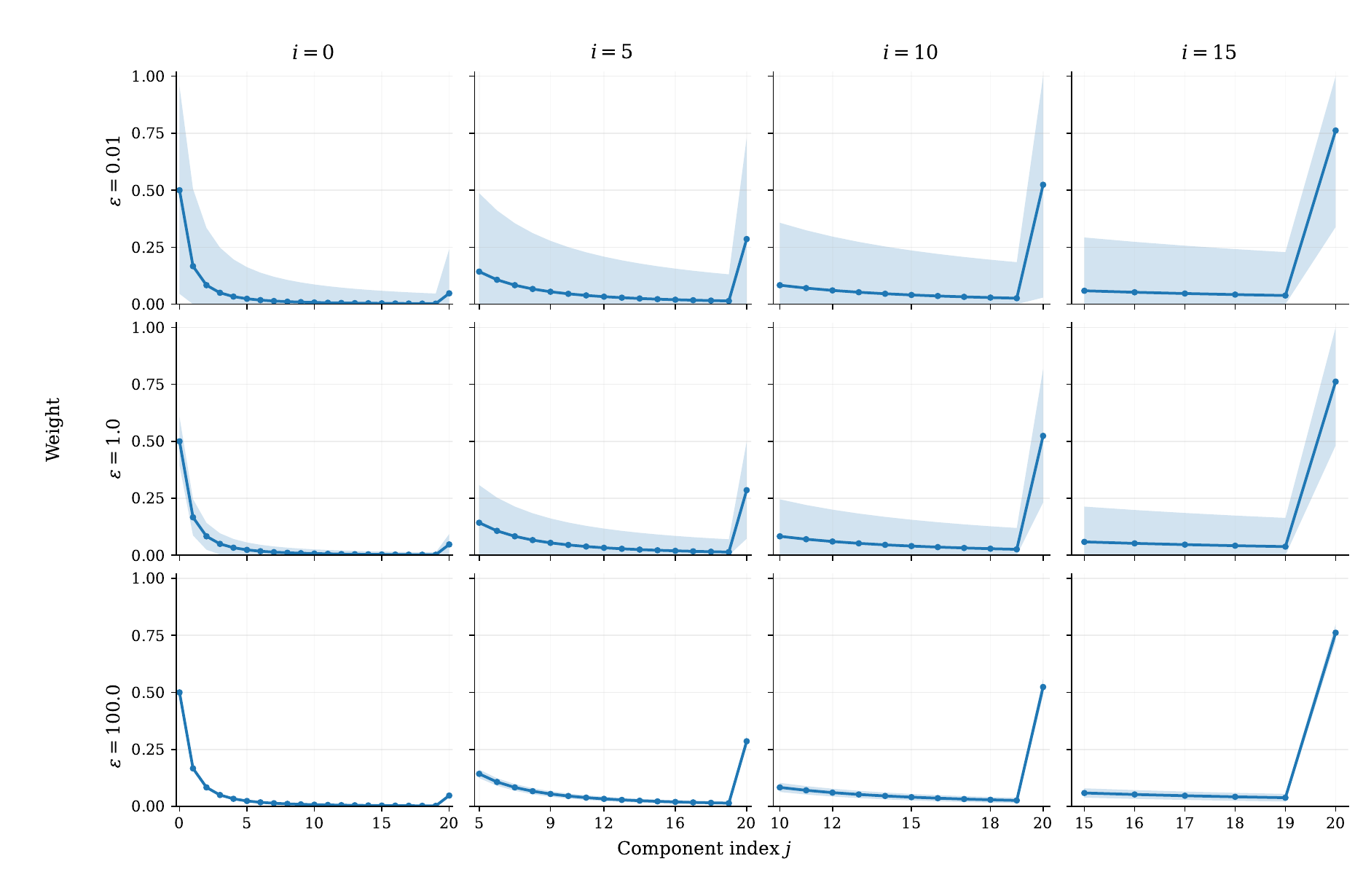}
    \caption{Visualization of the Self-Conditioning weights $L_{i,j}$ from \Cref{propbeaut:selfconditioning}. Each panel shows the exact mean profile $\mathbb{E}[L_{i,j}]$ (solid line) together with a $\pm 1$ standard-deviation band for the Dirichlet law in \Cref{propbeaut:selfconditioning}, for $\varepsilon \in \{0.01,1,100\}$ and $i \in \{0,5,10,15\}$. The mean curve is independent of $\varepsilon$.}
    \label{fig:weight_grid}
\end{figure}

Let us now investigate the high and low temperature limits of \Cref{propbeaut:selfconditioning}. 

\begin{corollarybeaut}{Limiting behavior of intrinsic Self-Conditioning}{coroselfconditioning}
Let $L_i= (L_{i,j})_{j=i}^M$ be the Dirichlet-distributed weight vector defined in \Cref{propbeaut:selfconditioning}. We define the unscaled parameters $c_j = \frac{1}{1-\alpha_{t_j}} - \frac{1}{1-\alpha_{t_{j+1}}}$ for $j \in \{i, \dots, M-1\}$ and $c_M = \frac{1}{1-\alpha_{t_M}}$. Let $C_i = \sum_{j=i}^M c_j = \frac{1}{1-\alpha_{t_i}}$. 
We have the following limiting cases. 
\begin{enumerate}\item \textbf{Low-temperature limit (Deterministic flow):} As $\vareps \to \infty$, $L_i$ converges in distribution to a deterministic vector:
\begin{equation}L_i \overset{d}{\to} \bar{L}_i , \qquad \text{where} \quad \bar{L}_{i,j} = \frac{c_j}{C_i} .\end{equation}
\item \textbf{High-temperature limit (Discrete jumps):} As $\vareps \to 0$, $L_i$ converges in distribution to a categorical distribution over the canonical basis $(u_j)_{j=i}^M$ of the simplex over indices $j \in \{i,...,M\}$:
\begin{equation}L_i \overset{d}{\to}  u_{J_i}, \qquad \text{where} \quad J_i \sim \textup{Cat}\left(\left(\frac{c_j}{C_i}\right)_{j=i}^M \right).
\end{equation}
\end{enumerate}
\end{corollarybeaut}

Therefore, we recover the Discrete Diffusion model behavior when $\vareps \to 0$ as in that case only one component $j^\star \in \{i, \dots, M\}$ is non-zero. This means that all contributions at steps other than $t_{j^\star}$ are neglected. This is not the case when $\vareps > 0$ 
and in the limit $\vareps \to +\infty$, the mixing of the clean-state draws is deterministic.

\begin{proof}
Recall that \Cref{propbeaut:selfconditioning} assumes the schedule
$c_t=\vareps/(1-\alpha_t)$, and gives
\begin{equation}
    L_i \sim \Dir\left( c_{t_i} - c_{t_{i+1}}, \dots, c_{t_{M-1}} - c_{t_M}, c_{t_M} \right).
\end{equation}
In terms of the unscaled parameters $\tilde c_j$ of the statement, namely
$\tilde c_j = \frac{1}{1-\alpha_{t_j}} - \frac{1}{1-\alpha_{t_{j+1}}}$ for
$j \in \{i,\dots,M-1\}$ and $\tilde c_M = \frac{1}{1-\alpha_{t_M}}$, we have
\begin{equation}
\label{eq:unscaled_identification}
    c_{t_j} - c_{t_{j+1}} = \vareps\, \tilde c_j
    \quad (j \leq M-1),
    \qquad
    c_{t_M} = \vareps\, \tilde c_M ,
\end{equation}
so that $L_i \sim \Dir\left(\vareps \tilde c_i, \dots, \vareps \tilde c_M\right)$.
Since $u \mapsto \alpha_u$ is non-increasing, each $\tilde c_j \geq 0$, and the
total concentration telescopes:
\begin{equation}
    \sum\nolimits_{j=i}^{M} \vareps\, \tilde c_j
    = \vareps\, C_i
    = \frac{\vareps}{1-\alpha_{t_i}}
    = c_{t_i} \in (0,\infty)
    \qquad \text{for } i \geq 1 .
\end{equation}
Writing $\bar L_i = \left(\tilde c_j / C_i\right)_{j=i}^{M} \in \Delta_{M-i+1}$, which
does not depend on $\vareps$, this reads
\begin{equation}
\label{eq:Li_scaled_form}
    L_i \sim \Dir\left(\left(\vareps C_i\right) \bar L_i\right) .
\end{equation}
So the following results follow.

\emph{(1) Low-temperature limit.} As $\vareps \to \infty$ we have $\vareps C_i \to \infty$,
so $L_i \to \bar L_i$ in probability,
hence in distribution. Explicitly, $\Ebb[L_{i,j}] = \tilde c_j / C_i$ for every $\vareps$,
while
\begin{equation}
    \Var\left[L_{i,j}\right]
    =
    \frac{\left(\tilde c_j/C_i\right)\left(1 - \tilde c_j/C_i\right)}{\vareps C_i + 1}
    \xrightarrow[\vareps \to \infty]{} 0 .
\end{equation}

\emph{(2) High-temperature limit.} As $\vareps \to 0$ we have $\vareps C_i \to 0$, so we obtain
\begin{equation}
    L_i \overset{d}{\longrightarrow}
    \sum\nolimits_{j=i}^{M} \frac{\tilde c_j}{C_i}\, \updelta_{u_j} ,
\end{equation}
where $(u_j)_{j=i}^M$ denotes the canonical basis of $\rset^{M-i+1}$. Equivalently,
$L_i \overset{d}{\to} u_{J_i}$ with
$J_i \sim \Cat\left(\left(\tilde c_j/C_i\right)_{j=i}^M\right)$. Indices $j$ with
$\tilde c_j = 0$, which occur when $\alpha$ is constant on $[t_j,t_{j+1}]$, satisfy
$L_{i,j}=0$ almost surely and receive zero mass in the limit, consistently with the
statement.
\end{proof}

\subsection{Proof of \Cref{propbeaut:temperature-limits}}

In this section, we investigate the temperature limits of SDMs. The complete design space is summarized in \Cref{fig:sdm-design-space}. We recall that we set $c_t = \vareps / (1-\alpha_t)$. 

\begin{figure}[t]
\centering
\resizebox{1\linewidth}{!}{%
\includegraphics[width=1.1\linewidth]{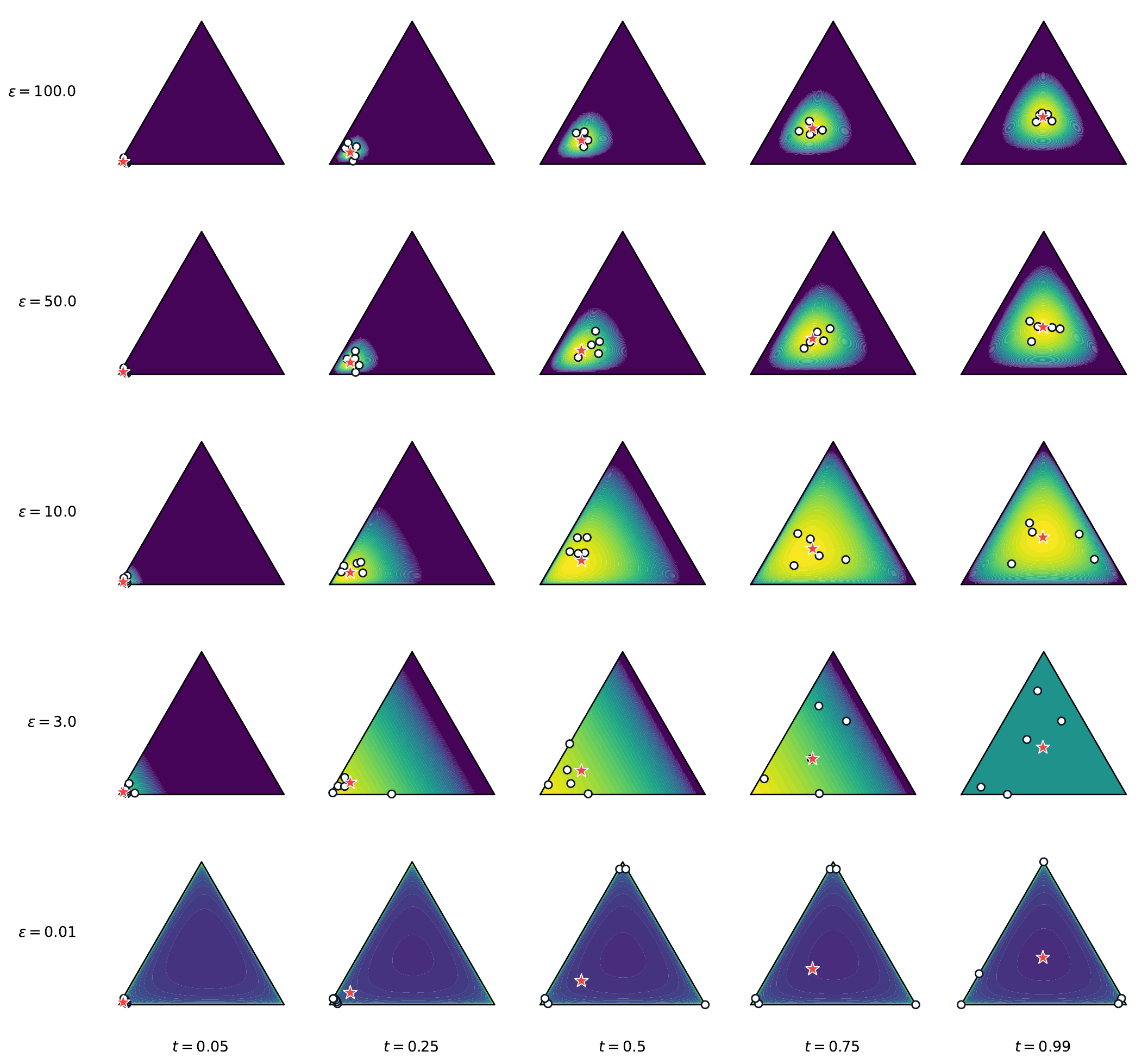}
}
\caption{Each figure corresponds to the density of $p_{t|0}(P_t|P_0) = \Dir(c_t(\alpha_t P_0 + (1-\alpha_t) \pi))$, with $\pi$ the uniform distribution and $c_t = \vareps / (1-\alpha_t)$ for different values of $t \in [0,1]$ and $\vareps >0$. We also display the mean (red star) of $p_{t|0}$, i.e., $\mathbb{E}[P_t|P_0] = \alpha_t P_0 + (1-\alpha_t) \pi$ and observe that this quantity is independent of $\vareps$. Finally, we also display 5 samples (white dots) of $p_{t|0}$. Note that as $\vareps \rightarrow 0$ the samples concentrate on the vertices of the simplex, thereby illustrating the convergence of SDMs to their Discrete Diffusion counterpart.}
    \label{fig:sdm-forward}
\end{figure}

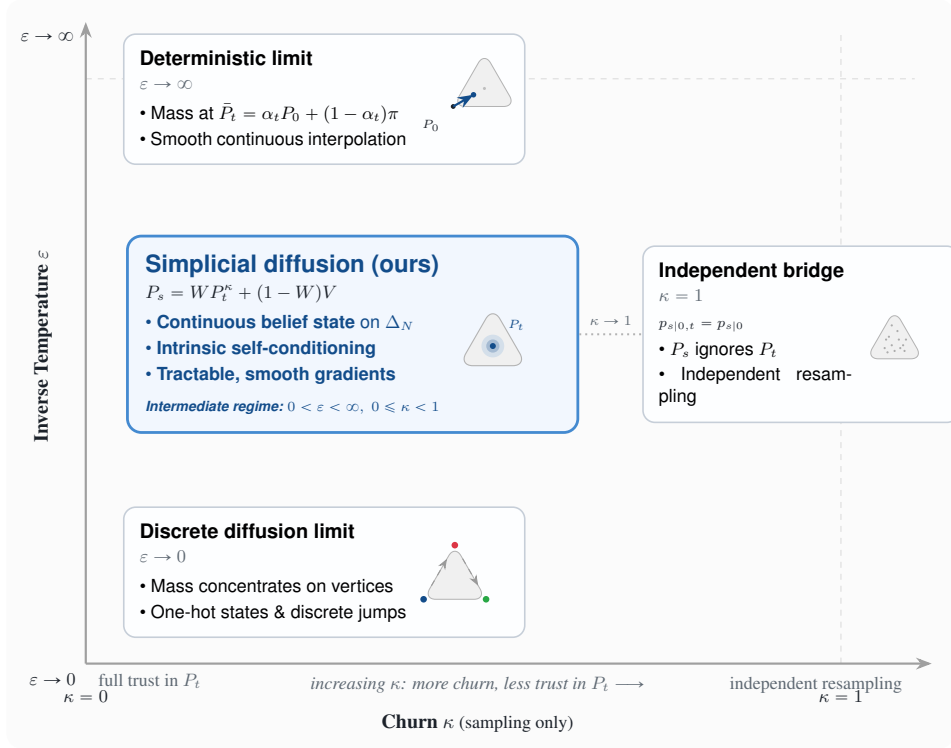
\begin{figure}[t]
\centering
\resizebox{0.95\linewidth}{!}{%

\definecolor{sdmblue}{RGB}{18, 76, 138}
\definecolor{sdmsoftblue}{RGB}{242, 247, 254}
\definecolor{sdmborderblue}{RGB}{68, 138, 208}
\definecolor{cardbg}{RGB}{253, 254, 255}
\definecolor{bordergray}{RGB}{198, 206, 216}
\definecolor{darktext}{RGB}{28, 34, 42}
\definecolor{subtext}{RGB}{95, 105, 115}
\definecolor{accentred}{RGB}{220, 53, 69}
\definecolor{accentgreen}{RGB}{40, 167, 69}

\begin{tikzpicture}[
  font=\sffamily,
  >=Stealth,
  card/.style={
    draw=bordergray,
    fill=cardbg,
    rounded corners=6pt,
    line width=0.8pt,
    inner sep=8pt
  },
  herocard/.style={
    draw=sdmborderblue,
    fill=sdmsoftblue,
    rounded corners=7pt,
    line width=1.4pt,
    inner sep=9pt
  }
]
  \fill[blue!1!gray!3, rounded corners=8pt] (-0.3, -0.6) rectangle (16.4, 12.6);
  \coordinate (O) at (1.1, 0.9);
  \coordinate (Xmax) at (16.0, 0.9);
  \coordinate (Ymax) at (1.1, 12.2);
  \draw[dashed, draw=gray!25, line width=0.6pt] (1.1, 11.2) -- (15.7, 11.2);
  \draw[dashed, draw=gray!25, line width=0.6pt] (14.4, 0.9) -- (14.4, 12.0);
  \draw[->, line width=1.1pt, draw=gray!75] (O) -- (Xmax);
  \draw[->, line width=1.1pt, draw=gray!75] (O) -- (Ymax);
  \node[anchor=north east, font=\small, text=darktext] at ($(O) + (-0.05, -0.05)$) {$\varepsilon \to 0$};
  \node[anchor=north, font=\small, text=darktext] at ($(O) + (0, -0.35)$) {$\kappa = 0$};
  
  \node[anchor=east, font=\small, text=darktext] at ($(Ymax) + (-0.08, -0.25)$) {$\varepsilon \to \infty$};
  \node[anchor=north, font=\small, text=darktext] at (14.4, 0.55) {$\kappa = 1$};
  \node[rotate=90, anchor=center, font=\bfseries\normalsize, text=darktext] at (0.35, 6.5) {Inverse Temperature $\varepsilon$};
  \node[anchor=center, font=\bfseries\normalsize, text=darktext] at (8.0, -0.15) {Churn $\kappa$ \textnormal{\small(sampling only)}};
  \node[anchor=north west, font=\footnotesize, text=subtext] at ($(O) + (0.1, -0.05)$) {full trust in $P_t$};
  \node[anchor=north, font=\footnotesize\itshape, text=subtext] at (8.0, 0.8) {increasing $\kappa$: more churn, less trust in $P_t \longrightarrow$};
  \node[anchor=north east, font=\footnotesize, text=subtext] at (15.6, 0.8) {independent resampling};
  \node[card, anchor=north] (determ) at (5.3, 12.0) {
    \begin{minipage}{6.5cm}
      \begin{minipage}[c]{4.7cm}
        {\bfseries\normalsize Deterministic limit}\\[1pt]
        {\footnotesize\color{subtext}$\varepsilon \to \infty$}\\[4pt]
        {\footnotesize
        \textbullet~Mass at $\bar{P}_t = \alpha_t P_0 + (1-\alpha_t)\pi$\\[2pt]
        \textbullet~Smooth continuous interpolation}
      \end{minipage}%
      \begin{minipage}[c]{1.7cm}
        \centering
        \begin{tikzpicture}[scale=0.55, baseline=-0.3cm]
          \coordinate (A) at (0, 1.155);
          \coordinate (B) at (-1, -0.577);
          \coordinate (C) at (1, -0.577);
          \fill[fill=gray!12, draw=gray!60, line width=0.5pt] (A) -- (B) -- (C) -- cycle;
          \coordinate (PI) at (0, 0);
          \fill[gray!50] (PI) circle (1.4pt);
          \fill[darktext] (B) circle (2.2pt) node[below left=-2pt, font=\tiny] {$P_0$};
          \draw[->, line width=1.1pt, sdmblue] (B) -- ($(B)!0.65!(PI)$) coordinate (PT);
          \fill[sdmblue] (PT) circle (2.5pt);
        \end{tikzpicture}
      \end{minipage}
    \end{minipage}
  };
  \node[card, anchor=south] (discrete) at (5.3, 1.35) {
    \begin{minipage}{6.5cm}
      \begin{minipage}[c]{4.7cm}
        {\bfseries\normalsize Discrete diffusion limit}\\[1pt]
        {\footnotesize\color{subtext}$\varepsilon \to 0$}\\[4pt]
        {\footnotesize
        \textbullet~Mass concentrates on vertices\\[2pt]
        \textbullet~One-hot states \& discrete jumps}
      \end{minipage}%
      \begin{minipage}[c]{1.7cm}
        \centering
        \begin{tikzpicture}[scale=0.55, baseline=-0.3cm]
          \coordinate (A) at (0, 1.155);
          \coordinate (B) at (-1, -0.577);
          \coordinate (C) at (1, -0.577);
          \fill[fill=gray!12, draw=gray!60, line width=0.5pt] (A) -- (B) -- (C) -- cycle;
          \fill[accentred] (A) circle (3pt);
          \fill[sdmblue] (B) circle (3pt);
          \fill[accentgreen] (C) circle (3pt);
          \draw[->, line width=0.55pt, dashed, gray!80] ($(B)!0.25!(A)$) -- ($(B)!0.75!(A)$);
          \draw[->, line width=0.55pt, dashed, gray!80] ($(A)!0.25!(C)$) -- ($(A)!0.75!(C)$);
        \end{tikzpicture}
      \end{minipage}
    \end{minipage}
  };
  \node[card, anchor=west] (bridge) at (10.9, 6.7) {
    \begin{minipage}{5.0cm}
      \begin{minipage}[c]{3.4cm}
        {\bfseries\normalsize Independent bridge}\\[1pt]
        {\footnotesize\color{subtext}$\kappa = 1$}\\[2pt]
        {\scriptsize\color{darktext}$p_{s|0,t} = p_{s|0}$}\\[3pt]
        {\footnotesize
        \textbullet~$P_s$ ignores $P_t$\\[2pt]
        \textbullet~Independent resampling}
      \end{minipage}%
      \begin{minipage}[c]{1.5cm}
        \centering
        \begin{tikzpicture}[scale=0.5, baseline=-0.3cm]
          \coordinate (A) at (0, 1.155);
          \coordinate (B) at (-1, -0.577);
          \coordinate (C) at (1, -0.577);
          \fill[fill=gray!12, draw=gray!60, line width=0.5pt] (A) -- (B) -- (C) -- cycle;
          \foreach \x/\y in {
            -0.2/0.2, 0.1/0.4, 0.3/-0.1, -0.35/-0.2, 0.0/-0.3,
            -0.1/0.6, 0.4/0.1, -0.3/0.1, 0.2/-0.2, -0.1/-0.1,
            0.15/0.1, -0.25/-0.35, 0.35/-0.3
          }{
            \fill[gray!65] (\x, \y) circle (1.1pt);
          }
        \end{tikzpicture}
      \end{minipage}
    \end{minipage}
  };
  \node[herocard, anchor=center] (ours) at (5.8, 6.7) {
    \begin{minipage}{7.3cm}
      \begin{minipage}[c]{5.3cm}
        {\bfseries\large\color{sdmblue} Simplicial diffusion (ours)}\\[2pt]
        {\small\bfseries\color{darktext} $P_s = W P_t^\kappa + (1-W)V$}\\[4pt]
        {\footnotesize
        \color{sdmblue}\textbullet~\textbf{Continuous belief state} on $\Delta_N$\\[2pt]
        \color{sdmblue}\textbullet~\textbf{Intrinsic self-conditioning}\\[2pt]
        \color{sdmblue}\textbullet~\textbf{Tractable, smooth gradients}}\\[5pt]
        {\scriptsize\color{sdmblue}\bfseries\itshape Intermediate regime: $0 < \varepsilon < \infty,\; 0 \le \kappa < 1$}
      \end{minipage}%
      \begin{minipage}[c]{1.9cm}
        \centering
        \begin{tikzpicture}[scale=0.6, baseline=-0.3cm]
          \coordinate (A) at (0, 1.155);
          \coordinate (B) at (-1, -0.577);
          \coordinate (C) at (1, -0.577);
          \fill[fill=gray!12, draw=gray!60, line width=0.55pt] (A) -- (B) -- (C) -- cycle;
          \fill[sdmblue!18] (0, 0) circle (9.5pt);
          \fill[sdmblue!38] (0, 0) circle (5.5pt);
          \fill[sdmblue] (0, 0) circle (2.5pt);
          \node[above right=-2pt, font=\tiny\bfseries, text=sdmblue] at (0,0) {$P_t$};
        \end{tikzpicture}
      \end{minipage}
    \end{minipage}
  };
  \draw[dotted, line width=1.0pt, draw=gray!60] (ours.east) -- (bridge.west)
    node[midway, above=1pt, font=\scriptsize, text=subtext] {$\kappa \to 1$};
\end{tikzpicture}%
}
\caption{The design space of Simplex Diffusion Models (SDMs). The forward process is controlled by the inverse temperature $\varepsilon$: as $\varepsilon \to 0$, the simplex mass concentrates on vertices and recovers discrete-diffusion behavior; as $\varepsilon \to \infty$, the forward process concentrates around the deterministic interpolation $\bar P_t = \alpha_t P_0 + (1-\alpha_t)\pi$. The reverse process is controlled by the churn parameter $\kappa$: $\kappa=0$  fully trusts the current belief $P_t$, while $\kappa=1$ yields the independent bridge $p_{s\mid 0,t}=p_{s\mid 0}$. SDMs operate in the intermediate regime, maintaining a continuous belief state throughout sampling.}
    \label{fig:sdm-design-space}
\end{figure}

We use the representation from \Cref{propbeaut:interpolatingpath},
\begin{equation}
    P_t
    =
    W_t P_0
    +
    \left(1-W_t \right)V_t,
\end{equation}
where
\begin{equation}
    W_t
    \sim
    \Beta(\vareps h_t,\vareps),
    \qquad
    V_t
    \sim
    \Dir(\vareps\pi),
    \qquad
    h_t=\frac{\alpha_t}{1-\alpha_t},
\end{equation}
and $W_t$ and $V_t$ are independent.

\paragraph{High-Temperature Limit.}
For fixed $t\in(0,1)$, standard results for the Beta and Dirichlet distributions give
\begin{equation}
    W_t
    \overset{d}{\longrightarrow}
    B_t,
    \qquad
    B_t\sim\Ber(\alpha_t),
\end{equation}
and
\begin{equation}
    V_t
    \overset{d}{\longrightarrow}
    V,
    \qquad
    V
    \sim
    \sum_{i=1}^N \pi_i\,\updelta_{e_i}.
\end{equation}
The limiting variables are independent. Hence, by the continuous
mapping theorem,
\begin{equation}
    P_t
    \overset{d}{\longrightarrow}
    B_tP_0+(1-B_t)V.
\end{equation}
Since $P_0=e_{x_0}$ and $V$ is supported on the canonical basis
vectors, the limiting law is supported on the vertices of $\Delta_N$.

For $\kappa=0$, the backward transition is
\begin{equation}
    P_s
    =
    W_{s,t}^{0}P_t
    +
    \left(1-W_{s,t}^{0}\right)e_{x_0},
\end{equation}
with
\begin{equation}
    W_{s,t}^{0}
    \sim
    \Beta\!\left(
        \frac{\vareps}{1-\alpha_t},
        \frac{\vareps}{1-\alpha_s}
        -
        \frac{\vareps}{1-\alpha_t}
    \right).
\end{equation}
Applying again the small-concentration Beta limit gives
\begin{equation}
    W_{s,t}^{0}
    \overset{d}{\longrightarrow} 
    \Ber\!\left(
        \frac{1-\alpha_s}{1-\alpha_t}
    \right).
\end{equation}
Therefore, for every fixed $P_t\in\Delta_N$, we have in the limit
\begin{equation}
    P_s
    \overset{d}{\longrightarrow}
    W_{s,t}^{0}P_t
    +
    \left(1-W_{s,t}^{0}\right)e_{x_0},\qquad  W_{s,t}^{0}
    \sim
    \Ber\!\left(
        \frac{1-\alpha_s}{1-\alpha_t}
    \right).
\end{equation}
In particular, when $P_t$ is a simplex vertex, the limiting transition
is supported on simplex vertices.

\paragraph{Low-Temperature Limit.}
For fixed $t\in(0,1)$, standard results give
\begin{equation}
    W_t
    \overset{\mathbb P}{\longrightarrow}
    \frac{h_t}{1+h_t}
    =
    \alpha_t,
\end{equation}
and
\begin{equation}
    V_t
    \overset{\mathbb P}{\longrightarrow}
    \pi.
\end{equation}
Using once again the representation of $P_t$ and the continuous
mapping theorem,
\begin{equation}
    P_t
    \overset{\mathbb P}{\longrightarrow}
    \alpha_tP_0+(1-\alpha_t)\pi,
\end{equation}
which concludes the proof.

\subsection{Gaussian Fluctuations}
\label{sec:gaussian_limit_linear}

The next result identifies
the fluctuations around the deterministic path at scale $\vareps^{-1/2}$. For this result, we introduce the tangent space of the simplex $\mathrm{T}\Delta_N = \{ x \in \rset^N , \ \mathbf{1}^\top x = 0 \}$. 

\begin{propositionbeaut}{Gaussian fluctuation limit}{linear-gaussian-limit}
Let $U \in \rset^{N \times (N-1)}$ be an orthonormal basis for the tangent space $\mathrm{T}\Delta_N$.
Let $t \in (0,1]$ and let $P_t \sim \Dir(\beta_t(P_0,\pi))$ with $\pi_i>0$ for all $i$. Let $X_0 = \sqrt{\vareps}U^\top(P_0-\pi)$ and $X_t = \sqrt{\vareps}U^\top(P_t-\pi)$. We have 
\begin{equation}
    X_t = \alpha_t X_0 + (1-\alpha_t) \xi_t^\vareps,
\end{equation}
where $\xi_t^\vareps = \frac{\sqrt{\vareps}}{1-\alpha_t} U^\top (P_t - \bar{P}_t)$. Define $\Sigma_t := U^\top\left[\mathrm{diag}(\pi)-\pi\pi^\top+\alpha_t(P_0-\pi)(P_0-\pi)^\top\right]U$. As $\vareps \to \infty$,
\begin{equation}
    \xi_t^\vareps \overset{d}{\to} \mathcal{N}\left(0, \Sigma_t \right) .
\end{equation}
In addition, the spectrum of $\Sigma_t$ is contained in $[\min_i \pi_i,2]$ for all $t \in (0,1]$. 
\end{propositionbeaut}

To connect with standard continuous Gaussian diffusion models, we project the scaled displacement of the forward state from the prior $\pi$ onto an orthonormal basis of the simplex tangent space. 

\begin{propositionbeaut}{Asymptotic normality and covariance structure}{dirichlet-clt}
Assume that $\pi$ has strictly positive components. 
Let $t \in (0,1]$ and let $P_t \sim \Dir(\beta_t(P_0,\pi))$. Define $\bar{P}_t = \alpha_t P_0 + (1-\alpha_t) \pi$ and $V_t = (1-\alpha_t)^{-1} (\mathrm{diag}(\bar{P}_t) - \bar{P}_t \bar{P}_t^\top)$. As $\vareps \to \infty$, we have
\begin{equation}
    \sqrt{\vareps}(P_t - \bar{P}_t) \overset{d}{\to} \mathcal{N}\left(0, (1-\alpha_t)^2V_t\right) .
\end{equation}
In addition, letting $V_\pi = \mathrm{diag}(\pi) - \pi \pi^\top$, we have
\begin{equation}
    V_t = V_\pi + \alpha_t (P_0 - \pi)(P_0 - \pi)^\top .
\end{equation}
\end{propositionbeaut}

\begin{proof}
For the rest of the proof, we define $\tilde{V}_t  = \mathrm{diag}(\bar{P}_t) - \bar{P}_t \bar{P}_t^\top$. Hence, we have that $\tilde{V}_t = (1-\alpha_t) V_t$. 
We leverage the Gamma representation of the Dirichlet distribution. Let $\mu_t = \pi + h_t P_0$ with $h_t = \frac{\alpha_t}{1-\alpha_t}$, such that the concentration parameters are $\beta_t^\vareps = \vareps \mu_t$. We can write $P_t \overset{d}{=} Y^\vareps / (\mathbf{1}^\top Y^\vareps)$, where $Y^\vareps$ is a vector of independent random variables $Y_i^\vareps \sim \text{Gamma}(\vareps \mu_{t,i}, 1)$. 

Using the Central Limit Theorem, as $\vareps \to \infty$, we have 
$$
    \sqrt{\vareps}\left(\frac{Y^\vareps}{\vareps} - \mu_t\right) \overset{d}{\to} \mathcal{N}\left(0, \mathrm{diag}(\mu_t)\right) .
$$
We apply the multivariate Delta method with $g(x) = \frac{x}{\mathbf{1}^\top x}$. Let $c_t = \mathbf{1}^\top \mu_t = 1 + h_t = \frac{1}{1-\alpha_t}$. Notice that $g(\mu_t) = \frac{\mu_t}{c_t} = \bar{P}_t$. The Jacobian of $g$ at $\mu_t$ is given by 
$$
    \nabla g(\mu_t) = \frac{1}{\mathbf{1}^\top \mu_t} \Id - \frac{\mu_t \mathbf{1}^\top}{(\mathbf{1}^\top \mu_t)^2} = \frac{1}{c_t} \left( \Id - \bar{P}_t \mathbf{1}^\top \right) .
$$
The asymptotic covariance is $\nabla g(\mu_t) \mathrm{diag}(\mu_t) \nabla g(\mu_t)^\top$ which can be simplified in
\begin{align}
    \frac{1}{c_t^2} \left( I - \bar{P}_t \mathbf{1}^\top \right) \mathrm{diag}(\mu_t) \left( I - \mathbf{1} \bar{P}_t^\top \right) &= \frac{1}{c_t} \left( I - \bar{P}_t \mathbf{1}^\top \right) \mathrm{diag}(\bar{P}_t) \left( I - \mathbf{1} \bar{P}_t^\top \right) \\
    &= \frac{1}{c_t} \left( \mathrm{diag}(\bar{P}_t) - \bar{P}_t \bar{P}_t^\top - \bar{P}_t \bar{P}_t^\top + \bar{P}_t (\mathbf{1}^\top \bar{P}_t) \bar{P}_t^\top \right) \\
    &= \frac{1}{c_t} (\mathrm{diag}(\bar{P}_t) - \bar{P}_t \bar{P}_t^\top) \\
    &= (1-\alpha_t)\tilde{V}_t = (1-\alpha_t)^2 V_t.
\end{align}
To obtain the second part of the proposition, we expand $\tilde{V}_t = \mathrm{diag}(\bar{P}_t) - \bar{P}_t \bar{P}_t^\top$ using $\bar{P}_t = \alpha_t P_0 + (1-\alpha_t) \pi$. Since $\mathrm{diag}(P_0) = P_0 P_0^\top$, because $P_0 = e_{x_0}$, we have
\begin{align}
    \tilde{V}_t &= \alpha_t \mathrm{diag}(P_0) + (1-\alpha_t) \mathrm{diag}(\pi) - \left[ \alpha_t^2 P_0 P_0^\top + \alpha_t(1-\alpha_t)(P_0 \pi^\top + \pi P_0^\top) + (1-\alpha_t)^2 \pi \pi^\top \right] \nonumber \\
    &= (\alpha_t - \alpha_t^2) P_0 P_0^\top + (1-\alpha_t) \mathrm{diag}(\pi) - \alpha_t(1-\alpha_t)(P_0 \pi^\top + \pi P_0^\top) - (1-\alpha_t)^2 \pi \pi^\top \nonumber \\
    &= (1-\alpha_t) \left[ \mathrm{diag}(\pi) - \pi \pi^\top + \alpha_t P_0 P_0^\top - \alpha_t(P_0 \pi^\top + \pi P_0^\top) + \alpha_t \pi \pi^\top \right] \nonumber \\
    &= (1-\alpha_t) \left[ V_\pi + \alpha_t(P_0 - \pi)(P_0 - \pi)^\top \right] ,
\end{align}
which concludes the proof since $\tilde{V}_t = (1-\alpha_t) V_t$.
\end{proof}

\begin{propositionbeaut}{Spectral stability of the projected covariance}{spectral-bounds}
Let $U \in \rset^{N \times (N-1)}$ be an orthonormal basis for the tangent space $\mathrm{T}\Delta_N$. If $\pi$ has strictly positive components, 
the spectrum of $U^\top V_t U$ is uniformly bounded for all $t \in (0,1]$:
\begin{equation}
    0 < \min_i \pi_i \leq \lambda_{\min}(U^\top V_t U) \leq \lambda_{\max}(U^\top V_t U) \leq 2 .
\end{equation}
\end{propositionbeaut}

\begin{proof}
 Let $v \in \rset^{N-1}$ such that $\| v \|_2 = 1$ and $u=Uv$. It follows that $\|u\|_2 = \|v\|_2 = 1$ and $u^\top \mathbf{1} = 0$. Using the decomposition of $V_t$ from \Cref{propbeaut:dirichlet-clt}, we get 
\begin{equation}
    \label{eq:rayleigh_decomp}
    v^\top (U^\top V_t U) v = u^\top V_\pi u + \alpha_t (u^\top (P_0 - \pi))^2 .
\end{equation}
We analyze the two terms in \eqref{eq:rayleigh_decomp} separately.  First, we have
\begin{align}
    u^\top V_\pi u &= \sum_{i=1}^N \pi_i u_i^2 - \left(\sum_i \pi_i u_i\right)^2 \\
    &= \sum_{i=1}^N \pi_i (u_i - \pi^\top u)^2 \\
    &\geq \min_i \pi_i \left( 1  + N (\pi^\top u)^2 \right) \geq \min_i \pi_i ,
\end{align}
where we have used that $u^\top u=1$ and $u^\top \mathbf{1} =0$.
Combining this result and \eqref{eq:rayleigh_decomp}, we get that $ \lambda_{\min}(U^\top V_t U) \geq \min_i \pi_i > 0$.

For the upper-bound, we use that 
\begin{equation}
    \lambda_{\max}(U^\top V_t U) \leq \mathrm{Tr}(U^\top V_t U) \leq \mathrm{Tr}(V_t) \label{eq:upper_bound_trace}.
\end{equation}
Using that $P_0 = e_{x_0}$, we have that 
\begin{equation}
    \mathrm{Tr}(V_\pi) = 1 - \| \pi \|^2 , \qquad \mathrm{Tr}((P_0 - \pi)(P_0 - \pi)^\top) = \| P_0 - \pi \|^2 = 1 - 2 \pi^\top P_0 + \| \pi \|^2 . 
\end{equation}
Therefore, combining this result, $\alpha_t \leq 1$ and \eqref{eq:upper_bound_trace} we get that 
\begin{equation}
    \lambda_{\max}(U^\top V_t U) \leq 1 - \| \pi \|^2 + 1  - 2 \pi^\top P_0 + \| \pi \|^2 \leq 2 ,
\end{equation}
which concludes the proof.
\end{proof}

Finally, we conclude this section with the proof of \Cref{propbeaut:linear-gaussian-limit}.


\begin{proof}
Scaling the perturbation $P_t - \pi = \alpha_t (P_0 - \pi) + (P_t - \bar{P}_t)$ by $\sqrt{\vareps} U^\top$ yields the linear form. By \Cref{propbeaut:dirichlet-clt}, $\sqrt{\vareps}(P_t - \bar{P}_t) \overset{d}{\to} \mathcal{N}\left(0, (1-\alpha_t)^2V_t\right)$. Applying the linear transformation $\frac{1}{(1-\alpha_t)} U^\top$ directly yields
$$
    \xi_t^\vareps \overset{d}{\to} \mathcal{N}\left(0, \frac{1}{(1-\alpha_t)^2} U^\top (1-\alpha_t)^2 V_t U\right) = \mathcal{N}\left(0, U^\top V_t U\right) ,
$$
which concludes the proof. The second part of the proof is a direct consequence of \Cref{propbeaut:spectral-bounds}.
\end{proof}

\section{Variational Lower Bound}\label{app:ELBO}
We derive a variational lower bound for the inference procedure described in  \Cref{sec:training_inference_simplicial}.

Fix a grid $0=t_0<t_1<\cdots<t_M=1$ and write $P_k=P_{t_k}$. For
$j\in\mcx$ and $k=1,\ldots,M-1$, define
\begin{equation}
\label{eq:elbo_kernel_notation}
    R_k^j(P_k\mid P_{k+1})=p_{t_k\mid0,t_{k+1}} ( P_k\mid e_j,P_{k+1}).
\end{equation}
The learned reverse kernel is
\begin{equation}
\label{eq:fixed_grid_reverse_kernel}
    K_k^\theta(P_k\mid P_{k+1})
    =
    \sum_{j=1}^N
    \hat P_{\theta,j}(t_{k+1},P_{k+1})
    R_k^j(P_k\mid P_{k+1}),
\end{equation}
and the endpoint decoder is
$p_\theta(x\mid P_1)=\hat P_\theta(t_1,P_1)_x$.

For any $x\in\mcx$, define the bridge path law
\begin{equation}
\label{eq:bridge_path_law}
    Q_x(P_{1:M})
    :=
    \Dir(P_M;c_1\pi)
    \prod_{k=1}^{M-1}
    R_k^x(P_k\mid P_{k+1}).
\end{equation}
Because $\alpha_1=0$, the terminal marginal is
$\Dir(c_1\pi)=\Dir(\beta_{t_M}(e_x,\pi))$. One can easily check by backward induction that
\begin{equation}
\label{eq:bridge_path_marginals}
    P_k\sim
    \Dir\!\left(\beta_{t_k}(e_x,\pi)\right)
    \quad\text{under }Q_x,
    \qquad k=1,\ldots,M.
\end{equation}

\begin{propositionbeaut}{Cross-Entropy ELBO}{exact_augmented_elbo}
Let $p_\theta(x)$ denote the marginal likelihood of the fixed-grid sampler
with terminal law $\Dir(c_1\pi)$, reverse kernels
\eqref{eq:fixed_grid_reverse_kernel}, and endpoint decoder
$\hat P_\theta(t_1,P_1)$. For every $x_0\in\mcx$,
\begin{equation}
    \log p_\theta(x_0)
    \geq
    \mathcal L_{\mathrm{ELBO}}(\theta;x_0),
\end{equation}
where
\begin{equation}
\label{eq:categorical_elbo_compact}
    \mathcal L_{\mathrm{ELBO}}(\theta;x_0)
    =
    \sum_{k=1}^{M}
    \mathbb E_{
        P_k\sim\Dir(\beta_{t_k}(e_{x_0},\pi))
    }
    \left[
        \log \hat P_\theta(t_k,P_k)_{x_0}
    \right].
\end{equation}
\end{propositionbeaut}

\begin{proof}
Since \eqref{eq:fixed_grid_reverse_kernel} is a nonnegative mixture, for every
$P_{k+1}$,
\begin{equation}
\label{eq:component_kernel_bound}
    K_k^\theta(P_k\mid P_{k+1})
    \geq
    \hat P_\theta(t_{k+1},P_{k+1})_{x_0}
    R_k^{x_0}(P_k\mid P_{k+1}).
\end{equation}
Retaining this single component at every reverse step
gives
\begin{equation}
\label{eq:selected_component_bound}
    p_\theta(x_0)
    \geq
    \mathbb E_{Q_{x_0}}
    \left[
        \prod_{k=1}^{M}
        \hat P_\theta(t_k,P_k)_{x_0}
    \right].
\end{equation}
Taking logarithms and applying Jensen's inequality,
\begin{align}
    \log p_\theta(x_0)
    &\geq
    \log
    \mathbb E_{Q_{x_0}}
    \left[
        \prod_{k=1}^{M}
        \hat P_\theta(t_k,P_k)_{x_0}
    \right]
    \nonumber\\
    &\geq
    \sum_{k=1}^{M}
    \mathbb E_{Q_{x_0}}
    \left[
        \log \hat P_\theta(t_k,P_k)_{x_0}
    \right].
\end{align}
Using the marginals \eqref{eq:bridge_path_marginals} proves
\eqref{eq:categorical_elbo_compact}.
\end{proof}

Thus, sampling $k$ uniformly from the inference grid and then sampling
$P_k\sim\Dir(\beta_{t_k}(e_{x_0},\pi))$ gives an unbiased estimator of
$-\mathcal L_{\mathrm{ELBO}}(\theta;x_0)/M$.

\paragraph{Alternative bound.}
Keeping the full mixture before applying Jensen gives a tighter but generally
intractable mixture-KL bound, which we do not consider further.

\section{Distillation}
\label{app:distillation}

\subsection{Simplex Distribution Matching Distillation}

Throughout this section, we restrict attention to time pairs for which the Dirichlet parameters below are strictly positive.

\paragraph{Distribution Matching Distillation.} We consider a multistep and simplicial version of Distribution Matching Distillation (DMD) \citep{luo2023diff,yin2024one}.  We introduce a generator $G: \mcp \times [0,1] \times \Delta_N \times \mathcal U \to \Delta_N$, where $\mathcal U$ is a noise space and $\mcp$ is a real vector space which represents the parameter space of the parametric function $G_\eta$. 
The main goal of the distillation method we describe below is to find $G_\eta$ such that, for some random variable $U \sim \Pbb_U$ taking values in $\mathcal U$, we have $G_\eta(t, P_t, U) \sim \hat{p}_{0|t}(\cdot | P_t)$, where $p_{0|t}$ is given by Bayes's rule,  \eqref{eq:forward_process_simplicial} and 
\begin{equation}
    \int_{\Delta_N} p_t(P_t) p_{0|t}(P_0|P_t) \rmd P_t = \int_{\Delta_N} p_t(P_t) \hat{p}_{0|t}(P_0|P_t) \rmd P_t . 
\end{equation}
In particular, $G_\eta(t, P_t, U)$ should take values on the vertices of $\Delta_N$, since its prior distribution $p_\text{data}(P) = \sum_{i=1}^{N} \pdata{, i} \updelta_{e_i}(P)$ (see \eqref{eq:form_pi_0}) only has mass on the vertices. In what follows, we make no such assumption on the form of $G_\eta$ and only ensure that $G_\eta$ takes values in $\Delta_N$. 

We denote $\rmd G_\eta(t, P_t, U): \mcp \to \rset^N$ the derivative of $G_\eta$ with respect to $\eta$, evaluated at $(\eta, t, P_t, U)$.
Similarly, we denote $\rmD G_\eta(t, P_t, U)$ such that for any $h \in \mcp$
\begin{equation}
\label{eq:derivative_g}
    \rmd G_\eta(t, P_t, U)(h) = \rmD G_\eta(t, P_t, U) h . 
\end{equation}

We  recall that for every $t \in [0,1]$, we have that $P_t$ is given by \eqref{eq:forward_process_simplicial}. 
For notational simplicity, we define $\Pbb_t$ the distribution of $P_t$ with $P_t \sim p_{t|0}(\cdot \ | \ P_0)$ and $P_0 = e_{x_0}$ with $x_0 \sim \pdata{}$. More precisely for any test function $f \in \rmc^\infty(\Delta_N, \rset)$ we have that 

\begin{equation}
    \int_{\Delta_N} f(P)\,\rmd \Pbb_t(P)
    =
    \int_{\Delta_N}\int_{\Delta_N}
    f(P_t)\,\rmd p_{t|0}(P_t\mid P_0)\,\rmd p_\text{data}(P_0).
\end{equation}

Similarly, we denote  $\Pbb_t^{\eta,s}$ the distribution of $P_t$ with $P_t \sim p_{t|0}(\cdot \ | \ P_0^s)$ where $P_0^s = G_\eta(s, P_s,U)$ where $P_s \sim p_{s|0}(\cdot \ | \ P_0)$ and $P_0 \sim p_\text{data} \left(\cdot \right)$. More precisely for any test function $f \in \rmc^\infty(\Delta_N, \rset)$ we have that 

\begin{align}
    &\int_{\Delta_N} f(P)\,\rmd \Pbb_t^{\eta,s}(P) \\
    &\qquad =
    \int_{\Delta_N}\int_{\Delta_N}\int_{\mathcal U}
    \int_{\Delta_N}\int_{\Delta_N}
    f(P_t)\,
    \rmd p_{t|0}(P_t\mid P_0^s)\,
    \rmd \updelta_{G_\eta(s,P_s,U)}(P_0^s)\,
    \rmd \Pbb_U(U)\,
    \rmd p_{s|0}(P_s\mid P_0)\,
    \rmd p_\text{data}(P_0).
\end{align}
We denote by $\Pbb_{0|t}^{\eta,s}$ the conditional distribution of $P_0^s$ given $P_t$ under this construction, and by $\Pbb_{0|t}$ the conditional distribution of $P_0$ given $P_t$ under the data-forward joint distribution.

We consider the following distribution matching loss

\begin{equation}
    \mcl(\eta) = \int_0^1 \int_0^1  \KL(\Pbb_t^{\eta,s} | \Pbb_t) \rmd \Qbb(s,t) . \label{eq:distillation-loss}
\end{equation}

Note that in the original setting of DMD \citep{luo2023diff,yin2024one}, we have that $\rmd \Qbb(s,t) = \updelta_1(s) \rmd \Qbb(t)$.
In the multistep regime however, we consider a general distribution $\Qbb$.
In what follows, we are going to first derive the gradient of the loss function and then give an equivalent expression for \eqref{eq:distillation-loss} which can be readily implemented. 

First, we recall the forward process
\begin{equation}
\label{eq:forward_process_simplicial_appendix}
    p_{t|0}(P_t|P_0) = \Dir(P_t;\beta_t(P_0,\pi)) , \qquad \beta_t(P_0,\pi) = a_tP_0+b_t\pi = c_t (\alpha_t P_0+(1-\alpha_t) \pi),
\end{equation}
where $a_t=c_t\alpha_t$ and $b_t=c_t(1-\alpha_t)$, as in \Cref{sec:simplicialdiffusion}.
In particular, using the fact that a Dirichlet random variable can be generated using Gamma random variables (see \Cref{app:basics}), then for any $t \in [0,1]$ there exists $F_t$ such that
$F_t(P_0, V)$ has distribution $p_{t|0}(\cdot | P_0)$, with $V = \{V_i\}_{i=1}^{N}$ which represents $N$ independent uniform random variables used to generate the $N$ Gamma random variates. In addition, $F_t$ is differentiable with respect to $P_0$. We denote $\rmd F_t(P_0, V): \ \rset^N \to \rset^N$ the differential of $F_t$ with respect to $P_0$ evaluated at $(P_0, V)$. Finally, we introduce $\rmD F_t(P_0, V) \in \rset^{N \times N}$ such that for any $h \in \rset^N$ we have
\begin{equation}
\label{eq:derivative_f}
    \rmd F_t(P_0, V)(h) = \rmD F_t(P_0, V) h . 
\end{equation}

\paragraph{Surrogate Loss.} While \eqref{eq:distillation-loss} is a valid loss, it is not tractable. However, leveraging stop-gradient techniques and an equivalent of Tweedie's identity in the case of Dirichlet distributions, we will be able to derive a surrogate tractable loss with identical gradients in \Cref{propbeaut:surrogate_loss}. We  start with the following lemma. 

\begin{lemmabeaut}{Chain rule}{chain_rule}
    Let $s, t \in [0,1]$, $P_s \in \Delta_N$, $U \sim \Pbb_U$ and  $V = \{V_i\}_{i=1}^{N}$, $N$ independent uniform random variables.
    Let $f: \Delta_N \to \rset$ be a differentiable function. 
    Define $g: \ \mcp \to \rset$ such that $g(\eta) = f(F_t(G_\eta(s, P_s, U), V))$. 
    We have
    \begin{equation}
        \nabla_\eta g(\eta) = \rmD G_\eta(s, P_s, U)^\top \rmD F_t(G_\eta(s, P_s, U),V)^\top \nabla f(F_t(G_\eta(s, P_s, U),V)) . 
    \end{equation}
\end{lemmabeaut}

Next, we consider the following lemma which shows that Dirichlet distribution also enjoys some form of Tweedie's identity. 

\begin{lemmabeaut}{Tweedie meets Dirichlet}{tweedie_dirichlet}
    Let $p_X$ be a distribution over $\Delta_N$ and $p_{Y|X}$ a Dirichlet distribution with parameter $\alpha(X) \in (0,+\infty)^N$. Then, we have
    \begin{equation}
    \nabla \log p_Y(y) = \left\lbrace \frac{\mathbb{E}[\alpha_i(X) \ | Y] - 1}{y_i} \right\rbrace_{i=1}^N . 
\end{equation}
\end{lemmabeaut}

\begin{proof}
We have that $p_X$ is a distribution over $\Delta_N$ and $p_{Y|X}$ is a Dirichlet distribution with parameter $\alpha(X) \in (0,+\infty)^N$. 
We have that for any $y \in \Delta_N$ 
\begin{equation}
    p_{Y|X}(y|X) = \frac{\Gamma(\sum_{i=1}^N \alpha_i(X))}{\prod_{i=1}^N \Gamma(\alpha_i(X))} \prod_{i=1}^N y_i^{\alpha_i(X)-1} .  
\end{equation}
Then, we have that
\begin{equation}
    \nabla_y \log p_{Y|X}(y|X) = \{ (\alpha_i(X) - 1)/ y_i \}_{i=1}^N . 
\end{equation}
In addition, we have that 
\begin{equation}
    \nabla \log p_Y(y) = \int \nabla_y \log p_{Y|X}(y|x)~ p_{X|Y}(x|y) \rmd x . 
\end{equation}
Finally, we get that
\begin{equation}
    \nabla \log p_Y(y) = \left\lbrace \frac{\mathbb{E}[\alpha_i(X) \ | Y] - 1}{y_i} \right\rbrace_{i=1}^N . 
\end{equation}
\end{proof}

Combining the standard pathwise gradient identity for the KL divergence with \Cref{lemmabeaut:chain_rule} and \Cref{lemmabeaut:tweedie_dirichlet}, we obtain the following proposition. 

\begin{propositionbeaut}{}{}
Let $\mcl(\eta)$ be given by \eqref{eq:distillation-loss}. 
Then, we have that
\begin{align}
    \nabla_\eta \mcl(\eta) &= \int a_t\, \rmD G_\eta(s, P_s, U)^\top \rmD F_t(G_\eta(s, P_s, U), V)^\top \\
    & \qquad \qquad \times \frac{\mathbb{E}_{\Pbb_{0|t}^{\eta,s}}[P_0^s | F_t(G_\eta(s, P_s, U), V)] - \mathbb{E}_{\Pbb_{0|t}}[P_0 | F_t(G_\eta(s, P_s, U), V)]}{F_t(G_\eta(s, P_s, U), V)} \\
    & \qquad \qquad \qquad \qquad \times \rmd \Qbb(s,t)  \rmd p_{\textup{data}}(P_0) \rmd p_{s|0}(P_s|P_0) \rmd \Pbb_U(U)  \rmd V  ,
\end{align}
where the division is coordinate-wise. 
\end{propositionbeaut}

We can therefore find an equivalent loss to \eqref{eq:distillation-loss} which is tractable and yields the same gradients. 

\begin{propositionbeaut}{Surrogate loss}{surrogate_loss}
    Let $\hat{\mcl}(\eta)$ be given by
    \begin{align}
        \hat{\mcl}(\eta) &= \int a_t\, \log(F_t(G_\eta(s, P_s, U), V))^\top \\
        & \qquad \qquad \stopgrad \left( \mathbb{E}_{\Pbb_{0|t}^{\eta,s}}[P_0^s | F_t(G_\eta(s, P_s, U), V)] - \mathbb{E}_{\Pbb_{0|t}}[P_0 | F_t(G_\eta(s, P_s, U), V)] \right)  \\
        & \qquad \qquad \qquad \qquad \times \rmd \Qbb(s,t)  \rmd p_{\textup{data}}(P_0) \rmd p_{s|0}(P_s|P_0) \rmd \Pbb_U(U) \rmd V.
    \end{align}
    Then $\hat{\mcl}$ and $\mcl$ have the same gradients.
\end{propositionbeaut}

\paragraph{Algorithm and Links with Literature.} Leveraging the loss given in \Cref{propbeaut:surrogate_loss}, we consider an algorithm to learn a distilled model.
The full algorithm is given in \Cref{alg:simplicial_dmd}. 
With exact conditional means, i.e., the denoisers are exact,  and a valid pathwise
derivative of the corruption map, \Cref{alg:simplicial_dmd} is valid. In practice, learned denoisers approximate these conditional means,
yielding an approximate gradient. To differentiate through the forward process, we here differentiate through the fast $\GammaDist$ implementation of \Cref{app:samplingdirichlet}. Note that one might be able to leverage the implementation of \cite{greaves2026extended} to obtain numerical gradients which are more precise. We leave this exploration for future work. 
It is apparent that \Cref{alg:simplicial_dmd} follows the same pattern as many existing distillation algorithms such as DMD \citep{yin2024one} but also Multistep Moment matching Distillation (MMD) \citep{salimans2024multistep}. In fact, our algorithm can be interpreted as a simplicial extension of the MMD algorithm. In both cases, we train a \emph{generator} network $G_\eta$ as well as an \emph{auxiliary} denoiser approximating $\mathbb{E}_{\Pbb_{0|t}^{\eta,s}}[P_0^s | P_t]$. Another line of related work is Universal Distillation Matching (UDM) and its discrete counterpart Inverse Distillation Language Model (IDLM) \citep{kornilov2025universal,li2026idlm}. While closely related to our method, we highlight the key differences between these methods and our contribution in \Cref{sec:links_idlm}.

\begin{algorithm}
\caption{Simplex DMD Training}
\label{alg:simplicial_dmd}
\begin{algorithmic}[1]
\Require Teacher denoiser $D_{\mathrm{teach}}(P_t, t) \approx \Ebb_{\Pbb_{0|t}}[P_0 \mid P_t]$
\Require Generator $G_\eta(s, P_s, U)$, auxiliary denoiser $D_\phi(P_t, s, t) \approx \Ebb_{\Pbb_{0|t}^{\eta,s}}[P_0^s \mid P_t]$
\Require Data distribution $\pdata{}$, noise distribution $\Pbb_U$, time distribution $\Qbb$
\Require Forward process parameters $(\alpha_t)_{t\in[0,1]}$, $(c_t)_{t\in[0,1]}$, and $\pi$, with $a_t=c_t\alpha_t$
\Require Learning rates $\mathrm{lr}_\eta,\mathrm{lr}_\phi$

\Repeat

\State \textbf{1. Sample data and times}
\State $x_0 \sim \pdata{}$, \quad $P_0 \gets e_{x_0}$
\State $(s, t) \sim \Qbb$

\State \textbf{2. Teacher forward to time $s$} \Comment{No gradient through this step}
\For{$i = 1, \ldots, N$}
    \State $G_i^{(s)} \gets \textsc{GammaSample}(\beta_s(P_0,\pi)_i)$ \Comment{\Cref{alg:gamma_sampling}}
\EndFor
\State $P_s \gets (G_1^{(s)}, \ldots, G_N^{(s)}) / {\textstyle\sum_j G_j^{(s)}}$

\State \textbf{3. Generator (one denoising step)}
\State $U \sim \Pbb_U$
\State $P_0^s \gets G_\eta(s, P_s, U)$

\State \textbf{4. Student forward to time $t$} \Comment{Differentiable through $P_0^s$}
\For{$i = 1, \ldots, N$}
    \State $G_i^{(t)} \gets \textsc{GammaSample}(\beta_t(P_0^s,\pi)_i)$ \Comment{\Cref{alg:gamma_sampling}}
\EndFor
\State $P_t \gets (G_1^{(t)}, \ldots, G_N^{(t)}) / {\textstyle\sum_j G_j^{(t)}}$

\State \textbf{5. Score difference (stopped gradients)}
\State $\Delta \gets \stopgrad\bigl(D_\phi(P_t, s, t)\bigr) - \stopgrad\bigl(D_{\mathrm{teach}}(P_t, t)\bigr)$

\State \textbf{6. Generator update}
\State $\hat{\mcl}_\eta \gets a_t \displaystyle\sum_{k=1}^{N} \log(P_{t,k}) \cdot \Delta_k$ \Comment{Gradient flows only through $\log P_t$}
\State $\eta \gets \eta - \mathrm{lr}_\eta \cdot \nabla_\eta \hat{\mcl}_\eta$

\State \textbf{7. Auxiliary denoiser update}
\State $\hat{\mcl}_\phi \gets - \sum_{k=1}^N \log(D_\phi(\stopgrad(P_t),\, s, t))_k  \cdot \stopgrad(P_0^s)_k $
\State $\phi \gets \phi - \mathrm{lr}_\phi \cdot \nabla_\phi \hat{\mcl}_\phi$

\Until{converged}
\end{algorithmic}
\end{algorithm}
  
\subsection{Link with Inverse Distilled Language Models}
\label{sec:links_idlm}

In this section, we draw connections between the Simplex DMD framework
and the Inverse Distilled Language Model (IDLM) approach
\citep{kornilov2025universal,li2026idlm}. We show that both methods share a
common structure but differ in how the gradient of the generative loss is
handled, with the Simplex DMD framework resolving a fundamental difficulty
encountered in IDLM.

\paragraph{IDLM Generative Loss.}
Adapting the formulation of \cite{kornilov2025universal,li2026idlm} to the
notation of the present paper, the IDLM generative loss can be written as
\begin{align}
\label{eq:idlm_gen_loss}
    \mcl^{\mathrm{G}}(\eta) &= \int
      G_\eta(s, P_s, U)^\top
      \bigl(\log D_\phi(F_t(G_\eta(s, P_s, U), V),\, s, t)
            - \log D_{\mathrm{teach}}(F_t(G_\eta(s, P_s, U), V),\, t)
      \bigr) \nonumber \\
    &\qquad\qquad\qquad\qquad\qquad \times  \rmd \Qbb(s,t)   \rmd p_{\textup{data}}(P_0) 
     \rmd p_{s|0}(P_s|P_0)
     \rmd \Pbb_U(U) \rmd V,
\end{align}
where $D_\phi$ plays the role of the auxiliary model and
$D_{\mathrm{teach}}$ the teacher model from \cite{li2026idlm}.
The generator $G_\eta$ appears \emph{twice} in this expression: as
the denoised sample $P_0^s = G_\eta(s, P_s, U)$ and as the input to
the forward reparameterization
$P_t = F_t(G_\eta(s, P_s, U), V)$. As a result, the gradient
of $\mcl^{\mathrm{G}}$ with respect to $\eta$ splits into two terms:
\begin{align}
    \nabla_\eta \mcl^{\mathrm{G}}(\eta)
    &= \int \rmD G_\eta^\top
       \underbrace{%
         \bigl(\log D_\phi(P_t, s, t) - \log D_{\mathrm{teach}}(P_t, t)\bigr)
       }_{\text{direct gradient}}
       \,\rmd(\cdots)
       \label{eq:direct_grad} \\
    &\quad + \int \rmD G_\eta^\top\,\rmD F_t^\top
       \underbrace{%
         \left(
           \frac{\partial \log D_\phi}{\partial P_t}
           - \frac{\partial \log D_{\mathrm{teach}}}{\partial P_t}
         \right)^\top G_\eta
       }_{\text{indirect gradient}}
       \,\rmd(\cdots) .
       \label{eq:indirect_grad}
\end{align}
As discussed in \cite{li2026idlm,hoogeboom2026beyond}, the indirect
gradient~\eqref{eq:indirect_grad} is empirically high-variance, and
current implementations discard it entirely. The resulting update,
however, is not generally the gradient of any well-defined objective.

\paragraph{Comparison with Simplex DMD.}
By contrast, the surrogate loss $\hat{\mcl}(\eta)$ of
 \Cref{propbeaut:surrogate_loss} yields a \emph{single} gradient
term
\begin{equation}
\label{eq:dmd_gradient_single}
    \nabla_\eta \hat{\mcl}(\eta)
    = \int a_t\,\rmD G_\eta^\top\,\rmD F_t^\top\,
      \frac{%
        \stopgrad\bigl(D_\phi(P_t, s, t) - D_{\mathrm{teach}}(P_t, t)\bigr)
      }{P_t}
      \,\rmd(\cdots) ,
\end{equation}
which is the \emph{exact} gradient of the KL divergence
$\mcl(\eta)$ defined in~\eqref{eq:distillation-loss}.
We highlight several structural differences with the IDLM
gradient~\eqref{eq:direct_grad}--\eqref{eq:indirect_grad}.

\begin{enumerate}
    \item \emph{Single gradient path.}
    In~\eqref{eq:dmd_gradient_single}, the gradient flows exclusively
    through the forward map $F_t$. There is no direct term
    because the DMD loss is derived from
    $\KL(\Pbb_t^{\eta,s} \| \Pbb_t)$ via the Stein score identity for
    pushforward measures, which produces a gradient that is entirely
    of the indirect type. 

    \item \emph{Exact gradient with stop-gradient.}
    The $\stopgrad$ operator on $\Delta$ in~\eqref{eq:dmd_gradient_single}
    is not a heuristic approximation:
    \Cref{propbeaut:surrogate_loss} guarantees
    $\nabla_\eta \hat{\mcl} = \nabla_\eta \mcl$.
    In contrast, discarding the indirect gradient in IDLM is an
    approximation without such a guarantee.

    \item \emph{Auxiliary model.}
    The IDLM auxiliary loss
    $-\int P_0^{s,\top} \log D_\phi(P_t, s, t)\,\rmd(\cdots)$
    is a cross-entropy between the generator output and the auxiliary
    model. This is identical in structure to the auxiliary denoiser loss in
    Step~7 of \Cref{alg:simplicial_dmd}, confirming the
    correspondence between the auxiliary model of \cite{li2026idlm}
    and the auxiliary denoiser $D_\phi$.

    \item \emph{Differentiability.}
    In the simplicial framework, $F_t$ is differentiable via
    the $\Gamma$-reparameterization
    (\Cref{alg:gamma_sampling}), so the Jacobian
    $\rmD F_t$ in~\eqref{eq:dmd_gradient_single} can be
    computed exactly by automatic differentiation. In discrete models
    \citep{li2026idlm}, the forward corruption is non-differentiable,
    necessitating biased approximations. This is one of the key
    advantages of the simplicial formulation.
\end{enumerate}

\appendixpart[]{Connection with the Literature}

\section{Extended Related Work} 

\subsection{General overview}

In this section, we present an overview of simplex diffusion models and their applications to discrete data generation. For Discrete Diffusions, we refer the reader to the original paper of \citet{austin2021structured} and more recent extensions \citep{campbell2022continuous,benton2024denoising,lou2023discrete,sahoo2024simple,shi2024simplified}.

Several simplex diffusion models are obtained by defining a forward stochastic process directly on the simplex. Then, relying on tools from time-reversal \citep{haussmann1986time}, similar to \citet{song2020score} in Euclidean state spaces, one can define a generative process. In \citet{richemond2022categorical}, inspired by \citet{baker2018large}, the forward process corresponds to a set of $N$ Cox--Ingersoll--Ross processes that converge to a Gamma distribution and to a Dirichlet distribution after renormalization. \citet{floto2023diffusion} consider another forward process based on a softmax transformation of an Ornstein--Uhlenbeck process.
One of the main advantages of \citet{floto2023diffusion} is that for any $t \in [0,1]$ the distribution $p_{t|0}$ is a logistic normal distribution which is easy to sample. Dirichlet Diffusion Score Models (DDSMs) \citep{avdeyev2023dirichlet} consider instead a Jacobi diffusion process as a forward process that converges to a $\Beta$ distribution. By then leveraging a stick-breaking construction, the authors obtain a forward process converging to a Dirichlet distribution. \citet{benton2024denoising} instead use the Wright-Fisher diffusion as a forward process, a diffusion on the simplex originating from genetics. Recently, in \citet{chandra2025unification}, it was shown that discrete, continuous and Wright-Fisher diffusions could arise from a similar particle perspective, where the Wright-Fisher diffusion is obtained as a limit where the number of particles goes to infinity and we consider reproduction within the particles' evolution.

Yet another approach is to define diffusion models taking into account the special geometry of the simplex. Doing so, it is then possible to define diffusion models on an associated \emph{statistical manifold} using general techniques from Flow Matching and  Riemannian diffusion models \citep{de2022riemannian,huang2022riemannian,chen2023flow}.
While in \citet{cheng2024categorical,davis2024fisher,williams2025simplex}, the authors derive the metrics on the simplex using the Fisher--Rao connection, \citet{boll2024generative,boll2025generative} consider $e$-connections, see \cite[Appendix E.2]{davis2024fisher} for a discussion of the choice of metrics. 
In \citet{han2023ssd,mahabadi2024tess,jo2025continuous}, the authors consider a Gaussian diffusion in the space of logits embeddings. 
Those different approaches, along with the one ignoring the geometry of the simplex and simply performing a linear diffusion in that space \citep{dunn2024mixed}, can be unified using the concept of $\alpha$-divergence to define the statistical manifold (see \citealp{cheng2025alpha}).

Another approach incorporating the geometry of the simplex into the generation process is to consider a \emph{constrained} forward process using either projection or reflection of the original Stochastic Differential Equation \citep{liu2023mirror,fishman2023diffusion,fishman2023metropolis,lou2023reflected}.

The closest work to ours is the Dirichlet Flow Matching approach introduced by \citet{stark2024dirichlet}. In this work, the authors introduced a forward process akin to ours. The main difference between the two approaches arises from inference. 
While \citet{stark2024dirichlet} leverage a flow perspective, our approach is akin to DDIM \citep{song2020denoising}. We also provide a different temperature-parameterized process. Note that very recently \citet{boget2026unrestrainedsimplexdenoisingdiscrete} have proposed the same forward process as \citet{stark2024dirichlet}. They introduce an inference procedure which, like ours, allows for stochasticity but assumes $p_{s|0,t}(P_s|P_0,P_t)=p_{s|0}(P_s|P_0)$ as in star-shaped diffusion models \citep{okhotin2023star} (corresponding to $\kappa=1$ in our case). Another concurrent work (Simplax) \citep{sakurai2026simplex} also uses Dirichlet simplex states, but their denoiser and sampling procedure operate on categorical samples rather than the continuous simplex state itself.

We conclude by mentioning that \citet{haviv2024wasserstein} introduce Wasserstein Flow Matching, i.e., define Flow Matching on the space of distributions, similar to the simplicial approach lifting distributions on a given state space to the space of distributions of distributions. Finally, we note that the recent work of \citet{song2025shortlisting} introducing Shortlisting Models (SLMs) also claims a simplex approach as they ``aim to preserve the core principle of simplex-based methods, \emph{gradual information growth}''. To do so, they operate on the space of probability vectors. Starting from a single full probability vector, they progressively eliminate categories until the final denoised state is one-hot.

\subsection{Simplex Relaxation for Discrete Diffusion}

In this section, we describe the approach of \citet{sakurai2026simplex} and compare it with our framework. We first outline how they construct their forward process alongside their simplex relaxation. Next, we discuss their chosen training loss. Finally, we investigate their sampling procedure and show that it can be simplified to bypass the simplex relaxation, thereby highlighting a fundamental difference from Simplex Diffusion Models.

\paragraph{Forward Process.} \citet{sakurai2026simplex} first consider a forward process on the categorical space given for any $t \in [0,1]$ and $x_t, x_0 \in \{1, \dots, N\}$  by
\begin{equation}
    p_{t|0}(x_t|x_0) = \Cat(x_t ; \alpha_t x_0 + (1-\alpha_t) \pi) ,
\end{equation}
where $\pi \in \Delta_N$ is a probability distribution.
Note that, as emphasized by \citet{sakurai2026simplex}, one can define a compatible backward bridge transition for this forward rule by defining for any $s, t \in [0,1]$ and $s \leq t$ and $x_0, x_s, x_t \in \{1, \dots, N\}$ by 
\begin{equation}
    p_{s|0,t}(x_s|x_0,x_t) = \Cat\left( x_s ; \frac{\left[\frac{\alpha_t}{\alpha_s} x_t + \left( 1 - \frac{\alpha_t}{\alpha_s}\right) \langle x_t, \pi \rangle \1 \right] \odot (\alpha_s x_0 + (1-\alpha_s) \pi)}{\langle x_t, \alpha_t x_0 + (1-\alpha_t) \pi \rangle}\right) . 
\end{equation}
We denote $r_{s|0,t}$ the mean of $p_{s|0,t}$. 
In particular, we have that for any $x_0, x_s \in \{1, \dots, N\}$
\begin{equation}
    p_{s|0}(x_s|x_0) = \sum_{x_t=1}^{N} p_{s|0,t}(x_s|x_0,x_t) p_{t|0}(x_t|x_0) . 
\end{equation}
One of the main innovations of \citet{sakurai2026simplex} is to introduce the simplex-valued variable $w_t \in \Delta_N$ for any $t \in [0,1]$ with
\begin{equation}
    p(w_t|x_t, x_0)=  \Dir(w_t ; \eta_t (\alpha_t x_0 + (1-\alpha_t) \pi) + x_t) . 
\end{equation}
\paragraph{Training Loss.} The training loss they consider is given for a given $s, t \in [0,1]$ with $s \leq t$, $x_0 \in \{1, \dots, N\}$ and $w_t \in \Delta_N$
\begin{equation}
    \mathcal{L}_{s,t}  = \sum q(\tilde{x}_t | w_t) \KL(p_{s|0,t}(x_s|x_0, \tilde{x}_t)| p(x_s |\hat{x}_{\theta}(x_t), \tilde{x}_t)) . 
\end{equation}
This can be simplified (see \cite[Proposition 5]{sakurai2026simplex}) in 
\begin{align}
    \mathcal{L}_{s,t} &= \langle w_t, \log (\alpha_t \hat{x}_\theta(x_t) + (1-\alpha_t) \pi) - \log (\alpha_t x_0 + (1-\alpha_t) \pi) \rangle  \\
    & \qquad + \langle \rho_{s|0,t}, \log (\alpha_s x_0 + (1-\alpha_s) \pi) - \log (\alpha_s \hat{x}_\theta(x_t) + (1-\alpha_s) \pi) \rangle
\end{align}
where $\rho_{s|0,t}$ is defined by \cite[Equation 11]{sakurai2026simplex} 
\begin{align}
    \rho_{s|0,t} &= (\alpha_s x_0 + (1-\alpha_s) \pi)\\
    & \qquad \odot \left[ \frac{\alpha_t}{\alpha_s}(w_t \oslash (\alpha_t x_0 + (1-\alpha_t) \pi)) + \left( 1- \frac{\alpha_t}{\alpha_s}\right)  \langle w_t, \pi \oslash (\alpha_t x_0 + (1-\alpha_t) \pi) \rangle \1 \right] . 
\end{align}
In contrast, we use a simple cross-entropy loss \eqref{eq:xentropy}, and discuss a true ELBO in \Cref{app:ELBO}.

\paragraph{Sampling.} Once the model is trained, they sample from the model as follows. 
Let $s,t \in [0,1]$ with $s \leq t$. Assume that we have access to a pair $(x_t, w_t)$. Then, they let the denoiser predict $x_\theta(t, x_t)$. 
Then, they sample $x_s \sim \Cat(\rho_{s|0,t})$ and $w_s \sim \Dir(\eta_s (\alpha_s x_\theta(t, x_t) + (1-\alpha_s)\pi) + x_s)$. 
In \cite[Proposition 1]{sakurai2026simplex}, it is shown that $x_s \sim p_{s|0,t}(x_s|x_\theta(t, x_t), w_t)$.
Therefore, we have that for any test function $f: \ \{1, \dots, N\} \to \rset$
\begin{equation}
    \mathbb{E}[f(x_s) | x_t] = \sum f(x_s) \mathbb{E}[p_{s|0,t}(x_s|x_0, w_t)|x_t] p_{0|t}(x_0|x_t),
\end{equation}
where the expectation is w.r.t. $w_t$.
In addition, we have that 
\begin{align}
    &\mathbb{E}[p_{s|0,t}(x_s|x_0, w_t)|x_t] = (\alpha_s x_0 + (1-\alpha_s) \pi)\\
    & \qquad \odot \left[ \frac{\alpha_t}{\alpha_s}(\mathbb{E}[w_t|x_t] \oslash (\alpha_t x_0 + (1-\alpha_t) \pi)) + \left( 1- \frac{\alpha_t}{\alpha_s}\right)  \langle \mathbb{E}[w_t|x_t], \pi \oslash (\alpha_t x_0 + (1-\alpha_t) \pi) \rangle \1 \right] .
\end{align}
We have that 
\begin{equation}
    \mathbb{E}[w_t |x_t] = \frac{\eta_t}{1 + \eta_t } (\alpha_t x_0 + (1-\alpha_t) \pi) + \frac{1}{1 + \eta_t} x_t .  
\end{equation}
Therefore, we have that 
\begin{equation}
    \mathbb{E}[p_{s|0,t}(x_s|x_0, w_t)|x_t] = \frac{\eta_t}{1+\eta_t} (\alpha_s x_0 + (1-\alpha_s)\pi) + \frac{1}{1+\eta_t} r_{s|0,t} . 
\end{equation}
Hence, we get that 
\begin{equation}
    \mathbb{E}[f(x_s) | x_t] = \sum f(x_s) \sum \left( \frac{\eta_t}{1+\eta_t} p_{s|0}(x_s|x_0) + \frac{1}{1+\eta_t} p_{s|0,t}(x_s|x_0,x_t) \right) p_{0|t}(x_0|x_t) . 
\end{equation}
Therefore, we can interpret the transition proposed in \citet{sakurai2026simplex} as a \emph{pure discrete} backward transition \emph{with remasking} with remasking levels controlled by $\eta_t$.  

In contrast, in our framework, we do not maintain a discrete state during the generation and instead only track a simplex state and only sample from the terminal simplex state.

\section{Extended Background}
\subsection{Discrete Diffusion Models}
\label{sec:ext_bg_ddms}
In this section, we outline the training and sampling procedures for Discrete Diffusion models. For simplicity, we describe processes over scalar variables, and the extension to sequences is similar to \Cref{sec:multiple_tokens}. Refer to \citet{austin2021structured, campbell2022continuous, sahoo2024simple, shi2024simplified, schiff2025simpleguidancemechanismsdiscrete, vonrutte2025generalizedinterpolatingdiscretediffusion, gourevitch2026uniformdiffusionmodelsrevisited} for the derivations. Discrete Diffusion Models define a corruption process \eqref{eq:discrete_forward} in terms of a prior $\pi \in \Delta_N$.
Prior work mainly focuses on the \emph{absorbing} (or masked) prior $\pi^\text{mask} = m$, where $m$ is the one-hot embedding of a special \masktoken token, and the \emph{uniform} prior $\pi^\text{unif} = \mathbf{1}/N$. Several works study mixtures of $\pi^\text{mask}$ and $\pi^\text{unif}$ \citep{fathi2025unifyingautoregressivediffusionbasedsequence, vonrutte2025generalizedinterpolatingdiscretediffusion, liu2026balancingunderstandinggenerationdiscrete, wang2026generalizeddiscretediffusionselfcorrection, zhang2026denoising}, or data-dependent priors \citep{Alamdari2023, vignac2023digressdiscretedenoisingdiffusion,
qin2025defogdiscreteflowmatching} but it is not clear whether elaborate priors are necessary at scale for language modeling \citep{sahoo2026scalingmaskeddiffusionlanguage, vonrutte2026scalingbehaviordiscretediffusion}. 
\paragraph{Sampling.}
Let us refer to $p_{s|0, t}$ as the \emph{bridge} \citep{gourevitch2026uniformdiffusionmodelsrevisited}:
\begin{equation}
\label{eq:appendix_bridge_bayes}
\begin{aligned}
    p_{s|0, t}(x_s \mid x_0, x_t) &= \frac{p_{t|s}(x_t \mid x_s) p_{s|0}(x_s \mid x_0)}{p_{t|0}(x_t \mid x_0)} \\
    &= \operatorname{Cat}\left(x_s ; \frac{\left[\alpha_{t|s} e_{x_t} + (1 - \alpha_{t|s})(e_{x_t}^\top \pi) \mathbf{1}\right] \odot \left[\alpha_s e_{x_0} + (1 - \alpha_s)\pi\right]}{\alpha_t (e_{x_t}^\top e_{x_0}) + (1 - \alpha_t)(e_{x_t}^\top \pi)}\right).
\end{aligned}
\end{equation}
where $\alpha_{t|s} = \nicefrac{\alpha_t}{\alpha_s}$. While the bridge \eqref{eq:appendix_bridge_bayes} is defined for discrete variables $x_0, x_s, x_t$, 
it is possible to define transitions $\hat{p}_{s|t}^\theta(x_s|x_t)$ by replacing $e_{x_0}$ with the predictions of a denoiser $\mathbf{x}_\theta(t, x_t)$. Thus, with a slight abuse of notation, one can write $\hat{p}_{s|t}^\theta(x_s|x_t) = p_{s|0, t}(x_s | \mathbf{x}_\theta(t, x_t), x_t)$. After expanding \eqref{eq:appendix_bridge_bayes} with the absorbing and uniform priors, we find that
\begin{equation}
    p_{s|0,t}^\text{mask}(x_s \mid x_0, x_t) = \begin{cases}
    \operatorname{Cat}(x_s ; e_{x_0}) & \text{if } x_t = x_0 \\[1.5ex]
    \operatorname{Cat}\left(x_s ; \frac{\alpha_s - \alpha_t}{1 - \alpha_t} e_{x_0} + \frac{1 - \alpha_s}{1 - \alpha_t} m\right) & \text{if } x_t = m,
    \end{cases}
\end{equation}
and
\begin{equation}
    p_{s|0,t}^\text{unif}(x_s \mid x_0, x_t) = \operatorname{Cat}\left(x_s ; \frac{N \alpha_t (e_{x_t}^\top e_{x_0}) e_{x_t} + (\alpha_{t|s} - \alpha_t) e_{x_t} + (\alpha_s - \alpha_t) e_{x_0} + D_{s,t} \mathbf{1} / N}{N \alpha_t (e_{x_t}^\top e_{x_0}) + 1 - \alpha_t}\right),
\end{equation}

where $D_{s,t} := (1 - \alpha_{t|s})(1 - \alpha_s)$. For a time grid $0=t_0 < t_1 < ... < t_n = 1$, the standard ancestral sampler applies the transition $\hat{p}_{s|t}^\theta(x_s|x_t)$ $n$ times to obtain $x_0$:
\begin{equation}
x_{1} \sim \operatorname{Cat}(\pi), \qquad x_{t_{i-1}} \sim \hat{p}_{t_{i-1}|t_i}^\theta(\cdot \mid x_{t_i}) \quad \text{for } i = n, \dots, 1.
\end{equation}
Alternatively, Predictor-Corrector samplers \citep{grathwohl2021oopsitookgradient, campbell2022continuous, lezama2023discrete, sun2023discretelangevinsamplerwasserstein, campbell2024generativeflowsdiscretestatespaces, kim2026finetuningmaskeddiffusionprovable, wang2026remaskingdiscretediffusionmodels, liu2025thinkgeneratediscretediffusion, zhao2025informedcorrectorsdiscretediffusion, deschenaux2026diffusiondualitychapterii, gourevitch2026uniformdiffusionmodelsrevisited} also exist. We describe the variant used in our experiments in \Cref{sec:pc_sampler}.
\paragraph{Temperature.}
All samplers can scale the logits $\zeta$ of the denoiser by a temperature $T > 0$, i.e., they use the prediction
\begin{equation}
    \mathrm{softmax}(\zeta / T)
    \label{eq:logit_temperature}
\end{equation}
instead of $\mathrm{softmax}(\zeta)$. $T = 1$ recovers the denoiser prediction, and $T \to 0$ approaches greedy decoding. For SDMs, the scaled prediction replaces $\hat P_\theta(t, P_t)$ in the transition of \Cref{propbeaut:simplicialtransition}.
\paragraph{Training.}
As for Variational (continuous) Diffusion Models \citep{sohl2015deep,ho2020denoising,kingma2023variationaldiffusionmodels, kingma2023understandingdiffusionobjectiveselbo}, Discrete Diffusion Models can be trained by minimizing an expected Negative Evidence Lower Bound (NELBO). Specifically, one can first derive an expression for the standard discrete-time expected NELBO with $n$ noise levels:
\begin{equation}
\label{eq:discrete_diffusion_nelbo}
L_n(x_0; \theta) = \mathbb{E}\left[ -\log \hat{p}_{0|t_1}^\theta(x_0 \mid x_{t_1}) + \sum_{i=2}^n \operatorname{KL}\left(p_{t_{i-1}|0,t_i}(\cdot \mid x_0, x_{t_i}) \parallel \hat{p}_{t_{i-1}|t_i}^\theta(\cdot \mid x_{t_i})\right) \right],
\end{equation}
and find the limit of \eqref{eq:discrete_diffusion_nelbo} as $n \rightarrow \infty$ to obtain a continuous-time objective. For the absorbing corruption processes, the limit of \eqref{eq:discrete_diffusion_nelbo} converges to \citep{ou2024absorbingdiscretediffusionsecretly, sahoo2024simple, shi2024simplified}:
\begin{equation}
    \label{eq:mdm_nelbo}
    L^{\text{mask}}_\infty (x_0; \theta) = - \int_0^1 \frac{\alpha_t'}{1 - \alpha_t} \mathbb{E}_{p_{t|0}} \left[ e_{x_0}^\top \log \mathbf{x}_\theta(t, x_t) \right] \, \mathrm{d}t.
\end{equation}
With uniform corruption, \eqref{eq:discrete_diffusion_nelbo} converges to \citep{schiff2025simpleguidancemechanismsdiscrete, sahoo2025diffusionduality}:
\begin{equation}
    \label{eq:udm_elbo}
    \begin{aligned}
    L^{\text{unif}}_\infty(x_0; \theta) = -\int_0^1 \frac{\alpha_t'}{N \alpha_t} \mathbb{E}_{p_{t|0}} \Bigg[ 
    &\frac{N}{\bar{x}_{x_t}} - \frac{N}{e_{x_t}^\top \mathbf{x}_\theta(t, x_t)} \\
    & - (\kappa_t \mathbb{I}_{x_t = x_0} + \mathbb{I}_{x_t \neq x_0}) \left( N \log (e_{x_t}^\top \mathbf{x}_\theta(t, x_t)) - \mathbf{1}^\top \log \mathbf{x}_\theta(t, x_t) \right) \\
    &- N \frac{\alpha_t}{1 - \alpha_t} \left( \log(e_{x_t}^\top \mathbf{x}_\theta(t, x_t)) - \log(e_{x_0}^\top \mathbf{x}_\theta(t, x_t)) \right) \mathbb{I}_{x_t \neq x_0} \\
    &- \left( (N-1)\kappa_t \mathbb{I}_{x_t = x_0} - \frac{1}{\kappa_t} \mathbb{I}_{x_t \neq x_0} \right) \log \kappa_t 
    \Bigg] \mathrm{d}t,
    \end{aligned}
\end{equation}
where $\kappa_t := \frac{1 - \alpha_t}{N \alpha_t + 1 - \alpha_t}$ and $\mathbb{I}$ is the indicator function.
\subsection{Predictor-Corrector Sampler}
\label{sec:pc_sampler}
Our PC baselines for MDMs and UDMs use a \emph{pure} predict-and-renoise sampler.
Instead of the ancestral transition $\hat p^\theta_{s|t}$, each step (1) samples $\hat x_0$ from the plug-in posterior at target time $0$, i.e., the bridge \eqref{eq:appendix_bridge_bayes} with $s = 0$ and $e_{x_0}$ replaced by the denoiser prediction, and (2) re-applies the forward process \eqref{eq:discrete_forward} to $\hat x_0$ at the next time.
Let $(t_i)_{i=0}^n$ be the sampling time grid (\Cref{sec:adaptive_time_schedule}), and let $\mathbf{x}_\theta(t, x_t)$ denote the denoiser, with logits scaled by the temperature $T$ of \eqref{eq:logit_temperature}.
Starting from $x_{t_n} \sim \operatorname{Cat}(\pi)$, we repeat for $i = n, \dots, 1$:
\begin{align}
    \textbf{(Predict)} \qquad & \hat x_0 \sim \operatorname{Cat}\bigl(\mathbf{x}_\theta(t_i, x_{t_i})\bigr), \\
    \textbf{(Re-noise)} \qquad & x_{t_{i-1}} \sim p_{t_{i-1}|0}(\cdot \mid \hat x_0) = \operatorname{Cat}\bigl(\alpha_{t_{i-1}} e_{\hat x_0} + (1-\alpha_{t_{i-1}})\pi\bigr),
\end{align}
independently for each position, and we return $\hat x_0$ at the last step ($t_0 = 0$, $\alpha_{t_0} = 1$).
The sampler therefore uses the bridge $p_{s|0,t}(x_s \mid x_0, x_t) = p_{s|0}(x_s \mid x_0)$ with the plug-in prediction, as in star-shaped diffusion \citep{okhotin2023star}.
Every step discards $x_{t_i}$ except through the prediction $\hat x_0$.
For MDMs, this means that tokens unmasked at earlier steps can be masked again, which lets the sampler revise earlier choices.
\subsection{Self-Conditioning and Loopholing}
\label{sec:ext_sc_loopholing}
Because Discrete Diffusion models operate directly on discrete state spaces, the rich categorical distribution predicted by the denoiser collapses into a single discrete token value at each sampling step. Consequently, sampling discards the uncertainty of the denoiser. Therefore, it is common to resort to Self-Conditioning (SC; \citealp{chen2022analog}) or Loopholing \citep{jo2025loopholing} to propagate continuous information across sampling steps. We implement SC following \citet{jo2025loopholing}.
\paragraph{Architecture.}
We decompose the denoiser into three parts. An embedding layer $e(\cdot)$ maps the input state $x_t$ to one vector per position. A backbone $f_\theta$ maps the embedded input and the time $t$ to the last hidden representation $h$, which we call the latent. An output head $g$ (a linear projection followed by a softmax) maps $h$ to the predicted distribution over clean tokens. Without SC, the denoiser computes
\begin{equation}
    h = f_\theta\bigl(e(x_t), t\bigr), \qquad \mathbf{x}_\theta(t, x_t) = g(h).
\end{equation}
With SC, the denoiser additionally receives a latent $h^\text{prev}$ and adds it to the input embedding after a LayerNorm $\mathrm{LN}(\cdot)$:
\begin{equation}
\label{eq:sc_loopholing}
    h = f_\theta\bigl(e(x_t) + \mathrm{LN}(h^\text{prev}), t\bigr), \qquad \mathbf{x}_\theta(t, x_t, h^\text{prev}) = g(h).
\end{equation}
Setting $h^\text{prev} = 0$ recovers a denoiser without context. Following \citet{jo2025loopholing}, we initialize the scale and shift of $\mathrm{LN}$ to zero, so that SC initially leaves the input unchanged.
\paragraph{Training.}
Let $\mathrm{sg}(\cdot)$ denote the stop-gradient operator, which acts as the identity in the forward pass and blocks gradients in the backward pass. With probability $p_\text{SC} = 0.9$, we apply SC: a first forward pass with $h^\text{prev} = 0$ and without gradient produces a latent $h^0$, and a second forward pass \eqref{eq:sc_loopholing} with $h^\text{prev} = \mathrm{sg}(h^0)$ produces the prediction on which we compute the loss. With probability $1 - p_\text{SC}$, we use a single forward pass with $h^\text{prev} = 0$. This avoids unrolling the sampling trajectory during training.
\paragraph{Matching the Training FLOPs.}
We count the cost of a forward pass as $1$ and that of a forward and backward pass as $3$. A training step with SC then costs on average $3 + p_\text{SC} \cdot 1 = 3 + 0.9 = 3.9$ forward-equivalents, against $3$ without SC. To match the training FLOPs of models trained without SC, we train SC models for a fraction $3 / 3.9 \approx 0.77$ of the training steps.
\paragraph{Sampling.}
At the first sampling step, $h^\text{prev} = 0$. The latent $h$ computed at step $t_i$ is then passed as $h^\text{prev}$ to step $t_{i-1}$. Next to the sampled discrete state, each step thus carries a deterministic continuous state. This adds one LayerNorm and one addition per step, and no extra forward pass.

\subsection{Dirichlet Flow Matching}
\label{sec:ext_dirichlet_fm}
This section contains additional background on Euclidean and Dirichlet Flow Matching. To ensure consistency with the Discrete Diffusion literature, we denote the noise distribution by $p_1$ and the empirical data distribution by $p_0$.
\paragraph{Continuous Normalizing Flows.}
Continuous Normalizing Flows (CNFs; \citealp{chen2018neural, grathwohl2018ffjord}) are generative models on $\mathbb{R}^d$ that transport samples from a tractable prior $p_1 = p_{\text{noise}}$ to an unknown data distribution $p_0= p_{\text{data}}$. To match the rest of the paper, we reverse the usual CNF time convention, in which $t=0$ is noise. The time-dependent velocity field $u_t^\theta \colon \mathbb{R}^d \to \mathbb{R}^d$ induces a flow map $\phi_t \colon \mathbb{R}^d \to \mathbb{R}^d$ via the Ordinary Differential Equation (ODE):
\begin{equation}
    \label{eq:cnf_ode}
    \frac{\mathrm{d} x_t}{\mathrm{d} t} = u_t^\theta(x_t), \quad x_1 \sim p_1,
\end{equation}
where $\theta$ parameterizes a neural network. Integrating \eqref{eq:cnf_ode} from $t=1$ to $t=0$ maps approximately noise $x_1$ to data $x_0 = \phi_0(x_1)$, defining intermediate densities $p_t = [\phi_t]_\sharp p_1$ along the probability path $\{p_t\}_{t \in [0, 1]}$. 
\paragraph{Flow Matching.}
Rather than leaving the velocity field $u_t^\theta$ unconstrained and optimizing $\theta$ through expensive ODE simulation, Flow Matching (FM) constructs a generative process using conditional probability paths $p_{t|0}(x_t \mid x_0)$ and conditional velocity fields $u_{t|0}(x_t \mid x_0)$, conditioned on clean data samples $x_0 \sim p_0$. From this conditional pair, we define 
\begin{equation}
    \label{eq:marginal_path}
    p_t(x_t) = \int p_{t|0}(x_t \mid x_0) \, p_0(x_0) \, \mathrm{d}x_0.
\end{equation}
and
\begin{equation}
    \label{eq:marginal_velocity}
    u_t(x_t) = \int u_t(x_t \mid x_0) \, p_{0|t}(x_0 \mid x_t) \, \mathrm{d}x_0 = \int  u_t(x_t \mid x_0) \frac{p_{t|0}(x_t \mid x_0) p_0(x_0)}{p_t(x_t)} dx_0.
\end{equation}
It can then be shown that these quantities (see e.g., \citet{lipman2022flow} ) satisfy
\begin{equation}
    \label{eq:continuity}
    \frac{\partial p_t(x)}{\partial t} + \nabla \cdot \big( p_t(x) u_t(x) \big) = 0.
\end{equation}
Satisfying \eqref{eq:continuity} ensures that integrating an ODE of drift $u_t$ from $t=1$ to $t=t'$ produces a sample $x_{t'} \sim p_{t'}$.
%
\paragraph{Extension to the Simplex.}
Dirichlet Flow Matching (DFM; \citealp{stark2024dirichlet}) extends FM to the simplex $\Delta_N$ to model categorical data. In particular, DFM tackles the issue of contracting support on $\Delta_N$ by defining conditional probability paths with full support:
\begin{equation}
    p_{t|0}(P_t | P_0) = \Dir \big(P_t; \mathbf{1} + h_t P_0 \big).
\end{equation}
\citet{stark2024dirichlet} show that when paired with the following conditional velocity field $u_{t|0}$, $(p_{t|0}, u_{t|0})$ satisfy the continuity equation. For $P_0$ the one-hot embedding $e_i$ of category $i$, 
\begin{equation}
\label{eq:dfm_c}
\begin{aligned}
    &u_{t|0}(P_t | P_0) = C(P_{t,i}, t) (e_i - P_t), \\
    &C(P_{t,i}, t) = -\tilde{I}_{P_{t,i}}(t+1, N-1) \frac{\mathcal{B}(t+1, N-1)}{(1 - P_{t,i})^{N-1} P^t_{t,i}},
\end{aligned}
\end{equation}
where $\mathcal{B}$ denotes the beta function, $\tilde{I}_x(a, b) = \frac{\partial}{\partial a} I_x(a, b)$ the partial derivative of the regularized incomplete beta function, and $N$ the number of categories. Like us, \citet{stark2024dirichlet} train a denoiser $p^\theta_{0|t}(x_0|x_t): (\Delta_N)^L \mapsto (\Delta_N)^L$ with Cross-Entropy. However, their sampling dynamics differs, as they marginalize the conditional as in \eqref{eq:marginal_velocity}, replacing the true posterior by the learned one. Therefore \citet{stark2024dirichlet} propose a deterministic sampler, while ours is not, even in the case $\kappa = 0$ (\Cref{sec:simplicialdiffusion}).

\newpage

\appendixpart[]{Experimental Setup, Ablations and Results}

\section{Extended Experimental Details}
\subsection{Time Samplers and Sampling Schedules}
\label{sec:adaptive_time_sampler}
We distinguish two choices. The \emph{time sampler} is the distribution of $t$ during training: uniform ($^\dagger$) or adaptive ($^\ddagger$). The \emph{sampling schedule} is the time grid used at inference: linear, cosine or adaptive. The two can be combined freely, except that the adaptive schedule requires the adaptive time sampler, since it reuses the loss profile fitted during training.
\paragraph{Motivation.}
By default, recent Discrete Diffusion draws the time $t$ uniformly from $t \in [0, 1]$ during training. However, recent \emph{continuous} diffusion language models require \emph{adaptive} time samplers \citep{dieleman2022continuous, pynadath2025candi, batzolis2026cobitlanguagemodelingbitstream, chemseddine2026sphericalflowssamplingcategorical, chen2026langflowcontinuousdiffusionrivals, deschenaux2026languagemodelinghypersphericalflows, lee2026one, potaptchik2026discreteflowmaps, raya2026noiseschedulinginformationguidedallocation, roos2026categoricalflowmaps, yang2026continuousdiffusionscalescompetitively} to perform well. During training, we use the derivative of the loss profile $t \mapsto \frac{\rmd}{\rmd t}  L_t$ as a proxy for where the network learns the most. Assuming that the true loss increases monotonically with $t$, the derivative $\frac{\rmd{L_t}}{\rmd t} \ge 0$ directly defines an importance density $q(t) \propto \frac{\rmd{L_t}}{\rmd t}$. 
\paragraph{High-Level Algorithm.}
We maintain a ring buffer to store recent $(t, L_t)$ pairs. During the first 1k training steps, we start with $t \sim \mathcal{U}(0, 1)$ to fill the buffer. Afterwards, every $50$ training steps, we approximate the loss profile using the content of the ring buffer, with either B-splines \citep{cox1972numerical, deboor1972calculating} or with the mean per bucket, and differentiate via finite differences to obtain the empirical density $\hat q(t) \propto \frac{\rmd L_t}{\rmd t}$. We track an EMA $\hat{q}_\text{EMA}$ of the successive densities $\hat q$ with momentum $0.9$ for stability. Let $\hat{F}_\text{EMA}$ denote the CDF obtained by numerically integrating $\hat{q}_\text{EMA}$. Finally, we evaluate $\hat{F}_\text{EMA}$ on a fine grid of 1k values and store $(\hat{F}_\text{EMA}(t), t)$. We evaluate $t = \hat{F}^{-1}_\text{EMA}(u)$ with linear interpolation for $u \sim \mathcal{U}(0, 1)$ during training (inverse transform sampling).
\paragraph{Estimating the Loss Profile.}
\citet{deschenaux2026languagemodelinghypersphericalflows} estimates $L_t$ by fitting a global cubic B-spline with ridge regression. In preliminary experiments, we found that a simpler bucketed piecewise-linear estimator led to stronger performance. Therefore, we partition $[0, 1]$ into $B = 50$ uniform bins, compute the mean loss within each bin independently, and linearly interpolate, placing the estimated means at the center of each bin. We compute the density as $\hat{q}(t) = \max \left( 0, \frac{\rmd L}{\rmd t} \right)$. While the true $\frac{\rmd L}{\rmd t}$ should be non-negative in principle, $\max$ removes numerical artifacts that would make the density negative. 
\paragraph{Sampling Time Schedules.}
\label{sec:adaptive_time_schedule}
We define the sampling time grid as $0 = t_0 < t_1 < \dots < t_n = 1$ and compare three schedules.
The first is the linear grid $t_i = \frac{i}{n}$.
The second is the cosine grid $t_i = \cos\left(\frac{\pi}{2}\left(1 - \frac{i}{n}\right)\right)$ \citep{chang2022maskgitmaskedgenerativeimage, shi2024simplified}.
The third is an adaptive grid that reuses the loss profile fitted during training.
Specifically, we take the final CDF $\hat F_\text{EMA}$ of the adaptive time sampler, stored as a lookup table of 1k points, and place the grid at uniform quantiles:
\begin{equation}
    t_i = \hat F^{-1}_\text{EMA}\!\left(\tfrac{i}{n}\right), \qquad i = 0, \dots, n,
\end{equation}
evaluated by linear interpolation, exactly as during training.
Since $\hat q_\text{EMA} \propto \max(0, \frac{\rmd L_t}{\rmd t})$, the grid spends more denoising steps on the noise levels where the loss changes fastest.
These are the same noise levels at which the adaptive sampler concentrates training.
The grid needs no extra computation or tuning at inference.
\subsection{Denoiser Input for SDMs}
\label{sec:sdm_input_variants}
The SDM denoiser receives a sequence of simplex states $P_t = (P_t^1, \dots, P_t^L) \in (\Delta_N)^L$ (\Cref{sec:multiple_tokens}). As in \Cref{sec:ext_sc_loopholing}, an embedding layer $e$ maps the input to one vector in $\rset^d$ per position, a backbone $f_\theta$ maps the embedded sequence and the time $t$ to the last hidden representation $h$, and an output head $g$ maps $h$ to the predicted distribution $\hat P_\theta(t, P_t)$. Let $E \in \rset^{N \times d}$ be the input embedding matrix, whose $i$-th row $E_i$ embeds token $i \in \mcx$. We consider two embeddings of $P_t$, each with or without Self-Conditioning.
\paragraph{Expectation.}
We feed the expected embedding
\begin{equation}
    e(P_t)^\ell = (P_t^\ell)^\top E = \sum\nolimits_{i=1}^N P_{t,i}^\ell \, E_i = \mathbb{E}_{X \sim \Categorical(P_t^\ell)}\bigl[E_X\bigr] .
    \label{eq:sdm_input_expectation}
\end{equation}
This input uses the full belief state: two states with the same most likely token but different uncertainty can map to different inputs. 
Note that we are not ensured that we do not necessarily have that $e(P) = e(Q) $ implies $P=Q$, i.e. the embedding is not necessarily injective. 
Since $P_t^\ell$ has strictly positive entries almost surely, \eqref{eq:sdm_input_expectation} requires a dense product with $E$, which costs $2Nd$ FLOPs per token, as much as the output projection. On TinyGSM ($N \approx 49$k, $d = 768$), this adds $\approx 29\%$ to the FLOPs of a forward pass of our 12-layer DiT.
\paragraph{Argmax.}
To avoid this product, we embed only the most likely token, both during training and at sampling:
\begin{equation}
    e(P_t)^\ell = E_{\hat x_t^\ell}, \qquad \hat x_t^\ell = \argmax_{i} P_{t,i}^\ell ,
    \label{eq:sdm_input_argmax}
\end{equation}
which is a table lookup, as for MDMs and UDMs. The denoiser then sees only $\hat x_t^\ell$, but the sampler is unchanged: the transition of \Cref{propbeaut:simplicialtransition} still acts on the full state $P_t$, so the belief state is still carried from step to step.
\paragraph{Self-Conditioning.}
We also combine both inputs with SC, as described in \Cref{sec:ext_sc_loopholing}, with the simplex state $P_t$ in place of $x_t$: the backbone receives $e(P_t) + \mathrm{LN}(h^\text{prev})$. SC gives the denoiser a rich summary of its previous prediction, which complements the argmax input in particular (Argmax + SC).
\paragraph{Which Input Is Used Where.}
On Sudoku, SDMs use the expectation input, with and without SC (\Cref{tab:sudoku_hard_table_main}). On TinyGSM, we report the expectation, Argmax and Argmax + SC inputs (\Cref{tab:tinygsm_sdm_sdm_t10_s512,tab:tinygsm_sdm_sdm_t10_s64,tab:tinygsm_sdm_sdm_t01_s512,tab:tinygsm_sdm_sdm_t01_s64}).
\subsection{Dirichlet Flow Matching Baseline}
\label{app:dfm_baseline}

To understand the benefits of our simplex based method, we also compared to the Dirichlet Flow Matching (DFM) method, as described in \Cref{sec:ext_dirichlet_fm}. The original implementation provided by the authors at \url{https://github.com/HannesStark/dirichlet-flow-matching} does not include the ability to run on any of the datasets we experimented on. Furthermore, the original implementation and experiments deal with small vocab sizes (only up to 160 categories) whereas our text experiments use up to 50k tokens. We therefore re-implemented and extended the DFM method to provide a fair comparison to our approach. We made two key changes to DFM. The first was to use the adaptive time sampler during training which we found helped performance on TinyGSM for our method. 
In the original presentation in \citet{stark2024dirichlet}, eq (14) presents the corruption distribution as (in the authors' original notation)
\begin{equation}
    p_t(\mathbf{x} | \mathbf{x}_1 = \mathbf{e}_i) = \text{Dir}(\mathbf{x}; \mathbf{\alpha} = \mathbf{1} + t \cdot \mathbf{e}_i).
\end{equation}
In the original implementation during training, $t$ is sampled as an exponential random variable with scale parameter $\alpha_{\text{scale}}$, $t \sim \text{Exp}(\frac{1}{\alpha_{\text{scale}}})$. We instead sample $t$ with the adaptive time sampler described in \Cref{sec:adaptive_time_sampler}.

Secondly, we improved the sampling implementation to handle much larger numbers of categories than the original implementation. Sampling in DFM requires integrating the conditional velocity field \eqref{eq:dfm_c} towards target vertices $e_i$ on the simplex $\Delta_N$. Writing $b = P_{t,i} \in [0, 1]$ for the coordinate along category $i$, the velocity is $u_{t|0}(P_t \mid P_0 = e_i) = C(b, t) (e_i - P_t)$ with
\begin{equation}
    C(b, t) = -\tilde{I}_b(t+1, N-1) \frac{\mathcal{B}(t+1, N-1)}{(1 - b)^{N-1} b^t},
    \label{eq:dfm_c_def}
\end{equation}
where $\mathcal{B}$ and $\tilde{I}$ are defined in \Cref{sec:ext_dirichlet_fm}. Since the first argument is $t+1$, we have $\tilde{I}_b(t+1, N-1) = \frac{\partial}{\partial t} I_b(t+1, N-1)$.

In the original implementation of \citet{stark2024dirichlet}, $\tilde{I}_b(t+1, N-1)$ is calculated through computing $I_b(t+1, N-1)$ at linearly spaced intervals of $b \in [0, 1]$ and $t \in [t_\text{min}, t_\text{max}]$ and using numerical differences to approximate the gradient. This breaks down for large vocab sizes because if we consider a sample from the uniform prior $\Dir(\mathbf{1})$ then the coordinate $b$ is distributed according to $\text{Beta}(1, N-1)$ which has mean $1/N \approx 2 \times 10^{-5}$. Therefore we need more precision in this $b$ range than a linear spacing of $b$ would provide. The original implementation used $\Delta b = 10^{-3}$ which immediately skips over the high probability region at $2 \times 10^{-5}$ for large vocabulary sizes.  We instead use a non-uniform geometric discretization grid that is denser around the prior mean $1/N$.
\begin{equation}
 b_k = b_{\text{min}} \left( \frac{b_{\text{max}}}{b_{\text{min}}} \right)^{\frac{k}{M-1}}
\end{equation}
with $M = 1000$, $b_{\text{min}} = 10^{-7}$ and $b_{\text{max}} = 0.999$.

Furthermore, the original implementation uses $\texttt{scipy.special.betainc}$ to provide the values of $I_b(t+1, N-1)$. However, for $I_b$ close to $1.0$, values quickly round to exactly $1.00$ when represented as floats thus giving $0$ numerical gradient. We avoid this by using the complementary Beta CDF function \texttt{scipy.special.betaincc} to compute gradients in the regime where $I_b > 0.5$ because this keeps numerical values closer to $0.0$ where they have more precision.

Finally, as $b$ gets larger in \eqref{eq:dfm_c_def}  $\tilde{I}_b \rightarrow 0$ while $\frac{1}{(1-b)^{N-1}} \rightarrow \infty$. These two effects should approximately cancel leaving a well-conditioned $C(b,t)$ however when the terms are represented numerically, underflow and overflow can result in attempting to compute $0 \times \infty$. To avoid this situation, we can reformulate $C(b, t)$ into an exact integral where $(1 - b)^{N - 1}$ is canceled analytically

\begin{align}
    C(b, t) &= \frac{\mathcal{B}(t+1, N - 1)}{(1 - b)^{N - 1} b^t} \frac{\partial}{\partial t} \left[ \frac{1}{\mathcal{B}(t+1, N - 1)} \int_b^1 s^t (1 - s)^{N - 2} ds \right] \label{eq:step1} \\
    &= \frac{1}{(1 - b)^{N - 1} b^t} \left[ \int_b^1 s^t (1 - s)^{N - 2} \ln(s) ds \right. \nonumber \\
    &\qquad\qquad\qquad \left. - \frac{\frac{\partial}{\partial t}\mathcal{B}(t+1, N - 1)}{\mathcal{B}(t+1, N - 1)} \int_b^1 s^t (1 - s)^{N - 2} ds \right] \label{eq:step2} \\
    &= \frac{1}{(1 - b)^{N - 1} b^t} \int_b^1 s^t (1 - s)^{N - 2} \left[ \psi(t + N) - \psi(t + 1) + \ln(s) \right] ds \label{eq:step3} \\
    &= \frac{1}{(1 - b)^{N - 1} b^t} \int_0^1 (b + (1 - b)u)^t (1 - b)^{N - 1} (1 - u)^{N - 2} \nonumber \\
    &\qquad\qquad\qquad \times \left[ \psi(t + N) - \psi(t + 1) + \ln(b + (1 - b)u) \right] du \label{eq:step4} \\
    &= \int_0^1 \left( 1 + \frac{1 - b}{b} u \right)^t (1 - u)^{N - 2} \left[ \psi(t + N) - \psi(t + 1) + \ln(b + (1 - b) u) \right] du, \label{eq:step5}
\end{align}

where $\psi(z) = \frac{\rmd}{\rmd z} \ln \Gamma(z) = \frac{\Gamma'(z)}{\Gamma(z)}$ is the digamma function. In the above derivation, we have used in ~\eqref{eq:step3} that $\frac{\partial}{\partial t} \ln \mathcal{B}(t+1, N-1) = \psi(t+1) - \psi(t+N)$, and in ~\eqref{eq:step4} the  substitution $s = b + (1 - b)u$, allowing $(1 - b)^{N - 1}$ to cancel in ~\eqref{eq:step5}. For $b \gtrsim 0.01$, we numerically integrate this integral using 48-node Gauss-Legendre quadrature instead of the numerical gradient approach. For numerical integration, we make the substitution $(1-u)^{N-2} = e^{-w}$ $\implies$ $u = 1 - e^{\frac{-w}{N-2}}$ and integrate over the range $w \in [0, 50]$. This prevents standard quadrature on $[0,1]$ from stepping over the narrow region near zero where $(1-u)^{N-2}$ is concentrated before decaying to zero.

These three improvements to the DFM sampling algorithm allow us to safely sample at the $50k$ vocabulary size scale.

To verify our implementation, we first re-ran the toy experiment from \citet{stark2024dirichlet} where the model is tasked to reproduce a synthetic categorical distribution. We train our re-implemented model with 40 categories, sample it and compute the KL-divergence to the ground truth. We obtain a KL-divergence of 0.03 approximately matching the value from \citet[Figure 4]{stark2024dirichlet}.

We integrate the marginal velocity field with the explicit Euler method in $\alpha$ from $\alpha_\text{min}=1$ to $\alpha_\text{max}=4\alpha_\text{scale}$, project onto the simplex after every step, and return the argmax of the final prediction. We use the same number of steps as for the other methods (180 on Sudoku; 64 or 512 on TinyGSM), so one step is one NFE.

\subsection{Sudoku}
\label{sec:appendix-setup-sudoku}
\paragraph{Data Generation.}
We generate 9x9 Sudoku puzzles with 30 visible cells using the backtracking generator of \citet{alp2024sudoku}. We produce partial grids by iteratively removing cells and verifying that the solution remains unique. We ensure that the training and validation sets share no full grids. We use 200k training and 5k validation puzzles (instead of the 48k/2k split in \citet{deschenaux2026languagemodelinghypersphericalflows, kim2026finetuningmaskeddiffusionprovable}) since we observed mild overfitting in preliminary experiments, especially when using the adaptive time sampler. We report the exact-match accuracy on the validation set.
\paragraph{Tokenization.}
We represent each puzzle as a sequence of 180 tokens drawn from a vocabulary of size 14. We map digits $\{1, \dots, 9\}$ directly to their numerical values, empty cells $\texttt{\_}$ to ID 0, row separators $\texttt{|}$ to ID 10, sequence boundaries \bostoken{} to ID 11, (unused) padding to ID 12, and the \masktoken to ID 13. To accommodate both autoregressive and diffusion algorithms, we concatenate the unsolved puzzle and its complete solution into a sequence of 180 tokens:
\begin{equation}
\Big( \underbrace{\text{\bostoken}, \, 5, 3, \texttt{\_}, \dots, 2, \, \texttt{|}, \dots, \texttt{|}, \, \texttt{\_}, \dots, 9}_{\text{Prompt / Unsolved Puzzle } (90\text{ tokens})}, \; \underbrace{\text{\bostoken}, \, 5, 3, 4, \dots, 2, \, \texttt{|}, \dots, \texttt{|}, \, 3, \dots, 9}_{\text{Target / Complete Solution } (90\text{ tokens})} \Big).
\end{equation}
During training, we only corrupt the target tokens and keep the first half clean as conditioning. During inference, we start from a clean prompt and denoise only the second half.
\paragraph{Architecture and Hyperparameters.}
We train an 8-layer Diffusion Transformer (DiT) \citep{peebles2023scalablediffusionmodelstransformers} with hidden dimension 512, 8 attention heads, an MLP expansion ratio of 4, 1D Rotary Positional Embeddings \citep{su2023roformerenhancedtransformerrotary}, and untied input/output embeddings (28.6M total parameters). We apply a 0.1 dropout rate to the output projection \citep{sahoo2024simple}. We train with Adam ($\beta_1 = 0.9, \beta_2 = 0.999, \epsilon = 10^{-8}$), gradient clipping at maximum norm $1.0$, and no weight decay. The learning rate warms up linearly to $3 \times 10^{-4}$ over 2.5k steps and remains constant afterwards. We maintain an Exponential Moving Average (EMA) of the parameters with decay $0.9999$ for evaluation for all methods. We implement time conditioning via Adaptive LayerNorm (AdaLN) \citep{peebles2023scalablediffusionmodelstransformers}. The time-independent variants (AR and absorbing diffusion) receive a constant (zero) vector in place of time-conditioning, to share the exact same architecture with the other approaches. We train for 50k steps with a global batch size of 256 in full 32-bit precision ($\approx 77\%$ of the steps, i.e.\ 38{,}462, for SC variants, to match the training FLOPs; \Cref{sec:ext_sc_loopholing}).
\subsection{TinyGSM}
\label{sec:appendix-setup-tinygsm}
\paragraph{Tokenization.}
We train our models on TinyGSM \citep{liu2023tinygsmachieving80gsm8k}, a dataset containing approximately 11.8M synthetic grade-school math word problems associated with executable Python programs producing the correct answer. We evaluate the models zero-shot on the 1319 test problems of GSM8K \citep{cobbe2021trainingverifierssolvemath} by executing a generated program and verifying its numerical output. Following \citet{deschenaux2026languagemodelinghypersphericalflows}, we tokenize TinyGSM with the SmolLM tokenizer \citep{allal2025smollm2smolgoesbig} (49k tokens). Unlike GPT-2, SmolLM pretrains on code, and thus its tokenizer compresses Python programs better. \citet{kim2026stoptrainingworstprogressive} originally used the Qwen2 tokenizer \citep{qwen2}, but its 151k tokens induce huge embedding tables, and therefore we chose SmolLM as a compact alternative.
\paragraph{Architecture and Hyperparameters.}
We train a 12-layer Diffusion Transformer (DiT) \citep{peebles2023scalablediffusionmodelstransformers} with hidden dimension 768, 12 attention heads, an MLP expansion ratio of 4, 1D Rotary Positional Embeddings \citep{su2023roformerenhancedtransformerrotary}, and untied input and output embeddings (167.9M total parameters). Following \citet{sahoo2024simple}, we apply a 0.1 dropout rate to the output projection. We optimize all models using Adam \citep{kingma2014adam} ($\beta_1 = 0.9, \beta_2 = 0.999, \epsilon = 10^{-8}$) without weight decay, clipping gradient norms at $1.0$. We linearly warm up the learning rate to $3 \times 10^{-4}$ over 2{,}500 steps and hold it constant thereafter. We maintain an Exponential Moving Average (EMA) of the parameters with decay $0.9999$ and evaluate the EMA weights across all methods, including the autoregressive baseline. We implement time conditioning via Adaptive LayerNorm (AdaLN) \citep{peebles2023scalablediffusionmodelstransformers} with a conditioning dimension of 128. The time-independent models (autoregressive and absorbing Discrete Diffusion) receive a constant zero vector as time conditioning, to keep the architecture the same across experiments. We train each model for 250k steps with a global batch size of 512 in full 32-bit floating-point precision ($\approx 77\%$ of the steps for SC variants, to match the training FLOPs; \Cref{sec:ext_sc_loopholing}). We train all SDMs with the adaptive time sampler (\Cref{sec:adaptive_time_sampler}). During evaluation, we run 512 sampling steps for all diffusion variants unless stated otherwise.
\paragraph{AST Diversity Score.}
\label{app:ast_diversity}
We measure how structurally different the $K=5$ programs sampled for each GSM8K problem are by comparing their abstract syntax trees (ASTs).
For each sample, we keep the text from the first function definition onward (or the content of a Markdown code block, if present), parse it with Python's \texttt{ast} module, and discard samples that do not parse.
We then normalize each tree: we remove docstrings, rename user-defined variables, functions and arguments to canonical identifiers within each scope (\texttt{VAR0}, \texttt{VAR1}, \dots) while keeping built-in names such as \texttt{sum} or \texttt{range}, and replace keyword-argument names and import aliases by fixed placeholders.
Each node is labeled by its node type; names also carry their canonical identifier, constants their type and value, and arithmetic operators their type.

For two normalized trees $T_a$ and $T_b$ with $|T_a|$ and $|T_b|$ nodes, we compute the tree edit distance $\mathrm{TED}(T_a, T_b)$ \citep{zhang1989simple} with unit costs for insertion, deletion and relabeling, and define the similarity
\begin{equation}
  \mathrm{sim}(T_a, T_b) = \max\Bigl(0,\, 1 - \frac{\mathrm{TED}(T_a, T_b)}{\max(|T_a|, |T_b|)}\Bigr) \in [0, 1] .
\end{equation}
For a problem $i$ whose set of parsable samples $\mathcal{V}_i$ has at least two elements, the diversity is one minus the mean pairwise similarity,
\begin{equation}
  D_i = 1 - \binom{|\mathcal{V}_i|}{2}^{-1} \sum_{a < b \in \mathcal{V}_i} \mathrm{sim}(T_a, T_b) .
\end{equation}
We report $100 \cdot \frac{1}{|\mathcal{I}|}\sum_{i \in \mathcal{I}} D_i$, where $\mathcal{I}$ is the set of problems with at least two parsable samples; higher is more diverse.
\emph{AST Div.\ (Correct)} is the same score computed only on the samples whose execution returns the reference answer, averaged over the problems with at least two such samples. It measures whether a model finds structurally different \emph{correct} solutions, rather than rewarding diversity that comes from incorrect programs.
\subsection{OpenWebText}
\label{sec:appendix-setup-owt}
\paragraph{Tokenization.}
We evaluate unconditional language modeling on OpenWebText (OWT) \citep{Gokaslan2019OpenWeb}. Following prior work, we tokenize OWT using the GPT-2 tokenizer \citep{Radford2019LanguageMA} (50257 tokens) and pack documents into sequences of 1024 tokens. Unlike Sudoku and TinyGSM, which require conditional prefix completion, we corrupt all 1024 positions during training and denoise sequences from scratch during inference. We measure the sample quality with the Generative Perplexity (Gen.\ PPL) under a pretrained GPT-2-large model and the unigram entropy. Specifically, we produce Pareto curves, following \citet{pynadath2025candi}.
\paragraph{Architecture and Hyperparameters.}
We use the same 12-layer DiT architecture as in TinyGSM. We use the exact same optimizer, learning rate schedule, AdaLN time conditioning, EMA decay rate, and global batch size as well. To follow the academic literature, we train each model for 1M steps. We sample the diffusion models with 64 steps, and the AR model token by token (1024 steps).
\subsection{Language Understanding}
\label{sec:appendix-setup-lu}
We evaluate our model on language understanding benchmarks such as ARC-easy \citep{clark2018think}, PIQA \citep{bisk2020piqa} and HellaSwag \citep{zellers2019hellaswag}. The goal of this benchmark is to evaluate the accuracy of a language model to choose the correct continuation when faced with multiple choices for a given context prompt. This is a significant departure from the \emph{sample quality} metrics that we report for Sudoku, TinyGSM and OWT. 
 
We compare our method to GPT-2 \citep{Radford2019LanguageMA}, the retrained LLaMA baseline \citep{touvron2023llama} of \citet{vonrutte2025generalizedinterpolatingdiscretediffusion}, Generalized Interpolating Discrete Diffusion (GIDDs) \citep{vonrutte2025generalizedinterpolatingdiscretediffusion}, Masked Diffusion Models (MDMs) as reported in \citet{deschenaux2025beyond} and Partition Generative Models (PGMs) \citep{deschenaux2026partition}. 

In each scenario we are given a context $s$ of length $L_s$ and possible continuations $c^{(1)}, \dots, c^{(n)}$ ($n=2$ in the case of PIQA, and $n=4$ in the case of ARC-easy and HellaSwag). For a continuation $c = (c^1, \dots, c^{L_c})$ of length $L_c$, we denote by $x = (s, c) \in \mcx^L$ the full sequence, with $L = L_s + L_c$. For AR models, the score of a continuation is given by
\begin{equation}
    S_{\text{AR}}(s, c) = \sum_{\ell=1}^{L_c} \log p_\theta(c^\ell \mid s, c^1, \dots, c^{\ell-1}) . 
\end{equation}
In the case of MDMs, the score is given by 
\begin{equation}
    S_{\text{MDM}}(s, c) =  -\sum_{\ell=1}^{L} \int_0^1 \frac{\alpha_t'}{1-\alpha_t} \langle x^\ell, \log \mathbf{x}_\theta(t, x_t)^\ell \rangle \rmd t . \label{eq:mdm_scoring}
\end{equation}
In \citet{vonrutte2025generalizedinterpolatingdiscretediffusion}, the integral is discretized on a uniform grid with 128 NFE, whereas in \citet{deschenaux2025beyond,deschenaux2026partition} the authors sample uniformly $1024$ times in the interval $[0,1]$. 

In our case, the noising process retains significant information about the initial sample contrary to MDMs. Therefore, we consider an additional regularization term emphasizing that we care about the continuation and not only the whole text plausibility. In practice, we consider 
\begin{equation}
    S_{\text{SDM}}(s, c) =  \sum_{\ell=L_s+1}^{L} \int_0^1 q_t \langle x^\ell, \log x_{\theta}(t, x_t)^\ell \rangle \rmd t - w \sum_{\ell=L_s+1}^{L} \int_0^1 q_t \langle \tilde{x}^\ell, \log x_{\theta}(t, \tilde{x}_t)^\ell \rangle \rmd t , \label{eq:scoring_language_understanding}
\end{equation}
where $\tilde{x}$ is the same as $x$ except that the context is replaced by $\texttt{"Answer"}$. 
Given a score $S$, we define the \emph{accuracy} as follows. Given a context $s$ and possible continuations $c^{(1)}, \dots, c^{(n)}$ for this context with ground-truth $c^{(1)}$, we define the accuracy as $\text{acc} = \updelta_1\bigl(\argmax_{j \in \{1, \dots, n\}} S(s, c^{(j)})\bigr)$, where we normalize each score by the byte length of its continuation. For each task we report the average accuracy over the test set. 

The quantity $q_t$ is the density we consider to reweight the time during training with the ring buffer, see \Cref{sec:adaptive_time_sampler} for details. This practice is consistent with the recommendation of \citet{holtzman2021surface} who introduced \emph{Domain Conditional Pointwise Mutual Information}. For completeness, we sweep over $w \in [0,1]$ and remark that $w=0$ corresponds to disregarding the domain regularization. 

We first evaluate the results for a checkpoint obtained while training with the hyper-parameters of \Cref{sec:appendix-setup-owt}. 
However, we note that the training procedure on OWT noises the whole sequence of 1024 tokens. Hence, when we noise only the continuation while keeping the context clean in \eqref{eq:scoring_language_understanding}, the model is severely out of distribution. We therefore also train a short adaptation run, which changes how the sequence is corrupted on OWT. To train the adapted model, we initialize the training run with the baseline weights and train the model for 20,000 iterations with global batch size 256 and EMA decay rate 0.999. We consider a learning rate of $1\times10^{-4}$ with cosine decay to 0 after 20,000 iterations. Finally, during the adaptation run, we train on OWT sequences of length 128 and corrupt them as at evaluation: each sequence is split into a clean context and a noised continuation, with lengths $(L_s, L_c)$ drawn from the $1{,}024$ pairs measured on the tokenized evaluation prompts. 
\subsection{Unconditional molecular generation}
\label{sec:appendix-setup-genmol}

\paragraph{Experimental setup.}

We evaluate unconditional \textit{de novo} molecular generation following the GenMol benchmark protocol~\citep{lee2025genmol} using the SAFE (Sequential Attachment-based Fragment Embedding) molecular representation~\citep{noutahi2024safe}.
We train on the SAFE-GPT dataset~\citep{noutahi2024safe} (v2 version) using the SAFE tokenizer, which has a vocabulary size of $|\mathcal{V}|=1{,}880$. All sequences are padded or truncated to a fixed length of $L=256$. In all diffusion experiments, we treat padding tokens as ordinary tokens during forward corruption, reverse denoising, and in the loss, whereas in the autoregressive (AR) baseline, padding tokens are ignored in the cross-entropy loss.

For all experiments, we use a $12$-layer BERT~\citep{devlin-etal-2019-bert} transformer ($d_{\text{model}}=768$, $12$ attention heads, MLP dimension $3{,}072$) similar to the $87\text{M}$-parameter BERT-Base architecture of GenMol~\citep{lee2025genmol}. We use Post-LayerNorm layers as in MaskGIT~\citep{chang2022maskgitmaskedgenerativeimage}, but replace its MLM prediction head with a single zero-initialized linear output layer that is not tied to the input embeddings. For masked diffusion, similar to~\citet{sahoo2024simple}, we do not condition the network on the diffusion time, while for uniform and Simplex diffusion, we study both variants -- with and without time conditioning. Time-conditioned variants use Sinusoidal time embedding with SiLU activation and embedding dimension $128$, the time conditioning is provided via adaptive layer normalization (AdaLN).
Diffusion models use bidirectional attention, while the AR model uses causal attention. We disable dropout on the attention probabilities and apply a hidden dropout rate of $0.1$ to the embedding output and to every attention and MLP sublayer output. For Simplex Diffusion, the denoiser receives the expectation embedding~\eqref{eq:sdm_input_expectation} of each simplex state. We do not use Self-Conditioning in any of the diffusion models. SDM and masked diffusion are trained with cross-entropy objectives, while uniform diffusion is trained with the ELBO objective.

All models are trained for up to $400$k iterations with a global batch size of $2{,}048$ using AdamW ($\text{lr}=3\times 10^{-4}$ with a $2{,}500$-step linear warmup followed by a constant schedule, weight decay $0.0$, $\beta_1=0.9,\beta_2=0.999,\epsilon=10^{-8}$), global gradient-norm clipping at $1.0$, and an EMA of the parameters with decay $0.9999$; all evaluations use the EMA parameters.
 During training, diffusion times are drawn from $[10^{-4},1]$: by stratified uniform sampling for MDM and UDM, and by the adaptive time sampler of \Cref{sec:adaptive_time_schedule} for SDM.
 
During training, we monitored the evaluation metrics (see below) and noticed that some runs (including AR ones) exhibited collapse (all the evaluation metrics went to zero, while the loss went up). We therefore used early stopping. For each training run, we select the checkpoint with the highest Quality (see below) metric, computed on $1{,}000$ generated molecules. For this evaluation, every method used the standard sampler (see below) with temperature $T=1$ and $256$ sampling steps, and we used $\kappa=0$ for Simplex diffusion. The checkpoints were saved every $500$ steps. In addition, for Simplex Diffusion, we swept over the concentration schedules: \emph{Constant-linear}($\nu_0=0.2,\nu_1=0.5,\ell=0.8)$, \emph{Constant-linear}($\nu_0=0.2,\nu_1=0.75,\ell=0.8)$, \emph{Constant-linear} ($\nu_0=0.4,\nu_1=0.75,\ell=0.2)$,  \emph{Constant}($\nu=0.5$), \emph{Constant}($\nu=0.25$), where the \emph{Constant-linear}($\nu_0=0.4,\nu_1=0.75,\ell=0.2$) schedule achieved the best Quality metric and is used for all reported Simplex Diffusion evaluations.

\paragraph{Evaluation Metrics.}
Following~\citet{lee2025genmol}, we generate $M=1{,}000$ molecules from scratch for each of $3$ sampling seeds, and report the mean $\pm$ standard deviation over seeds. Generated SAFE strings are decoded using the post-processing from~\citep{lee2025genmol}. We report four metrics:
\textbf{Validity} (percentage of the $M$ generated sequences that decode to chemically valid molecules),
\textbf{Uniqueness} (percentage of unique canonical SMILES among valid molecules),
\textbf{Diversity} (one minus the mean pairwise Tanimoto similarity of Morgan fingerprints with radius $2$ and $2{,}048$ bits, over the unique valid molecules), and
\textbf{Quality} (percentage of generated samples that are simultaneously valid, unique and drug-like. Drug-like molecules are defined as those satisfying quantitative estimate of drug-likeness $\mathrm{QED}\geq 0.6$~\citep{bickerton2012quantifying} and synthetic accessibility $\mathrm{SA}\leq 4$~\citep{ertl2009estimation}, respectively, 
following~\citet{lee2025genmol}.).

\paragraph{Sampling and Hyperparameter Sweeps.}
For the autoregressive (AR) baseline, we generate sequences token-by-token using temperature-scaled softmax sampling with temperature $T$. For all diffusion models, we evaluate sampling budgets of $\{32,64,256,1024\}$ steps and three sampling schedules: \textit{linear}, \textit{cosine}, and \textit{adaptive} (\Cref{sec:adaptive_time_schedule}). The detailed results in
\Cref{tab:safegpt_dropout0_ordinary,tab:safegpt_dropout0_conf}
use $256$ steps.

We use two diffusion sampling strategies: (i)~\textbf{standard (ordinary) sampling} (\Cref{tab:safegpt_dropout0_ordinary}), where all sequence positions at step $t \to s$ are updated solely via the reverse bridge transition kernel using temperature-scaled clean predictions sampled from $\hat{P}_0^\ell = \operatorname{Softmax}(z_\theta^\ell(P_t)/T)$ (with churn $\kappa \in [0, 1]$ in \Cref{propbeaut:simplicialtransition} for SDM);
(ii)~\textbf{confidence-based (\textit{conf.}) sampling} (\Cref{tab:safegpt_dropout0_conf}), where each position $\ell \in \{1, \dots, L\}$ is scored by its maximum predicted probability $c_\ell = \max_{v \in \mcx} \hat{P}_{0,v}^\ell$. At step $t \to s$, we commit the union of the $\lfloor \alpha_s^p L \rfloor$ most confident positions (top-fraction schedule with power $p$)
directly to their most likely token $\arg\max_{v \in \mcx} \hat{P}_{0,v}^\ell$ (a simplex vertex for SDM). The remaining positions follow the standard update~(i). Here, $\alpha_s$ is defined in~\eqref{eq:discrete_forward} and is evaluated at the target time $s$ of the transition from $t$ to $s$. Note that the end at the step going from $t \to s$ we have a mixed of clean and noisy tokens in the sample. In the case of UDM and MDM those samples are \emph{in distribution} while in the case of SDM they are severely \emph{out of distribution}. Nevertheless we observe benefits from using this confidence based sampler even without retraining. Additional training of the denoiser with such strategy (e.g.\ training with a fraction $\alpha_t^p$ of positions replaced by clean vertices) could, in principle, further improve Quality and remains an interesting avenue for future work. For MDM, unmasked tokens are never re-masked, and when the top-fraction schedule is active, positions outside the committed set stay masked. For SDM and UDM, the committed set is recomputed at every step.

\section{Additional Experimental Results}
\label{app:additional-exp}

\subsection{Sudoku}
\label{app:additional-exp-sudoku}
We report the full Sudoku results for SDMs with $\kappa \in \{0, 0.2, 1\}$ in \Cref{tab:sudoku_sdm_churn_0_0,tab:sudoku_sdm_churn_0_2,tab:sudoku_sdm_churn_1_0}, for the baselines in \Cref{tab:sudoku_baselines_appendix}, and for our re-implementation of DFM in \Cref{tab:sudoku_hard_dfm_apdx}. The full tables are in \Cref{app:full_tables_sudoku}.
\paragraph{Churn and Concentration Schedule.}
SDMs perform best with $\kappa = 1$. Without SC, every concentration schedule reaches 86--92\% at $\kappa = 1$, whereas no schedule exceeds 79.4\% at $\kappa = 0$. The best result (91.6\%) comes from $\nu_t^\text{cst.-lin.}$ with $\nu_0 = 0.2$, $\nu_1 = 0.75$, $\ell = 0.2$. With SC, SDMs are robust to the churn: all three values reach 98.6--99.2\%.
\paragraph{Comparison with the Baselines.}
\begin{wraptable}{r}{0.45\textwidth}
    \vspace{-12pt}
    \caption{\textbf{Accuracy (\%)} on Sudoku in 180 steps comparing Dirichlet Flow Matching and our SDM method. SDMs use $\nu_t^\text{cst.-lin.}$ with $\nu_0 = 0.4$, $\nu_1 = 0.75$, $\ell = 0.8$, as in \Cref{tab:sudoku_hard_table_main}.}
    \label{tab:sudoku_hard_dfm_apdx}
    \centering
    \small
    \setlength{\tabcolsep}{12pt}
    \renewcommand{\arraystretch}{1.05}
    {%
        \newcommand{\tabrow}{\hspace*{0.8em}}
        \begin{tabular}{l c}
        \toprule
        Model & Linear \\
        \midrule
        \multicolumn{2}{@{}l@{}}{\textit{Continuous}} \\
        \tabrow DFM (reimpl.) & $76.7_{{\color{gray}\scriptscriptstyle\pm 0.75}}$  \\ 
        \rowcolor{gray!15}
        \tabrow SDM ($\kappa = 1.0$) & $88.4_{{\color{gray}\scriptscriptstyle\pm 2.2}}$ \\
        \rowcolor{gray!15}
        \tabrow SDM (+SC; $\kappa = 1.0$) & $\underline{\textbf{99.1}}_{{\color{gray}\scriptscriptstyle\pm 0.2}}$  \\
        \bottomrule
        \vspace{-20pt}
        \end{tabular}%
    }
\end{wraptable}
We compare against SDMs with $\nu_t^\text{cst.-lin.}$ with $\nu_0 = 0.4$, $\nu_1 = 0.75$, $\ell = 0.8$ and $\kappa = 1$, as in \Cref{tab:sudoku_hard_table_main}. Without SC, SDMs (88.4\%) outperform DFM (76.7\%) and FLMs (75.3\%), and match $\mathbb{S}$-FLMs (87.3\%). Among discrete models without SC, only UDMs with predictor-corrector sampling (95.8\%) exceed them. With SC, SDMs reach 99.1\% ($\kappa = 1$), on par with the best discrete model (MDM with SC, 99.3\%).
\paragraph{Adaptive Time Sampler.}
Without SC, the adaptive time sampler improves ancestral sampling across models: MDMs (60.1\% $\to$ 69.4\%), UDMs (73.7\% $\to$ 82.1\%), FLMs (73.4\% $\to$ 75.3\%), $\mathbb{S}$-FLMs (83.5\% $\to$ 87.3\%) and SDMs at $\kappa = 0$ (73.9\% $\to$ 77.1\%). At $\kappa = 1$, SDMs perform similarly with both samplers (within 3 points). With SC, the adaptive sampler increases the variance across seeds for both SDMs and UDMs, and it degrades predictor-corrector sampling for the discrete baselines. We therefore train SDMs with the uniform time sampler in the main text.
\paragraph{Comparison with DFM.}
With the improvements of \Cref{app:dfm_baseline}, DFM reaches 76.7\% on Sudoku, well below SDMs with $\nu_t^\text{cst.-lin.}$ with $\nu_0 = 0.4$, $\nu_1 = 0.75$, $\ell = 0.8$ and $\kappa = 1$ (88.4\%, and 99.1\% with SC; \Cref{tab:sudoku_hard_dfm_apdx}).
\subsection{TinyGSM}
\label{app:additional-exp-tinygsm}
Full tables are in \Cref{app:full_tables_tinygsm}.
\paragraph{Effect of the Churn $\kappa$.}
\label{app:tinygsm_churn}
\Cref{tab:tinygsm_sdm_sdm_t10_s512,tab:tinygsm_sdm_sdm_t10_s64,tab:tinygsm_sdm_sdm_t01_s512,tab:tinygsm_sdm_sdm_t01_s64} report SDM accuracy for $\kappa \in \{0, 0.2, 1\}$.
The tables cover three sampling grids and several concentration schedules (\Cref{sec:concentration_schedule}), as well as the three input variants (expected embedding, Argmax, Argmax + SC), two temperatures and two step budgets.
\Cref{fig:tinygsm_churn_T1p0_s512,fig:tinygsm_churn_T0p1_s512,fig:tinygsm_churn_T1p0_s64,fig:tinygsm_churn_T0p1_s64} sweep $\kappa \in \{0, 0.1, \dots, 1\}$ for two concentration schedules.
In all configurations of the tables, $\kappa = 1$ gives the highest accuracy.
The gain is largest for the expected-embedding input with the adaptive grid.
For $\nu^\text{cst.-lin.}$ with $\nu_0 = 0.2$, $\nu_1 = 0.5$, $\ell = 0.2$, accuracy rises from $12.6\%$ to $45.8\%$ at $T = 1$ and from $19.3\%$ to $49.0\%$ at $T = 0.1$ (512 steps).
The trend is not always monotone.
In some configurations, $\kappa = 0.2$ is below $\kappa = 0$, mostly with the Argmax inputs at 64 steps (e.g., $15.6\% \to 10.0\% \to 22.5\%$ for Argmax, $\nu = 0.5$, cosine grid, $T = 1$).
%
\begin{figure*}
    \centering
    \includegraphics[width=\textwidth]{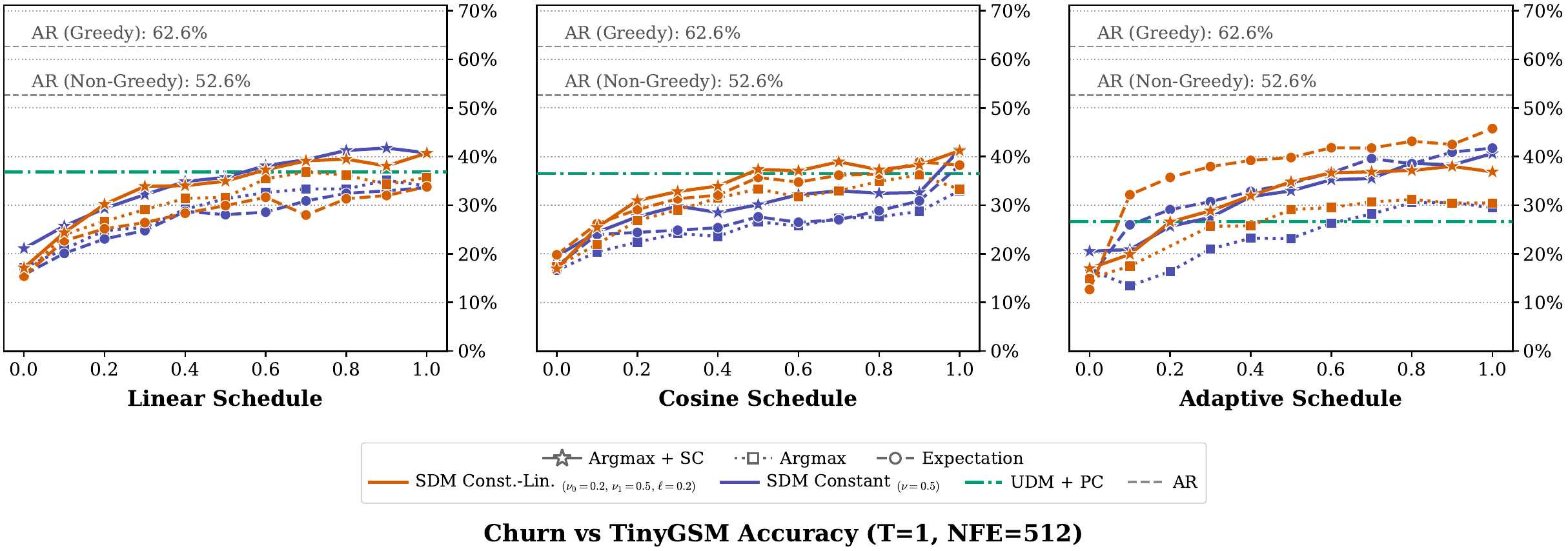}
    \caption{\textbf{Churn $\kappa$ vs.\ TinyGSM accuracy at $T = 1.0$ with 512 steps} for the \textit{Linear} (left), \textit{Cosine} (center) and \textit{Adaptive} (right) sampling schedules. We compare two concentration schedules, \texttt{Const.-Lin.} ($\nu_0=0.2, \nu_1=0.5, \ell=0.2$, orange) and \texttt{Constant} ($\nu=0.5$, indigo), and three input variants: \textit{Argmax + SC} ($\star$), \textit{Argmax} ($\blacksquare$) and \textit{Expectation} $P_t$ ($\bullet$). Horizontal lines show AR decoding (greedy: $62.6\%$; sampling: $52.6\%$) and UDM with the Predictor-Corrector (PC) sampler. Accuracy generally increases with $\kappa$. With the adaptive schedule, \texttt{Const.-Lin.} with the expectation input rises from $12.6\%$ at $\kappa=0$ to $\mathbf{45.8\%}$ at $\kappa=1$, $9.0$ points above the best UDM with PC ($36.8\%$). Exact values for $\kappa \in \{0, 0.2, 1\}$ are in \Cref{tab:tinygsm_sdm_sdm_t10_s512}.}
    \label{fig:tinygsm_churn_T1p0_s512}
\end{figure*}
\begin{figure*}
    \centering
    \includegraphics[width=\textwidth]{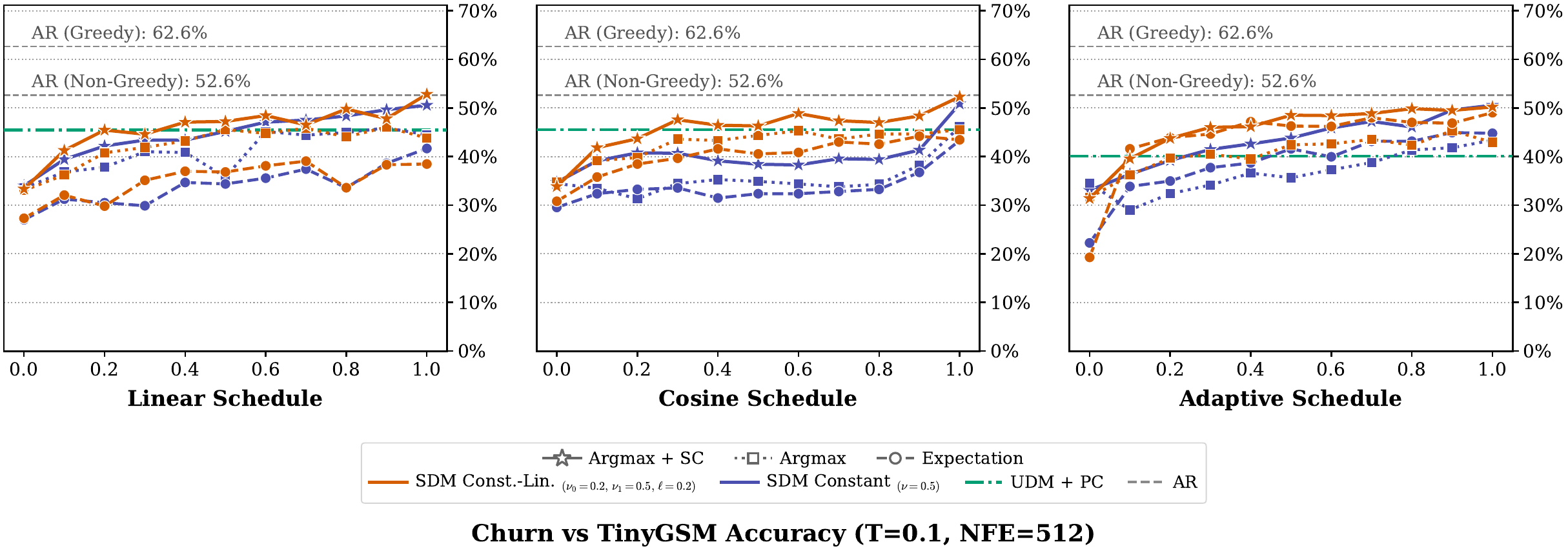}
    \caption{\textbf{Churn $\kappa$ vs.\ TinyGSM accuracy at $T = 0.1$ with 512 steps} for the \textit{Linear} (left), \textit{Cosine} (center) and \textit{Adaptive} (right) sampling schedules. We compare two concentration schedules, \texttt{Const.-Lin.} ($\nu_0=0.2, \nu_1=0.5, \ell=0.2$, orange) and \texttt{Constant} ($\nu=0.5$, indigo), and three input variants: \textit{Argmax + SC} ($\star$), \textit{Argmax} ($\blacksquare$) and \textit{Expectation} $P_t$ ($\bullet$). Horizontal lines show AR decoding (greedy: $62.6\%$; sampling: $52.6\%$) and UDM with the Predictor-Corrector (PC) sampler. Accuracy generally increases with $\kappa$, and \textit{Argmax + SC} is the strongest input under the linear and cosine schedules. At $\kappa=1$, \texttt{Constant} with \textit{Argmax + SC} reaches $\mathbf{50.9\%}$ (cosine), above the best UDM with PC ($45.5\%$). With the adaptive schedule, \texttt{Const.-Lin.} with the expectation input rises from $19.3\%$ at $\kappa=0$ to $49.0\%$ at $\kappa=1$. Exact values for $\kappa \in \{0, 0.2, 1\}$ are in \Cref{tab:tinygsm_sdm_sdm_t01_s512}.}
    \label{fig:tinygsm_churn_T0p1_s512}
\end{figure*}
\begin{figure*}
    \centering
    \includegraphics[width=\textwidth]{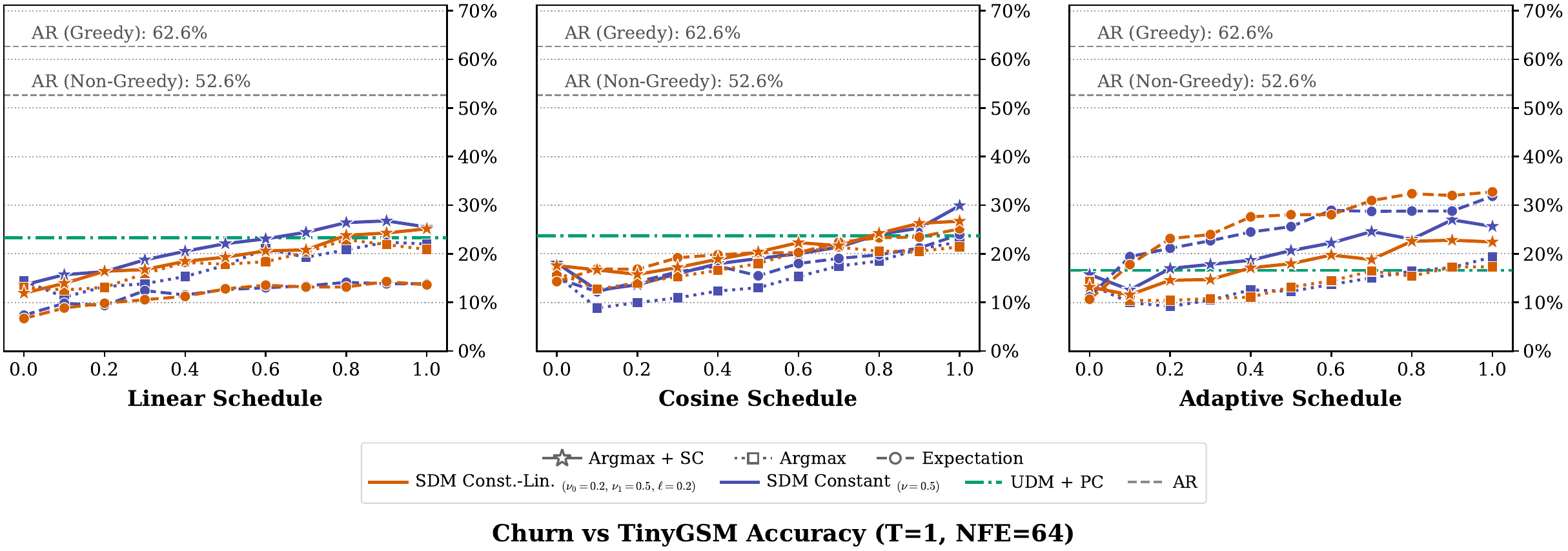}
    \caption{\textbf{Churn $\kappa$ vs.\ TinyGSM accuracy at $T = 1.0$ with 64 steps} for the \textit{Linear} (left), \textit{Cosine} (center) and \textit{Adaptive} (right) sampling schedules. We compare two concentration schedules, \texttt{Const.-Lin.} ($\nu_0=0.2, \nu_1=0.5, \ell=0.2$, orange) and \texttt{Constant} ($\nu=0.5$, indigo), and three input variants: \textit{Argmax + SC} ($\star$), \textit{Argmax} ($\blacksquare$) and \textit{Expectation} $P_t$ ($\bullet$). Horizontal lines show AR decoding (greedy: $62.6\%$; sampling: $52.6\%$) and UDM with the Predictor-Corrector (PC) sampler. With few steps and high temperature, the expectation input combined with the adaptive schedule and $\kappa=1$ performs best: \texttt{Const.-Lin.} rises from $10.6\%$ at $\kappa=0$ to $\mathbf{32.7\%}$ at $\kappa=1$ and \texttt{Constant} from $11.2\%$ to $31.8\%$, compared with $14.7\%$ and $13.8\%$ under the linear schedule and $23.7\%$ for the best UDM with PC. Exact values for $\kappa \in \{0, 0.2, 1\}$ are in \Cref{tab:tinygsm_sdm_sdm_t10_s64}.}
    \label{fig:tinygsm_churn_T1p0_s64}
\end{figure*}
\begin{figure*}
    \centering
    \includegraphics[width=\textwidth]{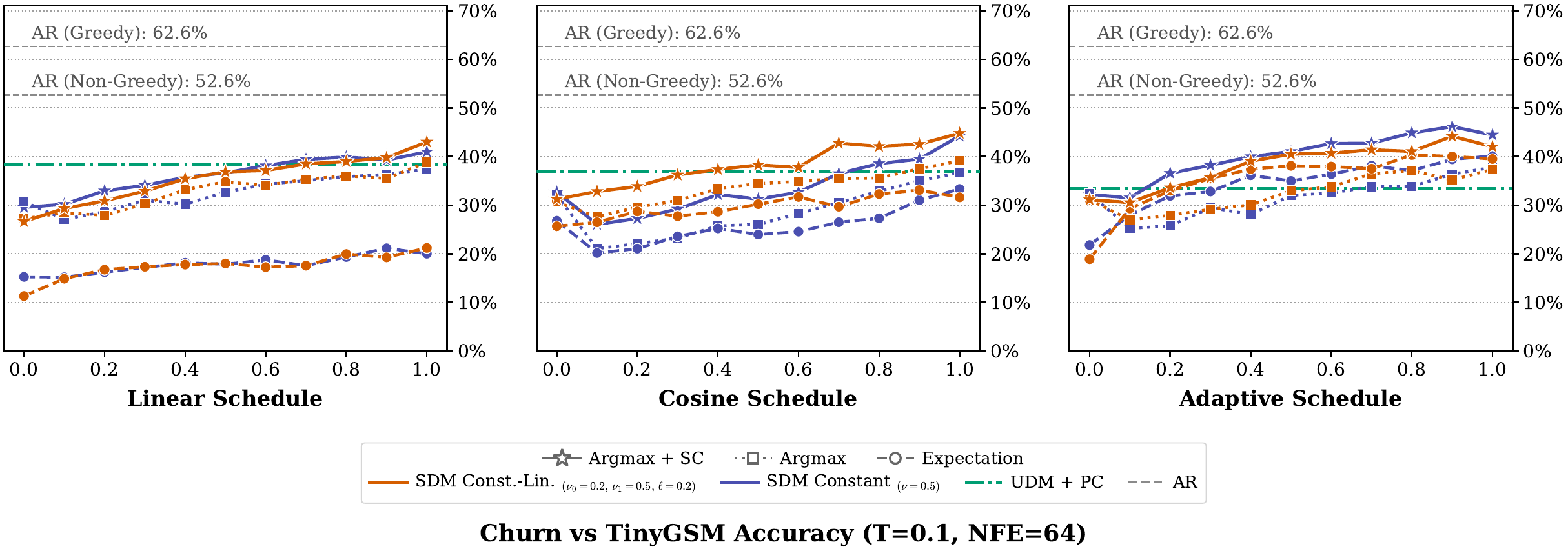}
    \caption{\textbf{Churn $\kappa$ vs.\ TinyGSM accuracy at $T = 0.1$ with 64 steps} for the \textit{Linear} (left), \textit{Cosine} (center) and \textit{Adaptive} (right) sampling schedules. We compare two concentration schedules, \texttt{Const.-Lin.} ($\nu_0=0.2, \nu_1=0.5, \ell=0.2$, orange) and \texttt{Constant} ($\nu=0.5$, indigo), and three input variants: \textit{Argmax + SC} ($\star$), \textit{Argmax} ($\blacksquare$) and \textit{Expectation} $P_t$ ($\bullet$). Horizontal lines show AR decoding (greedy: $62.6\%$; sampling: $52.6\%$) and UDM with the Predictor-Corrector (PC) sampler. At $\kappa=1$, \textit{Argmax + SC} reaches up to $\mathbf{44.5\%}$ (\texttt{Constant}, adaptive), above the best UDM with PC ($38.3\%$). With the adaptive schedule, the expectation input also improves strongly with churn, from $21.8\%$ to $40.1\%$ (\texttt{Constant}) and from $18.9\%$ to $39.5\%$ (\texttt{Const.-Lin.}). Exact values for $\kappa \in \{0, 0.2, 1\}$ are in \Cref{tab:tinygsm_sdm_sdm_t01_s64}.}
    \label{fig:tinygsm_churn_T0p1_s64}
\end{figure*}
\paragraph{Blockwise Generation.}
We compare full-sequence SDMs with a block-autoregressive variant \citep{han2023ssd, arriola2025block} that generates blocks of 32 tokens from left to right, with the linear, cosine or adaptive grid (\Cref{fig:tinygsm_block_vs_fullseq_nfe}). Both use $\kappa = 1$, with the expected-embedding and Argmax + SC inputs. Full-sequence SDMs are more accurate at every NFE budget, despite the stronger left-to-right bias of block generation.
\begin{figure*}
    \centering
    \includegraphics[width=\textwidth]{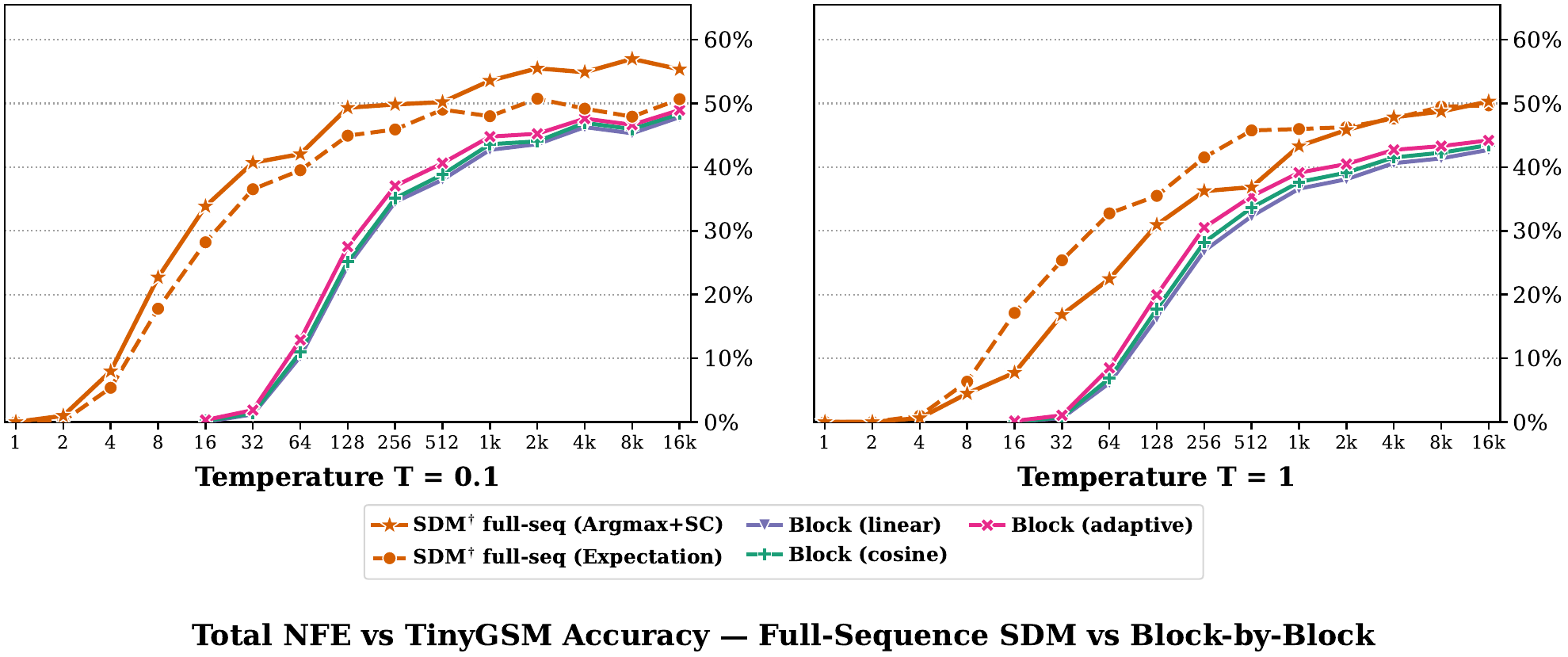}
    \caption{\textbf{Full-sequence vs.\ block-wise SDM generation on TinyGSM.} Accuracy vs.\ total NFE at $T = 0.1$ (left) and $T = 1$ (right) for block-wise generation with block size 32 under the linear, cosine and adaptive grids, compared with full-sequence SDMs (expected-embedding and Argmax + SC inputs, $\kappa = 1$). Block-wise generation is less accurate at every budget.}
    \label{fig:tinygsm_block_vs_fullseq_nfe}
\end{figure*}
\paragraph{Adaptive Grid.}
The adaptive grid (\Cref{sec:adaptive_time_schedule}) gives the largest gains for SDMs with the expected-embedding input.
These gains are largest at small step budgets: $+18.0$ points at 64 steps and $T = 1$ (\Cref{tab:tinygsm_sdm_sdm_t10_s64}).
\paragraph{Training Sampler and Grid for the Baselines.}
\begin{wraptable}{r}{0.45\textwidth}
    \vspace{-12pt}
    \caption{\textbf{Accuracy (\%)} on TinyGSM (Adaptive schedule) comparing Dirichlet Flow Matching and our SDM method across step budgets and logit temperature $T$. For SDM, we use the expectation input, $\nu_t^\text{cst.-lin.}$ ($\nu_0=0.2$, $\nu_1=0.5$, $\ell=0.2$) and $\kappa = 1$ (\Cref{tab:tinygsm_sdm_sdm_t10_s512,tab:tinygsm_sdm_sdm_t10_s64,tab:tinygsm_sdm_sdm_t01_s512,tab:tinygsm_sdm_sdm_t01_s64}).}
    \label{tab:tinygsm_dfm_apdx}
    \centering
    \small
    \setlength{\tabcolsep}{12pt}
    \renewcommand{\arraystretch}{1.05}
    {%
        \newcommand{\tabrow}{\hspace*{0.8em}}
        \begin{tabular}{l cc}
        \toprule
        Model & 64 & 512 \\
        \midrule
        \multicolumn{3}{@{}l@{}}{\textit{$T = 1.0$}} \\
        \tabrow DFM (reimpl.) & 5.4 &  6.1 \\
        \rowcolor{gray!15}
        \tabrow SDM ($\kappa = 1.0$) & 32.7  & 45.8 \\
        \midrule
        \multicolumn{3}{@{}l@{}}{\textit{$T = 0.1$}} \\
        \tabrow DFM (reimpl.) & 12.1 & 13.3 \\
        \rowcolor{gray!15}
        \tabrow SDM ($\kappa = 1.0$) & 39.5 & 49.0 \\
        \bottomrule
        \end{tabular}%
    }
\end{wraptable}
\Cref{tab:tinygsm_baselines_t10_s512,tab:tinygsm_baselines_t10_s64,tab:tinygsm_baselines_t01_s512,tab:tinygsm_baselines_t01_s64} report every baseline trained with the uniform ($^\dagger$) and the adaptive ($^\ddagger$) time sampler, and sampled with the linear, cosine and adaptive grids.
The adaptive training sampler improves the SC variants in 11 of 12 settings (e.g., UDM + SC from $31.7\%$ to $39.7\%$ at $T = 0.1$ with 512 steps) and FLMs in all four settings.
It lowers UDM + PC in all four settings (e.g., from $45.5\%$ to $40.1\%$), consistent with our Sudoku results.
For discrete baselines, the adaptive grid changes accuracy by between $-5.8$ and $+5.8$ points relative to the linear grid.
For every baseline, we sweep the training time sampler (uniform or adaptive), the loss (CE or ELBO for MDMs), the sampler (ancestral, PC, SC) and the sampling grid (linear, cosine, and adaptive when trained with the adaptive sampler), at the same temperatures and step budgets as SDMs, and compare SDMs against the best of these configurations. SDMs have a larger search space, since we additionally sweep the concentration schedule (4--5 per input) and the churn $\kappa \in \{0, 0.2, 1\}$; however, a single configuration ($\nu_0=0.2$, $\nu_1=0.5$, $\ell=0.2$, $\kappa=1$, adaptive grid) is the best Expectation SDM in three of the four settings.
\paragraph{Comparison with DFM.}
We trained our re-implemented DFM model (\Cref{app:dfm_baseline}) on TinyGSM, using the exact same network architecture and training setup as our method. We again find that DFM performs worse than our method (\Cref{tab:tinygsm_dfm_apdx}).
\paragraph{Comparison with Spherical Flow.}
\label{app:spherical_flows_tinygsm}
\citet{chemseddine2026sphericalflowssamplingcategorical} train Spherical Flow (SF) on TinyGSM with the same architecture, tokenizer, context length and training budget as ours, and evaluate it at $T = 1$.
\Cref{tab:tinygsm_spherical_flows} compares their reported numbers with SDMs at matched NFE.
Both SDMs and SFs benefit from increased stochasticity during generation. With the least stochasticity (ODE for Spherical Flow, $\kappa = 0$ for SDMs), neither exceeds $13\%$.
Without SC, SDMs outperform Spherical Flow with predictor-corrector (PC) sampling with 64 ($32.7\%$ vs.\ $26.9\%$) and 512 ($45.8\%$ vs.\ $32.4\%$) NFEs.
SDMs without SC perform similarly to Spherical Flow with PC and SC ($32.7\%$ vs.\ $35.6\%$ at 64 NFE, $45.8\%$ vs.\ $41.7\%$ at 512 NFE).

\begin{wraptable}{r}{0.45\textwidth}
    \vspace{-12pt}
    \caption{\textbf{Accuracy (\%) of SDMs and Spherical Flow on TinyGSM} at \mbox{$T = 1$}. Spherical Flow numbers come from \citet{chemseddine2026sphericalflowssamplingcategorical} (Table 20), who train with the same architecture, tokenizer, context length and training budget as ours. Each PC entry uses their best sampler configuration for that budget. For SDMs, we always use $\nu_t^\text{cst.-lin.}$ ($\nu_0=0.2$, $\nu_1=0.5$, $\ell=0.2$) and the adaptive grid, as in \Cref{fig:tinygsm-nfe-vs-accuracy}.}
    \label{tab:tinygsm_spherical_flows}
    \centering
    \small
    \setlength{\tabcolsep}{5pt}
    \renewcommand{\arraystretch}{1.05}
    {%
        \newcommand{\tabrow}{\hspace*{0.8em}}
        \begin{tabular}{l cc}
        \toprule
        Model & 64 NFE & 512 NFE \\
        \midrule
        \multicolumn{3}{@{}l@{}}{\textit{Spherical Flow}} \\
        \tabrow ODE & 6.1 & 6.4 \\
        \tabrow PC & 26.9 & 32.4 \\
        \tabrow PC + SC & 35.6 & 41.7 \\
        \midrule
        \multicolumn{3}{@{}l@{}}{\textit{SDMs}} \\
        \tabrow Expectation ($\kappa = 0$) & 10.6 & 12.6 \\
        \rowcolor{gray!15}
        \tabrow Expectation ($\kappa = 1$) & 32.7 & 45.8 \\
        \tabrow Argmax + SC ($\kappa = 1$) & 22.4 & 36.9 \\
        \bottomrule
        \end{tabular}%
    }
    \vspace{-25pt}
\end{wraptable}

\subsection{Distillation on TinyGSM}
\label{app:distillation_exp}
\paragraph{Teacher.} We distill the SDM trained with the cross-entropy loss \eqref{eq:xentropy}, the expected-embedding input (no SC), $\alpha_t = 1-t$, a uniform prior $\pi$, the schedule $\nu_t^\text{cst.-lin.}$ with $\nu_0=0.2$, $\nu_1=0.5$, $\ell=0.2$ (\Cref{sec:concentration_schedule}), and the adaptive time sampler (\Cref{sec:adaptive_time_sampler}). This is the SDM (Expectation) of \Cref{fig:tinygsm-nfe-vs-accuracy}. We sample distilled models with $\kappa = 1$, the adaptive grid and $T = 1$.
\paragraph{Hyperparameters.} We sweep 36 configurations (\Cref{tab:tinygsm_distilled_ast_diversity_t10}) and the sampling churn $\kappa \in \{0, 0.2, 1\}$. We do not hold out a validation split: configurations are evaluated, and the reported one is selected, on the GSM8K test set, so the $32.1\%$ in the main text is a best-of-sweep number. The conclusion does not depend on this selection: with $\kappa = 1$, all 36 configurations reach between $27.4\%$ and $32.1\%$ pass@1 with 8 steps (median $28.9\%$), above IDLM with 128 steps ($21.4\%$). The generator loss weight $\lambda_\text{gen}$ has little to no effect.
\paragraph{Comparison with IDLM.} We take the IDLM results from \citet{li2026idlm} without re-training. They use the same tokenizer, metric (pass@1) and temperature ($T=1$) as us (\Cref{tab:tinygsm_distill_idlm}). The remaining differences are the distillation objective, the diffusion process (masked vs.\ simplex) and the teacher.
\paragraph{Accuracy vs.\ Steps and Diversity.} \Cref{fig:tinygsm_distilled_vs_sdm_nfe} compares distilled and undistilled SDMs across sampling budgets. \Cref{tab:tinygsm_baselines_ast_diversity} reports the multi-sample accuracy and AST diversity of the baselines. Distillation trades diversity for speed: with 8 NFEs at $T=1$, the distilled SDM reaches an AST diversity of $17.5$, vs.\ $35.3$ for the undistilled SDM with 512 NFEs and $36.7$ for AR.
\begin{table}[t]
    \caption{\textbf{Accuracy (\%)} of distilled models on TinyGSM across sampling steps ($T=1$). SDMs use the concentration schedule defined in terms of the total variance with $\nu_t^\text{cst.-lin.}$ (\mbox{$\nu_0=0.2$, $\nu_1=0.5$, $\ell=0.2$}).}
    \label{tab:tinygsm_distill_idlm}
    \centering
    \small
    \setlength{\tabcolsep}{5pt}
    \begin{tabular}{l cccc}
    \toprule
    Model & 8 & 32 & 64 & 128 \\
    \midrule
    IDLM-MDLM \citep{li2026idlm} & -- & 12.8 & 14.9 & 19.9 \\
    IDLM-Duo \citep{li2026idlm} & -- & 15.4 & 19.0 & 21.4 \\
    SDM (undistilled) & 6.3 & 25.4 & 32.7 & 35.5 \\
    \rowcolor{gray!15}
    SDM (distilled) & \textbf{32.1} & \textbf{37.6} & \textbf{37.8} & \textbf{39.4} \\
    \bottomrule
    \end{tabular}
\end{table}
\begin{table*}[t]
    \centering
    \scriptsize
    \setlength{\tabcolsep}{6.2pt}
    \renewcommand{\arraystretch}{1.12}
    \caption{\textbf{Distilled SDMs on TinyGSM} (8 steps, $\kappa = 1$, adaptive grid, $T = 1$, $K=5$ samples per problem). We sweep the learning rate $\in \{3\times10^{-6}, 10^{-5}, 3\times10^{-5}\}$, the Adam parameters $(\beta_1, \beta_2) \in \{0, 0.9\} \times \{0.95, 0.999\}$ and the generator loss weight $\lambda_{\text{gen}} \in \{1, 2, 5\}$ (36 configurations), and report the 25 configurations with highest pass@1.}
    \label{tab:tinygsm_distilled_ast_diversity_t10}
    \begin{tabular}{c c c c ccccc}
        \toprule

        $\text{lr}$ & $\beta_1$ & $\beta_2$ & $\lambda_{\text{gen}}$ & pass@1 (\%) $\uparrow$ & pass@2 (\%) $\uparrow$ & pass@5 (\%) $\uparrow$ & AST Div. $\uparrow$ & AST Div. (Correct) $\uparrow$ \\
        \midrule
        \texttt{1e-05} & 0.0 & 0.95 & 1.0 & \textbf{31.9} & \textbf{40.0} & 48.7 & 17.5 & 7.5 \\
        \texttt{1e-05} & 0.0 & 0.95 & 2.0 & \textbf{31.9} & \textbf{40.0} & 48.7 & 17.5 & 7.5 \\
        \texttt{1e-05} & 0.0 & 0.95 & 5.0 & \textbf{31.9} & \textbf{40.0} & 48.7 & 17.5 & 7.5 \\
        \texttt{1e-05} & 0.9 & 0.999 & 1.0 & 31.8 & \textbf{40.0} & 49.6 & 17.4 & 7.6 \\
        \texttt{1e-05} & 0.9 & 0.999 & 2.0 & 31.8 & \textbf{40.0} & 49.6 & 17.4 & 7.6 \\
        \texttt{1e-05} & 0.9 & 0.999 & 5.0 & 31.7 & \textbf{40.0} & 49.6 & 17.4 & 7.6 \\
        \texttt{3e-06} & 0.0 & 0.999 & 2.0 & 30.1 & 39.6 & \textbf{49.9} & 19.3 & \textbf{8.7} \\
        \texttt{3e-06} & 0.0 & 0.999 & 1.0 & 29.9 & 39.3 & 49.8 & 19.3 & \textbf{8.7} \\
        \texttt{3e-06} & 0.0 & 0.999 & 5.0 & 29.9 & 39.3 & 49.8 & 19.3 & \textbf{8.7} \\
        \texttt{1e-05} & 0.0 & 0.999 & 1.0 & 29.5 & 38.1 & 48.5 & 18.2 & 8.3 \\
        \texttt{1e-05} & 0.0 & 0.999 & 2.0 & 29.5 & 38.1 & 48.5 & 18.2 & 8.3 \\
        \texttt{1e-05} & 0.0 & 0.999 & 5.0 & 29.5 & 38.1 & 48.5 & 18.2 & 8.3 \\
        \texttt{3e-05} & 0.9 & 0.95 & 1.0 & 29.2 & 37.1 & 46.6 & \textbf{20.5} & 8.1 \\
        \texttt{3e-05} & 0.9 & 0.95 & 2.0 & 29.2 & 37.1 & 46.6 & \textbf{20.5} & 8.1 \\
        \texttt{3e-05} & 0.9 & 0.95 & 5.0 & 29.2 & 37.1 & 46.6 & \textbf{20.5} & 8.1 \\
        \texttt{3e-06} & 0.9 & 0.999 & 1.0 & 28.9 & 38.3 & 48.7 & 17.5 & 8.4 \\
        \texttt{3e-06} & 0.9 & 0.999 & 2.0 & 28.9 & 38.2 & 48.7 & 17.6 & 8.3 \\
        \texttt{3e-06} & 0.9 & 0.999 & 5.0 & 28.9 & 38.3 & 48.7 & 17.5 & 8.4 \\
        \texttt{3e-05} & 0.0 & 0.95 & 1.0 & 28.8 & 37.0 & 44.9 & 19.6 & 8.4 \\
        \texttt{3e-05} & 0.0 & 0.95 & 2.0 & 28.8 & 36.9 & 44.9 & 19.6 & 8.2 \\
        \texttt{3e-05} & 0.0 & 0.95 & 5.0 & 28.8 & 36.9 & 44.9 & 19.6 & 8.2 \\
        \texttt{3e-05} & 0.0 & 0.999 & 1.0 & 28.6 & 35.8 & 45.3 & 19.6 & 8.4 \\
        \texttt{3e-05} & 0.0 & 0.999 & 2.0 & 28.6 & 35.8 & 45.3 & 19.6 & 8.5 \\
        \texttt{3e-05} & 0.0 & 0.999 & 5.0 & 28.6 & 35.8 & 45.3 & 19.6 & 8.4 \\
        \texttt{3e-05} & 0.9 & 0.999 & 1.0 & 28.6 & 36.6 & 44.7 & 18.4 & 8.0 \\
        \bottomrule
    \end{tabular}
\end{table*}
\begin{figure*}
    \centering
    \includegraphics[width=\textwidth]{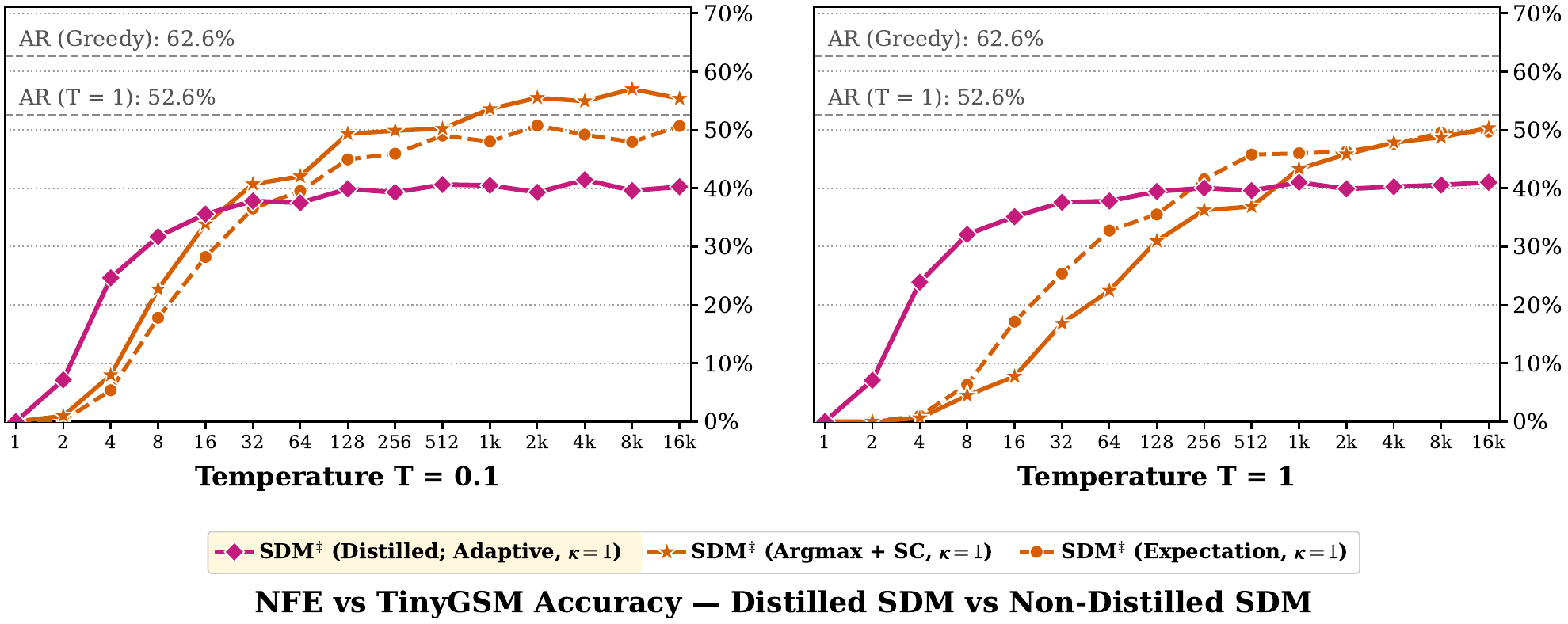}
    \caption{\textbf{Distilled vs.\ undistilled SDMs on TinyGSM} across sampling steps (NFE $\in [1, 2^{14}]$, $\kappa = 1$, adaptive grid) at $T = 0.1$ (left) and $T = 1$ (right). At $T = 1$, the distilled SDM solves $32.1\%$ of the problems with 8 NFEs, against $6.3\%$ for the undistilled SDM with the expected embedding, and $39.4\%$ with 128 NFEs. Beyond 128 NFEs it plateaus at $40$--$41\%$, whereas the undistilled SDM keeps improving ($45.8\%$ at 512 and $49.7\%$ at 16k NFEs).}
    \label{fig:tinygsm_distilled_vs_sdm_nfe}
\end{figure*}

\subsection{Auto-Guidance on TinyGSM}
\label{app:autoguidance}

While Classifier-Free Guidance (CFG; \citealp{ho2022classifier}) is a standard mechanism for improving sample fidelity in diffusion models, it requires training with conditioning dropout and cannot be applied directly to unconditional generation or pre-trained models without dropout tokens. To sharpen generation quality at inference time without modifying the training objective, we adapt \emph{Auto-Guidance} \citep{karras2024guiding} to Simplex Diffusion Models (SDMs). Specifically, we guide the fully converged model $\theta_{\text{final}}$ away from an earlier, unconverged checkpoint $\theta_K$ saved during training. More precisely, we modify the logits and probabilities of the Simplex Diffusion Model  as follows
\begin{equation}
    \tilde{\zeta}(x_t, t)
    = \zeta_{\theta_{\text{final}}}(x_t, t)
      + w \bigl( \zeta_{\theta_{\text{final}}}(x_t, t) - \zeta_{\theta_K}(x_t, t) \bigr),
    \qquad
    \hat{p}_0(x_t, t) = \operatorname{softmax}\!\left( \frac{\tilde{\zeta}(x_t, t)}{T} \right),
    \label{eq:simplicial_autoguidance}
\end{equation}
where $w \ge 0$ controls the guidance strength ($w = 0$ recovers the un-guided baseline).

We apply Auto-Guidance to the TinyGSM SDM with the expected embedding (no SC; $\nu_0=0.2$, $\nu_1=0.5$, $\ell=0.2$), sampled with 64 steps. The guided model is the EMA checkpoint at $250$k steps, and the guiding model uses the raw (non-EMA) parameters at step $K \in \{0, 10\text{k}, 20\text{k}, 30\text{k}, 50\text{k}, 100\text{k}, 150\text{k}, 200\text{k}, 249.5\text{k}\}$. We sweep $w \in \{0.05, 0.1, 0.15, 0.2, 0.25, 0.35, 0.5, 0.75, 1.0, 1.5\}$ and logit temperatures $T \in \{0.01, 0.1, 1.0\}$ for each churn $\kappa \in \{0.0, 0.2, 1.0\}$, with 3 evaluation seeds. \Cref{fig:tinygsm_autoguidance_by_churn} shows TinyGSM pass@1 accuracy for $w \in \{0.1, 0.25, 0.5, 1.0\}$. We observe that guiding against an early-to-intermediate checkpoint ($K \in [20\text{k}, 50\text{k}]$)  maximizes the accuracy for all churn levels $\kappa$, while for checkpoints close to convergence, the accuracy gets closer to the unguided baseline.

\Cref{fig:tinygsm_autoguidance_summary} summarizes the best auto-guided accuracy ($w^\star$) alongside the accuracy gain ($\Delta\%$) over the un-guided baseline ($w=0$) across logit temperatures $T \in \{0.01, 0.1, 1.0\}$. We observe improvements over the baseline in all frameworks we investigate.
In particular, the Auto-Guidance has the strongest effect with low churn and high temperature. 

\begin{figure}[!ht]
    \centering
    \includegraphics[width=\textwidth]{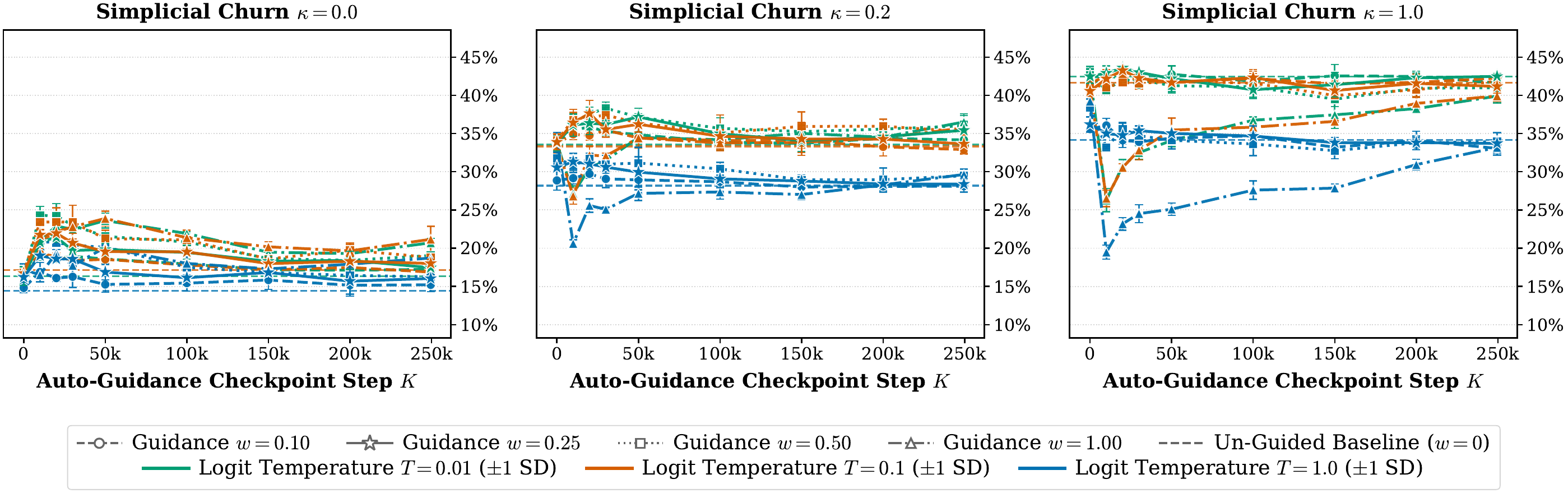}
    \vspace{-0.4em}
    \caption{\textbf{TinyGSM accuracy (\%) vs.\ Auto-Guidance checkpoint step $K$ for different churn regimes ($\kappa \in \{0.0, 0.2, 1.0\}$), with 64 sampling steps.} We compare logit temperatures $T=0.01$ (green), $T = 0.1$ (orange) and $T = 1.0$ (blue) for different guidance weights $w \in \{0.1, 0.25, 0.5, 1.0\}$ against the corresponding un-guided baselines ($w=0$, dashed horizontal lines). Error bars denote a variation of one standard deviation across 3 evaluation seeds. Early-to-intermediate checkpoints ($K \in [20\text{k}, 50\text{k}]$) consistently achieve peak accuracy.}
    \label{fig:tinygsm_autoguidance_by_churn}
\end{figure}

\begin{figure}[!ht]
    \centering
    \includegraphics[width=0.95\textwidth]{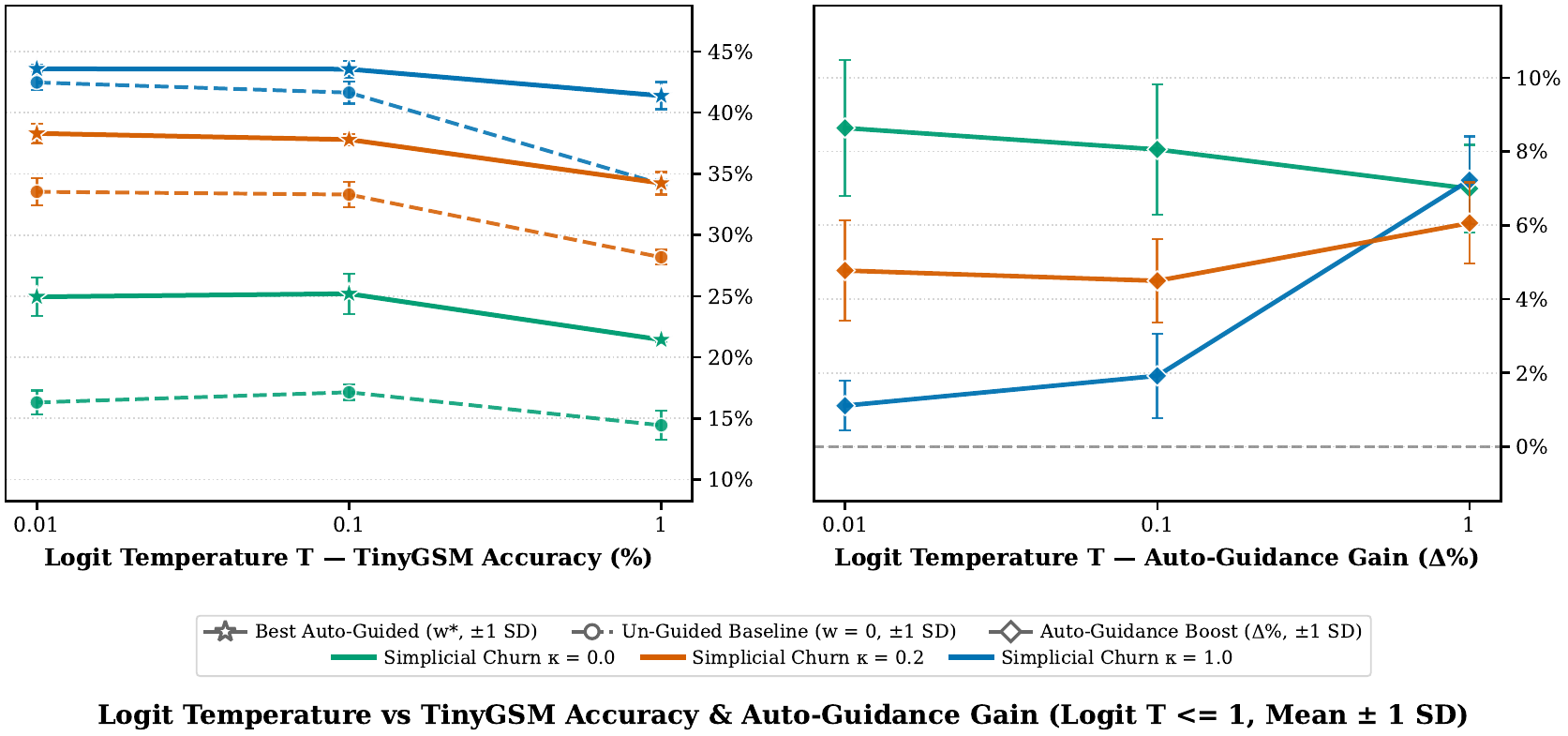}
    \vspace{-0.4em}
    \caption{\textbf{Logit temperature sensitivity and Auto-Guidance gain on TinyGSM ($T \in \{0.01, 0.1, 1.0\}$).} \textbf{Left:} Validation accuracy (\%) of the best auto-guided model ($w^\star$, solid lines with stars) vs.\ the un-guided baseline ($w=0$, dashed lines with circles) for $\kappa \in \{0.0, 0.2, 1.0\}$. \textbf{Right:} Accuracy gain ($\Delta\%$) from Auto-Guidance. Error bars denote a variation of one standard deviation across 3 evaluation seeds.
    }
    \label{fig:tinygsm_autoguidance_summary}
\end{figure}

\subsection{OpenWebText}
\label{app:additional-exp-owt}

\paragraph{Inference Interventions.} In the following paragraph, we show how different interventions can push the Pareto frontier in terms of Generative Perplexity and unigram entropy, see \citet{pynadath2026generative} for a discussion on those evaluations. We explore those interventions for the four different classes of models we investigate. Namely, we propose interventions for Autoregressive (AR) models, Masked Diffusion Models (MDMs), Uniform Diffusion Models (UDMs) and Simplex Diffusion Models (SDMs). Most of those techniques can be interpreted as being part of a logit shaping pipeline. The model was trained according to \Cref{sec:appendix-setup-owt}.

First, we consider some nucleus sampling techniques. For completeness, we recall the process of nucleus sampling adapted from \citet{holtzman2019curious}. In the case of one token we denote $\{\zeta_1, \dots, \zeta_{N}\}$, the proposed logits, i.e., we have that $\log \hat{P}_{0|t}(\cdot|P_t) = \{\zeta_1, \dots, \zeta_{N}\}$ in the case of SDMs for instance. Next, we denote $\{\zeta_{\varphi(1)}, \dots, \zeta_{\varphi(N)}\}$ the sorted set of logits (in descending order, i.e., $\zeta_{\varphi(j)} \geq \zeta_{\varphi(j+1)}$ for any $j \in \{1, \dots, N-1\}$). Next, we denote $k \in \{1, \dots, N\}$ the smallest integer such that $\sum_{j=1}^k p_{\varphi(j)} \geq p$, where $p$ is a hyperparameter and $\{p_1, \dots, p_{N}\} = \mathrm{softmax}(\{\zeta_1, \dots, \zeta_{N}\})$. We denote $\phi$ the inverse of $\varphi$, i.e. for any $i \in \{1, \dots, N\}$, $\varphi(\phi(i)) = i$. For any $j \in \{1, \dots, N\}$, we denote $\hat{\zeta}_j = \zeta_j$ if $\phi(j) \leq k$ and $\hat{\zeta}_j = -\infty$ otherwise.  Next, we consider $\topp$ such that 
\begin{equation}
    \topp(p, \{\zeta_1, \dots, \zeta_{N}\}) = \{\hat{\zeta}_1, \dots, \hat{\zeta}_{N}\} . \label{eq:topp}
\end{equation}
Another intervention we consider is temperature scaling of the logits. More precisely, we consider
\begin{equation}
\temp(T, \{\zeta_1, \dots, \zeta_{N}\}) =  \{\zeta_1/T, \dots, \zeta_{N}/T\} . \label{eq:temp}
\end{equation}
For the first intervention we consider $T$ in \eqref{eq:temp} to be in the following set 
\begin{equation}
\{0.20,0.30,0.40,0.50,0.60,0.65,0.70,0.75,0.80,0.85,0.90,0.95,1.00,1.05,1.10,1.18,1.28\} . \label{eq:sweep_t}
\end{equation}
and $p \in \{0.92, 0.96\}$ in \eqref{eq:topp}.  This yields $34$ sweeps for each family UDM, MDM, SDM and AR after this intervention.

For the second intervention, which only applies to diffusion methods, we consider temperature annealing. Let $t \in [0,1]$ be the diffusion time of the forward process, we consider a similar sweep as before but instead consider a temperature annealing procedure where
\begin{equation}
    T = T_{\startt} + (T_{\finish} - T_{\startt}) t^{1.5} ,
\end{equation}
with $T_{\finish}$ given by \eqref{eq:sweep_t} and $T_{\startt} = 0.8 T_{\finish}$. We do not claim that the power relationship with coefficient $1.5$ and $T_{\startt} = 0.8 T_{\finish}$ are optimal but we found those values to give good results in early experiments. Note that a linear temperature annealing was already considered in \citet{team2026diffusiongemma} with $T_{\startt} = 0.4$ and $T_{\finish}=0.8$. Similar power law temperature annealing schedule were also identified in \citet{chang2022maskgitmaskedgenerativeimage}. Since AR has no corruption time, we instead anneal its temperature over token positions.  We always consider $\temp$ and then follow it by $\topp$.

For the third intervention, which is by far the most influential one, we consider a \emph{frequency penalty}. The frequency penalty is a \emph{sequence} based penalty. In what follows, in the case of UDM and MDM, $x_t \in \{1, \dots, N\}^L$ where $L$ is the sequence length and $N$ is the vocabulary size (in the case of MDM we assume that one of the token is the masked one denoted $\masktoken$).  We denote the count variable $\{c_{t,v}\}_{v=1}^{N} \in \nset^N$ which is defined 
for any $v \in \{1, \dots, N\}$ with $v \neq \masktoken$
\begin{equation}
    c_{t,v} = \sum_{j=1}^L \updelta_{v}(x_{t,j}) . \label{eq:count_mdm_udm}
\end{equation}
In addition, in the case $v=\masktoken$, $c_{t,v}=0$. In other words, $c_{t,v}$ is the count of tokens which have values $v$ in the sequence $x_t$, except for potentially the mask token. In the case of the simplex diffusion model, we slightly modify the definition of the count in \eqref{eq:count_mdm_udm} and we define 
for any $v \in \{1, \dots, N\}$ with $v \neq \masktoken$
\begin{equation}
    c_{t,v} = \sum_{j=1}^L \updelta_{v}(\argmax \ P_{t,j}) . \label{eq:count_simplex}
\end{equation}
Let $\{c_{t,v}\}_{v=1}^{N} \in \nset^N$ be defined either with \eqref{eq:count_mdm_udm} or \eqref{eq:count_simplex}. We introduce the frequency penalty $\freq$ given by
\begin{equation}
    \freq(\lambda, \{\zeta_1, \dots, \zeta_{N}\}) = \{\zeta_1 - \lambda \log(1 + c_{t,1}), \dots, \zeta_{N} - \lambda \log(1 + c_{t,N})\} . 
\end{equation}
Repetition penalties were also considered in \citet{xu2206learning}. 
We always consider $\temp$ and then follow it by $\freq$ and then by $\topp$. In the case of AR one can define a similar intervention on the tokens already generated. We consider a regularization penalty $\lambda \in \{1.5, 3.0, 4.5, 6.0, 8.0\}$ for all the sweeps. 

For the fourth and final intervention, we consider another token penalty but at the \emph{local} level, contrary to the \emph{global} level of the frequency penalty of the third intervention. In particular, we consider $\{\hat{c}_{t,j,v}\} \in \rset^{L \times N}$ given for any $j \in \{1, \dots, L\}$ and $v \in \{1, \dots, N\}$ with $v \neq \masktoken$ by
\begin{equation}
    \hat{c}_{t,j,v} = \log\left(1 + \frac{\alpha_t N}{1-\alpha_t}\right) \updelta_{v}(x_{t,j}) . \label{eq:count_local_mdm_udm}
\end{equation}
We set $\hat{c}_{t,j,v} = 0$ for $v = \masktoken$. Finally, similarly to \eqref{eq:count_simplex}, we can define the simplicial counterpart of \eqref{eq:count_local_mdm_udm} as 
\begin{equation}
    \hat{c}_{t,j,v} = \log\left(1 + \frac{\alpha_t N}{1-\alpha_t}\right) \updelta_{v}(\argmax \ P_{t,j}) . \label{eq:count_local_simplex}
\end{equation}
We are now ready to define the \emph{local frequency} intervention $\freqloc$ given by  
\begin{equation}
    \freqloc(\gamma, \{\zeta_1, \dots, \zeta_{N}\}) = \{\zeta_1 - \gamma \hat{c}_{t,j,1}, \dots, \zeta_{N} - \gamma \hat{c}_{t,j,N} \} , 
\end{equation}
where here we have assumed that $\{\zeta_1, \dots, \zeta_{N}\}$  are the logits at position $j \in \{1, \dots, L\}$. We consider $\gamma \in \{2.0, 2.5\}$ for SDM only. We do not consider the intervention for AR as it already yields models which are scoring better than real data on the Generative Perplexity and Entropy Pareto frontier. In the case of UDM the intervention did not improve the Pareto frontier and is therefore not reported. In the case of MDM the local frequency constraint does not change the prediction since we do not consider remasking. Therefore, we only report SDMs results.

In \Cref{fig:owt_pareto_head_to_head}, we illustrate the obtained Pareto frontiers after all those interventions. Note that the AR Pareto frontier scores \emph{better} than the real data on OpenWebText, thereby putting into question the validity of generative frontiers on OWT as a meaningful benchmark for language generation. In \Cref{fig:owt_pareto_interventions_h1} and \Cref{fig:owt_pareto_interventions_h2}, we show the effect of each intervention on the Pareto frontier for all four models that we are investigating. The best hyperparameters for each model at matched entropy are reported in \Cref{tab:owt_pareto_h1_vertical_cuts} and at matched Generative Perplexity are reported in \Cref{tab:owt_pareto_h1_horizontal_cuts}.

For each evaluation point, we generate 128 unconditional sequences of length 1024 (evaluated in batches of 2 over 64 batches) using 64 sampling steps for all diffusion models and 1024 steps (token-by-token) for the autoregressive baseline.

\begin{figure*}[t]
  \centering
  \includegraphics[width=0.49\textwidth]{\figdir/owt_pareto_head_to_head_h1.pdf}\hfill
  \includegraphics[width=0.49\textwidth]{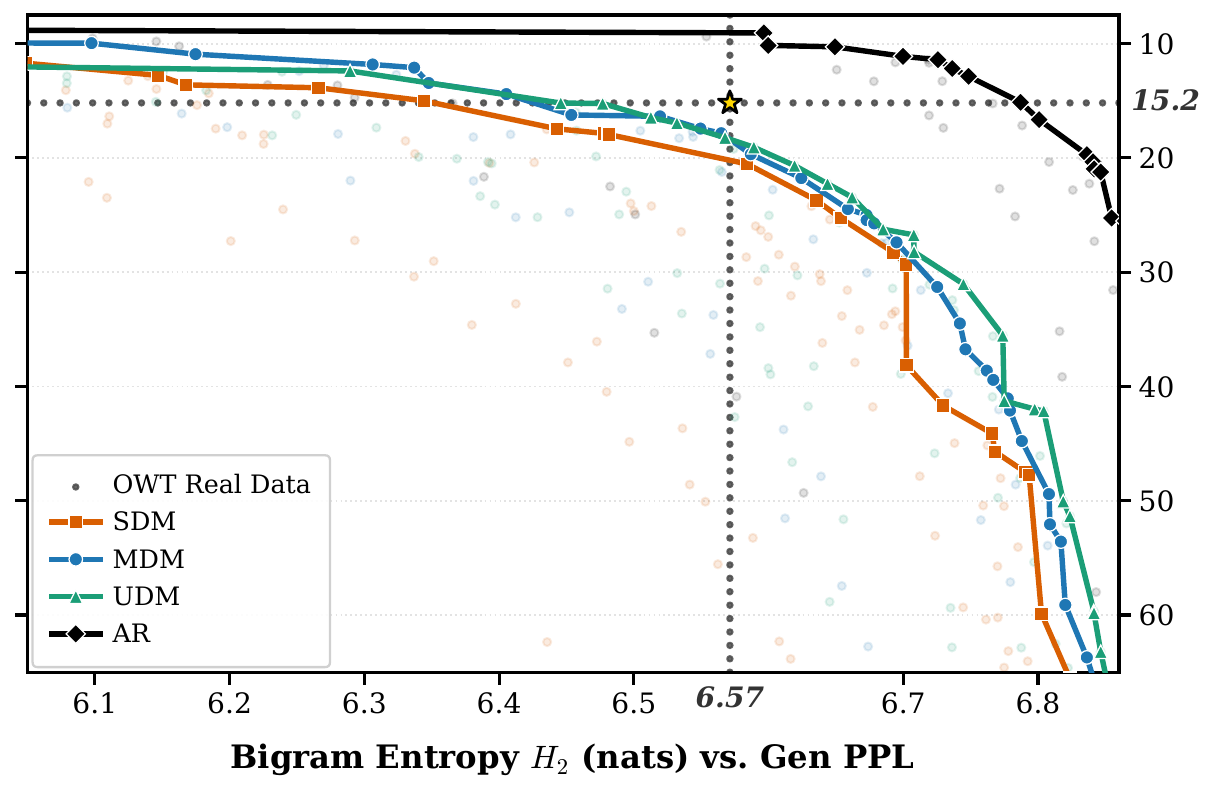}
  \vspace{-4pt}
  \caption{%
    \textbf{OpenWebText ($L=1024$) Pareto frontiers across generative families.}
    Best cumulative Pareto frontiers of Generative Perplexity (scored by GPT-2 Large; \emph{top is lower/better}) versus sequence token diversity (\emph{right is higher/better}) for Simplex Diffusion Models (\textbf{SDM}, orange), Masked Diffusion (\textbf{MDM}, blue), Uniform Discrete Diffusion (\textbf{UDM}, green), and the Autoregressive baseline (\textbf{AR}, black).
    \textbf{Left:} Unigram token entropy $H_1$ (nats).
    \textbf{Right:} Bigram token entropy $H_2$ (nats).
    Dotted reference lines and the gold star ($\star$) denote the measured OpenWebText validation distribution ($H_1 = 5.46\text{ nats}$, $H_2 = 6.57\text{ nats}$, $\text{Gen PPL} = 15.19$).%
  }
  \label{fig:owt_pareto_head_to_head}
\end{figure*}
\begin{figure*}[t]
  \centering
  \includegraphics[width=0.49\textwidth]{\figdir/owt_pareto_interventions_sdm_h1.pdf}\hfill
  \includegraphics[width=0.49\textwidth]{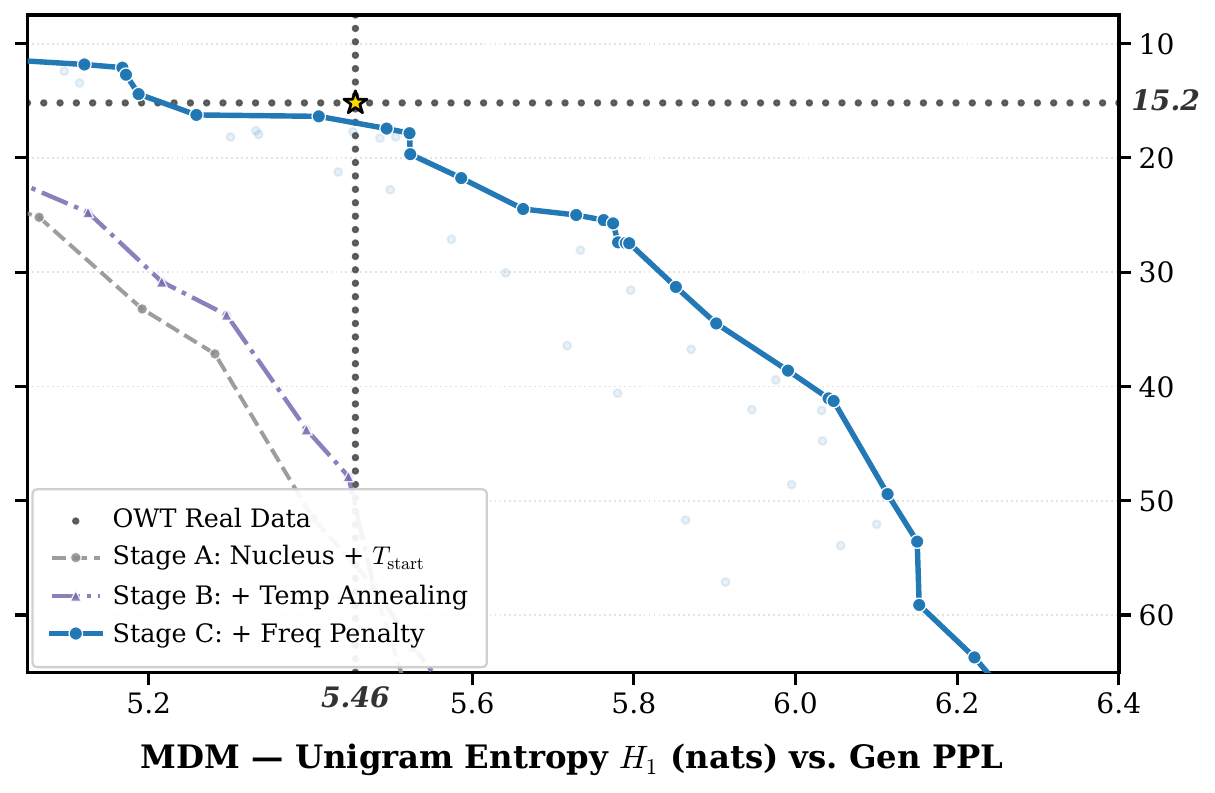}\\[6pt]
  \includegraphics[width=0.49\textwidth]{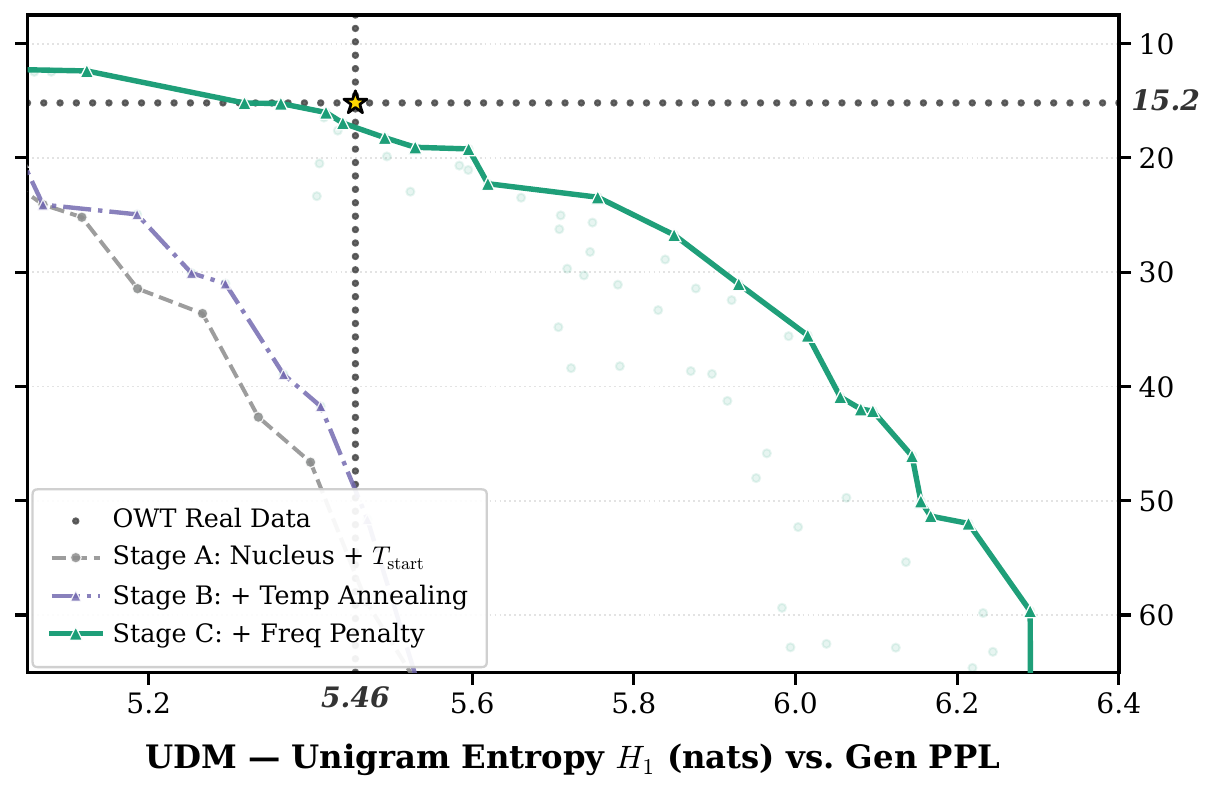}\hfill
  \includegraphics[width=0.49\textwidth]{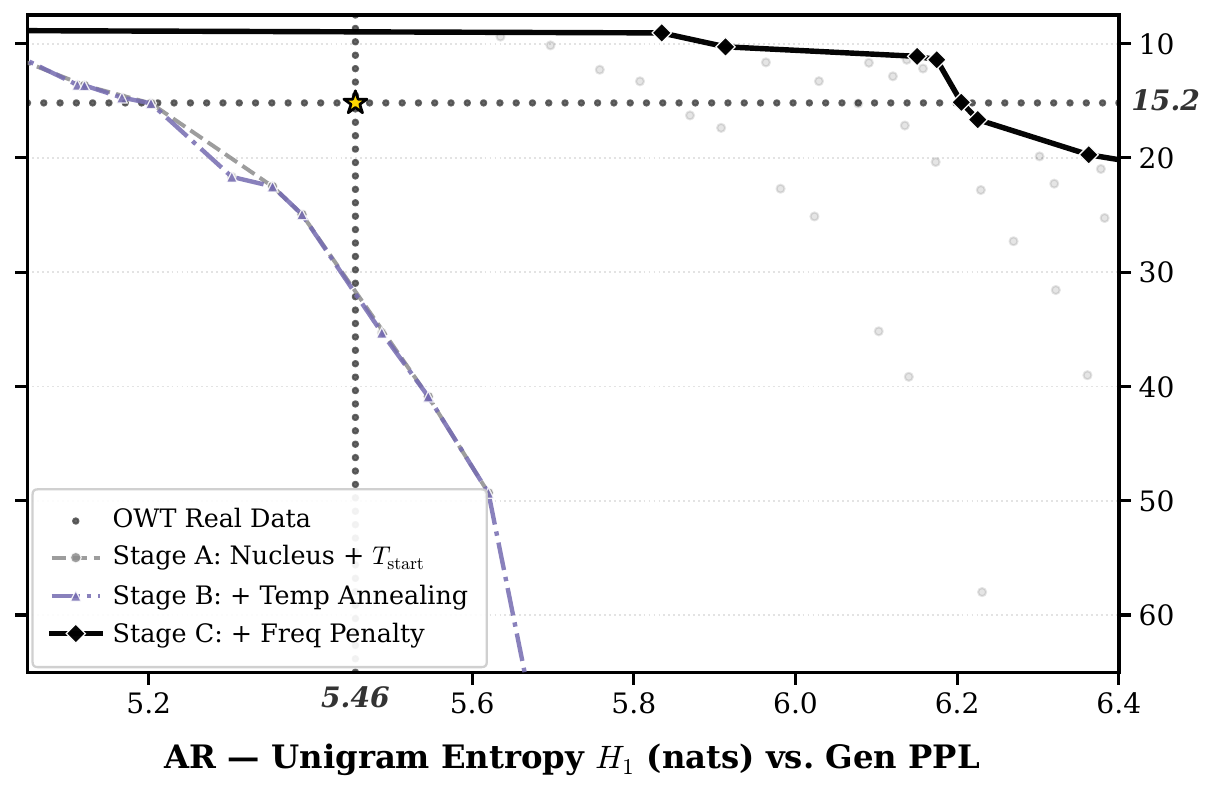}
  \vspace{-4pt}
  \caption{%
    \textbf{Progressive intervention ablations on Unigram Entropy $H_1$ (nats) vs.\ Generative Perplexity.}
    Cumulative Pareto frontiers per family across the intervention stack:
    \textbf{Stage A} (nucleus truncation $p \in \{0.92, 0.96\}$ + initial temperature $T_{\text{start}}$),
    \textbf{Stage B} ($+$ power-law temperature annealing $T_{\text{start}} = 0.80\,T_{\text{end}}$, exponent $1.5$),
    \textbf{Stage C} ($+$ sequence-level frequency penalty $\lambda \in \{1.5, 3.0, 4.5, 6.0, 8.0\}$), and
    \textbf{Stage D} ($+$ local frequency penalty $\gamma$, \textbf{SDM} only).%
  }
  \label{fig:owt_pareto_interventions_h1}
\end{figure*}
\begin{figure*}[t]
  \centering
  \includegraphics[width=0.49\textwidth]{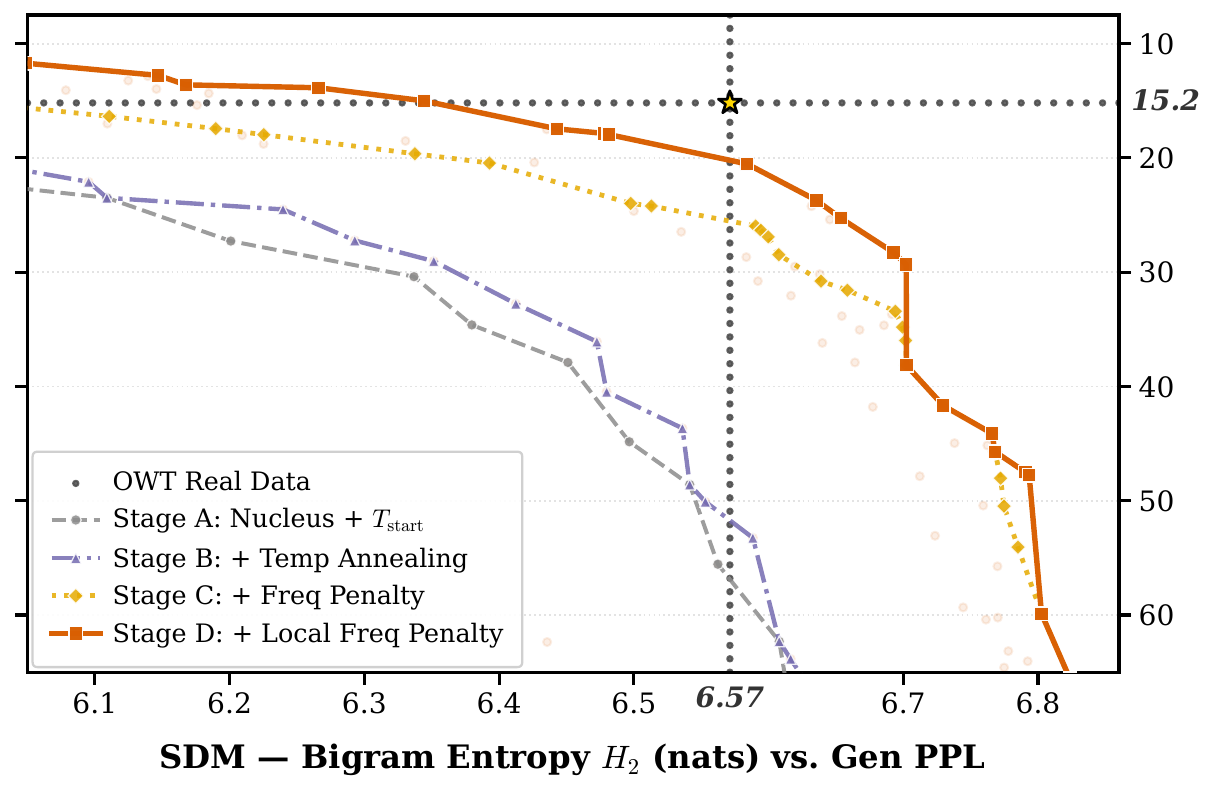}\hfill
  \includegraphics[width=0.49\textwidth]{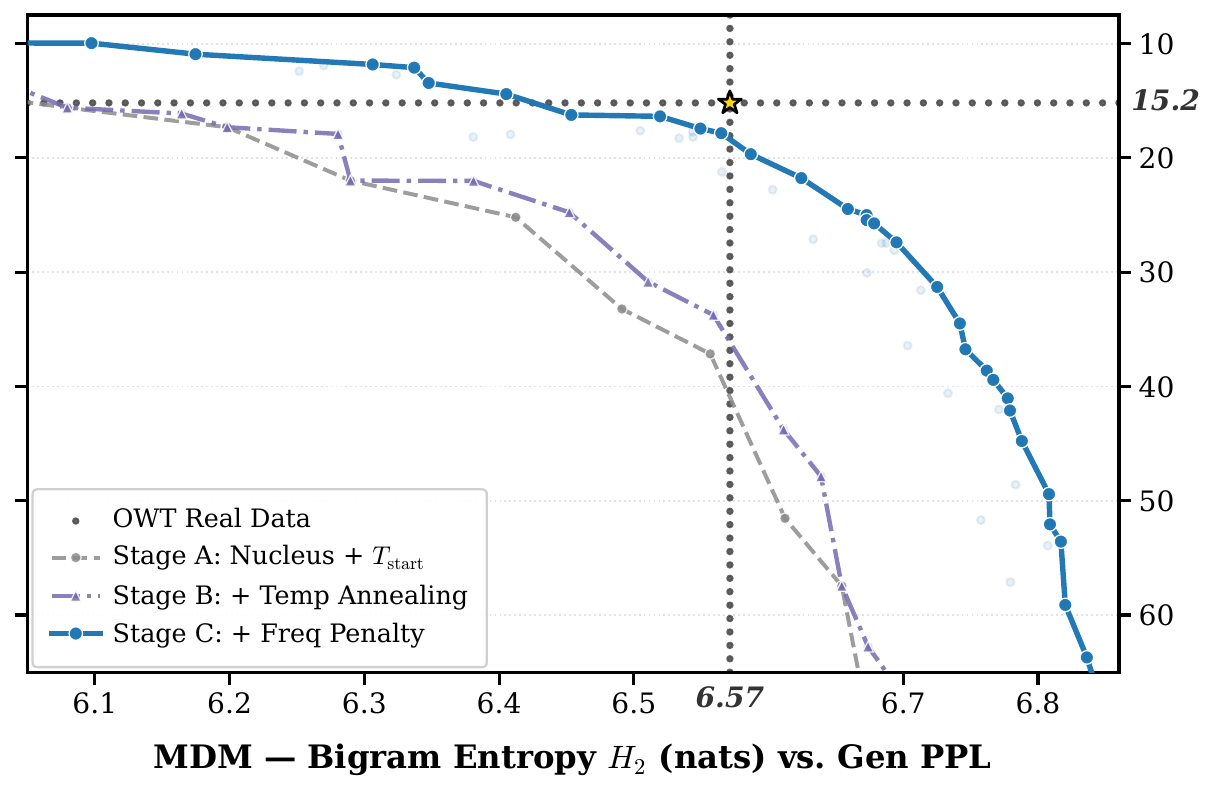}\\[6pt]
  \includegraphics[width=0.49\textwidth]{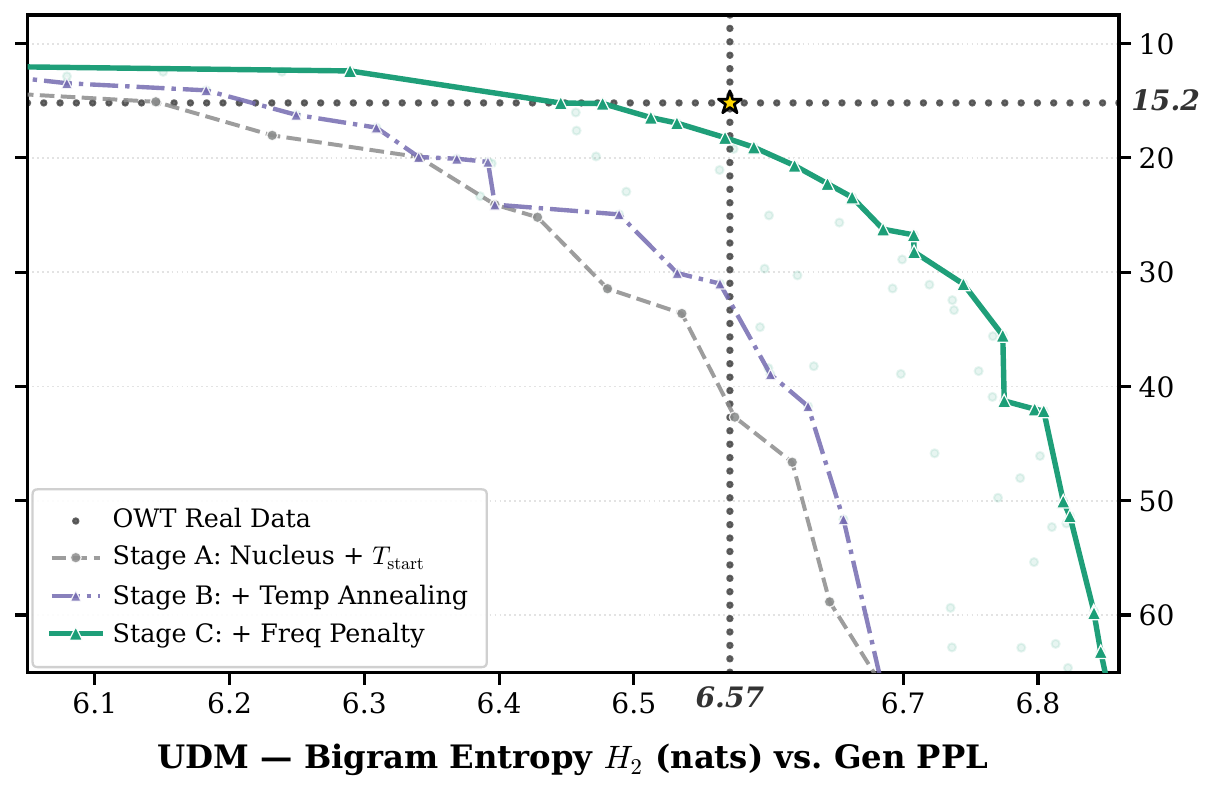}\hfill
  \includegraphics[width=0.49\textwidth]{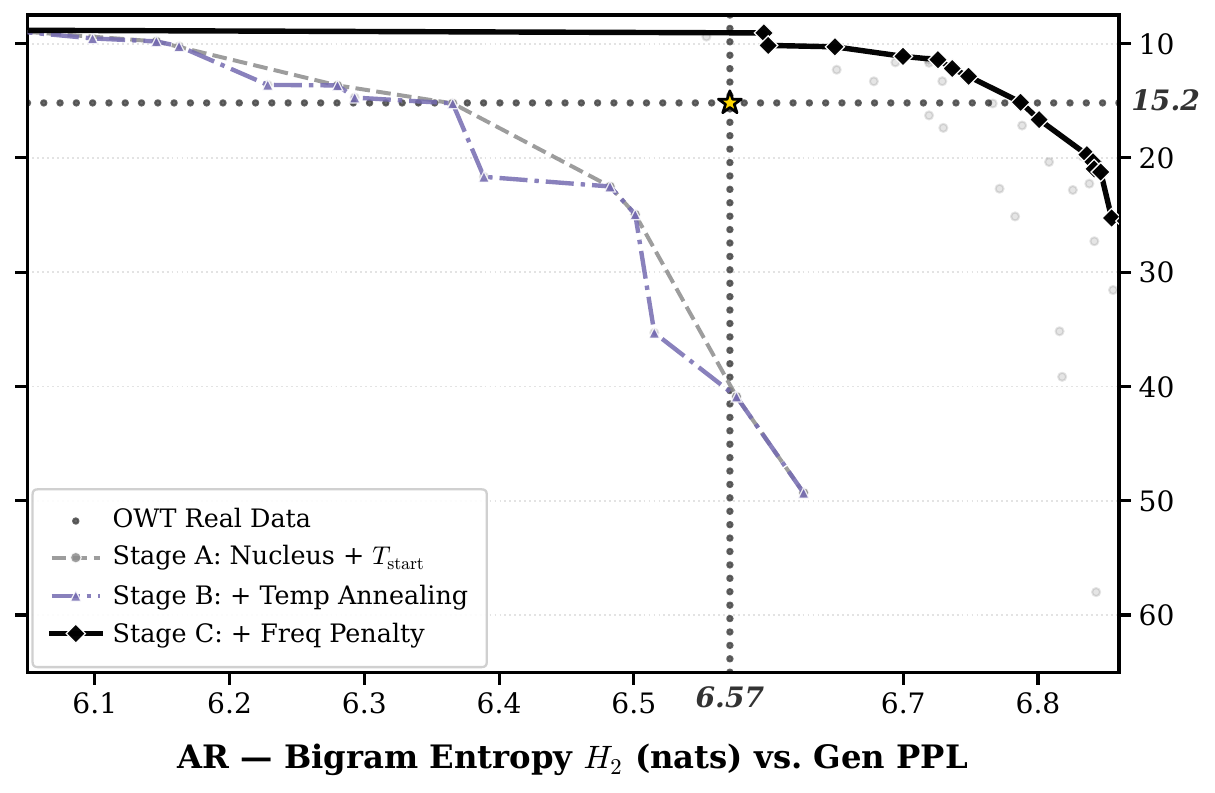}
  \vspace{-4pt}
  \caption{%
    \textbf{Progressive intervention ablations on Bigram Entropy $H_2$ (nats) vs.\ Generative Perplexity.}
    Order-sensitive bigram diversity control ($H_2$) across the same nested intervention stages (\textbf{Stage A} $\to$ \textbf{Stage D}) for \textbf{SDM}, \textbf{MDM}, \textbf{UDM}, and \textbf{AR}.
    Both the sequence frequency penalty (\textbf{Stage C}) and the local frequency penalty (\textbf{Stage D}) preserve gains along the bigram entropy $H_2$ axis.%
  }
  \label{fig:owt_pareto_interventions_h2}
\end{figure*}

\begin{table}[t!]
\centering
\scriptsize
{\setlength{\aboverulesep}{0pt}\setlength{\belowrulesep}{0pt}
\renewcommand{\arraystretch}{1.18}
\setlength{\tabcolsep}{6pt}
\caption{\textbf{Vertical Cuts on Unigram Entropy ($H_1 \ge \tau$): Minimum Generative Perplexity ($\text{GenPPL} \downarrow$) on OpenWebText.} AR is our model with the same architecture as the diffusion models, trained on OWT  (1024 steps). The best diffusion method is \textbf{bold}, second best is \underline{underlined}. The associated hyperparameters $(p,\, T_{\text{start}},\, \lambda,\, \gamma)$ are shown in grey below each entry. The real OWT validation data ($\star$) has $H_1 = 5.46$ and $\text{GenPPL} = 15.2$.}
\label{tab:owt_pareto_h1_vertical_cuts}
\begin{tabular}{@{}l !{\vrule width 0.65pt} ccc !{\vrule width 0.3pt} c@{}}
\toprule
\rule{0pt}{2.4ex}\rule[-1.0ex]{0pt}{0pt}\textbf{Threshold ($H_1 \ge \tau$)} & \textbf{SDM (Simplex)} & \textbf{MDM (Masked)} & \textbf{UDM (Uniform)} & \textbf{AR (1024 steps)} \\
\midrule
\rule{0pt}{2.5ex}$H_1 \ge 5.20$ & \underline{14.0} & 14.7 & \textbf{13.5} & 8.9 \\
\rule[-1.3ex]{0pt}{0pt}{\tiny \textcolor{gray}{\textit{Hyperparameters $(p, T_{\text{start}}, \lambda, \gamma)$}}} & {\tiny \textcolor{gray}{$(.96, .40, 3.0, 2.0)$}} & {\tiny \textcolor{gray}{$(.96, .70, 1.5, 0.0)$}} & {\tiny \textcolor{gray}{$(.92, .60, 1.5, 0.0)$}} & {\tiny \textcolor{gray}{$(.96, .65, 3.0, 0.0)$}} \\
\rule{0pt}{2.3ex}$H_1 \ge 5.35$ & \underline{15.8} & 16.3 & \textbf{15.2} & 8.9 \\
\rule[-1.3ex]{0pt}{0pt}{\tiny \textcolor{gray}{\textit{Hyperparameters $(p, T_{\text{start}}, \lambda, \gamma)$}}} & {\tiny \textcolor{gray}{$(.92, .30, 6.0, 2.5)$}} & {\tiny \textcolor{gray}{$(.92, .60, 3.0, 0.0)$}} & {\tiny \textcolor{gray}{$(.96, .60, 1.5, 0.0)$}} & {\tiny \textcolor{gray}{$(.96, .65, 3.0, 0.0)$}} \\
\rule{0pt}{2.3ex}\textbf{$\mathbf{H_1 \ge 5.46}$ {\tiny ($\star$ OWT = 15.2)}} & \underline{17.0} & \textbf{16.9} & 17.3 & 9.0 \\
\rule[-1.3ex]{0pt}{0pt}{\tiny \textcolor{gray}{\textit{Hyperparameters $(p, T_{\text{start}}, \lambda, \gamma)$}}} & {\tiny \textcolor{gray}{$(.92, .30, 6.0, 2.5)$}} & {\tiny \textcolor{gray}{$(.92, .50, 4.5, 0.0)$}} & {\tiny \textcolor{gray}{$(.92, .70, 1.5, 0.0)$}} & {\tiny \textcolor{gray}{$(.96, .65, 3.0, 0.0)$}} \\
\rule{0pt}{2.3ex}$H_1 \ge 5.65$ & \textbf{20.6} & 24.0 & \underline{22.5} & 9.0 \\
\rule[-1.3ex]{0pt}{0pt}{\tiny \textcolor{gray}{\textit{Hyperparameters $(p, T_{\text{start}}, \lambda, \gamma)$}}} & {\tiny \textcolor{gray}{$(.96, .30, 8.0, 2.0)$}} & {\tiny \textcolor{gray}{$(.92, .70, 3.0, 0.0)$}} & {\tiny \textcolor{gray}{$(.96, .60, 3.0, 0.0)$}} & {\tiny \textcolor{gray}{$(.96, .65, 3.0, 0.0)$}} \\
\rule{0pt}{2.3ex}$H_1 \ge 5.85$ & \textbf{26.2} & 31.2 & \underline{26.7} & 9.3 \\
\rule[-1.3ex]{0pt}{0pt}{\tiny \textcolor{gray}{\textit{Hyperparameters $(p, T_{\text{start}}, \lambda, \gamma)$}}} & {\tiny \textcolor{gray}{$(.92, .40, 6.0, 2.0)$}} & {\tiny \textcolor{gray}{$(.92, .65, 4.5, 0.0)$}} & {\tiny \textcolor{gray}{$(.96, .70, 3.0, 0.0)$}} & {\tiny \textcolor{gray}{$(.92, .70, 3.0, 0.0)$}} \\
\bottomrule
\end{tabular}}
\end{table}

\begin{table}[t!]
\centering
\scriptsize
{\setlength{\aboverulesep}{0pt}\setlength{\belowrulesep}{0pt}
\renewcommand{\arraystretch}{1.18}
\setlength{\tabcolsep}{6pt}
\caption{\textbf{Horizontal Cuts on Generative Perplexity ($\text{GenPPL} \le \tau_{\text{PPL}}$): Maximum Unigram Entropy ($H_1 \uparrow$) on OpenWebText.} AR is our model with the same architecture as the diffusion models, trained on OWT (1024 steps). The best diffusion method is \textbf{bold}, second best is \underline{underlined}. The associated hyperparameters $(p,\, T_{\text{start}},\, \lambda,\, \gamma)$ are shown in grey below each entry. The real OWT validation data ($\star$) has $\text{GenPPL} = 15.2$ and $H_1 = 5.46$.}
\label{tab:owt_pareto_h1_horizontal_cuts}
\begin{tabular}{@{}l !{\vrule width 0.65pt} ccc !{\vrule width 0.3pt} c@{}}
\toprule
\rule{0pt}{2.4ex}\rule[-1.0ex]{0pt}{0pt}\textbf{Budget ($\text{GenPPL} \le \tau_{\text{PPL}}$)} & \textbf{SDM (Simplex)} & \textbf{MDM (Masked)} & \textbf{UDM (Uniform)} & \textbf{AR (1024 steps)} \\
\midrule
\rule{0pt}{2.5ex}$\text{GenPPL} \le 12.5$ & 5.09 & \textbf{5.17} & \underline{5.13} & 6.18 \\
\rule[-1.3ex]{0pt}{0pt}{\tiny \textcolor{gray}{\textit{Hyperparameters $(p, T_{\text{start}}, \lambda, \gamma)$}}} & {\tiny \textcolor{gray}{$(.92, .30, 4.5, 2.5)$}} & {\tiny \textcolor{gray}{$(.96, .50, 3.0, 0.0)$}} & {\tiny \textcolor{gray}{$(.96, .50, 1.5, 0.0)$}} & {\tiny \textcolor{gray}{$(.96, .40, 8.0, 0.0)$}} \\
\rule{0pt}{2.3ex}$\text{GenPPL} \le 14.0$ & \underline{5.20} & 5.18 & \textbf{5.23} & 6.20 \\
\rule[-1.3ex]{0pt}{0pt}{\tiny \textcolor{gray}{\textit{Hyperparameters $(p, T_{\text{start}}, \lambda, \gamma)$}}} & {\tiny \textcolor{gray}{$(.92, .40, 3.0, 2.0)$}} & {\tiny \textcolor{gray}{$(.96, .40, 4.5, 0.0)$}} & {\tiny \textcolor{gray}{$(.96, .50, 1.5, 0.0)$}} & {\tiny \textcolor{gray}{$(.96, .40, 8.0, 0.0)$}} \\
\rule{0pt}{2.3ex}\textbf{$\mathbf{\text{GenPPL} \le 15.2}$ {\tiny ($\star$ OWT $H_1=5.46$)}} & \underline{5.29} & 5.22 & \textbf{5.32} & 6.21 \\
\rule[-1.3ex]{0pt}{0pt}{\tiny \textcolor{gray}{\textit{Hyperparameters $(p, T_{\text{start}}, \lambda, \gamma)$}}} & {\tiny \textcolor{gray}{$(.96, .40, 3.0, 2.0)$}} & {\tiny \textcolor{gray}{$(.92, .70, 1.5, 0.0)$}} & {\tiny \textcolor{gray}{$(.96, .50, 1.5, 0.0)$}} & {\tiny \textcolor{gray}{$(.92, .65, 4.5, 0.0)$}} \\
\rule{0pt}{2.3ex}$\text{GenPPL} \le 18.0$ & \textbf{5.54} & \underline{5.52} & 5.48 & 6.29 \\
\rule[-1.3ex]{0pt}{0pt}{\tiny \textcolor{gray}{\textit{Hyperparameters $(p, T_{\text{start}}, \lambda, \gamma)$}}} & {\tiny \textcolor{gray}{$(.96, .30, 6.0, 2.5)$}} & {\tiny \textcolor{gray}{$(.96, .50, 4.5, 0.0)$}} & {\tiny \textcolor{gray}{$(.96, .65, 1.5, 0.0)$}} & {\tiny \textcolor{gray}{$(.96, .65, 4.5, 0.0)$}} \\
\rule{0pt}{2.3ex}$\text{GenPPL} \le 22.0$ & \textbf{5.72} & 5.59 & \underline{5.62} & 6.43 \\
\rule[-1.3ex]{0pt}{0pt}{\tiny \textcolor{gray}{\textit{Hyperparameters $(p, T_{\text{start}}, \lambda, \gamma)$}}} & {\tiny \textcolor{gray}{$(.92, .40, 4.5, 2.0)$}} & {\tiny \textcolor{gray}{$(.96, .65, 3.0, 0.0)$}} & {\tiny \textcolor{gray}{$(.96, .50, 3.0, 0.0)$}} & {\tiny \textcolor{gray}{$(.96, .50, 8.0, 0.0)$}} \\
\bottomrule
\end{tabular}}
\end{table}

\subsection{Language Understanding}
\label{sec:language_understanding}

We describe the setup in \Cref{sec:appendix-setup-lu}.

The results sweeping on the weight $w$ in \eqref{eq:scoring_language_understanding} and comparing the baseline model with the adapted model  are reported in \Cref{fig:pmi-sweep}. The adapted model outperforms the baseline model, which confirms that the baseline is out of distribution when only the continuation is noised. Finally, our main results and comparison with benchmarks are presented in \Cref{tab:pmi_benchmarks}. 

\begin{figure}[h]
\resizebox{1.00\linewidth}{!}{%
\includegraphics[width=0.99\linewidth]{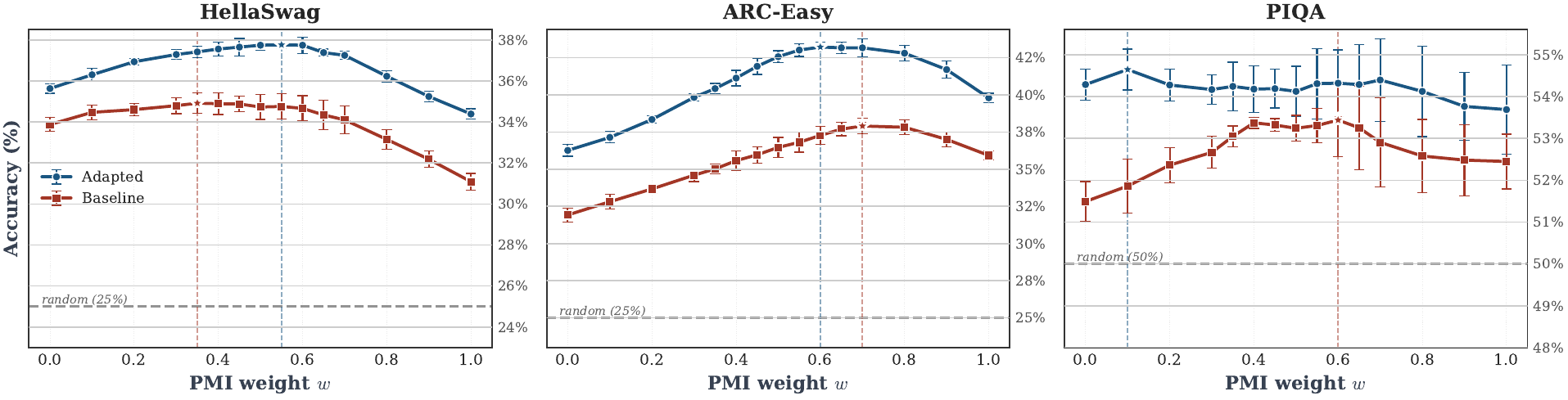}
}
\caption{We present a sweep on the Pointwise Mutual Information (PMI) weight $w$ in \eqref{eq:scoring_language_understanding}. For each of the three tasks that we investigate HellaSwag \citep{zellers2019hellaswag}, ARC-easy \citep{clark2018think} and PIQA \citep{bisk2020piqa}, we report the influence of $w$ on the final results. Note that $w=0$ corresponds to the setup which is most comparable to the MDM scoring \eqref{eq:mdm_scoring} (even though MDM noise is applied on the whole sequence while ours focus on the continuation). We denote $w^\star$ the best weighting for each task. The optimal weight differs across tasks ($w^\star = 0.1$ for PIQA, $0.6$ for ARC-Easy and $0.55$ for HellaSwag) and is selected on the evaluation split. The results are averaged over 5 different seeds. In red we report the results for the baseline out of distribution OWT checkpoint. In blue we report the results after the short adaptation run.}
    \label{fig:pmi-sweep}
\end{figure}

\begin{table}[h]
    \centering
    \small
    \setlength{\tabcolsep}{6.0pt}
    \renewcommand{\arraystretch}{1.08}
    \caption{\textbf{Zero-shot accuracy (\%)} on language understanding benchmarks. We evaluate SDMs by sampling 5 times with a different random seed, and  report the $\text{mean}_{\pm\text{std}}$. The best score per column is \textbf{bolded} and the second best is \underline{underlined}. The LLaMA baseline is taken from \citet{vonrutte2025generalizedinterpolatingdiscretediffusion}.}
    \label{tab:pmi_benchmarks}
    \begin{tabular}{l ccc}
    \toprule
    Model & PIQA & ARC-Easy & HellaSwag \\
    \midrule
    \multicolumn{4}{@{}l@{}}{\textit{Autoregressive baselines}} \\
    \hspace*{0.8em} LLaMA & \underline{62.7} & 40.5 & 33.1 \\
    \hspace*{0.8em} GPT-2 & \textbf{62.9} & \textbf{43.8} & 28.9 \\
    \midrule
    \multicolumn{4}{@{}l@{}}{\textit{Discrete Diffusion models}} \\
    \hspace*{0.8em} MDM & 54.1 & 31.0 & 31.1 \\
    \hspace*{0.8em} PGM & 58.9 & 40.4 & 33.2 \\
    \hspace*{0.8em} GIDD & 56.4 & 31.0 & 31.9 \\
    \midrule
    \multicolumn{4}{@{}l@{}}{\textit{Ours (Simplex Diffusion)}} \\
    \hspace*{0.8em} $w = 0$ & 54.3$_{\pm 0.4}$ & 36.3$_{\pm 0.4}$ & \underline{35.6}$_{\pm 0.2}$ \\
    \hspace*{0.8em} $w = w^\star$ & 54.6$_{\pm 0.5}$ & \underline{43.2}$_{\pm 0.3}$ & \textbf{37.9}$_{\pm 0.2}$ \\
    \bottomrule
    \end{tabular}
\end{table}

\subsection{Unconditional molecular generation}
\label{sec:appendix-genmol}

The experimental setup is described in~\Cref{sec:appendix-setup-genmol}.

\paragraph{Sampling hyperparameters selection.} To select the optimal inference hyperparameters for each method, sampler, and schedule, we swept over:
\begin{itemize}[leftmargin=1.5em, itemsep=1pt, topsep=2pt]
    \item \textbf{Autoregressive (AR)}: softmax temperature $T \in \{0.0,\allowbreak 0.5,\allowbreak 0.6,\allowbreak 0.7,\allowbreak 0.8,\allowbreak 0.9,\allowbreak 1.0\}$.
    \item \textbf{Discrete Diffusion (MDM \& UDM)}: logit temperature $T \in \{0.001,\allowbreak 0.005,\allowbreak 0.01,\allowbreak 0.03,\allowbreak 0.05,\allowbreak 0.08,\allowbreak 0.1,\allowbreak 0.15,\allowbreak 0.2,\allowbreak 0.3,\allowbreak 0.5,\allowbreak 0.7,\allowbreak 0.8,\allowbreak 1.0\}$, top-fraction schedule power $p \in \{0.5,\allowbreak 1.0,\allowbreak 1.5,\allowbreak 2.0,\allowbreak 3.0\}$.
    \item \textbf{Simplex Diffusion (SDM)}: logit temperature $T \in \{0.001,\allowbreak 0.005,\allowbreak 0.01,\allowbreak 0.03,\allowbreak 0.05,\allowbreak 0.08,\allowbreak 0.1,\allowbreak 0.15,\allowbreak 0.2,\allowbreak 0.3,\allowbreak 0.5,\allowbreak 0.7,\allowbreak 0.8,\allowbreak 1.0\}$, churn $\kappa \in \{0.0,\allowbreak 0.2,\allowbreak 1.0\}$, top-fraction schedule power $p \in \{0.5,\allowbreak 1.0,\allowbreak 1.5,\allowbreak 2.0,\allowbreak 3.0\}$.
\end{itemize}

\begin{table*}[t]
\centering
\small
\setlength{\tabcolsep}{6pt}
\renewcommand{\arraystretch}{0.95}
\caption{\textbf{Selected hyperparameters for Simplex Diffusion (with time conditioning, \texttt{SDM\_TE}) across noise schedules and sampling steps} (molecular generation; $N \in \{32, 64, 256, 1024\}$) corresponding to \Cref{fig:sdm_te_quality_diversity_vs_steps}. $T$: logit sampling temperature; $\kappa$: Simplex churn parameter; $p$: top-fraction power for $\alpha_s^p$ to select the most confident subset.}
\label{tab:sdm_te_multistep_hypers}
\begin{tabular}{lr cc ccc}
\toprule
& & \multicolumn{2}{c}{\textbf{Standard Sampler}} & \multicolumn{3}{c}{\textbf{Confidence-Based Sampler}} \\
\cmidrule(lr){3-4} \cmidrule(lr){5-7}
\textbf{Schedule} & \textbf{Steps ($N$)} & \textbf{Temperature ($T$)} & \textbf{Churn ($\kappa$)} & \textbf{Temperature ($T$)} & \textbf{Churn ($\kappa$)} & \textbf{Power ($p$)} \\
\midrule
Linear     &   32 & $0.01$   & $1.0$    & $0.01$   & $1.0$    & $1.0$    \\
           &   64 & $0.005$  & $1.0$    & $0.05$   & $1.0$    & $1.0$    \\
           &  256 & $0.05$   & $0.2$    & $0.05$   & $1.0$    & $1.0$    \\
           & 1024 & $0.001$  & $0.2$    & $0.05$   & $1.0$    & $1.0$    \\
\midrule
Cosine     &   32 & $0.03$   & $1.0$    & $0.03$   & $1.0$    & $1.0$    \\
           &   64 & $0.05$   & $1.0$    & $0.05$   & $1.0$    & $1.0$    \\
           &  256 & $0.10$   & $0.0$    & $0.05$   & $1.0$    & $1.0$    \\
           & 1024 & $0.15$   & $1.0$    & $0.08$   & $1.0$    & $1.0$    \\
\midrule
Adaptive   &   32 & $0.005$  & $0.0$    & $0.05$   & $1.0$    & $1.0$    \\
           &   64 & $0.001$  & $0.0$    & $0.05$   & $1.0$    & $1.0$    \\
           &  256 & $0.005$  & $0.2$    & $0.03$   & $1.0$    & $1.0$    \\
           & 1024 & $0.03$   & $1.0$    & $0.005$  & $1.0$    & $1.0$    \\
\bottomrule
\end{tabular}
\end{table*}

\begin{table}[t]
\centering
\small
\setlength{\tabcolsep}{6pt}
\renewcommand{\arraystretch}{0.95}
\caption{\textbf{Selected hyperparameters for diffusion methods under standard sampling across steps} ($N \in \{32, 64, 256, 1024\}$) corresponding to \Cref{fig:methods_standard_quality_diversity_vs_steps}. Simplex methods use the adaptive schedule; Masked and Uniform use the cosine schedule. $T$: logit sampling temperature; $\kappa$: Simplex churn parameter.}
\label{tab:diffusion_methods_standard_multistep_hypers}
\begin{tabular}{llr cc}
\toprule
\textbf{Method} & \textbf{Schedule} & \textbf{Steps ($N$)} & \textbf{Temperature ($T$)} & \textbf{Churn ($\kappa$)} \\
\midrule
Masked Diffusion (MDM)             & Cosine   &   32 & $0.05$   & ---      \\
                                   &          &   64 & $0.01$   & ---      \\
                                   &          &  256 & $0.005$  & ---      \\
                                   &          & 1024 & $0.03$   & ---      \\
\midrule
Uniform Diffusion (UDM)            & Cosine   &   32 & $0.05$   & ---      \\
                                   &          &   64 & $0.001$  & ---      \\
                                   &          &  256 & $0.01$   & ---      \\
                                   &          & 1024 & $0.01$   & ---      \\
\midrule
Uniform (with time conditioning)   & Cosine   &   32 & $0.005$  & ---      \\
                                   &          &   64 & $0.001$  & ---      \\
                                   &          &  256 & $0.05$   & ---      \\
                                   &          & 1024 & $0.05$   & ---      \\
\midrule
Simplex Diffusion (SDM)            & Adaptive &   32 & $0.08$   & $0.2$    \\
                                   &          &   64 & $0.10$   & $0.2$    \\
                                   &          &  256 & $0.10$   & $1.0$    \\
                                   &          & 1024 & $0.05$   & $1.0$    \\
\midrule
Simplex (with time conditioning)   & Adaptive &   32 & $0.005$  & $0.0$    \\
                                   &          &   64 & $0.001$  & $0.0$    \\
                                   &          &  256 & $0.005$  & $0.2$    \\
                                   &          & 1024 & $0.03$   & $1.0$    \\
\bottomrule
\end{tabular}
\end{table}

\begin{table}[t]
\centering
\small
\setlength{\tabcolsep}{5pt}
\renewcommand{\arraystretch}{0.95}
\caption{\textbf{Selected hyperparameters for diffusion methods under confidence-based sampling across steps} (molecular generation; $N \in \{32, 64, 256, 1024\}$) corresponding to \Cref{fig:methods_conf_quality_diversity_vs_steps}. Simplex methods use the adaptive schedule; Masked and Uniform diffusion use the cosine schedule. $T$: logit sampling temperature; $\kappa$: Simplex churn parameter; $p$: top-fraction power for $\alpha_s^p$ to select the most confident subset.}
\label{tab:diffusion_methods_conf_multistep_hypers}
\begin{tabular}{llr ccc}
\toprule
\textbf{Method} & \textbf{Schedule} & \textbf{Steps ($N$)} & \textbf{Temperature ($T$)} & \textbf{Churn ($\kappa$)} & \textbf{Power ($p$)} \\
\midrule
Masked Diffusion (MDM)             & Cosine   &   32 & $0.01$   & ---      & $1.5$    \\
                                   &          &   64 & $0.01$   & ---      & $1.5$    \\
                                   &          &  256 & $0.01$   & ---      & $1.0$    \\
                                   &          & 1024 & $0.01$   & ---      & $2.0$    \\
\midrule
Uniform Diffusion (UDM)            & Cosine   &   32 & $0.10$   & ---      & $2.0$    \\
                                   &          &   64 & $0.03$   & ---      & $1.5$    \\
                                   &          &  256 & $0.03$   & ---      & $1.0$    \\
                                   &          & 1024 & $0.01$   & ---      & $0.5$    \\
\midrule
Uniform (with time conditioning)   & Cosine   &   32 & $0.01$   & ---      & $1.0$    \\
                                   &          &   64 & $0.01$   & ---      & $1.0$    \\
                                   &          &  256 & $0.01$   & ---      & $1.0$    \\
                                   &          & 1024 & $0.01$   & ---      & $1.0$    \\
\midrule
Simplex Diffusion (SDM)            & Adaptive &   32 & $0.05$   & $1.0$    & $1.0$    \\
                                   &          &   64 & $0.01$   & $1.0$    & $1.0$    \\
                                   &          &  256 & $0.05$   & $1.0$    & $1.0$    \\
                                   &          & 1024 & $0.08$   & $1.0$    & $1.0$    \\
\midrule
Simplex (with time conditioning)   & Adaptive &   32 & $0.05$   & $1.0$    & $1.0$    \\
                                   &          &   64 & $0.05$   & $1.0$    & $1.0$    \\
                                   &          &  256 & $0.03$   & $1.0$    & $1.0$    \\
                                   &          & 1024 & $0.005$  & $1.0$    & $1.0$    \\
\bottomrule
\end{tabular}
\end{table}

\begin{table}[t]
\centering
\small
\setlength{\tabcolsep}{5pt}
\renewcommand{\arraystretch}{0.95}
\caption{\textbf{Optimal sampling hyperparameters for each method and schedule} (molecular generation) in \Cref{tab:safegpt_dropout0_ordinary} (standard sampling) and \Cref{tab:safegpt_dropout0_conf} (confidence-based sampling, \textit{conf.}) using $256$ sampling steps. $T$: logit temperature; $\kappa$: Simplex churn (\Cref{propbeaut:simplicialtransition}); $p$: top-fraction power for $\alpha_s^p$ to select the most confident subset.}
\label{tab:safegpt_dropout0_hypers}
\begin{tabular}{llcc}
\toprule
\textbf{Method} & \textbf{Schedule} & \textbf{Standard Sampling Hypers} & \textbf{Confidence (\textit{conf.}) Sampling Hypers} \\
\midrule
Autoregressive                       & ---      & $T=0.7$                & --- \\
\addlinespace[2pt]
\emph{Simplex}             & Linear   & $T=0.03,\; \kappa=1.0$ & $T=0.05,\; \kappa=1.0,\;  p=1.0$ \\
\emph{Simplex}             & Cosine   & $T=0.01,\; \kappa=1.0$ & $T=0.05,\; \kappa=1.0,\;  p=1.0$ \\
\emph{Simplex}             & Adaptive & $T=0.10,\; \kappa=1.0$ & $T=0.05,\; \kappa=1.0,\;  p=1.0$ \\
\addlinespace[2pt]
\emph{Simplex (time cond.)} & Linear   & $T=0.05,\; \kappa=0.2$ & $T=0.05,\; \kappa=1.0,\;  p=1.0$ \\
\emph{Simplex (time cond.)} & Cosine   & $T=0.10,\; \kappa=0.0$ & $T=0.05,\; \kappa=1.0,\;  p=1.0$ \\
\emph{Simplex (time cond.)} & Adaptive & $T=0.005,\; \kappa=0.2$ & $T=0.03,\; \kappa=1.0,\;  p=1.0$ \\
\addlinespace[2pt]
Uniform                    & Linear   & $T=0.01$               & $T=0.03,\;  p=1.5$ \\
Uniform                    & Cosine   & $T=0.01$               & $T=0.03,\;  p=1.0$ \\
Uniform                    & Adaptive & $T=0.10$               & $T=0.03,\;  p=1.0$ \\
\addlinespace[2pt]
Uniform (time cond.)        & Linear   & $T=0.01$               & $T=0.01,\;  p=1.0$ \\
Uniform (time cond.)        & Cosine   & $T=0.05$               & $T=0.01,\; p=1.0$ \\
Uniform (time cond.)        & Adaptive & $T=0.08$               & $T=0.01,\;  p=1.0$ \\
\addlinespace[2pt]
Masked                     & Linear   & $T=0.01$               & $T=0.01,\;  p=1.5$ \\
Masked                     & Cosine   & $T=0.005$              & $T=0.01,\;  p=1.0$ \\
Masked                     & Adaptive & $T=0.001$              & $T=0.01,\;  p=1.0$ \\
\bottomrule
\end{tabular}
\end{table}

The optimal sampling hyperparameters corresponding to Figure~\ref{fig:sdm_te_quality_diversity_vs_steps} are reported in Table~\ref{tab:sdm_te_multistep_hypers}. The optimal sampling hyperparameters corresponding to Figure~\ref{fig:methods_standard_quality_diversity_vs_steps} are reported in Table~\ref{tab:diffusion_methods_standard_multistep_hypers}, while those for Figure~\ref{fig:methods_conf_quality_diversity_vs_steps} are reported in Table~\ref{tab:diffusion_methods_conf_multistep_hypers}. The optimal sampling hyperparameters corresponding to each row of \Cref{tab:safegpt_dropout0_ordinary,tab:safegpt_dropout0_conf} with $256$ sampling steps, are reported in \Cref{tab:safegpt_dropout0_hypers} and are selected based on the quality metric.

\begin{figure}[!ht]
    \centering
    \includegraphics[width=\textwidth]{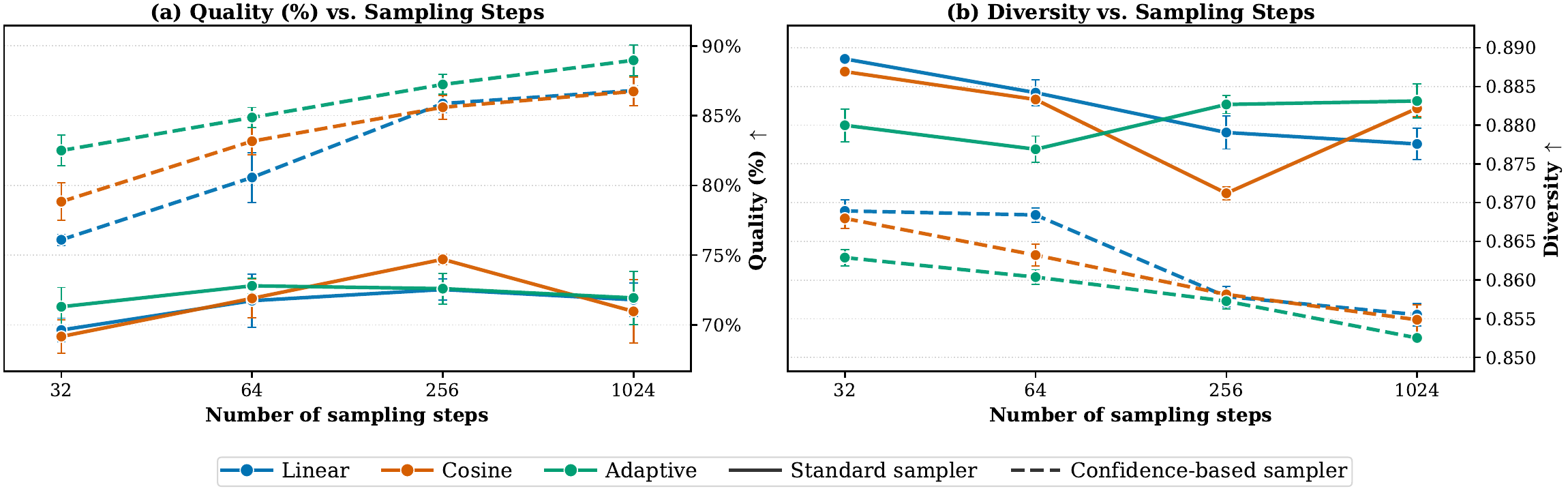}
    \vspace{-0.4em}
    \caption{\textbf{Sampling-budget and sampling time schedule comparison for time-conditioned SDM.}
Quality and Diversity (molecular generation) are shown for standard sampling (solid) and confidence-based sampling (dashed), using linear (blue), cosine (orange), and adaptive (green) inference grids.
Increasing the sampling budget improves Quality under confidence-based sampling while reducing Diversity.
Standard sampling retains higher Diversity, with smaller and nonmonotonic changes in Quality.
Error bars show one standard deviation across three sampling seeds.}
    \label{fig:sdm_te_quality_diversity_vs_steps}
\end{figure}

\begin{figure}[!ht]
    \centering
    \includegraphics[width=\textwidth]{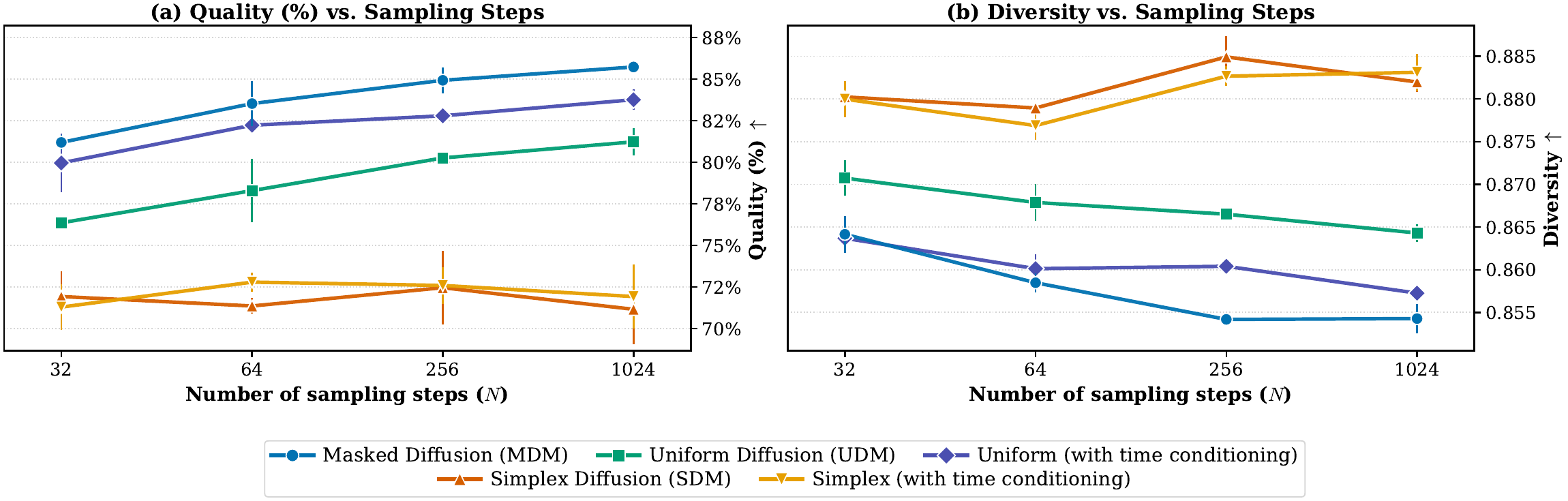}
    \vspace{-0.4em}
    \caption{\textbf{Quality and Diversity versus sampling steps under standard sampling} (molecular generation). 
    Comparison across all five diffusion methods, using adaptive time schedule for Simplex diffusion and cosine for the rest. 
    Error bars show one standard deviation across three sampling seeds.}
    \label{fig:methods_standard_quality_diversity_vs_steps}
\end{figure}

\begin{figure}[!ht]
    \centering
    \includegraphics[width=\textwidth]{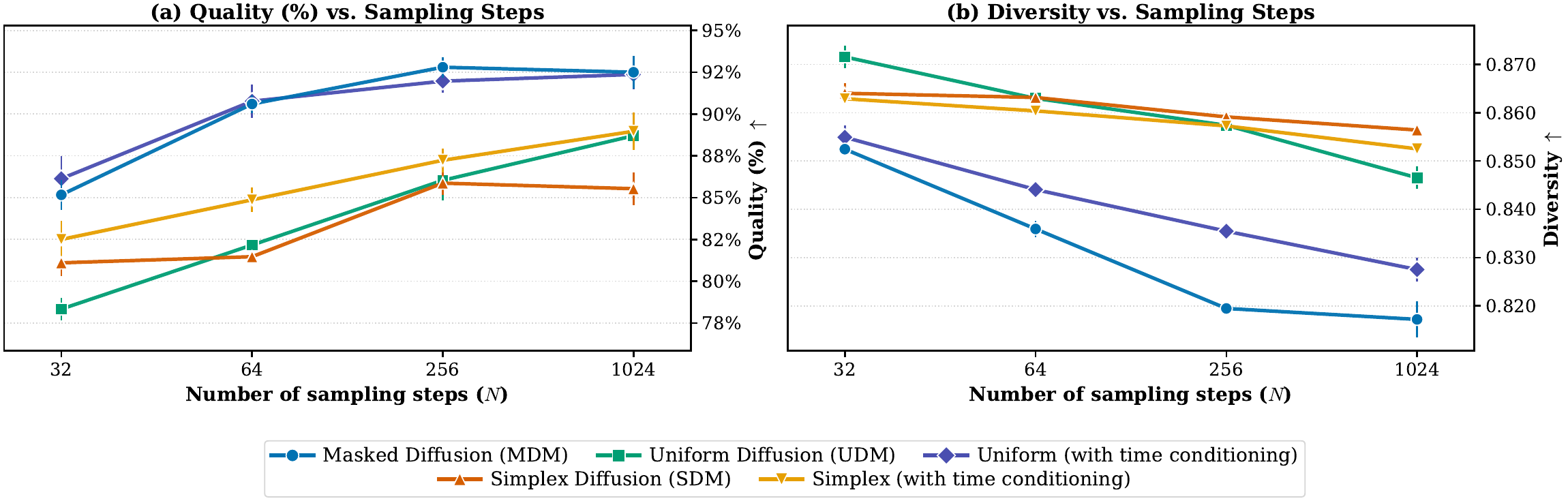}
    \vspace{-0.4em}
    \caption{\textbf{Quality and Diversity versus sampling steps under confidence-based sampling} (molecular generation).
    Comparison across all five diffusion methods, using adaptive time schedule for Simplex diffusion and cosine for the rest. 
    Error bars show one standard deviation across three sampling seeds.}
    \label{fig:methods_conf_quality_diversity_vs_steps}
\end{figure}

\begin{figure}[!ht]
    \centering
    \includegraphics[width=\textwidth]{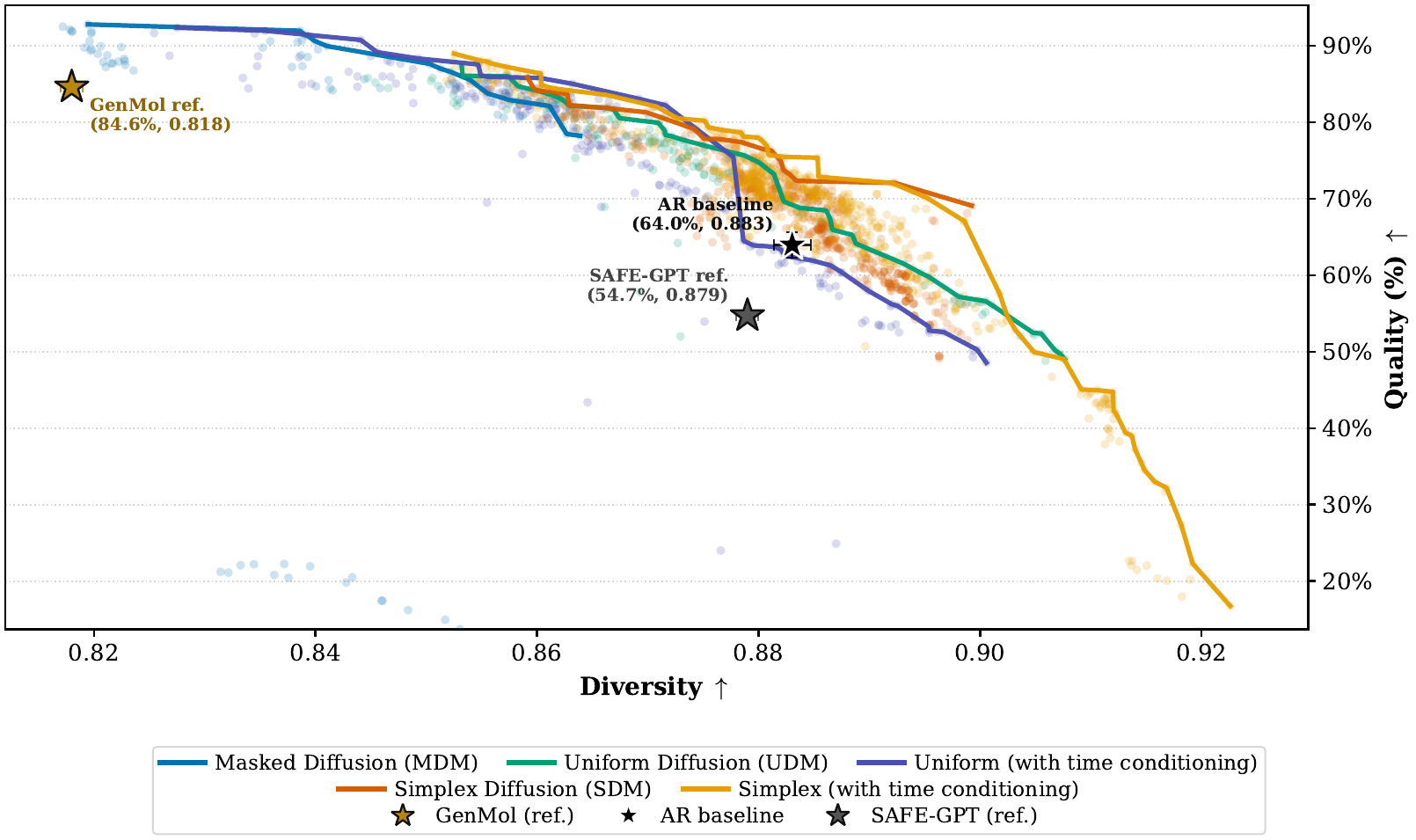}
    \vspace{-0.4em}
    \caption{\textbf{Empirical Quality--Diversity Pareto frontiers across diffusion methods with confidence based sampling and hyperparameter sweeps.}
Molecule Quality (\%), $\uparrow$, versus Diversity ($\uparrow$) across all evaluated sampling hyperparameter configurations. 
All markers denote individual exploration runs sweeping sampling temperatures, number of sampling steps, top-fraction power for selecting the most confident subset and churn parameters, see~\Cref{sec:appendix-setup-genmol} for more details.
Solid curves trace the empirical Pareto frontier of non-dominated configurations for each diffusion family, where we use adaptive time schedule for Simplex diffusion and cosine for the rest. Distinct stars mark external and autoregressive baselines: \textbf{GenMol} (gold star), the autoregressive \textbf{AR} model (black star), and \textbf{SAFE-GPT} (gray star). Simplex Diffusion Pareto frontier leans more towards top right, with many points located in high diversity regime.}
    \label{fig:quality_vs_diversity_sweeps_pareto}
\end{figure}

\begin{table}[t]
\centering
\small
\setlength{\tabcolsep}{4.5pt} 
\renewcommand{\arraystretch}{0.95} 
\caption{\textbf{Unconditional \textit{de novo} molecular generation on SAFE-GPT with standard sampling} following the GenMol evaluation protocol~\citep{lee2025genmol} ($1{,}000$ molecules/seed, mean $\pm$ std over 3 seeds). \textbf{Bold}: best; \underline{underline}: second best among our runs. All our runs use $256$ sampling steps.}
\label{tab:safegpt_dropout0_ordinary}
\begin{tabular}{llcccc}
\toprule
\textbf{Method} & \textbf{Schedule} & \textbf{Quality (\%)} $\uparrow$ & \textbf{Validity (\%)} $\uparrow$ & \textbf{Uniqueness (\%)} $\uparrow$ & \textbf{Diversity} $\uparrow$ \\
\midrule
\multicolumn{6}{l}{\textit{Published Reference~\citep{lee2025genmol}}} \\
SAFE-GPT                                & ---      & $54.7 \pm 0.3$               & $94.0 \pm 0.4$               & $100.0 \pm 0\phantom{.0}$    & $0.879 \pm 0.001$ \\
GenMol                                  & ---      & $84.6 \pm 0.8$               & $100.0 \pm 0.0$              & $99.7 \pm 0.1$               & $0.818 \pm 0.001$ \\
\midrule
\multicolumn{6}{l}{\textit{Our Run Comparison}} \\
Autoregressive (AR)                     & ---      & $63.97 \pm 1.67$             & $96.37 \pm 0.21$             & $99.76 \pm 0.06$             & $\underline{0.883 \pm 0.002}$ \\
\addlinespace[2pt]
\emph{Simplex}                & Linear   & $73.57 \pm 1.86$             & $99.63 \pm 0.12$             & $99.83 \pm 0.12$             & $0.879 \pm 0.001$ \\
\emph{Simplex}                & Cosine   & $72.33 \pm 0.15$             & $99.40 \pm 0.40$             & $\mathbf{99.97 \pm 0.06}$    & $0.881 \pm 0.000$ \\
\emph{Simplex}                & Adaptive & $72.47 \pm 2.21$             & $99.63 \pm 0.15$             & $99.80 \pm 0.17$             & $\mathbf{0.885 \pm 0.002}$ \\
\addlinespace[2pt]
\emph{Simplex (time cond.)}    & Linear   & $72.53 \pm 0.76$             & $99.00 \pm 0.10$             & $99.70 \pm 0.30$             & $0.879 \pm 0.002$ \\
\emph{Simplex (time cond.)}    & Cosine   & $74.70 \pm 0.36$             & $98.90 \pm 0.44$             & $97.64 \pm 0.76$             & $0.871 \pm 0.001$ \\
\emph{Simplex (time cond.)}    & Adaptive & $72.60 \pm 1.10$             & $98.83 \pm 0.06$             & $99.73 \pm 0.23$             & $\underline{0.883 \pm 0.001}$ \\
\addlinespace[2pt]
Uniform                       & Linear   & $80.60 \pm 1.41$             & $99.27 \pm 0.21$             & $98.89 \pm 0.44$             & $0.865 \pm 0.001$ \\
Uniform                       & Cosine   & $80.27 \pm 0.32$             & $99.30 \pm 0.46$             & $98.42 \pm 0.72$             & $0.867 \pm 0.001$ \\
Uniform                       & Adaptive & $75.70 \pm 1.87$             & $98.43 \pm 0.42$             & $98.58 \pm 0.27$             & $0.873 \pm 0.003$ \\
\addlinespace[2pt]
Uniform (time cond.)           & Linear   & $83.03 \pm 1.50$             & $\mathbf{99.80 \pm 0.00}$    & $99.73 \pm 0.21$             & $0.859 \pm 0.003$ \\
Uniform (time cond.)           & Cosine   & $82.80 \pm 0.36$             & $99.57 \pm 0.06$             & $99.87 \pm 0.12$             & $0.860 \pm 0.001$ \\
Uniform (time cond.)           & Adaptive & $73.80 \pm 0.53$             & $98.37 \pm 0.12$             & $98.68 \pm 0.54$             & $0.873 \pm 0.002$ \\
\addlinespace[2pt]
Masked                        & Linear   & $\underline{84.83 \pm 0.68}$ & $99.67 \pm 0.12$             & $\underline{99.90 \pm 0.10}$ & $0.855 \pm 0.002$ \\
Masked                        & Cosine   & $\mathbf{84.93 \pm 0.78}$    & $\underline{99.77 \pm 0.32}$ & $99.77 \pm 0.06$             & $0.854 \pm 0.001$ \\
Masked                        & Adaptive & $81.17 \pm 1.65$             & $99.20 \pm 0.00$             & $99.83 \pm 0.15$             & $0.863 \pm 0.001$ \\
\bottomrule
\end{tabular}
\end{table}

\begin{table}[t]
\centering
\small
\setlength{\tabcolsep}{4.5pt} 
\renewcommand{\arraystretch}{0.95} 
\caption{\textbf{Unconditional \textit{de novo} molecular generation on SAFE-GPT with confidence-based (\textit{conf.}) sampling} following the GenMol evaluation protocol~\citep{lee2025genmol} ($1{,}000$ molecules/seed, mean $\pm$ std over 3 seeds). \textbf{Bold}: best; \underline{underline}: second best among our runs. All our runs use $256$ sampling steps.}
\label{tab:safegpt_dropout0_conf}
\begin{tabular}{llcccc}
\toprule
\textbf{Method} & \textbf{Schedule} & \textbf{Quality (\%)} $\uparrow$ & \textbf{Validity (\%)} $\uparrow$ & \textbf{Uniqueness (\%)} $\uparrow$ & \textbf{Diversity} $\uparrow$ \\
\midrule
\multicolumn{6}{l}{\textit{Published Reference~\citep{lee2025genmol}}} \\
SAFE-GPT                                & ---      & $54.7 \pm 0.3$               & $94.0 \pm 0.4$               & $100.0 \pm 0\phantom{.0}$    & $0.879 \pm 0.001$ \\
GenMol                                  & ---      & $84.6 \pm 0.8$               & $100.0 \pm 0.0$              & $99.7 \pm 0.1$               & $0.818 \pm 0.001$ \\
\midrule
\multicolumn{6}{l}{\textit{Our Run Comparison}} \\
Autoregressive (AR)                     & ---      & $63.97 \pm 1.67$             & $96.37 \pm 0.21$             & $99.76 \pm 0.06$             & $\mathbf{0.883 \pm 0.002}$ \\
\addlinespace[2pt]
\emph{Simplex}                & Linear   & $83.87 \pm 1.36$             & $99.80 \pm 0.10$             & $\underline{99.87 \pm 0.12}$ & $0.859 \pm 0.000$ \\
\emph{Simplex}                & Cosine   & $83.57 \pm 1.59$             & $99.77 \pm 0.15$             & $\underline{99.87 \pm 0.15}$ & $0.862 \pm 0.001$ \\
\emph{Simplex}                & Adaptive & $85.87 \pm 0.70$             & $99.63 \pm 0.31$             & $\mathbf{99.93 \pm 0.06}$    & $0.859 \pm 0.001$ \\
\addlinespace[2pt]
\emph{Simplex (time cond.)}    & Linear   & $85.87 \pm 0.67$             & $99.63 \pm 0.35$             & $99.80 \pm 0.10$             & $0.858 \pm 0.001$ \\
\emph{Simplex (time cond.)}    & Cosine   & $85.60 \pm 0.85$             & $99.67 \pm 0.06$             & $99.70 \pm 0.20$             & $0.858 \pm 0.000$ \\
\emph{Simplex (time cond.)}    & Adaptive & $87.23 \pm 0.71$             & $99.87 \pm 0.12$             & $99.77 \pm 0.06$             & $0.857 \pm 0.001$ \\
\addlinespace[2pt]
Uniform                       & Linear   & $85.23 \pm 0.93$             & $99.37 \pm 0.23$             & $98.56 \pm 0.59$             & $0.858 \pm 0.003$ \\
Uniform                       & Cosine   & $86.03 \pm 1.18$             & $99.63 \pm 0.06$             & $97.69 \pm 0.27$             & $0.857 \pm 0.001$ \\
Uniform                       & Adaptive & $83.60 \pm 0.26$             & $99.27 \pm 0.31$             & $97.01 \pm 0.51$             & $\underline{0.863 \pm 0.001}$ \\
\addlinespace[2pt]
Uniform (time cond.)           & Linear   & $89.60 \pm 0.36$             & $\mathbf{99.93 \pm 0.12}$    & $99.20 \pm 0.40$             & $0.846 \pm 0.002$ \\
Uniform (time cond.)           & Cosine   & $91.97 \pm 0.70$             & $\underline{99.90 \pm 0.10}$ & $99.30 \pm 0.20$             & $0.835 \pm 0.001$ \\
Uniform (time cond.)           & Adaptive & $91.50 \pm 1.11$             & $99.87 \pm 0.06$             & $99.37 \pm 0.15$             & $0.839 \pm 0.002$ \\
\addlinespace[2pt]
Masked                        & Linear   & $\mathbf{92.80 \pm 0.26}$    & $\underline{99.90 \pm 0.10}$ & $99.53 \pm 0.38$             & $0.820 \pm 0.000$ \\
Masked                        & Cosine   & $\mathbf{92.80 \pm 0.61}$    & $99.77 \pm 0.06$             & $99.83 \pm 0.12$             & $0.819 \pm 0.001$ \\
Masked                        & Adaptive & $\underline{92.00 \pm 1.05}$ & $99.77 \pm 0.12$             & $99.83 \pm 0.06$             & $0.823 \pm 0.001$ \\
\bottomrule
\end{tabular}
\end{table}

\paragraph{Sampling budgets and sampling time schedules.} We start from an ablation over sampling time schedules for Simplex Diffusion for molecular generation. In Figure~\ref{fig:sdm_te_quality_diversity_vs_steps}, we report performance of Simplex Diffusion (with time conditioning) as a function of sampling steps using either standard or confidence-based sampler. We see that depending on the sampling budget, the type of a sampler and the evaluation metric chosen, the choice of a sampling time schedule may differ. For the next ablation, we select adaptive time schedule for Simplex diffusion since it offers overall a good trade-off between different metrics and different sampling budgets.

\paragraph{Comparison across model variants.} Next, we compare performance of Simplex Diffusion to masked and uniform diffusion. For these methods we use cosine time schedule since it led to the best Quality. We report results in Figure~\ref{fig:methods_standard_quality_diversity_vs_steps} for the standard sampler and in Figure~\ref{fig:methods_conf_quality_diversity_vs_steps} for the confidence-based one. We see that both masked and uniform diffusion lead to higher Quality than Simplex Diffusion, while achieving lower Diversity. In case of confidence-based sampler the finding remains except for the fact that Simplex diffusion performs on-par with Uniform diffusion without time conditioning. These results demonstrate that different diffusion methods provide a trade-off between these two metrics.

\paragraph{Quality--diversity trade-offs.} We further highlight this point by plotting a Pareto frontier for different methods with confidence-based sampler in Figure~\ref{fig:quality_vs_diversity_sweeps_pareto}, where different points represent different sampling seeds, different number of sampling steps, different sampling temperatures, different churns and different top fraction power, see~\Cref{sec:appendix-setup-genmol} for details. We see that Simplex Diffusion Pareto frontier is situated more towards right and top, though it does not achieve the highest Quality, compared to Uniform and Masked diffusion. Understanding how to push this frontier further to the top right is a promising future research direction.

\paragraph{Detailed results at 256 steps.}
\Cref{tab:safegpt_dropout0_ordinary,tab:safegpt_dropout0_conf}
report all four evaluation metrics under standard and confidence-based sampling.
With confidence-based sampling and an adaptive grid, time-conditioned SDM achieves $87.23\%$ Quality and $0.857$ Diversity.
For comparison, time-conditioned UDM with a cosine grid achieves $91.97\%$ and $0.835$, while MDM with a cosine grid achieves $92.80\%$ and $0.819$.
Time conditioning increases the selected confidence-based Quality scores from $85.87\%$ to $87.23\%$ for SDM with an adaptive grid and from $86.03\%$ to $91.97\%$ for UDM with a cosine grid, accompanied by reductions in Diversity. These results highlight the same findings observed in the figures above.

All reported diffusion configurations exceed our AR baseline in Quality.
Time-conditioned SDM also achieves higher reported Quality and Diversity than the published GenMol reference, although the training and sampling protocols differ. Overall, the results demonstrate the applicability of SDMs to molecular generation and their competitive performance in higher-diversity regimes, while identifying a remaining gap in maximum Quality relative to masked diffusion. Finding ways to bridge the gap and to push the Pareto frontier for Simplex Diffusion is left for future work.

\section{Additional OWT Samples}
\label{sec:additional_samples}
\Cref{tab:qual_owt_wu17_sample1,tab:qual_owt_wu17_sample2,tab:qual_owt_wu17_sample3,tab:qual_owt_wu17_sample4,tab:qual_owt_wu17_sample5} show non-cherry-picked unconditional 1024-token samples from SDMs trained on OWT.

\begin{table*}[t!]
\centering
\scriptsize
\renewcommand{\arraystretch}{1.25}
\caption{Unconditional 1024-token generation from Simplex Diffusion Model trained on OWT (Sample 1).}
\label{tab:qual_owt_wu17_sample1}
\begin{tabular}{@{}p{\textwidth}@{}}
\toprule
\textbf{\textcolor{teal!80!black}{Sample 1 (1024 tokens)}} \hfill \textbf{GenPPL (GPT-2 Large):} \texttt{14.15} \quad|\quad \textbf{Unigram Entropy:} \texttt{5.27} \quad|\quad \textbf{Distinct-2:} \texttt{0.724} \\
\midrule
bring peace and stability Middle East,'' he told reporters. ``We're all going to be hopeful if we're part of an agreement that is retroactively \dots\ And I think the Democratic Party will be up for grabs if she takes it seriously.''\par\vspace{0.35em}
Republican strategist Steve Bannon also criticized Trump's comments earlier this week, saying that he had a ``waste for American diplomacy'' by ``trying to burn out sanctions relief with Iran.\par\vspace{0.35em}
New York Times columnist Hugh Hewitt, also a former State Department official and a member of the Republican National Committee, says Trump is trying to forge a deal with Iran without ever breaching international sanctions.\par\vspace{0.35em}
``I have a lot of miscalculations because I think they've succeeded in a bad engagement'' Hewitt told The New York Times, adding that the administration needs to understand the terms of the deal.\par\vspace{0.35em}
``I think I don't think this deal will be bad, but it's hard --- which is very hard, very hard --- to fix it,'' he added. ``What can we do? Can we depolarize the rest of the Middle East?''\par\vspace{0.35em}
\texttt{\_\_\_} \quad 10 p.m.\par\vspace{0.35em}
North Korea's government says it backs Republican presidential candidate Donald Trump for setting a new tone with his criticism of North Korea's state-controlled media, in a move at odds with Washington rhetoric over its pursuit of nuclear weapons.\par\vspace{0.35em}
South Korean Foreign Minister Kim Kyung-seo responded in a statement calling Trump ``reckless and reckless.'' He criticizes Trump on North Korea, calling him a ``highly dangerous man for the President of the United States.''\par\vspace{0.35em}
He also said he would not talk to North Korea without doubting its nuclear weapons abroad. ``Without such a strong leadership, any threat from the North Korean regime would be rejected,'' he said. South Korea has said it banned Pyongyang's nuclear programs in the past because it would stop North Koreans from leaving the country without their weapons.\par\vspace{0.35em}
\texttt{\_\_\_} \quad 9:30 p.m.\par\vspace{0.35em}
North Korean military leader Jang Song-thaek says Trump ``totally impoundately'' the U.S. nuclear weapons program because it was ``totally fragile'' after the United States dropped an H-bomb on the city of Nagasaki, according to an Air Force One video released Tuesday.\par\vspace{0.35em}
He also said he hoped that the U.S. would make a good return solely to the negotiating table.\par\vspace{0.35em}
He also called for North Korea to get serious about halting its program, although the U.S. has not sanctioned or formally approved its nuclear-weapons program since August 1950. He also urged the United States to learn a lesson from history.\par\vspace{0.35em}
Trump said on Tuesday that the U.S. must getting ready to restart its nuclear program. He said restarting its work with North Korea could mean more years to begin in the coming months.\par\vspace{0.35em}
``It will be very gradual,'' he said. ``And I think that will not be abrupt --- and it will will be gradual --- until the United States can get back on the table.''\par\vspace{0.35em}
\texttt{\_\_\_} \quad 8:30 p.m.\par\vspace{0.35em}
Yuri Yuri, South Korea's former deputy prime minister, says Donald Trump believes the U.S. should continue to unravel its nuclear program but is a ``preaching and grave mistake.''\par\vspace{0.35em}
His remarks followed a recent high-level meeting between Japan and South Korea at an annual summit in Seoul.\par\vspace{0.35em}
Yuri, who is South Korea's first president, has criticized the U.S. should apologize for dropping two atomic bombs on Hiroshima during World War II. He said the U.S. should take steps necessary to try to halt its nuclear program.\par\vspace{0.35em}
\texttt{\_\_\_} \quad 8:45 p.m.\par\vspace{0.35em}
Former Republican presidential candidate Donald Trump said on Tuesday that his country's nuclear and missile programs should hold a meeting to discuss ``a peaceful solution.''\par\vspace{0.35em}
Trump made similar comments about his country's relationship with the United States, which has long labeled its nuclear weapons and other NATO allies are ``obsolete.''\par\vspace{0.35em}
North Korean leader, Kim Il Jong-un, had earlier this month accused China of cutting off off the country's nuclear program during talks with other parties for peace talks.\par\vspace{0.35em}
Trump's comments have been seen as provocative by some over his tough stance toward China and his controversial unification policy.\par\vspace{0.35em}
But South Korean Foreign Minister Jang Song-thaek issued a statement Tuesday saying the issue was non-negotiable and adding that nuclear talks could not be reached unless both sides reached a peaceful political solution.\par\vspace{0.35em}
\texttt{\_\_\_\_\_} \quad 7:30 p.m. \quad The United \\
\bottomrule
\end{tabular}
\end{table*}

\begin{table*}[t!]
\centering
\scriptsize
\renewcommand{\arraystretch}{1.25}
\caption{Unconditional 1024-token generation from Simplex Diffusion Model trained on OWT (Sample 2).}
\label{tab:qual_owt_wu17_sample2}
\begin{tabular}{@{}p{\textwidth}@{}}
\toprule
\textbf{\textcolor{teal!80!black}{Sample 2 (1024 tokens)}} \hfill \textbf{GenPPL (GPT-2 Large):} \texttt{19.52} \quad|\quad \textbf{Unigram Entropy:} \texttt{5.59} \quad|\quad \textbf{Distinct-2:} \texttt{0.804} \\
\midrule
after it was advertised as a ``gender-only'' service run by a non-binary woman. The ad came less than two hours after a video featuring a YouTube user called Avoid Open Door Wicked, which featured 12 women, 14 men, 11 men and 10 women was posted online on its website.\par\vspace{0.35em}
``I have been harassed 400+ times because I am surrounded by a hostile environment directed at me,'' the 23-year-old woman wrote. ``I can't believe it,'' read one of the ads, which read ``Divorce is in my veil!'' Another added: ``You are not a misogynist, and you cannot use your own operating system to escape your own travails.''\par\vspace{0.35em}
``If you were a non-binary woman would you be captive for your own sexual desires? If you'd you were a woman woman would you be captive for your desires?'' asked ACLU attorney Jennifer Partridge, who filed the case through the Justice Department's nonprofit Philanthropy Law Project.\par\vspace{0.35em}
``It's reflective of what we kind of social network is about and whether it is a real issue,'' Partridge added.\par\vspace{0.35em}
Courtney ACLU attorneys argued that the company's discrimination against harassment and gender-based bias could violate her legal rights.\par\vspace{0.35em}
``We do not substantially suggest that gender discrimination is not a real issue,'' she wrote her brief. ``This ad does not fit the context of our social Facebook or Google ads.'' Partridge noted that the company's policy --- which requires customers to purchase or sell their products online --- suggests more for women than it does for men. The company also sells products in the same retail department stores.\par\vspace{0.35em}
``One could argue that this product based on gender-based disrespect is far more sanctimonious what people might buy from a store online,'' Partridge said in an email. ``This is not an abstract issue.''\par\vspace{0.35em}
The company's justification is that its site must engage users to `disconnect' amongst psychological issues. ``I don't think it should,'' Partridge said. ``I don't think another person should be discriminated against.''\par\vspace{0.35em}
The company has also said that its decision to remove names or corporate logos from all ads on the site because it doesn't want to deter any potential trespassing. Story continues below advertisement --- Hillary Clinton has maintained her commanding lead in the presidential race over Republican presidential nominee Donald Trump.\par\vspace{0.35em}
Clinton says Trump leads her by 39\%, according to a new NBC News/Wall Street Journal Journal poll that shows Republican presidential nominee Donald Trump with the largest-ever lead in any presidential election.\par\vspace{0.35em}
The poll, conducted by Public Opinion Strategies, a long-time polling firm, surveyed 1,000 likely voters. It has a margin of error at 3.0 points with the total sample of 1,000.\par\vspace{0.35em}
Clinton gets 43 percent of the vote behind Green Green Party candidate Jill Stein (38 percent), while Trump gets 41 (36 percent) and Jill Stein (35 percent). Under the case for both candidates, Stein would get 33.9 percent of the vote.\par\vspace{0.35em}
The poll predicts Clinton would win over Trump (34 percent), while Stein would get 9.7 percent ahead of Stein (8 percent) Stein/Garrabee (6 percent).\par\vspace{0.35em}
Despite last week's debate announcing her bid for the presidency, Clinton has implied that she would support any future Republican presidential candidate. In November, when she joined Barack Obama's national security team, she said she would need to cast her first female vote in the U.S. Senate to serve as president.\par\vspace{0.35em}
The poll also shows Trump leads all other Republican candidates with 34\% support, while Texas Sen. Ted Cruz leads the presumptive GOP nominee with just 15\%.\par\vspace{0.35em}
Trump has previously abandoned former Democratic Secretary of State Hillary Clinton after launching an unsuccessful bid to revive his campaign. He himself, however, has endorsed the presumptive GOP nominee, admitting that he still had no intention of securing support. He also suggested that he might be tempted to endorse Clinton if his party didn't back him up.\par\vspace{0.35em}
Story continues below advertisement.\par\vspace{0.35em}
The polls are conducted by landline and automated landline telephone interviews with 1,000 likely voters nationwide and were conducted between Oct. 19 through Nov. 20, 2012, with Barack Obama and Mitt Romney registered among likely voters. The margin of error is plus or minus 3.5 points.\par\vspace{0.35em}
Also on HuffPost: Two new skyscraper plans will bring some of Manhattan's tallest skyline to the rest of the world, after Madison Square Park was knocked out of Manhattan's tallest buildings earlier this year. Construction on the first three-story tower has been delayed in time for a full moon. The city's skyscraper straddled between New York, New York, Chicago and \\
\bottomrule
\end{tabular}
\end{table*}

\begin{table*}[t!]
\centering
\scriptsize
\renewcommand{\arraystretch}{1.25}
\caption{Unconditional 1024-token generation from Simplex Diffusion Model trained on OWT (Sample 3).}
\label{tab:qual_owt_wu17_sample3}
\begin{tabular}{@{}p{\textwidth}@{}}
\toprule
\textbf{\textcolor{teal!80!black}{Sample 3 (1024 tokens)}} \hfill \textbf{GenPPL (GPT-2 Large):} \texttt{19.53} \quad|\quad \textbf{Unigram Entropy:} \texttt{5.23} \quad|\quad \textbf{Distinct-2:} \texttt{0.645} \\
\midrule
Related Articles Section 491. Related Articles Section 492 Miscellaneous Related Articles Section 493 Miscellaneous. Section 493. Related Section 494 Miscellaneous Related Articles Section 494 Miscellaneous. Section 493 Miscellaneous Articles Section 494 Miscellaneous Articles Sections 495 Miscellaneous Sections 487 Miscellaneous Sections 488 Miscellaneous Sections 489 24/24 Sections 4811 Miscellaneous Sections 4812 24/24 4813 Miscellaneous Sections 4812 24/24 4813 Miscellaneous Section 4812.\par\vspace{0.35em}
Sec. 501 of Hawaii Revised, Hawaii Revised Code as follows:\par\vspace{0.35em}
(a) It shall be unlawful as a citizen of the United States or a foreign territory of the United States; (a) may conduct, including but not not limited to as a citizen of the United States; (b) as an ex-concitizen of the State of Hawaii or a foreign territory; or (b) as a non-citizen of Hawaii or a former citizen of the United States; (ii) perform other activities as defined in this title. If such an individual does not wish to sign up for annual free membership or holiday free walks in an area outside Hawaii, it shall be unlawful to use internet services, advertisements newspaper articles, or witherbecoming to advertise such activities. Any person who knowingly giving birth to another person under this same title (or any other person under this title or regulations) shall be penalized.\par\vspace{0.35em}
If a person who is denied health care coverage has been notified in advance that a person who does not give birth birth at his birthplace will he become a resident of a foreign territory, or a Territory of another state or vice versa, there is no penalty for violation of this above law. If it also is found that a person who denies giving birth occurs at a site of birth will he is a resident of a foreign territory, or Foreign Territory of the State.\par\vspace{0.35em}
To prohibit any person visiting a foreign territory or from becoming a resident of any territory outside the United States, I am hereby directing that such person knowingly violates the above law. The incurring any violation of the above law or regulations may be taken pursuant to the provisions of the Rawled ``Hawaii'' Virgin Islands Act of 2018.\par\vspace{0.35em}
Hawaiian and native Hawai'i people: Hawaiian is a Hawaiian language which is a native Hawaiian language that has been traditionally associated with the Hawai'i people. It is nevertheless traditional Hawaiian language spoken by native Hawai'i people. The Native Hawaiian people are descendants of Native Hawaiian people who are descended from Native Hawaiian tribes, and since then they have been interceded by tribes from other Native Hawaiian tribes.\par\vspace{0.35em}
A permanent Hawaiian resident or permanent Hawaiian resident who resides in Hawaii becomes a permanent Hawaiian resident shall continue to form a liaison with all Hawaiian community and volunteer volunteers to assist with navig land being navigated navigated within Hawaii. A permanent Hawaiian resident shall be required surrender reside in Hawaii court for specified time period, and if granted permanent Hawaiian resident order obtain permission from his residence in Hawaii for a specified time period. Permanent Hawaiian residents shall be required to reside before a Hawaii court. A permanent Hawaiian resident shall continue to become legal citizens of Poly Hawaiians because they are native Hawai'i residents who are non-institutionalized employment and who operate their own agricultural enterprises.\par\vspace{0.35em}
Article XV01 of the Civil Rights Act of Samoa: I am proposing to amend section article XV01 of the Civil Rights Act of Samoa and shall be amended by sections VII, II and Article XII, the Constitution of the Republic on January 1, 2016.\par\vspace{0.35em}
Article XV01 U.S. Constitution of Samoa --- The Pacific Ocean Rights Act of Samoa Act: This act was performed on January 1, 2016 while facilitating the Civil Rights Act of Samoa. This Act shall be amended in section I, II, sections XVI, and XVI, the Constitution of Samoa as amended by section VI.\par\vspace{0.35em}
The American Samoa Act: This Act shall be amended in section Article XIII, the Pacific Ocean Rights Act as amended by section XXVIII. This Act shall be amended in section Article XIV, the Constitution of the United States as amended by section VI.\par\vspace{0.35em}
Section 18017 U.S. Constitution of the United States of America: Nothing shall be made construed unlawful for any county, state, political subdivision of any country or any other State whatsoever to employ any person in the United States of America. Section 18017 U.S. Constitution of the United States of America includes persons: Native American, African American, Baltic American, Mexican-American Indians non-Indigenous Indian Americans, immigrants, citizens of the United States.\par\vspace{0.35em}
Corrections / Proions: Section 17:20 Sec. 2 --- the Constitution 1/3 dated January 1, 2017. Section 17:20 Sec. 3 --- the text of the Constitution 1/3, dated January 1, 2017. (Proclamation). Section 17:20 Sec. 3 \\
\bottomrule
\end{tabular}
\end{table*}

\begin{table*}[t!]
\centering
\scriptsize
\renewcommand{\arraystretch}{1.25}
\caption{Unconditional 1024-token generation from Simplex Diffusion Model trained on OWT (Sample 4).}
\label{tab:qual_owt_wu17_sample4}
\begin{tabular}{@{}p{\textwidth}@{}}
\toprule
\textbf{\textcolor{teal!80!black}{Sample 4 (1024 tokens)}} \hfill \textbf{GenPPL (GPT-2 Large):} \texttt{20.85} \quad|\quad \textbf{Unigram Entropy:} \texttt{5.60} \quad|\quad \textbf{Distinct-2:} \texttt{0.848} \\
\midrule
efficient price for carbon emitted from renewable energy rests with the European Commission, spokesman Richard Ritter said.\par\vspace{0.35em}
Commissions are come under a special agreement, which allows member states to be analysed and then sell all of the emissions emitted from other EU states.\par\vspace{0.35em}
Under this agreement member states would buy emission-free allowances. The member states could then decide to sell them only if they couldn't afford to buy additional items.\par\vspace{0.35em}
For emissions-free under the new agreement, those allowances from Germany would have to give up their market share, said Ritter.\par\vspace{0.35em}
``The Commission will henceforth refuse to use emission-free allowances as an exercise for Germany or for Jean Juniet,'' Ritter said. ``There are other alternatives I would advise that we should be able to bring them into the domestic market, but not at all.''\par\vspace{0.35em}
Ritter added that the EU would need to remove emission-free allowances from its domestic market before sending them back to emission-free markets. He added that such plans were approved by the German parliament.\par\vspace{0.35em}
German MEP Bart Schoto told reporters on Sunday that there is no need for a new deal. ``It's just a simple legislative process,'' he said. He suggested emission-free allowances should be sold.\par\vspace{0.35em}
``I don't think so much but but leaving is not really a problem,'' he said. ``If we think we will be able to keep up the cost of our growing economy then I think this pathway will be very difficult.''\par\vspace{0.35em}
Other measures must be taken through the new emissions agreement, Ritter said. The measures could include building incentives to build nuclear power plants, which could cost about \$1 billion a year plus subsidies for nuclear power plants in South Florida, which could raise fuel costs.\par\vspace{0.35em}
``We will have to expand our capacity if we can't get our policies back into place,'' he said. ``It's going to be a very tough decision and it will be taken out as soon as possible,'' he said.\par\vspace{0.35em}
Merkel's action `unemptive': Juncker, who was arrested in Brussels in connection with his meeting between the two leaders and German Chancellor Angela Merkel, met privately with former Chancellor Gerhard Schroeder, the former Bavarian prime minister and medical doctor Reinforce Eisenberg at a press conference in Berlin last week.\par\vspace{0.35em}
The commission's decision has been fraught with political tensions. However, Germany has appealed to Brussels against a form of one-party rule imposed by the European Union.\par\vspace{0.35em}
On Friday, Germany's state energy company (SDF) and its state utility EDF --- which has been the main target for calls for divesting coal, said it would not move forward. Eisenberg said it was highly unlikely that Merkel would make any effort to change her position.\par\vspace{0.35em}
``This is a step in a right direction,'' said Marguerte Schweiko, the head of the German daily newspaper Vorb Mauschaftung in an interview with Der Spiegel.\par\vspace{0.35em}
``The chancellor has clearly changed position\dots\ She has already shown herself capable of exerting power without the popular support of her friends, who control the world's most valuable fossil fuel reserves.''\par\vspace{0.35em}
``Merkel is an anti-nuclear campaigner\dots\ She has made it clear that she did not actually participate in any concrete action.''\par\vspace{0.35em}
Schweiko also said Merkel her actions are ``unemptive'', adding that ``the German government played a principal role in cultivation of tools for climate change''. She added that Merkel had pressure on many Germans. ``Mrs Merkel seems to have lost her credibility in many ways she cannot stand trusted,'' she said.\par\vspace{0.35em}
Earlier this week, the German Green Party, Christian Saar, issued a statement condemning ``polarly polarized'' climate negotiations. ``It is inconceivable that all parties can agree on a legally binding treaty---even within the framework of the parliament---but it documents highly polarised,'' the statement read.\par\vspace{0.35em}
The EU has started undertaking its own voluntary emissions trading scheme, making matters more pressing among some European countries that are realistic about climate change, as well as other countries that have warned against taking climate action.\par\vspace{0.35em}
Germany may have already signed up to reduce carbon emissions, meaning that still does not have the resources to do so. But it may also have an opportunity to kick-start talks in Paris on climate change.\par\vspace{0.35em}
Later this year, Europe and the major United States are expected to host talks on agriculture and climate change. Their involvement The Paris talks is scheduled to begin at the end of 2012, months after the Kyoto Protocol took place in Copenhagen in December 2009.\par\vspace{0.35em}
British Prime Minister David Cameron, who is in Copenhagen, and is due to discuss negotiations on climate change, had told reporters the Paris talks would ``have no end in sight yet''. Cameron told reporters earlier this week that Britain was ``not ruling out of Paris'', nor was an organiser of the negotiations. G20 leaders will hold a summit in Hamburg later next month \\
\bottomrule
\end{tabular}
\end{table*}

\begin{table*}[t!]
\centering
\scriptsize
\renewcommand{\arraystretch}{1.25}
\caption{Unconditional 1024-token generation from Simplex Diffusion Model trained on OWT (Sample 5).}
\label{tab:qual_owt_wu17_sample5}
\begin{tabular}{@{}p{\textwidth}@{}}
\toprule
\textbf{\textcolor{teal!80!black}{Sample 5 (1024 tokens)}} \hfill \textbf{GenPPL (GPT-2 Large):} \texttt{4.04} \quad|\quad \textbf{Unigram Entropy:} \texttt{3.56} \quad|\quad \textbf{Distinct-2:} \texttt{0.197} \\
\midrule
Commandments. 300. 325 325. The Son of God, His Son, His Father and the Commandments.\par\vspace{0.25em}
300. 325 325. The Son of God, The Son of the Holy and Father and the Commandments.\par\vspace{0.25em}
300. 325 325 Elijah Elijah is not commanded by The Lord of Christ, His Father and the Commandments.\par\vspace{0.25em}
300. 325 325 Elijah is not commanded by Jesus Christ, His Father, the Commandments.\par\vspace{0.25em}
300. 325 325. Of The Lord of Christ, His Father and the Commandments. \quad 300. 325. The Lord of Christ, His Son and the Commandments. \quad 300. 325 325. The Lord of Christ, His Father and the Commandments.\par\vspace{0.25em}
301. 325 The Lord of Christ, His Father and the Commandments. \quad 301. 325. The Lord of Christ, His Father, the Commandments. \quad 301. 325 325. Of The Lord's son, His Father, the Commandments.\par\vspace{0.25em}
301. the Lord Elijah is not commanded by The Lord of son, His Father, the Commandments. \quad 301. the Lord Elijah is not commanded by The Lord's son, Father, the Commandments.\par\vspace{0.25em}
301. the Lord Elijah is not being commanded by the Lord's Christ, His Son, His Father and the Commandments. \quad 301. the Lord Elijah is not being commanded by The Lord's son, His Father and the Commandments.\par\vspace{0.25em}
302. 304 308 308. The Son of God to His Father and the Commandment 351. \quad 303. 308 308 The Son of God, Maker Commandments.\par\vspace{0.25em}
304. 328 309. Jesus God to His Father and the Commandment 351. \quad 303. 328 309. Jesus Lord to be Father, Maker Commandments. \quad 303. 328 309. Jesus Lord to be Stronger.\par\vspace{0.25em}
305. 328 309. The Son of God. Maker Commandment. \quad 306. 328 309 309. The Son of Christ, His Father, the Commandments. \quad 305. 328 309. The Lord to be Strong, Maker Commandments.\par\vspace{0.25em}
307. 328. The Son of Christ, His Father, the Commandments. \quad 308. 328 328. The Son of Christ, His Father, the Commandments. \quad 309. 328 328. The Lord to be Stronger.\par\vspace{0.25em}
322. 329 330. The Son of Christ, His Father, the Commandments. \quad 323. 329 330. Jesus The Son of God, Maker Command. \quad 324. 329 330. Jesus Lord to be Strong. \quad 323. 329 330. The Son of God, Maker Commandments.\par\vspace{0.25em}
324. 329 331 Lose! Of The Son's God, His Father and the Commandments. \quad 326. 329 331. The Son of Christ, His Father and the Commandments.\par\vspace{0.25em}
325. 320 315 315. Jesus The Son of God, His Exalted Father and the Commandments. \quad 326. 320 320. Jesus the Son of Christ, His Father and the Commandments.\par\vspace{0.25em}
325 320 320. Jesus Christ to the Father, the Commandment 351. \quad 325 320 320. The Son of Christ to the Father and the Commandment 351. \quad 327. 320 320. The Son of Christ to the Father and the Commandments.\par\vspace{0.25em}
349. 321 321. The Son of God, Maker Command. \quad 350. 321 321. Jesus Christ to He Father, Maker Commandment. \quad 351. 321 321. Jesus the Son of God to His Father, the Commandment 351.\par\vspace{0.25em}
352 321 321. The Son of Christ to His Father, the Commandment 351. \quad 352. 321 321. The Son's God, His Exalted Father and the Commandment 351. \quad 353. 321 321. The Son of Christ to His Father, the Commandment 351.\par\vspace{0.25em}
352. 412 355. The Son of Christ to His Father, the Commandment 351. \quad 352. 412 355. The Son of God, Maker Commandment. \quad 353. 412 355. The Lord to be Strongerful. \quad 354. 412 355. The Lord of his Father, Maker Commandments.\par\vspace{0.25em}
355. 412 412. Jesus The Son of Christ, the Father and the Commandments. \quad 355. 412 412. The Son of Christ, the Father and the Commandments. \quad 356. 412 412. Jesus Christ to the Father, Maker Commandments.\par\vspace{0.25em}
357. 412 410. The Son of Christ, His God, the Commandment 351. \quad 356 412 412 410. Jesus Christ to the Father, Maker Commandments. \quad 357. 412 410. The Son of God, the Father and the Commandments.\par\vspace{0.25em}
358. 412 412 410. The Lord's God, the Father and Commanders. \quad 358. 412 410. The Son of God, Maker Commandments. \quad 357. 415 414. Jesus Christ to the Father, the Commandments. \quad 357. 412 414. The Son of Christ, the Father and the Commandments. \quad 358. 412 415. It Happen to \\
\bottomrule
\end{tabular}
\end{table*}

\clearpage
\section{Full Result Tables}
\label{app:full_tables}
\subsection{Sudoku}
\label{app:full_tables_sudoku}

\begin{table}[!htbp]
    \centering
    \small
    \setlength{\tabcolsep}{6.0pt}
    \renewcommand{\arraystretch}{1.08}
    \caption{\textbf{Accuracy (\%) on Sudoku} in 180 steps for \textbf{Simplex Diffusion Models} with Churn \mbox{$\kappa = 0.0$} across sampling schedules. Values are reported as $\text{mean}_{{\color{gray}\pm\text{std}}}$ over 5 seeds. Column-wise top scores are \textbf{bolded}, and overall best for $\kappa = 0.0$ is \underline{\textbf{underlined}}}
    \label{tab:sudoku_sdm_churn_0_0}
    \begin{tabular}{l ccc}
    \toprule
    Concentration Schedule & Linear & Cosine & Adaptive \\
    \midrule
    \multicolumn{4}{@{}l@{}}{\textit{Uniform time sampler}} \\
    \hspace*{0.8em} Const.-Lin. ($\nu_0=0.2, \nu_1=0.5, \ell=0.2$) & $62.8_{{\color{gray}\scriptscriptstyle\pm 5.0}}$ & $62.4_{{\color{gray}\scriptscriptstyle\pm 4.7}}$ & -- \\
    \hspace*{0.8em} Const.-Lin. ($\nu_0=0.2, \nu_1=0.75, \ell=0.2$) & $74.0_{{\color{gray}\scriptscriptstyle\pm 4.4}}$ & $73.7_{{\color{gray}\scriptscriptstyle\pm 4.3}}$ & -- \\
    \hspace*{0.8em} Const.-Lin. ($\nu_0=0.4, \nu_1=0.75, \ell=0.8$) & $73.9_{{\color{gray}\scriptscriptstyle\pm 4.7}}$ & $74.0_{{\color{gray}\scriptscriptstyle\pm 4.7}}$ & -- \\
    \hspace*{0.8em} Constant ($\nu = 0.5$) & $64.9_{{\color{gray}\scriptscriptstyle\pm 3.7}}$ & $64.6_{{\color{gray}\scriptscriptstyle\pm 3.1}}$ & -- \\
    \hspace*{0.8em} Constant ($\nu = 0.25$) & $56.0_{{\color{gray}\scriptscriptstyle\pm 6.2}}$ & $55.6_{{\color{gray}\scriptscriptstyle\pm 5.8}}$ & -- \\
    \midrule
    \multicolumn{4}{@{}l@{}}{\textit{Adaptive time sampler}} \\
    \hspace*{0.8em} Const.-Lin. ($\nu_0=0.2, \nu_1=0.5, \ell=0.2$) & $78.5_{{\color{gray}\scriptscriptstyle\pm 2.1}}$ & $78.1_{{\color{gray}\scriptscriptstyle\pm 2.3}}$ & $75.3_{{\color{gray}\scriptscriptstyle\pm 2.5}}$ \\
    \hspace*{0.8em} Const.-Lin. ($\nu_0=0.2, \nu_1=0.75, \ell=0.2$) & $79.4_{{\color{gray}\scriptscriptstyle\pm 2.6}}$ & $79.2_{{\color{gray}\scriptscriptstyle\pm 2.7}}$ & $77.2_{{\color{gray}\scriptscriptstyle\pm 3.1}}$ \\
    \hspace*{0.8em} Const.-Lin. ($\nu_0=0.4, \nu_1=0.75, \ell=0.8$) & $77.1_{{\color{gray}\scriptscriptstyle\pm 2.5}}$ & $76.9_{{\color{gray}\scriptscriptstyle\pm 2.4}}$ & $75.1_{{\color{gray}\scriptscriptstyle\pm 2.8}}$ \\
    \hspace*{0.8em} Constant ($\nu = 0.5$) & $78.0_{{\color{gray}\scriptscriptstyle\pm 2.0}}$ & $77.9_{{\color{gray}\scriptscriptstyle\pm 2.1}}$ & $74.2_{{\color{gray}\scriptscriptstyle\pm 2.0}}$ \\
    \hspace*{0.8em} Constant ($\nu = 0.25$) & $76.0_{{\color{gray}\scriptscriptstyle\pm 3.5}}$ & $75.5_{{\color{gray}\scriptscriptstyle\pm 3.6}}$ & $72.4_{{\color{gray}\scriptscriptstyle\pm 4.2}}$ \\
    \midrule
    \multicolumn{4}{@{}l@{}}{\textit{Self-Conditioning + Uniform time sampler}} \\
    \hspace*{0.8em} Const.-Lin. ($\nu_0=0.2, \nu_1=0.5, \ell=0.2$) & $97.4_{{\color{gray}\scriptscriptstyle\pm 2.3}}$ & $94.6_{{\color{gray}\scriptscriptstyle\pm 6.6}}$ & -- \\
    \hspace*{0.8em} Const.-Lin. ($\nu_0=0.2, \nu_1=0.75, \ell=0.2$) & $96.9_{{\color{gray}\scriptscriptstyle\pm 3.1}}$ & $94.0_{{\color{gray}\scriptscriptstyle\pm 9.1}}$ & -- \\
    \hspace*{0.8em} Const.-Lin. ($\nu_0=0.4, \nu_1=0.75, \ell=0.8$) & $\underline{\textbf{99.2}}_{{\color{gray}\scriptscriptstyle\pm 0.2}}$ & $\textbf{99.1}_{{\color{gray}\scriptscriptstyle\pm 0.2}}$ & -- \\
    \hspace*{0.8em} Constant ($\nu = 0.5$) & $98.7_{{\color{gray}\scriptscriptstyle\pm 0.6}}$ & $98.1_{{\color{gray}\scriptscriptstyle\pm 1.3}}$ & -- \\
    \hspace*{0.8em} Constant ($\nu = 0.25$) & $97.6_{{\color{gray}\scriptscriptstyle\pm 1.7}}$ & $96.5_{{\color{gray}\scriptscriptstyle\pm 3.4}}$ & -- \\
    \midrule
    \multicolumn{4}{@{}l@{}}{\textit{Self-Conditioning + adaptive time sampler}} \\
    \hspace*{0.8em} Const.-Lin. ($\nu_0=0.2, \nu_1=0.5, \ell=0.2$) & $91.6_{{\color{gray}\scriptscriptstyle\pm 17.3}}$ & $88.5_{{\color{gray}\scriptscriptstyle\pm 24.3}}$ & $89.7_{{\color{gray}\scriptscriptstyle\pm 21.7}}$ \\
    \hspace*{0.8em} Const.-Lin. ($\nu_0=0.2, \nu_1=0.75, \ell=0.2$) & $94.0_{{\color{gray}\scriptscriptstyle\pm 10.1}}$ & $85.4_{{\color{gray}\scriptscriptstyle\pm 24.8}}$ & $89.5_{{\color{gray}\scriptscriptstyle\pm 17.9}}$ \\
    \hspace*{0.8em} Const.-Lin. ($\nu_0=0.4, \nu_1=0.75, \ell=0.8$) & $97.3_{{\color{gray}\scriptscriptstyle\pm 4.5}}$ & $95.2_{{\color{gray}\scriptscriptstyle\pm 9.1}}$ & $96.6_{{\color{gray}\scriptscriptstyle\pm 4.9}}$ \\
    \hspace*{0.8em} Constant ($\nu = 0.5$) & $99.1_{{\color{gray}\scriptscriptstyle\pm 0.2}}$ & $98.9_{{\color{gray}\scriptscriptstyle\pm 0.5}}$ & $\textbf{98.9}_{{\color{gray}\scriptscriptstyle\pm 0.3}}$ \\
    \hspace*{0.8em} Constant ($\nu = 0.25$) & $98.1_{{\color{gray}\scriptscriptstyle\pm 2.9}}$ & $97.1_{{\color{gray}\scriptscriptstyle\pm 5.0}}$ & $97.5_{{\color{gray}\scriptscriptstyle\pm 4.1}}$ \\
    \bottomrule
    \end{tabular}
\end{table}

\begin{table}[!htbp]
    \centering
    \small
    \setlength{\tabcolsep}{6.0pt}
    \renewcommand{\arraystretch}{1.08}
    \caption{\textbf{Accuracy (\%) on Sudoku} in 180 steps for \textbf{Simplex Diffusion Models} with Churn \mbox{$\kappa = 0.2$} across sampling schedules. Values are reported as $\text{mean}_{{\color{gray}\pm\text{std}}}$ over 5 seeds. Column-wise top scores are \textbf{bolded}, and overall best for $\kappa = 0.2$ is \underline{\textbf{underlined}}}
    \label{tab:sudoku_sdm_churn_0_2}
    \begin{tabular}{l ccc}
    \toprule
    Concentration Schedule & Linear & Cosine & Adaptive \\
    \midrule
    \multicolumn{4}{@{}l@{}}{\textit{Uniform time sampler}} \\
    \hspace*{0.8em} Const.-Lin. ($\nu_0=0.2, \nu_1=0.5, \ell=0.2$) & $81.1_{{\color{gray}\scriptscriptstyle\pm 3.2}}$ & $81.4_{{\color{gray}\scriptscriptstyle\pm 3.3}}$ & -- \\
    \hspace*{0.8em} Const.-Lin. ($\nu_0=0.2, \nu_1=0.75, \ell=0.2$) & $85.3_{{\color{gray}\scriptscriptstyle\pm 3.2}}$ & $85.4_{{\color{gray}\scriptscriptstyle\pm 3.3}}$ & -- \\
    \hspace*{0.8em} Const.-Lin. ($\nu_0=0.4, \nu_1=0.75, \ell=0.8$) & $80.4_{{\color{gray}\scriptscriptstyle\pm 3.7}}$ & $79.1_{{\color{gray}\scriptscriptstyle\pm 3.5}}$ & -- \\
    \hspace*{0.8em} Constant ($\nu = 0.5$) & $83.5_{{\color{gray}\scriptscriptstyle\pm 2.0}}$ & $80.0_{{\color{gray}\scriptscriptstyle\pm 1.8}}$ & -- \\
    \hspace*{0.8em} Constant ($\nu = 0.25$) & $76.6_{{\color{gray}\scriptscriptstyle\pm 3.5}}$ & $77.6_{{\color{gray}\scriptscriptstyle\pm 3.5}}$ & -- \\
    \midrule
    \multicolumn{4}{@{}l@{}}{\textit{Adaptive time sampler}} \\
    \hspace*{0.8em} Const.-Lin. ($\nu_0=0.2, \nu_1=0.5, \ell=0.2$) & $88.3_{{\color{gray}\scriptscriptstyle\pm 1.4}}$ & $88.0_{{\color{gray}\scriptscriptstyle\pm 1.3}}$ & $85.7_{{\color{gray}\scriptscriptstyle\pm 1.6}}$ \\
    \hspace*{0.8em} Const.-Lin. ($\nu_0=0.2, \nu_1=0.75, \ell=0.2$) & $87.9_{{\color{gray}\scriptscriptstyle\pm 1.6}}$ & $87.6_{{\color{gray}\scriptscriptstyle\pm 1.7}}$ & $85.1_{{\color{gray}\scriptscriptstyle\pm 1.9}}$ \\
    \hspace*{0.8em} Const.-Lin. ($\nu_0=0.4, \nu_1=0.75, \ell=0.8$) & $85.2_{{\color{gray}\scriptscriptstyle\pm 1.8}}$ & $82.9_{{\color{gray}\scriptscriptstyle\pm 1.8}}$ & $82.1_{{\color{gray}\scriptscriptstyle\pm 2.1}}$ \\
    \hspace*{0.8em} Constant ($\nu = 0.5$) & $87.5_{{\color{gray}\scriptscriptstyle\pm 2.0}}$ & $83.6_{{\color{gray}\scriptscriptstyle\pm 1.6}}$ & $84.0_{{\color{gray}\scriptscriptstyle\pm 3.2}}$ \\
    \hspace*{0.8em} Constant ($\nu = 0.25$) & $86.3_{{\color{gray}\scriptscriptstyle\pm 2.2}}$ & $86.1_{{\color{gray}\scriptscriptstyle\pm 1.8}}$ & $84.4_{{\color{gray}\scriptscriptstyle\pm 2.4}}$ \\
    \midrule
    \multicolumn{4}{@{}l@{}}{\textit{Self-Conditioning + Uniform time sampler}} \\
    \hspace*{0.8em} Const.-Lin. ($\nu_0=0.2, \nu_1=0.5, \ell=0.2$) & $98.3_{{\color{gray}\scriptscriptstyle\pm 1.4}}$ & $95.8_{{\color{gray}\scriptscriptstyle\pm 4.5}}$ & -- \\
    \hspace*{0.8em} Const.-Lin. ($\nu_0=0.2, \nu_1=0.75, \ell=0.2$) & $97.3_{{\color{gray}\scriptscriptstyle\pm 2.4}}$ & $95.8_{{\color{gray}\scriptscriptstyle\pm 4.8}}$ & -- \\
    \hspace*{0.8em} Const.-Lin. ($\nu_0=0.4, \nu_1=0.75, \ell=0.8$) & $\underline{\textbf{98.6}}_{{\color{gray}\scriptscriptstyle\pm 0.2}}$ & $95.8_{{\color{gray}\scriptscriptstyle\pm 0.6}}$ & -- \\
    \hspace*{0.8em} Constant ($\nu = 0.5$) & $97.8_{{\color{gray}\scriptscriptstyle\pm 0.5}}$ & $92.7_{{\color{gray}\scriptscriptstyle\pm 1.1}}$ & -- \\
    \hspace*{0.8em} Constant ($\nu = 0.25$) & $98.5_{{\color{gray}\scriptscriptstyle\pm 0.6}}$ & $\textbf{97.5}_{{\color{gray}\scriptscriptstyle\pm 1.4}}$ & -- \\
    \midrule
    \multicolumn{4}{@{}l@{}}{\textit{Self-Conditioning + adaptive time sampler}} \\
    \hspace*{0.8em} Const.-Lin. ($\nu_0=0.2, \nu_1=0.5, \ell=0.2$) & $89.8_{{\color{gray}\scriptscriptstyle\pm 21.6}}$ & $86.6_{{\color{gray}\scriptscriptstyle\pm 28.2}}$ & $86.9_{{\color{gray}\scriptscriptstyle\pm 28.1}}$ \\
    \hspace*{0.8em} Const.-Lin. ($\nu_0=0.2, \nu_1=0.75, \ell=0.2$) & $92.7_{{\color{gray}\scriptscriptstyle\pm 13.8}}$ & $87.7_{{\color{gray}\scriptscriptstyle\pm 22.3}}$ & $88.5_{{\color{gray}\scriptscriptstyle\pm 21.9}}$ \\
    \hspace*{0.8em} Const.-Lin. ($\nu_0=0.4, \nu_1=0.75, \ell=0.8$) & $96.7_{{\color{gray}\scriptscriptstyle\pm 4.3}}$ & $92.0_{{\color{gray}\scriptscriptstyle\pm 8.3}}$ & $95.0_{{\color{gray}\scriptscriptstyle\pm 5.9}}$ \\
    \hspace*{0.8em} Constant ($\nu = 0.5$) & $98.2_{{\color{gray}\scriptscriptstyle\pm 0.2}}$ & $93.4_{{\color{gray}\scriptscriptstyle\pm 0.3}}$ & $97.1_{{\color{gray}\scriptscriptstyle\pm 2.0}}$ \\
    \hspace*{0.8em} Constant ($\nu = 0.25$) & $98.4_{{\color{gray}\scriptscriptstyle\pm 2.1}}$ & $97.0_{{\color{gray}\scriptscriptstyle\pm 3.6}}$ & $\textbf{97.2}_{{\color{gray}\scriptscriptstyle\pm 4.3}}$ \\
    \bottomrule
    \end{tabular}
\end{table}

\begin{table}[!htbp]
    \centering
    \small
    \setlength{\tabcolsep}{6.0pt}
    \renewcommand{\arraystretch}{1.08}
    \caption{\textbf{Accuracy (\%) on Sudoku} in 180 steps for \textbf{Simplex Diffusion Models} with Churn \mbox{$\kappa = 1.0$} across sampling schedules. Values are reported as $\text{mean}_{{\color{gray}\pm\text{std}}}$ over 5 seeds. Column-wise top scores are \textbf{bolded}, and overall best for $\kappa = 1.0$ is \underline{\textbf{underlined}}}
    \label{tab:sudoku_sdm_churn_1_0}
    \begin{tabular}{l ccc}
    \toprule
    Concentration Schedule & Linear & Cosine & Adaptive \\
    \midrule
    \multicolumn{4}{@{}l@{}}{\textit{Uniform time sampler}} \\
    \hspace*{0.8em} Const.-Lin. ($\nu_0=0.2, \nu_1=0.5, \ell=0.2$) & $89.3_{{\color{gray}\scriptscriptstyle\pm 2.3}}$ & $89.7_{{\color{gray}\scriptscriptstyle\pm 2.3}}$ & -- \\
    \hspace*{0.8em} Const.-Lin. ($\nu_0=0.2, \nu_1=0.75, \ell=0.2$) & $91.5_{{\color{gray}\scriptscriptstyle\pm 1.9}}$ & $91.6_{{\color{gray}\scriptscriptstyle\pm 1.6}}$ & -- \\
    \hspace*{0.8em} Const.-Lin. ($\nu_0=0.4, \nu_1=0.75, \ell=0.8$) & $88.4_{{\color{gray}\scriptscriptstyle\pm 2.2}}$ & $88.9_{{\color{gray}\scriptscriptstyle\pm 2.4}}$ & -- \\
    \hspace*{0.8em} Constant ($\nu = 0.5$) & $90.5_{{\color{gray}\scriptscriptstyle\pm 1.1}}$ & $90.6_{{\color{gray}\scriptscriptstyle\pm 1.2}}$ & -- \\
    \hspace*{0.8em} Constant ($\nu = 0.25$) & $86.9_{{\color{gray}\scriptscriptstyle\pm 2.1}}$ & $87.8_{{\color{gray}\scriptscriptstyle\pm 2.5}}$ & -- \\
    \midrule
    \multicolumn{4}{@{}l@{}}{\textit{Adaptive time sampler}} \\
    \hspace*{0.8em} Const.-Lin. ($\nu_0=0.2, \nu_1=0.5, \ell=0.2$) & $90.8_{{\color{gray}\scriptscriptstyle\pm 0.9}}$ & $90.8_{{\color{gray}\scriptscriptstyle\pm 1.0}}$ & $89.4_{{\color{gray}\scriptscriptstyle\pm 1.0}}$ \\
    \hspace*{0.8em} Const.-Lin. ($\nu_0=0.2, \nu_1=0.75, \ell=0.2$) & $90.5_{{\color{gray}\scriptscriptstyle\pm 1.1}}$ & $90.3_{{\color{gray}\scriptscriptstyle\pm 1.3}}$ & $88.8_{{\color{gray}\scriptscriptstyle\pm 1.5}}$ \\
    \hspace*{0.8em} Const.-Lin. ($\nu_0=0.4, \nu_1=0.75, \ell=0.8$) & $88.3_{{\color{gray}\scriptscriptstyle\pm 1.3}}$ & $88.5_{{\color{gray}\scriptscriptstyle\pm 1.2}}$ & $86.5_{{\color{gray}\scriptscriptstyle\pm 1.4}}$ \\
    \hspace*{0.8em} Constant ($\nu = 0.5$) & $90.6_{{\color{gray}\scriptscriptstyle\pm 1.5}}$ & $90.4_{{\color{gray}\scriptscriptstyle\pm 1.5}}$ & $88.8_{{\color{gray}\scriptscriptstyle\pm 1.8}}$ \\
    \hspace*{0.8em} Constant ($\nu = 0.25$) & $89.6_{{\color{gray}\scriptscriptstyle\pm 1.8}}$ & $89.8_{{\color{gray}\scriptscriptstyle\pm 1.8}}$ & $88.3_{{\color{gray}\scriptscriptstyle\pm 1.9}}$ \\
    \midrule
    \multicolumn{4}{@{}l@{}}{\textit{Self-Conditioning + Uniform time sampler}} \\
    \hspace*{0.8em} Const.-Lin. ($\nu_0=0.2, \nu_1=0.5, \ell=0.2$) & $98.9_{{\color{gray}\scriptscriptstyle\pm 0.8}}$ & $97.2_{{\color{gray}\scriptscriptstyle\pm 2.8}}$ & -- \\
    \hspace*{0.8em} Const.-Lin. ($\nu_0=0.2, \nu_1=0.75, \ell=0.2$) & $98.3_{{\color{gray}\scriptscriptstyle\pm 1.1}}$ & $96.9_{{\color{gray}\scriptscriptstyle\pm 3.4}}$ & -- \\
    \hspace*{0.8em} Const.-Lin. ($\nu_0=0.4, \nu_1=0.75, \ell=0.8$) & $\underline{\textbf{99.1}}_{{\color{gray}\scriptscriptstyle\pm 0.2}}$ & $\textbf{98.7}_{{\color{gray}\scriptscriptstyle\pm 0.3}}$ & -- \\
    \hspace*{0.8em} Constant ($\nu = 0.5$) & $98.5_{{\color{gray}\scriptscriptstyle\pm 0.6}}$ & $97.5_{{\color{gray}\scriptscriptstyle\pm 1.1}}$ & -- \\
    \hspace*{0.8em} Constant ($\nu = 0.25$) & $\underline{\textbf{99.1}}_{{\color{gray}\scriptscriptstyle\pm 0.3}}$ & $98.5_{{\color{gray}\scriptscriptstyle\pm 0.8}}$ & -- \\
    \midrule
    \multicolumn{4}{@{}l@{}}{\textit{Self-Conditioning + adaptive time sampler}} \\
    \hspace*{0.8em} Const.-Lin. ($\nu_0=0.2, \nu_1=0.5, \ell=0.2$) & $91.2_{{\color{gray}\scriptscriptstyle\pm 18.8}}$ & $87.5_{{\color{gray}\scriptscriptstyle\pm 26.6}}$ & $88.6_{{\color{gray}\scriptscriptstyle\pm 24.4}}$ \\
    \hspace*{0.8em} Const.-Lin. ($\nu_0=0.2, \nu_1=0.75, \ell=0.2$) & $93.9_{{\color{gray}\scriptscriptstyle\pm 11.9}}$ & $89.1_{{\color{gray}\scriptscriptstyle\pm 20.7}}$ & $90.3_{{\color{gray}\scriptscriptstyle\pm 18.7}}$ \\
    \hspace*{0.8em} Const.-Lin. ($\nu_0=0.4, \nu_1=0.75, \ell=0.8$) & $97.5_{{\color{gray}\scriptscriptstyle\pm 3.9}}$ & $94.8_{{\color{gray}\scriptscriptstyle\pm 8.9}}$ & $96.3_{{\color{gray}\scriptscriptstyle\pm 5.4}}$ \\
    \hspace*{0.8em} Constant ($\nu = 0.5$) & $98.8_{{\color{gray}\scriptscriptstyle\pm 0.2}}$ & $98.2_{{\color{gray}\scriptscriptstyle\pm 0.2}}$ & $\textbf{98.3}_{{\color{gray}\scriptscriptstyle\pm 0.2}}$ \\
    \hspace*{0.8em} Constant ($\nu = 0.25$) & $98.8_{{\color{gray}\scriptscriptstyle\pm 1.5}}$ & $97.8_{{\color{gray}\scriptscriptstyle\pm 3.3}}$ & $98.1_{{\color{gray}\scriptscriptstyle\pm 2.9}}$ \\
    \bottomrule
    \end{tabular}
\end{table}

\begin{table}[!htbp]
    \centering
    \small
    \setlength{\tabcolsep}{6.0pt}
    \renewcommand{\arraystretch}{1.08}
    \caption{\textbf{Baseline Models Accuracy (\%) on Sudoku} in 180 steps across sampling schedules for the Appendix. We evaluate Autoregressive (AR), Discrete Absorbing (MDM with CE and ELBO losses), Discrete Uniform (UDM), Predictor-Corrector (PC), Self-Conditioning (SC / loopholing), Flow Matching over one-hot encoding (FLM with uniform, Gauss-Hermite LUT, and adaptive time samplers), and Hyperspherical Flows ($\mathbb{S}$-FLM). $^\dagger$trained with uniform time sampler, $^\ddagger$trained with adaptive time sampler. Values are reported as $\text{mean}_{{\color{gray}\pm\text{std}}}$ over 5 seeds. We \textbf{bold} the best result and \underline{underline} the second best in each column.}
    \label{tab:sudoku_baselines_appendix}
    \begin{tabular}{l ccc}
    \toprule
    Model & Linear & Cosine & Adaptive \\
    \midrule
    \multicolumn{4}{@{}l@{}}{\textit{Autoregressive}} \\
    \hspace*{0.8em} AR (Sample) & $\hphantom{0}2.9_{{\color{gray}\scriptscriptstyle\pm 0.6}}$ & -- & -- \\
    \hspace*{0.8em} AR (Greedy) & $\hphantom{0}3.1_{{\color{gray}\scriptscriptstyle\pm 0.5}}$ & -- & -- \\
    \midrule
    \multicolumn{4}{@{}l@{}}{\textit{Discrete (Absorbing)}} \\
    \hspace*{0.8em} Ancestral$^\dagger$ (CE) & $60.1_{{\color{gray}\scriptscriptstyle\pm 3.8}}$ & $61.1_{{\color{gray}\scriptscriptstyle\pm 4.5}}$ & -- \\
    \hspace*{0.8em} Ancestral$^\ddagger$ (CE) & $69.4_{{\color{gray}\scriptscriptstyle\pm 4.8}}$ & $69.7_{{\color{gray}\scriptscriptstyle\pm 4.0}}$ & $70.0_{{\color{gray}\scriptscriptstyle\pm 4.4}}$ \\
    \hspace*{0.8em} Ancestral$^\dagger$ (ELBO) & $56.0_{{\color{gray}\scriptscriptstyle\pm 5.4}}$ & $55.8_{{\color{gray}\scriptscriptstyle\pm 5.8}}$ & -- \\
    \hspace*{0.8em} Ancestral$^\ddagger$ (ELBO) & $71.5_{{\color{gray}\scriptscriptstyle\pm 2.2}}$ & $71.7_{{\color{gray}\scriptscriptstyle\pm 1.6}}$ & $71.8_{{\color{gray}\scriptscriptstyle\pm 2.4}}$ \\
    \hspace*{0.8em} Predictor-Corrector$^\dagger$ & $85.3_{{\color{gray}\scriptscriptstyle\pm 1.2}}$ & $84.6_{{\color{gray}\scriptscriptstyle\pm 1.0}}$ & -- \\
    \hspace*{0.8em} Predictor-Corrector$^\ddagger$ & $33.8_{{\color{gray}\scriptscriptstyle\pm 41.4}}$ & $35.6_{{\color{gray}\scriptscriptstyle\pm 40.7}}$ & $69.6_{{\color{gray}\scriptscriptstyle\pm 31.6}}$ \\
    \hspace*{0.8em} Ancestral + SC$^\dagger$ (CE) & $91.2_{{\color{gray}\scriptscriptstyle\pm 0.7}}$ & $98.3_{{\color{gray}\scriptscriptstyle\pm 0.5}}$ & -- \\
    \hspace*{0.8em} Ancestral + SC$^\ddagger$ (CE) & $94.7_{{\color{gray}\scriptscriptstyle\pm 0.8}}$ & $\textbf{99.3}_{{\color{gray}\scriptscriptstyle\pm 0.2}}$ & $\underline{91.6}_{{\color{gray}\scriptscriptstyle\pm 1.0}}$ \\
    \hspace*{0.8em} Ancestral + SC$^\dagger$ (ELBO) & $90.5_{{\color{gray}\scriptscriptstyle\pm 0.8}}$ & $98.1_{{\color{gray}\scriptscriptstyle\pm 0.8}}$ & -- \\
    \hspace*{0.8em} Ancestral + SC$^\ddagger$ (ELBO) & $94.3_{{\color{gray}\scriptscriptstyle\pm 0.7}}$ & $\underline{99.2}_{{\color{gray}\scriptscriptstyle\pm 0.2}}$ & $90.9_{{\color{gray}\scriptscriptstyle\pm 1.0}}$ \\
    \midrule
    \multicolumn{4}{@{}l@{}}{\textit{Discrete (Uniform)}} \\
    \hspace*{0.8em} Ancestral$^\dagger$ (ELBO) & $73.7_{{\color{gray}\scriptscriptstyle\pm 3.4}}$ & $73.6_{{\color{gray}\scriptscriptstyle\pm 3.2}}$ & -- \\
    \hspace*{0.8em} Ancestral$^\ddagger$ (ELBO) & $82.1_{{\color{gray}\scriptscriptstyle\pm 1.7}}$ & $82.4_{{\color{gray}\scriptscriptstyle\pm 1.7}}$ & $81.5_{{\color{gray}\scriptscriptstyle\pm 1.6}}$ \\
    \hspace*{0.8em} Predictor-Corrector$^\dagger$ & $\underline{95.8}_{{\color{gray}\scriptscriptstyle\pm 1.0}}$ & $95.2_{{\color{gray}\scriptscriptstyle\pm 1.3}}$ & -- \\
    \hspace*{0.8em} Predictor-Corrector$^\ddagger$ & $92.9_{{\color{gray}\scriptscriptstyle\pm 0.9}}$ & $92.3_{{\color{gray}\scriptscriptstyle\pm 0.9}}$ & -- \\
    \hspace*{0.8em} Ancestral + SC$^\dagger$ (ELBO) & $\textbf{97.8}_{{\color{gray}\scriptscriptstyle\pm 1.7}}$ & $98.1_{{\color{gray}\scriptscriptstyle\pm 1.2}}$ & -- \\
    \hspace*{0.8em} Ancestral + SC$^\ddagger$ (ELBO) & $94.9_{{\color{gray}\scriptscriptstyle\pm 5.1}}$ & $95.5_{{\color{gray}\scriptscriptstyle\pm 4.8}}$ & $\textbf{93.9}_{{\color{gray}\scriptscriptstyle\pm 5.5}}$ \\
    \midrule
    \multicolumn{4}{@{}l@{}}{\textit{Euclidean / Spherical Flows}} \\
    \hspace*{0.8em} FLM$^\dagger$ (Uniform) & $73.4_{{\color{gray}\scriptscriptstyle\pm 5.2}}$ & $73.0_{{\color{gray}\scriptscriptstyle\pm 5.1}}$ & -- \\
    \hspace*{0.8em} FLM (Gauss-Hermite LUT) & $67.7_{{\color{gray}\scriptscriptstyle\pm 3.4}}$ & $67.4_{{\color{gray}\scriptscriptstyle\pm 3.2}}$ & -- \\
    \hspace*{0.8em} FLM$^\ddagger$ (Adaptive) & $75.3_{{\color{gray}\scriptscriptstyle\pm 3.4}}$ & $74.9_{{\color{gray}\scriptscriptstyle\pm 3.5}}$ & $69.2_{{\color{gray}\scriptscriptstyle\pm 4.2}}$ \\
    \hspace*{0.8em} $\mathbb{S}$-FLM$^\dagger$ & $83.5_{{\color{gray}\scriptscriptstyle\pm 2.2}}$ & $84.0_{{\color{gray}\scriptscriptstyle\pm 1.8}}$ & -- \\
    \hspace*{0.8em} $\mathbb{S}$-FLM$^\ddagger$ & $87.3_{{\color{gray}\scriptscriptstyle\pm 1.5}}$ & $87.3_{{\color{gray}\scriptscriptstyle\pm 1.7}}$ & $84.5_{{\color{gray}\scriptscriptstyle\pm 1.9}}$ \\
    \bottomrule
    \end{tabular}
\end{table}

\clearpage
\subsection{TinyGSM}
\label{app:full_tables_tinygsm}
%

\begin{table}[!htbp]
    \centering
    \small
    \setlength{\tabcolsep}{6.0pt}
    \renewcommand{\arraystretch}{1.08}
    \caption{\textbf{Baseline Models Accuracy (\%) on TinyGSM} ($T = 1.0$, 512 steps) across sampling schedules. We evaluate Autoregressive (AR), Discrete Absorbing (CE and ELBO losses), Discrete Uniform (ELBO loss), Predictor-Corrector, Self-Conditioning (Ancestral + SC), FLM, and $\mathbb{S}$-FLM. $^\dagger$trained with the uniform time sampler, $^\ddagger$trained with the adaptive time sampler. We \textbf{bold} the best result in each column and \underline{underline} the best result within each section (both if achieving both).}
    \label{tab:tinygsm_baselines_t10_s512}
    \begin{tabular}{l ccc}
    \toprule
    Model & Linear & Cosine & Adaptive \\
    \midrule
    \multicolumn{4}{@{}l@{}}{\textit{Autoregressive}} \\
    \hspace*{0.8em} AR (Sample) & 52.6 & -- & -- \\
    \hspace*{0.8em} AR (Greedy) & \textbf{\underline{62.6}} & -- & -- \\
    \midrule
    \multicolumn{4}{@{}l@{}}{\textit{Discrete (Absorbing)}} \\
    \hspace*{0.8em} Ancestral$^\dagger$ (CE) & 12.9 & 14.8 & -- \\
    \hspace*{0.8em} Ancestral$^\ddagger$ (CE) & 13.2 & 13.4 & 13.1 \\
    \hspace*{0.8em} Ancestral$^\dagger$ (ELBO) & 15.8 & 14.2 & -- \\
    \hspace*{0.8em} Ancestral$^\ddagger$ (ELBO) & 14.0 & 13.9 & 13.9 \\
    \hspace*{0.8em} Predictor-Corrector$^\dagger$ & 13.3 & 12.1 & -- \\
    \hspace*{0.8em} Predictor-Corrector$^\ddagger$ & 10.8 & 11.4 & 10.9 \\
    \hspace*{0.8em} Ancestral + SC$^\dagger$ (CE) & 23.0 & \underline{26.4} & -- \\
    \hspace*{0.8em} Ancestral + SC$^\ddagger$ (CE) & 22.1 & 23.2 & \underline{21.9} \\
    \hspace*{0.8em} Ancestral + SC$^\dagger$ (ELBO) & 20.5 & 22.2 & -- \\
    \hspace*{0.8em} Ancestral + SC$^\ddagger$ (ELBO) & \underline{23.5} & 22.2 & 21.1 \\
    \midrule
    \multicolumn{4}{@{}l@{}}{\textit{Discrete (Uniform)}} \\
    \hspace*{0.8em} Ancestral$^\dagger$ (ELBO) & 17.7 & 17.0 & -- \\
    \hspace*{0.8em} Ancestral$^\ddagger$ (ELBO) & 16.7 & 16.2 & 15.6 \\
    \hspace*{0.8em} Predictor-Corrector$^\dagger$ & \underline{36.8} & \textbf{\underline{36.5}} & -- \\
    \hspace*{0.8em} Predictor-Corrector$^\ddagger$ & 32.4 & 30.9 & \textbf{\underline{26.6}} \\
    \hspace*{0.8em} Ancestral + SC$^\dagger$ (ELBO) & 15.6 & 15.2 & -- \\
    \hspace*{0.8em} Ancestral + SC$^\ddagger$ (ELBO) & 19.2 & 19.0 & 18.6 \\
    \midrule
    \multicolumn{4}{@{}l@{}}{\textit{Euclidean / Spherical Flows}} \\
    \hspace*{0.8em} FLM (Gauss-Hermite LUT) & 0.6 & 0.5 & -- \\
    \hspace*{0.8em} FLM$^\dagger$ & 3.0 & 2.9 & -- \\
    \hspace*{0.8em} FLM$^\ddagger$ & 4.0 & 3.9 & 3.6 \\
    \hspace*{0.8em} $\mathbb{S}$-FLM$^\ddagger$ & 10.7 & 11.5 & 10.3 \\
    \hspace*{0.8em} $\mathbb{S}$-FLM$^\ddagger$ (argmax velocity) & \underline{15.0} & \underline{16.5} & \underline{13.2} \\
    \bottomrule
    \end{tabular}
\end{table}


\begin{table}[!htbp]
    \centering
    \small
    \setlength{\tabcolsep}{6.0pt}
    \renewcommand{\arraystretch}{1.08}
    \caption{\textbf{Baseline Models Accuracy (\%) on TinyGSM} ($T = 1.0$, 64 steps) across sampling schedules. We evaluate Autoregressive (AR), Discrete Absorbing (CE and ELBO losses), Discrete Uniform (ELBO loss), Predictor-Corrector, Self-Conditioning (Ancestral + SC), FLM, and $\mathbb{S}$-FLM. $^\dagger$trained with the uniform time sampler, $^\ddagger$trained with the adaptive time sampler. We \textbf{bold} the best result in each column and \underline{underline} the best result within each section (both if achieving both).}
    \label{tab:tinygsm_baselines_t10_s64}
    \begin{tabular}{l ccc}
    \toprule
    Model & Linear & Cosine & Adaptive \\
    \midrule
    \multicolumn{4}{@{}l@{}}{\textit{Discrete (Absorbing)}} \\
    \hspace*{0.8em} Ancestral$^\dagger$ (CE) & 9.7 & 11.1 & -- \\
    \hspace*{0.8em} Ancestral$^\ddagger$ (CE) & 8.3 & 11.7 & 11.2 \\
    \hspace*{0.8em} Ancestral$^\dagger$ (ELBO) & 10.0 & 13.3 & -- \\
    \hspace*{0.8em} Ancestral$^\ddagger$ (ELBO) & 8.5 & 11.8 & 12.6 \\
    \hspace*{0.8em} Predictor-Corrector$^\dagger$ & 10.7 & 10.2 & -- \\
    \hspace*{0.8em} Predictor-Corrector$^\ddagger$ & 9.8 & 11.4 & 7.2 \\
    \hspace*{0.8em} Ancestral + SC$^\dagger$ (CE) & 12.2 & \underline{19.9} & -- \\
    \hspace*{0.8em} Ancestral + SC$^\ddagger$ (CE) & 12.6 & 19.6 & 18.0 \\
    \hspace*{0.8em} Ancestral + SC$^\dagger$ (ELBO) & 12.6 & 17.2 & -- \\
    \hspace*{0.8em} Ancestral + SC$^\ddagger$ (ELBO) & \underline{12.9} & 19.3 & \textbf{\underline{18.7}} \\
    \midrule
    \multicolumn{4}{@{}l@{}}{\textit{Discrete (Uniform)}} \\
    \hspace*{0.8em} Ancestral$^\dagger$ (ELBO) & 14.4 & 14.5 & -- \\
    \hspace*{0.8em} Ancestral$^\ddagger$ (ELBO) & 14.9 & 15.6 & 14.6 \\
    \hspace*{0.8em} Predictor-Corrector$^\dagger$ & \textbf{\underline{23.3}} & \textbf{\underline{23.7}} & -- \\
    \hspace*{0.8em} Predictor-Corrector$^\ddagger$ & 20.5 & 20.1 & 16.6 \\
    \hspace*{0.8em} Ancestral + SC$^\dagger$ (ELBO) & 12.0 & 14.7 & -- \\
    \hspace*{0.8em} Ancestral + SC$^\ddagger$ (ELBO) & 17.0 & 18.2 & \underline{16.7} \\
    \midrule
    \multicolumn{4}{@{}l@{}}{\textit{Euclidean / Spherical Flows}} \\
    \hspace*{0.8em} FLM (Gauss-Hermite LUT) & 0.5 & 0.6 & -- \\
    \hspace*{0.8em} FLM$^\dagger$ & 2.5 & 2.5 & -- \\
    \hspace*{0.8em} FLM$^\ddagger$ & 3.3 & 3.5 & 4.0 \\
    \hspace*{0.8em} $\mathbb{S}$-FLM$^\ddagger$ & 4.5 & 9.7 & 9.5 \\
    \hspace*{0.8em} $\mathbb{S}$-FLM$^\ddagger$ (argmax velocity) & \underline{9.1} & \underline{14.2} & \underline{12.9} \\
    \bottomrule
    \end{tabular}
\end{table}


\begin{table}[!htbp]
    \centering
    \small
    \setlength{\tabcolsep}{6.0pt}
    \renewcommand{\arraystretch}{1.08}
    \caption{\textbf{Baseline Models Accuracy (\%) on TinyGSM} ($T = 0.1$, 512 steps) across sampling schedules. We evaluate Autoregressive (AR), Discrete Absorbing (CE and ELBO losses), Discrete Uniform (ELBO loss), Predictor-Corrector, Self-Conditioning (Ancestral + SC), FLM, and $\mathbb{S}$-FLM. $^\dagger$trained with the uniform time sampler, $^\ddagger$trained with the adaptive time sampler. We \textbf{bold} the best result in each column and \underline{underline} the best result within each section (both if achieving both). With argmax velocity, the $\mathbb{S}$-FLM sampler is deterministic given the initial noise and does not depend on the temperature, hence identical rows at $T=1$ and $T=0.1$.}
    \label{tab:tinygsm_baselines_t01_s512}
    \begin{tabular}{l ccc}
    \toprule
    Model & Linear & Cosine & Adaptive \\
    \midrule
    \multicolumn{4}{@{}l@{}}{\textit{Autoregressive}} \\
    \hspace*{0.8em} AR (Sample) & 52.6 & -- & -- \\
    \hspace*{0.8em} AR (Greedy) & \textbf{\underline{62.6}} & -- & -- \\
    \midrule
    \multicolumn{4}{@{}l@{}}{\textit{Discrete (Absorbing)}} \\
    \hspace*{0.8em} Ancestral$^\dagger$ (CE) & 31.0 & 31.1 & -- \\
    \hspace*{0.8em} Ancestral$^\ddagger$ (CE) & 30.4 & 30.0 & 30.3 \\
    \hspace*{0.8em} Ancestral$^\dagger$ (ELBO) & 32.0 & 32.0 & -- \\
    \hspace*{0.8em} Ancestral$^\ddagger$ (ELBO) & 32.3 & 33.5 & 32.0 \\
    \hspace*{0.8em} Predictor-Corrector$^\dagger$ & 39.6 & 37.7 & -- \\
    \hspace*{0.8em} Predictor-Corrector$^\ddagger$ & 39.8 & 39.2 & 38.9 \\
    \hspace*{0.8em} Ancestral + SC$^\dagger$ (CE) & 44.0 & 42.7 & -- \\
    \hspace*{0.8em} Ancestral + SC$^\ddagger$ (CE) & \underline{45.8} & \textbf{\underline{45.5}} & \textbf{\underline{42.7}} \\
    \hspace*{0.8em} Ancestral + SC$^\dagger$ (ELBO) & 39.0 & 38.0 & -- \\
    \hspace*{0.8em} Ancestral + SC$^\ddagger$ (ELBO) & 43.2 & 43.5 & 42.6 \\
    \midrule
    \multicolumn{4}{@{}l@{}}{\textit{Discrete (Uniform)}} \\
    \hspace*{0.8em} Ancestral$^\dagger$ (ELBO) & 33.0 & 32.6 & -- \\
    \hspace*{0.8em} Ancestral$^\ddagger$ (ELBO) & 34.0 & 31.2 & 34.6 \\
    \hspace*{0.8em} Predictor-Corrector$^\dagger$ & \underline{45.5} & \textbf{\underline{45.5}} & -- \\
    \hspace*{0.8em} Predictor-Corrector$^\ddagger$ & 40.1 & 39.8 & \underline{40.1} \\
    \hspace*{0.8em} Ancestral + SC$^\dagger$ (ELBO) & 31.7 & 30.3 & -- \\
    \hspace*{0.8em} Ancestral + SC$^\ddagger$ (ELBO) & 39.7 & 38.3 & 38.8 \\
    \midrule
    \multicolumn{4}{@{}l@{}}{\textit{Euclidean / Spherical Flows}} \\
    \hspace*{0.8em} FLM (Gauss-Hermite LUT) & 0.8 & 0.8 & -- \\
    \hspace*{0.8em} FLM$^\dagger$ & 6.4 & 6.4 & -- \\
    \hspace*{0.8em} FLM$^\ddagger$ & 8.3 & 8.3 & 8.3 \\
    \hspace*{0.8em} $\mathbb{S}$-FLM$^\ddagger$ & \underline{15.4} & 15.4 & \underline{14.7} \\
    \hspace*{0.8em} $\mathbb{S}$-FLM$^\ddagger$ (argmax velocity) & 15.0 & \underline{16.5} & 13.2 \\
    \bottomrule
    \end{tabular}
\end{table}


\begin{table}[!htbp]
    \centering
    \small
    \setlength{\tabcolsep}{6.0pt}
    \renewcommand{\arraystretch}{1.08}
    \caption{\textbf{Baseline Models Accuracy (\%) on TinyGSM} ($T = 0.1$, 64 steps) across sampling schedules. We evaluate Autoregressive (AR), Discrete Absorbing (CE and ELBO losses), Discrete Uniform (ELBO loss), Predictor-Corrector, Self-Conditioning (Ancestral + SC), FLM, and $\mathbb{S}$-FLM. $^\dagger$trained with the uniform time sampler, $^\ddagger$trained with the adaptive time sampler. We \textbf{bold} the best result in each column and \underline{underline} the best result within each section (both if achieving both). With argmax velocity, the $\mathbb{S}$-FLM sampler is deterministic given the initial noise and does not depend on the temperature, hence identical rows at $T=1$ and $T=0.1$.}
    \label{tab:tinygsm_baselines_t01_s64}
    \begin{tabular}{l ccc}
    \toprule
    Model & Linear & Cosine & Adaptive \\
    \midrule
    \multicolumn{4}{@{}l@{}}{\textit{Discrete (Absorbing)}} \\
    \hspace*{0.8em} Ancestral$^\dagger$ (CE) & 25.0 & 28.2 & -- \\
    \hspace*{0.8em} Ancestral$^\ddagger$ (CE) & 24.9 & 29.3 & 27.3 \\
    \hspace*{0.8em} Ancestral$^\dagger$ (ELBO) & 27.4 & 29.5 & -- \\
    \hspace*{0.8em} Ancestral$^\ddagger$ (ELBO) & 29.6 & 30.7 & 29.9 \\
    \hspace*{0.8em} Predictor-Corrector$^\dagger$ & 31.1 & 31.8 & -- \\
    \hspace*{0.8em} Predictor-Corrector$^\ddagger$ & 33.0 & 35.2 & 32.9 \\
    \hspace*{0.8em} Ancestral + SC$^\dagger$ (CE) & 31.5 & 38.7 & -- \\
    \hspace*{0.8em} Ancestral + SC$^\ddagger$ (CE) & \underline{35.9} & \textbf{\underline{40.3}} & \textbf{\underline{41.4}} \\
    \hspace*{0.8em} Ancestral + SC$^\dagger$ (ELBO) & 29.8 & 34.5 & -- \\
    \hspace*{0.8em} Ancestral + SC$^\ddagger$ (ELBO) & 35.6 & 39.8 & 41.1 \\
    \midrule
    \multicolumn{4}{@{}l@{}}{\textit{Discrete (Uniform)}} \\
    \hspace*{0.8em} Ancestral$^\dagger$ (ELBO) & 30.7 & 31.5 & -- \\
    \hspace*{0.8em} Ancestral$^\ddagger$ (ELBO) & 32.3 & 32.4 & 31.6 \\
    \hspace*{0.8em} Predictor-Corrector$^\dagger$ & \textbf{\underline{38.3}} & 37.0 & -- \\
    \hspace*{0.8em} Predictor-Corrector$^\ddagger$ & 35.7 & 35.2 & 33.5 \\
    \hspace*{0.8em} Ancestral + SC$^\dagger$ (ELBO) & 27.8 & 30.7 & -- \\
    \hspace*{0.8em} Ancestral + SC$^\ddagger$ (ELBO) & 34.7 & \underline{37.5} & \underline{37.1} \\
    \midrule
    \multicolumn{4}{@{}l@{}}{\textit{Euclidean / Spherical Flows}} \\
    \hspace*{0.8em} FLM (Gauss-Hermite LUT) & 0.7 & 0.8 & -- \\
    \hspace*{0.8em} FLM$^\dagger$ & 6.6 & 6.4 & -- \\
    \hspace*{0.8em} FLM$^\ddagger$ & 8.1 & 8.3 & 7.9 \\
    \hspace*{0.8em} $\mathbb{S}$-FLM$^\ddagger$ & \underline{9.4} & 14.1 & \underline{13.4} \\
    \hspace*{0.8em} $\mathbb{S}$-FLM$^\ddagger$ (argmax velocity) & 9.1 & \underline{14.2} & 12.9 \\
    \bottomrule
    \end{tabular}
\end{table}


\begin{table}[!htbp]
    \centering
    \scriptsize
    \setlength{\tabcolsep}{3.0pt}
    \renewcommand{\arraystretch}{1.02}
    \caption{\textbf{TinyGSM Simplex Diffusion (SDM) Accuracy (\%)} ($T = 1.0$, 512 steps) across noise concentration schedules, sampling schedules (Linear, Cosine, Adaptive), and churn (\mbox{$\kappa \in \{0.0, 0.2, 1.0\}$}). We compare Expectation ($P_t$), Argmax, and Argmax with Self-Conditioning (+ SC). All SDMs are trained with the adaptive time sampler. We \textbf{bold} the best result in each column and \underline{underline} the best result within each section (both if achieving both).}
    \label{tab:tinygsm_sdm_sdm_t10_s512}
    \begin{tabular}{l ccc ccc ccc}
    \toprule
    \multirow{2}{*}{Concentration Schedule} & \multicolumn{3}{c}{Linear} & \multicolumn{3}{c}{Cosine} & \multicolumn{3}{c}{Adaptive} \\
    \cmidrule(lr){2-4} \cmidrule(lr){5-7} \cmidrule(lr){8-10}
     & $\kappa=0$ & $\kappa=0.2$ & $\kappa=1.0$ & $\kappa=0$ & $\kappa=0.2$ & $\kappa=1.0$ & $\kappa=0$ & $\kappa=0.2$ & $\kappa=1.0$ \\
    \midrule
    \multicolumn{10}{@{}l@{}}{\textit{Expectation Input Processing ($P_t$)}} \\
    \hspace*{0.8em} Const.-Lin.~{\tiny ($\nu_0=0.2, \nu_1=0.5, \ell=0.2$)} & 15.4 & \underline{25.1} & \underline{34.7} & \textbf{\underline{19.8}} & \underline{29.1} & \underline{38.8} & 12.6 & \textbf{\underline{35.7}} & \textbf{\underline{45.8}} \\
    \hspace*{0.8em} Const.-Lin.~{\tiny ($\nu_0=0.2, \nu_1=0.75, \ell=0.2$)} & 15.1 & 24.5 & 31.7 & 17.9 & 28.4 & 36.8 & 11.5 & 29.2 & 38.8 \\
    \hspace*{0.8em} Constant~{\tiny ($\nu = 0.5$)} & \underline{15.7} & 23.1 & 33.6 & 17.5 & 24.4 & 38.2 & \underline{15.0} & 29.1 & 41.7 \\
    \hspace*{0.8em} Constant~{\tiny ($\nu = 0.25$)} & 12.1 & 19.4 & 26.6 & 14.2 & 25.8 & 32.5 & 10.2 & 33.9 & 43.3 \\
    \hspace*{0.8em} Constant~{\tiny ($\nu = 0.1$)} & 4.7 & 9.6 & 12.8 & 6.5 & 15.6 & 20.1 & 3.5 & 17.0 & 24.5 \\
    \midrule
    \multicolumn{10}{@{}l@{}}{\textit{Argmax Input Processing}} \\
    \hspace*{0.8em} Const.-Lin.~{\tiny ($\nu_0=0.2, \nu_1=0.5, \ell=0.2$)} & 15.5 & 26.7 & \underline{35.7} & 17.4 & 26.9 & 33.3 & 14.9 & 21.9 & 30.4 \\
    \hspace*{0.8em} Const.-Lin.~{\tiny ($\nu_0=0.2, \nu_1=0.75, \ell=0.2$)} & 18.3 & 26.3 & 34.2 & \underline{18.4} & 26.0 & \underline{37.4} & 16.0 & 20.2 & \underline{30.6} \\
    \hspace*{0.8em} Constant~{\tiny ($\nu = 0.5$)} & 17.1 & 24.8 & 33.9 & 16.6 & 22.4 & 33.0 & 16.7 & 16.4 & 29.5 \\
    \hspace*{0.8em} Constant~{\tiny ($\nu = 0.25$)} & \underline{19.3} & \underline{29.3} & 31.8 & 17.8 & \underline{27.3} & 35.0 & \underline{17.8} & \underline{23.0} & 30.3 \\
    \midrule
    \multicolumn{10}{@{}l@{}}{\textit{Argmax + Self-Conditioning (SC)}} \\
    \hspace*{0.8em} Const.-Lin.~{\tiny ($\nu_0=0.2, \nu_1=0.5, \ell=0.2$)} & 17.1 & 30.2 & 38.1 & 16.9 & 31.0 & 40.1 & 17.0 & 26.6 & 36.9 \\
    \hspace*{0.8em} Const.-Lin.~{\tiny ($\nu_0=0.2, \nu_1=0.75, \ell=0.2$)} & 16.9 & \textbf{\underline{32.5}} & \textbf{\underline{43.7}} & 18.9 & \textbf{\underline{32.8}} & \textbf{\underline{42.7}} & 16.9 & \underline{30.8} & 40.3 \\
    \hspace*{0.8em} Const.-Lin.~{\tiny ($\nu_0=0.4, \nu_1=0.75, \ell=0.8$)} & 13.7 & 25.6 & 34.7 & 17.4 & 24.1 & 33.5 & 11.2 & 20.2 & 32.5 \\
    \hspace*{0.8em} Constant~{\tiny ($\nu = 0.5$)} & \textbf{\underline{21.1}} & 29.3 & 40.8 & \underline{19.5} & 27.6 & 41.4 & \textbf{\underline{20.5}} & 25.6 & \underline{40.6} \\
    \hspace*{0.8em} Constant~{\tiny ($\nu = 0.25$)} & 17.6 & 31.1 & 37.0 & 18.2 & 29.0 & 38.9 & 17.6 & 29.3 & 37.8 \\
    \bottomrule
    \end{tabular}
\end{table}


\begin{table}[!htbp]
    \centering
    \scriptsize
    \setlength{\tabcolsep}{3.0pt}
    \renewcommand{\arraystretch}{1.02}
    \caption{\textbf{TinyGSM Simplex Diffusion (SDM) Accuracy (\%)} ($T = 1.0$, 64 steps) across noise concentration schedules, sampling schedules (Linear, Cosine, Adaptive), and churn (\mbox{$\kappa \in \{0.0, 0.2, 1.0\}$}). We compare Expectation ($P_t$), Argmax, and Argmax with Self-Conditioning (+ SC). All SDMs are trained with the adaptive time sampler. We \textbf{bold} the best result in each column and \underline{underline} the best result within each section (both if achieving both).}
    \label{tab:tinygsm_sdm_sdm_t10_s64}
    \begin{tabular}{l ccc ccc ccc}
    \toprule
    \multirow{2}{*}{Concentration Schedule} & \multicolumn{3}{c}{Linear} & \multicolumn{3}{c}{Cosine} & \multicolumn{3}{c}{Adaptive} \\
    \cmidrule(lr){2-4} \cmidrule(lr){5-7} \cmidrule(lr){8-10}
     & $\kappa=0$ & $\kappa=0.2$ & $\kappa=1.0$ & $\kappa=0$ & $\kappa=0.2$ & $\kappa=1.0$ & $\kappa=0$ & $\kappa=0.2$ & $\kappa=1.0$ \\
    \midrule
    \multicolumn{10}{@{}l@{}}{\textit{Expectation Input Processing ($P_t$)}} \\
    \hspace*{0.8em} Const.-Lin.~{\tiny ($\nu_0=0.2, \nu_1=0.5, \ell=0.2$)} & 6.7 & 9.8 & 14.7 & 14.3 & \underline{16.8} & 22.8 & 10.6 & \textbf{\underline{23.1}} & \textbf{\underline{32.7}} \\
    \hspace*{0.8em} Const.-Lin.~{\tiny ($\nu_0=0.2, \nu_1=0.75, \ell=0.2$)} & 4.9 & \underline{12.0} & \underline{16.1} & 14.1 & 16.1 & 20.4 & 10.3 & 15.3 & 25.5 \\
    \hspace*{0.8em} Constant~{\tiny ($\nu = 0.5$)} & \underline{7.4} & 9.4 & 13.8 & \underline{15.0} & 14.2 & \underline{24.0} & \underline{11.2} & 21.1 & 31.8 \\
    \hspace*{0.8em} Constant~{\tiny ($\nu = 0.25$)} & 4.2 & 5.9 & 8.6 & 11.4 & 14.6 & 18.9 & 8.0 & 18.2 & 30.1 \\
    \hspace*{0.8em} Constant~{\tiny ($\nu = 0.1$)} & 0.8 & 1.5 & 1.9 & 5.0 & 6.7 & 9.3 & 2.7 & 8.9 & 13.4 \\
    \midrule
    \multicolumn{10}{@{}l@{}}{\textit{Argmax Input Processing}} \\
    \hspace*{0.8em} Const.-Lin.~{\tiny ($\nu_0=0.2, \nu_1=0.5, \ell=0.2$)} & 13.1 & 13.1 & 20.9 & 16.2 & 14.0 & 21.4 & 14.4 & 10.4 & 17.3 \\
    \hspace*{0.8em} Const.-Lin.~{\tiny ($\nu_0=0.2, \nu_1=0.75, \ell=0.2$)} & 9.7 & \underline{13.5} & \underline{22.8} & 16.1 & \underline{14.7} & 22.4 & 14.1 & 11.0 & 17.1 \\
    \hspace*{0.8em} Constant~{\tiny ($\nu = 0.5$)} & \textbf{\underline{14.4}} & 13.2 & 22.0 & 15.6 & 10.0 & \underline{22.5} & 13.7 & 9.2 & \underline{19.3} \\
    \hspace*{0.8em} Constant~{\tiny ($\nu = 0.25$)} & 12.6 & 13.3 & 21.7 & \underline{16.8} & 13.5 & 21.6 & \underline{15.3} & \underline{11.9} & 17.4 \\
    \midrule
    \multicolumn{10}{@{}l@{}}{\textit{Argmax + Self-Conditioning (SC)}} \\
    \hspace*{0.8em} Const.-Lin.~{\tiny ($\nu_0=0.2, \nu_1=0.5, \ell=0.2$)} & 11.9 & \textbf{\underline{16.4}} & 24.1 & 17.6 & 15.8 & 25.0 & 13.2 & 14.6 & 22.4 \\
    \hspace*{0.8em} Const.-Lin.~{\tiny ($\nu_0=0.2, \nu_1=0.75, \ell=0.2$)} & 10.0 & \textbf{\underline{16.4}} & \textbf{\underline{25.7}} & 15.1 & \textbf{\underline{17.5}} & 28.1 & 14.6 & \underline{17.0} & \underline{26.4} \\
    \hspace*{0.8em} Const.-Lin.~{\tiny ($\nu_0=0.4, \nu_1=0.75, \ell=0.8$)} & 6.2 & 15.1 & 22.5 & 13.4 & 12.9 & 23.1 & 9.4 & 11.6 & 18.2 \\
    \hspace*{0.8em} Constant~{\tiny ($\nu = 0.5$)} & \underline{13.6} & 16.3 & 25.5 & \textbf{\underline{18.1}} & 13.6 & \textbf{\underline{29.9}} & \textbf{\underline{15.8}} & \underline{17.0} & 25.6 \\
    \hspace*{0.8em} Constant~{\tiny ($\nu = 0.25$)} & 13.5 & 15.5 & 22.0 & 14.8 & 16.4 & 24.7 & 13.5 & 16.0 & 22.1 \\
    \bottomrule
    \end{tabular}
\end{table}


\begin{table}[!htbp]
    \centering
    \scriptsize
    \setlength{\tabcolsep}{3.0pt}
    \renewcommand{\arraystretch}{1.02}
    \caption{\textbf{TinyGSM Simplex Diffusion (SDM) Accuracy (\%)} ($T = 0.1$, 512 steps) across noise concentration schedules, sampling schedules (Linear, Cosine, Adaptive), and churn (\mbox{$\kappa \in \{0.0, 0.2, 1.0\}$}). We compare Expectation ($P_t$), Argmax, and Argmax with Self-Conditioning (+ SC). All SDMs are trained with the adaptive time sampler. We \textbf{bold} the best result in each column and \underline{underline} the best result within each section (both if achieving both).}
    \label{tab:tinygsm_sdm_sdm_t01_s512}
    \begin{tabular}{l ccc ccc ccc}
    \toprule
    \multirow{2}{*}{Concentration Schedule} & \multicolumn{3}{c}{Linear} & \multicolumn{3}{c}{Cosine} & \multicolumn{3}{c}{Adaptive} \\
    \cmidrule(lr){2-4} \cmidrule(lr){5-7} \cmidrule(lr){8-10}
     & $\kappa=0$ & $\kappa=0.2$ & $\kappa=1.0$ & $\kappa=0$ & $\kappa=0.2$ & $\kappa=1.0$ & $\kappa=0$ & $\kappa=0.2$ & $\kappa=1.0$ \\
    \midrule
    \multicolumn{10}{@{}l@{}}{\textit{Expectation Input Processing ($P_t$)}} \\
    \hspace*{0.8em} Const.-Lin.~{\tiny ($\nu_0=0.2, \nu_1=0.5, \ell=0.2$)} & 27.3 & 36.0 & 40.4 & \underline{30.8} & 38.5 & 42.4 & 19.3 & \underline{43.9} & \underline{49.0} \\
    \hspace*{0.8em} Const.-Lin.~{\tiny ($\nu_0=0.2, \nu_1=0.75, \ell=0.2$)} & \underline{27.7} & \underline{36.4} & 38.2 & \underline{30.8} & \underline{39.3} & 42.5 & 19.8 & 40.8 & 44.3 \\
    \hspace*{0.8em} Constant~{\tiny ($\nu = 0.5$)} & 27.0 & 30.5 & \underline{41.7} & 29.5 & 33.3 & \underline{43.3} & \underline{22.2} & 35.0 & 44.8 \\
    \hspace*{0.8em} Constant~{\tiny ($\nu = 0.25$)} & 18.3 & 25.0 & 30.1 & 22.5 & 31.2 & 38.5 & 11.3 & 38.3 & 42.5 \\
    \hspace*{0.8em} Constant~{\tiny ($\nu = 0.1$)} & 7.7 & 10.2 & 13.8 & 9.3 & 17.0 & 22.1 & 2.9 & 17.5 & 21.0 \\
    \midrule
    \multicolumn{10}{@{}l@{}}{\textit{Argmax Input Processing}} \\
    \hspace*{0.8em} Const.-Lin.~{\tiny ($\nu_0=0.2, \nu_1=0.5, \ell=0.2$)} & 33.0 & 40.8 & 43.8 & \textbf{\underline{34.7}} & 39.8 & 45.5 & 31.5 & \underline{39.7} & 42.9 \\
    \hspace*{0.8em} Const.-Lin.~{\tiny ($\nu_0=0.2, \nu_1=0.75, \ell=0.2$)} & 31.6 & 42.2 & \underline{46.4} & 33.6 & 40.8 & 45.5 & 30.7 & 38.8 & \underline{45.4} \\
    \hspace*{0.8em} Constant~{\tiny ($\nu = 0.5$)} & 33.6 & 37.8 & 44.5 & 34.4 & 33.6 & \underline{46.1} & \textbf{\underline{34.5}} & 32.4 & 43.4 \\
    \hspace*{0.8em} Constant~{\tiny ($\nu = 0.25$)} & \textbf{\underline{34.2}} & \underline{42.7} & 44.9 & 33.2 & \underline{41.1} & 45.7 & 33.4 & 38.2 & 43.7 \\
    \midrule
    \multicolumn{10}{@{}l@{}}{\textit{Argmax + Self-Conditioning (SC)}} \\
    \hspace*{0.8em} Const.-Lin.~{\tiny ($\nu_0=0.2, \nu_1=0.5, \ell=0.2$)} & 33.3 & 45.5 & 50.0 & 33.9 & 43.7 & 48.9 & 31.4 & 43.8 & 50.2 \\
    \hspace*{0.8em} Const.-Lin.~{\tiny ($\nu_0=0.2, \nu_1=0.75, \ell=0.2$)} & 31.6 & \textbf{\underline{46.1}} & 49.6 & 33.4 & \textbf{\underline{45.1}} & 50.5 & 31.0 & \textbf{\underline{46.2}} & 50.3 \\
    \hspace*{0.8em} Const.-Lin.~{\tiny ($\nu_0=0.4, \nu_1=0.75, \ell=0.8$)} & 27.5 & 38.8 & 46.1 & 31.8 & 36.6 & 46.4 & 19.9 & 35.9 & 43.2 \\
    \hspace*{0.8em} Constant~{\tiny ($\nu = 0.5$)} & \underline{34.0} & 42.2 & \textbf{\underline{50.6}} & \textbf{\underline{34.7}} & 40.8 & \textbf{\underline{50.9}} & 33.1 & 39.2 & \textbf{\underline{50.6}} \\
    \hspace*{0.8em} Constant~{\tiny ($\nu = 0.25$)} & 31.7 & 41.8 & 48.7 & 32.6 & 43.6 & 49.4 & \underline{33.7} & 43.4 & \textbf{\underline{50.6}} \\
    \bottomrule
    \end{tabular}
\end{table}


\begin{table}[!htbp]
    \centering
    \scriptsize
    \setlength{\tabcolsep}{3.0pt}
    \renewcommand{\arraystretch}{1.02}
    \caption{\textbf{TinyGSM Simplex Diffusion (SDM) Accuracy (\%)} ($T = 0.1$, 64 steps) across noise concentration schedules, sampling schedules (Linear, Cosine, Adaptive), and churn (\mbox{$\kappa \in \{0.0, 0.2, 1.0\}$}). We compare Expectation ($P_t$), Argmax, and Argmax with Self-Conditioning (+ SC). All SDMs are trained with the adaptive time sampler. We \textbf{bold} the best result in each column and \underline{underline} the best result within each section (both if achieving both).}
    \label{tab:tinygsm_sdm_sdm_t01_s64}
    \begin{tabular}{l ccc ccc ccc}
    \toprule
    \multirow{2}{*}{Concentration Schedule} & \multicolumn{3}{c}{Linear} & \multicolumn{3}{c}{Cosine} & \multicolumn{3}{c}{Adaptive} \\
    \cmidrule(lr){2-4} \cmidrule(lr){5-7} \cmidrule(lr){8-10}
     & $\kappa=0$ & $\kappa=0.2$ & $\kappa=1.0$ & $\kappa=0$ & $\kappa=0.2$ & $\kappa=1.0$ & $\kappa=0$ & $\kappa=0.2$ & $\kappa=1.0$ \\
    \midrule
    \multicolumn{10}{@{}l@{}}{\textit{Expectation Input Processing ($P_t$)}} \\
    \hspace*{0.8em} Const.-Lin.~{\tiny ($\nu_0=0.2, \nu_1=0.5, \ell=0.2$)} & 11.3 & 16.7 & 20.5 & 25.7 & 28.7 & 32.0 & 18.9 & \underline{33.0} & 39.5 \\
    \hspace*{0.8em} Const.-Lin.~{\tiny ($\nu_0=0.2, \nu_1=0.75, \ell=0.2$)} & 11.7 & \underline{22.5} & \underline{28.8} & 25.8 & \underline{29.2} & \underline{33.6} & 19.0 & 32.1 & 38.8 \\
    \hspace*{0.8em} Constant~{\tiny ($\nu = 0.5$)} & \underline{15.3} & 16.2 & 20.0 & \underline{26.8} & 21.1 & 33.3 & \underline{21.8} & 31.9 & \underline{40.1} \\
    \hspace*{0.8em} Constant~{\tiny ($\nu = 0.25$)} & 6.2 & 7.1 & 9.5 & 18.0 & 19.9 & 25.1 & 12.1 & 24.0 & 31.0 \\
    \hspace*{0.8em} Constant~{\tiny ($\nu = 0.1$)} & 0.8 & 0.6 & 1.5 & 7.7 & 9.4 & 12.4 & 2.6 & 7.9 & 12.4 \\
    \midrule
    \multicolumn{10}{@{}l@{}}{\textit{Argmax Input Processing}} \\
    \hspace*{0.8em} Const.-Lin.~{\tiny ($\nu_0=0.2, \nu_1=0.5, \ell=0.2$)} & 27.2 & 27.9 & \underline{38.8} & 30.7 & 29.6 & 39.1 & 31.1 & 27.9 & 37.4 \\
    \hspace*{0.8em} Const.-Lin.~{\tiny ($\nu_0=0.2, \nu_1=0.75, \ell=0.2$)} & 23.4 & 30.1 & 38.5 & \underline{32.4} & \underline{30.9} & \underline{39.3} & 27.9 & \underline{29.2} & 37.4 \\
    \hspace*{0.8em} Constant~{\tiny ($\nu = 0.5$)} & \textbf{\underline{30.7}} & 28.6 & 37.4 & 32.1 & 22.1 & 36.7 & \textbf{\underline{32.4}} & 25.7 & \underline{37.8} \\
    \hspace*{0.8em} Constant~{\tiny ($\nu = 0.25$)} & 29.7 & \underline{30.8} & 35.4 & 31.9 & 28.6 & 37.8 & 32.1 & 27.5 & 37.4 \\
    \midrule
    \multicolumn{10}{@{}l@{}}{\textit{Argmax + Self-Conditioning (SC)}} \\
    \hspace*{0.8em} Const.-Lin.~{\tiny ($\nu_0=0.2, \nu_1=0.5, \ell=0.2$)} & 26.6 & 30.9 & 40.6 & 31.2 & 33.9 & 42.0 & 31.1 & 33.6 & 42.0 \\
    \hspace*{0.8em} Const.-Lin.~{\tiny ($\nu_0=0.2, \nu_1=0.75, \ell=0.2$)} & 20.8 & 31.4 & \textbf{\underline{42.0}} & 30.4 & \textbf{\underline{34.3}} & 43.4 & 28.5 & 36.1 & \textbf{\underline{46.0}} \\
    \hspace*{0.8em} Const.-Lin.~{\tiny ($\nu_0=0.4, \nu_1=0.75, \ell=0.8$)} & 13.3 & 29.2 & 36.9 & 26.7 & 25.4 & 38.8 & 18.6 & 28.7 & 37.9 \\
    \hspace*{0.8em} Constant~{\tiny ($\nu = 0.5$)} & \underline{29.6} & \textbf{\underline{33.0}} & 41.0 & \textbf{\underline{32.7}} & 27.3 & \textbf{\underline{44.2}} & \underline{32.2} & \textbf{\underline{36.6}} & 44.5 \\
    \hspace*{0.8em} Constant~{\tiny ($\nu = 0.25$)} & 27.5 & 29.9 & 39.4 & 30.1 & 32.5 & 40.8 & 30.7 & 33.0 & 42.5 \\
    \bottomrule
    \end{tabular}
\end{table}

\begin{table*}[!htbp]
    \centering
    \scriptsize
    \setlength{\tabcolsep}{6.0pt}
    \renewcommand{\arraystretch}{1.12}
    \caption{\textbf{Multi-Sample Accuracy and AST Structural Diversity ($K=5$ samples per problem) on TinyGSM} across sampling steps $\text{NFE} \in \{8, 64, 512\}$ and temperatures $T \in \{0.1, 1.0\}$. $^\dagger$trained with uniform time sampler (evaluated with linear sampling schedule), $^\ddagger$trained with adaptive time sampler (evaluated with adaptive sampling schedule). ``ELBO'' denotes the evidence lower bound objective, ``SC'' denotes Self-Conditioning, and ``PC'' denotes the Predictor-Corrector sampler. Arrows ($\uparrow$ / $\downarrow$) indicate whether higher or lower values are better. For each column, we \textbf{bold} the overall best result and \protect\underline{underline} the best result within each section. For 64 and 512 steps, pass@1 is the single-sample accuracy from \Cref{tab:tinygsm_baselines_t10_s512,tab:tinygsm_baselines_t10_s64,tab:tinygsm_baselines_t01_s512,tab:tinygsm_baselines_t01_s64} (baselines), \Cref{tab:tinygsm_sdm_sdm_t10_s512,tab:tinygsm_sdm_sdm_t10_s64,tab:tinygsm_sdm_sdm_t01_s512,tab:tinygsm_sdm_sdm_t01_s64} (SDMs, $\nu_0=0.2$, $\nu_1=0.5$, $\ell=0.2$) and \Cref{fig:tinygsm_distilled_vs_sdm_nfe} (distilled SDM); for 8 steps it is the accuracy of the first of the $K$ samples. For AR, this is greedy decoding in the $T=0.1$ row and sampling in the $T=1.0$ row.}
    \label{tab:tinygsm_baselines_ast_diversity}
    \begin{tabular}{l c c ccccc}
        \toprule
        Model & Temp. ($T$) & Steps (NFE) & pass@1 (\%) $\uparrow$ & pass@2 (\%) $\uparrow$ & pass@5 (\%) $\uparrow$ & AST Div. $\uparrow$ & AST Div. (Correct) $\uparrow$ \\
        \midrule
        \multicolumn{8}{@{}l@{}}{\textit{Autoregressive}} \\
        & 0.1 & 512 & \textbf{\underline{62.6}} & \textbf{\underline{66.3}} & 69.6 & 9.6 & 6.2 \\
        & 1.0 & 512 & 52.6 & 66.2 & \textbf{\underline{77.6}} & \underline{36.7} & \textbf{\underline{29.5}} \\
        \midrule
        \multicolumn{8}{@{}l@{}}{\textit{Absorbing Diffusion (MDM$^\dagger$; ELBO; Linear Schedule)}} \\
        & 0.1 & 8 & 4.5 & 7.7 & 13.3 & 40.3 & 8.0 \\
        & 0.1 & 64 & 27.4 & 37.7 & 52.3 & 33.4 & 19.2 \\
        & 0.1 & 512 & \underline{32.0} & \underline{43.2} & \underline{56.9} & 31.6 & 19.3 \\
        & 1.0 & 8 & 0.5 & 0.8 & 1.4 & \textbf{\underline{51.1}} & 14.3 \\
        & 1.0 & 64 & 10.0 & 15.6 & 28.6 & 47.7 & 22.6 \\
        & 1.0 & 512 & 15.8 & 24.1 & 38.6 & 44.8 & \underline{23.4} \\
        \midrule
        \multicolumn{8}{@{}l@{}}{\textit{Absorbing Diffusion (MDM$^\ddagger$; ELBO; SC; Adaptive Schedule)}} \\
        & 0.1 & 8 & 19.6 & 29.7 & 43.8 & 30.4 & 14.6 \\
        & 0.1 & 64 & 41.1 & 51.6 & 62.0 & 29.1 & 19.1 \\
        & 0.1 & 512 & \underline{42.6} & \underline{54.9} & \underline{65.0} & 28.3 & 18.9 \\
        & 1.0 & 8 & 4.5 & 8.3 & 17.7 & 38.9 & 14.6 \\
        & 1.0 & 64 & 18.7 & 28.9 & 43.6 & \underline{41.7} & 21.7 \\
        & 1.0 & 512 & 21.1 & 36.2 & 49.7 & 41.3 & \underline{24.1} \\
        \midrule
        \multicolumn{8}{@{}l@{}}{\textit{Uniform Diffusion (UDM$^\dagger$; ELBO; Linear Schedule)}} \\
        & 0.1 & 8 & 11.2 & 17.3 & 28.7 & 35.9 & 15.4 \\
        & 0.1 & 64 & 30.7 & 42.5 & 56.7 & 32.0 & 19.6 \\
        & 0.1 & 512 & \underline{33.0} & \underline{46.1} & \underline{59.6} & 31.9 & 19.7 \\
        & 1.0 & 8 & 3.2 & 5.3 & 10.3 & \underline{46.5} & 15.4 \\
        & 1.0 & 64 & 14.4 & 22.3 & 36.1 & 44.1 & \underline{23.6} \\
        & 1.0 & 512 & 17.7 & 25.6 & 40.5 & 42.9 & 23.1 \\
        \midrule
        \multicolumn{8}{@{}l@{}}{\textit{Uniform Diffusion (UDM$^\dagger$; PC; Linear Schedule)}} \\
        & 0.1 & 8 & 11.5 & 18.1 & 31.0 & 39.0 & 17.9 \\
        & 0.1 & 64 & 38.3 & 49.0 & 61.5 & 33.4 & 22.2 \\
        & 0.1 & 512 & \underline{45.5} & \underline{56.8} & \underline{67.1} & 30.6 & 21.3 \\
        & 1.0 & 8 & 4.4 & 7.6 & 13.8 & \underline{49.9} & 19.0 \\
        & 1.0 & 64 & 23.3 & 34.9 & 51.1 & 44.3 & 26.8 \\
        & 1.0 & 512 & 36.8 & 49.3 & 62.2 & 40.6 & \underline{27.6} \\
        \midrule
        \multicolumn{8}{@{}l@{}}{\textit{Simplex Diffusion (SDM$^\ddagger$; Expectation; $\kappa=1$; Adaptive Schedule)}} \\
        & 0.1 & 8 & 17.1 & 27.7 & 42.8 & 36.5 & 21.2 \\
        & 0.1 & 64 & 39.5 & 52.4 & 64.8 & 32.3 & 22.5 \\
        & 0.1 & 512 & \underline{49.0} & \underline{58.0} & \underline{68.2} & 30.4 & 21.6 \\
        & 1.0 & 8 & 5.4 & 9.7 & 17.9 & \underline{40.0} & 17.3 \\
        & 1.0 & 64 & 32.7 & 44.0 & 57.7 & 38.7 & \underline{25.9} \\
        & 1.0 & 512 & 45.8 & 55.1 & 66.6 & 35.3 & 24.8 \\
        \midrule
        \multicolumn{8}{@{}l@{}}{\textit{Simplex Diffusion (SDM$^\ddagger$; Argmax + SC; $\kappa=1$; Adaptive Schedule)}} \\
        & 0.1 & 8 & 19.9 & 29.3 & 41.4 & 29.2 & 13.9 \\
        & 0.1 & 64 & 42.0 & 55.7 & 64.5 & 26.8 & 17.5 \\
        & 0.1 & 512 & \underline{50.2} & \underline{60.0} & \underline{67.5} & 23.2 & 16.2 \\
        & 1.0 & 8 & 3.9 & 6.2 & 12.2 & 25.0 & 7.8 \\
        & 1.0 & 64 & 22.4 & 32.6 & 48.0 & \underline{38.3} & 19.6 \\
        & 1.0 & 512 & 36.9 & 50.1 & 62.6 & 35.3 & \underline{22.3} \\
        \midrule
        \multicolumn{8}{@{}l@{}}{\textit{Distilled Simplex Diffusion (SDM$^\ddagger$; Expectation; $\kappa=1$; Adaptive Schedule)}} \\
        & 0.1 & 8 & 31.9 & 39.2 & 48.1 & 17.4 & 7.9 \\
        & 0.1 & 64 & 37.5 & 45.7 & 53.1 & 18.8 & 9.9 \\
        & 0.1 & 512 & \underline{40.6} & 46.3 & 53.7 & 18.0 & \underline{10.8} \\
        & 1.0 & 8 & 31.9 & 40.0 & 48.7 & 17.5 & 7.5 \\
        & 1.0 & 64 & 37.8 & 45.8 & 53.6 & \underline{19.0} & 10.7 \\
        & 1.0 & 512 & 39.6 & \underline{46.5} & \underline{53.9} & 18.1 & 10.4 \\
        \bottomrule
    \end{tabular}
\end{table*}

\end{document}